\documentclass{article} % For LaTeX2e
\usepackage{iclr2027_conference,times}
\usepackage{amsmath,amsfonts,bm}

\def\eqref#1{equation~\ref{#1}}
\def\1{\bm{1}}

\def\bc{{\bm{c}}}

\DeclareMathAlphabet{\mathsfit}{\encodingdefault}{\sfdefault}{m}{sl}
\SetMathAlphabet{\mathsfit}{bold}{\encodingdefault}{\sfdefault}{bx}{n}

\def\gD{{\mathcal{D}}}
\def\gE{{\mathcal{E}}}
\def\gF{{\mathcal{F}}}

\def\gH{{\mathcal{H}}}

\def\gN{{\mathcal{N}}}

\def\gP{{\mathcal{P}}}

\def\gR{{\mathcal{R}}}
\def\gS{{\mathcal{S}}}
\def\gT{{\mathcal{T}}}

\def\gV{{\mathcal{V}}}

\def\gX{{\mathcal{X}}}

\def\gZ{{\mathcal{Z}}}

\def\sP{{\mathbb{P}}}

\def\od{{\mathrm{d}}}

\newcommand{\E}{{\mathbb{E}}}

\newcommand{\R}{{\mathbb{R}}}

\newcommand{\KL}{\mathrm{KL}}

\newcommand{\Var}{\mathrm{Var}}

\newcommand{\Cov}{\mathrm{Cov}}

\newcommand{\proj}{\mathrm{Proj}}

\newcommand{\tr}{\mathrm{Tr}}

\DeclareMathOperator{\ind}{\mathds{1}}

\newcommand{\Law}{\mathsf{Law}}

\newcommand{\odt}{{\mathrm{d}t}}

\newcommand{\nn}{{\mathcal{NN}}}

\newcommand{\TV}{{\mathsf{TV}}}

\newcommand{\Ord}{{\mathcal{O}}}

\newcommand{\clip}{{\mathrm{clip}}}

\newcommand{\iid}{{\text{i.i.d.}}}

\newcommand{\Ncov}{N_{\mathrm{cov}}}

\providecommand{\Unif}{\mathrm{Unif}}

\DeclareMathOperator{\Ball}{Ball}
\newcommand{\Lip}{\mathsf{Lip}}

\usepackage[english]{babel}
\usepackage{enumitem}
\setlist[itemize]{labelindent=0pt,leftmargin=*,itemsep=2pt,parsep=1pt,topsep=0pt,partopsep=0pt}
\usepackage[most]{tcolorbox}
\usepackage{xurl}

\usepackage{xcolor}
\usepackage{natbib}
\definecolor{darkblue}{RGB}{5, 30, 140}
\usepackage[
    colorlinks=true,
    linkcolor=darkblue,
    citecolor=darkblue,
    urlcolor=darkblue
]{hyperref}

\usepackage{amsmath}
\usepackage{amssymb}
\usepackage{amsthm}
\usepackage{amscd}
\usepackage{amsfonts}
\usepackage{mathrsfs}
\usepackage{dsfont}
\usepackage{graphicx}
\usepackage{tikz}
\usetikzlibrary{graphs, positioning, quotes, shapes.geometric}
\usetikzlibrary{arrows.meta,fit,positioning}
\usepackage{tabularx}
\usepackage{array}
\usepackage{booktabs}
\usepackage{float}
\usepackage{longtable}
\usepackage{needspace}
\usepackage{microtype, wasysym, comment,mathtools}
\usepackage{multicol}
\usepackage{multirow}
\usepackage{diagbox}
\usepackage{bm}
\usepackage{bbm}
\usepackage{multirow}
\usepackage{algorithm}
\usepackage{algpseudocode}
\usepackage{fancyhdr}
\usepackage{subcaption}
\usepackage{wrapfig}
\usepackage{framed}
\usepackage{color}
\definecolor{shadecolor}{rgb}{0.94, 0.97, 1.0}

\newtheorem{theorem}{Theorem}
\newtheorem{lemma}[theorem]{Lemma}
\newtheorem{proposition}[theorem]{Proposition}
\newtheorem{corollary}[theorem]{Corollary}
\newtheorem{definition}{Definition}
\newtheorem{assumption}{Assumption}

\newtheorem{example}{Example}

\usepackage{cleveref} % should be put after other packages
\crefname{section}{Section}{Sections}
\crefname{theorem}{Theorem}{Theorems}
\crefname{lemma}{Lemma}{Lemmas}
\crefname{equation}{}{}
\crefname{proposition}{Proposition}{Propositions}
\crefname{corollary}{Corollary}{Corollaries}
\crefname{claim}{Claim}{Claims}
\crefname{remark}{Remark}{Remarks}
\crefname{assumption}{Assumption}{Assumptions}
\crefname{definition}{Definition}{Definitions}
\crefname{appendix}{Appendix}{Appendices}
\crefname{algorithm}{Algorithm}{Algorithms}
\crefname{figure}{Figure}{Figures}
\crefname{table}{Table}{Tables}
\crefname{example}{Example}{Examples}
\crefname{condition}{Condition}{Conditions}

\numberwithin{theorem}{section}
\numberwithin{assumption}{section}
\numberwithin{remark}{section}
\numberwithin{definition}{section}
\numberwithin{example}{section}

\AddToHook{env/lemma/begin}{\crefalias{theorem}{lemma}}
\AddToHook{env/proposition/begin}{\crefalias{theorem}{proposition}}
\AddToHook{env/corollary/begin}{\crefalias{theorem}{corollary}}
\AddToHook{env/claim/begin}{\crefalias{theorem}{claim}}

\usepackage{etoc}
\etocdepthtag.toc{mtchapter}
\etocsettagdepth{mtchapter}{subsection}
\etocsettagdepth{mtappendix}{none}

\title{Diffusion Transformers are Provably Optimal In-context Generators} 
\newcommand{\authorblock}[4]{%
  \parbox[t]{0.31\textwidth}{%
    \centering
    \normalfont
    \textbf{#1}\\
    #2\\
    #3\\
    {\footnotesize\nolinkurl{#4}}
  }%
}

\author{
\authorblock{Guoji Fu}
  {National University of Singapore}
  {Singapore}
  {guoji.fu@u.nus.edu}
\hfill
\authorblock{Tomoya Wakayama}
  {RIKEN AIP}
  {Tokyo, Japan}
  {tomoya.wakayama@riken.jp}
\hfill
\authorblock{Ryotaro Kawata}
  {The University of Tokyo}
  {Tokyo, Japan}
  {kawata-ryotaro725@g.ecc.u-tokyo.ac.jp}
\AND
\authorblock{Atsushi Nitanda}
  {A*STAR CFAR\\Nanyang Technological University}
  {Singapore}
  {atsushi_nitanda@a-star.edu.sg}
\hfill
\authorblock{Wee Sun Lee}
  {National University of Singapore}
  {Singapore}
  {leews@comp.nus.edu.sg}
\hfill
\authorblock{Taiji Suzuki}
  {The University of Tokyo\\RIKEN AIP}
  {Tokyo, Japan}
  {taiji@mist.i.u-tokyo.ac.jp}
}

\iclrfinalcopy % Uncomment for camera-ready version, but NOT for submission.
\begin{document}

\addtocontents{toc}{\protect\setcounter{tocdepth}{-1}}
\maketitle

\begin{abstract}
Generative foundation models are attracting interest for their ability to produce desired outputs from demonstrations given at inference time, without updating parameters. However, since a few demonstrations cannot uniquely identify the intended task, the challenge is how to learn and sample from an output distribution that reflects this task uncertainty. In this work, we theoretically analyze how a Diffusion Transformer (DiT), pretrained across diverse tasks, learns and generates predictive distributions for a new query from demonstrations. We first show that the natural target to generate from finite demonstrations is not an output derived from estimating a single task, but rather a predictive distribution that captures the task uncertainty remaining after observing the demonstrations. We then prove that a DiT can learn this predictive distribution through score estimation, using attention to aggregate information from demonstrations and diffusion to generate samples. Owing to this property, with sufficient pretraining resources and diffusion sampling steps, the resulting DiT achieves the minimax optimal rate over a Hölder class of test-time tasks. These results imply that DiT acts as a statistically grounded in-context generator capable of generating distributions adapted to new tasks while retaining the uncertainty inherent in finite demonstrations.
\end{abstract}

\vspace{-5pt}
\section{Introduction}
\vspace{-5pt}

Generative models have advanced rapidly over the past few years with diffusion models in particular achieving remarkable success in synthesizing complex high-dimensional data such as images~\citep{ho2020ddpm,songetal2021,rombach2022ldm,chi2023diffusionpolicy}. At the same time, in-context learning (ICL) has shown that a pretrained model can adapt to a new task from a few demonstrations without updating its parameters~\citep{brown2020gpt3,dong-etal-2024-survey}. These developments motivate \emph{in-context generation}: using demonstrations to specify what kinds of outputs should be generated for a new query~\citep{bar2022visualprompting,wang2023promptdiffusion}. 
Fig.~\ref{fig:conceptual-overview} illustrates this idea in a robotic manipulation example: a user provides several examples of acceptable behaviors and asks the model to generate an appropriate action for a new object. 
More broadly, in image generation and other structured domains~\citep{fu2024incontext,gu2024analogist}, the desired behavior may depend on visual, geometric, or stylistic criteria that are cumbersome to specify perfectly in natural language or mathematical form. 
In such cases, demonstrations serve as a natural interface for specifying what should be generated, and in-context generation is the mechanism that turns this interface into a working system.

Yet demonstrations are not necessarily a complete specification of the underlying task. With only finitely many examples, multiple latent rules may remain compatible with the same context while implying different outputs for a new query, leaving uncertainty about what should be generated~\citep{xieetal2022,panwar2024incontext}. The output distribution that the context induces at a new query is then frequently non-degenerate or even multimodal. 
Collapsing this distribution to its mean can yield implausible outputs (Fig.~\ref{fig:conceptual-overview}; see also \citealp{Blau_2018_CVPR,chi2023diffusionpolicy}).
Selecting a single plausible candidate instead discards the relative probabilities of alternatives, preventing faithful reproduction of the distribution and potentially compromising downstream decision-making under uncertainty~\citep{Kohl2018A,Hoeting1999Bayesian,Gneiting2007Strictly}.
Thus, what in-context generation should deliver is not a single prediction but samples drawn from the predictive distribution itself, so that the uncertainty left by the finite demonstrations is preserved.

\begin{figure}[t]
    \centering
    \includegraphics[width=0.9\linewidth]{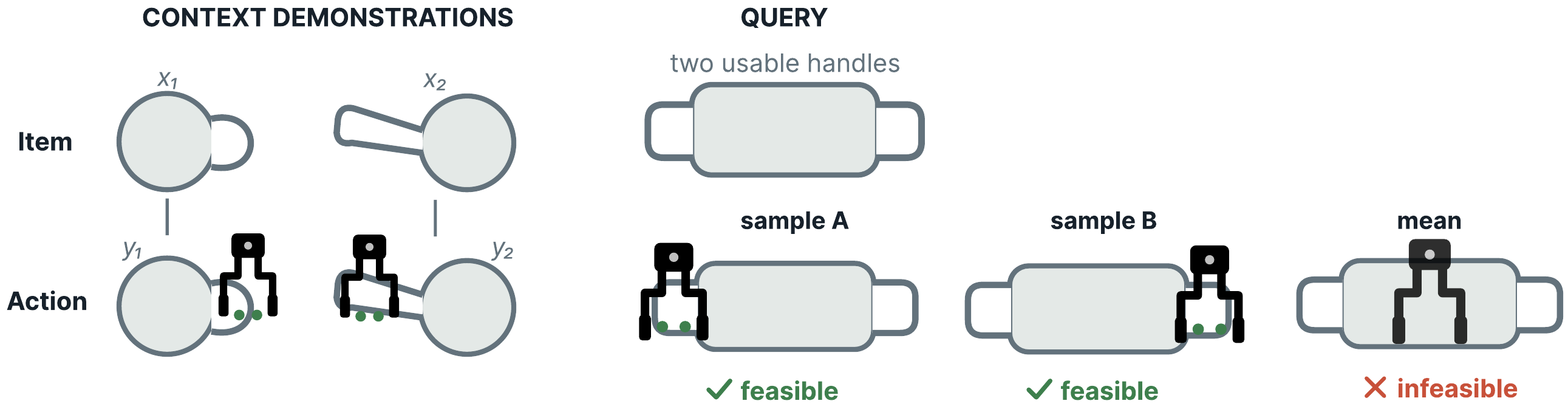}
    % \vspace{-10pt}
        \caption{\textbf{In-context generation from demonstrations.} A user specifies a desired behavior through a few examples rather than an explicit rule. Because a finite context may be consistent with multiple underlying rules, different outputs can remain plausible for the same query; collapsing the plausible outputs to a single average can produce an invalid result. }
    \label{fig:conceptual-overview}
    % \vspace{-10pt}
\end{figure}

Despite substantial theoretical progress on both in-context learning and diffusion models, existing frameworks do not cover this generative setting. Most finite-sample statistical analyses of ICL focus on point prediction under regression-type losses, studying how pretrained Transformers output optimal regression rules or Bayes predictors from a finite context~\citep{bai2023transformers,kim2024transformers,wakayamasuzuki2025}. In contrast, conditional generative modeling and conditional diffusion models aim to learn an entire output distribution given an observed covariate or conditioning variable~\citep{fuetal2024conditional,huetal2024conditionaldit}. These settings typically assume that, at the population level, the observed condition indexes a well-defined conditional distribution. In-context generation is different: the condition itself is a finite collection of demonstrations generated from an unobserved task, and provides only partial information about the test distribution. 
Consequently, it remains unclear how a predictive distribution should be formed from finitely many demonstrations, and how a pretrained generative model can learn and sample from this distribution across tasks.

In this paper, we formulate in-context generation as sampling from the distribution of outputs that remain plausible after finite demonstrations, rather than committing to a single inferred task or output. We identify this target with the posterior-predictive distribution and give a finite-sample theory of how a pretrained diffusion Transformer~\citep[DiT;][]{peebles2023dit} learns the corresponding prompt-to-distribution map and realizes it through diffusion sampling.
Our main findings are as follows.
\begin{itemize}
    \item \textbf{Pretraining learns the shared rule; demonstrations reduce task uncertainty.}
    %We separate the error from learning and sampling the shared prompt-to-distribution mapping from the uncertainty left by the current prompt: more pretraining episodes and computation can reduce the former, but cannot in general remove the latter.
    We identify two distinct error sources: error from learning and sampling the shared prompt-to-distribution mapping, and the uncertainty left by the current prompt. More pretraining data and computation can reduce the former, but cannot in general eliminate the latter.
    \item \textbf{Attention infers; diffusion generates.} Our construction shows that attention aggregates evidence from demonstrations into weights over plausible tasks, while diffusion converts the resulting predictive mixture into samples. With finite pretraining data and no task labels, a sufficiently large DiT and sampler can approximate the posterior-predictive distribution to any prescribed accuracy under our assumptions.
    \item \textbf{Minimax-optimal in-context generation via DiT.} We establish matching upper and lower bounds for finite-prompt uncertainty and show that, over the nondegenerate Hölder task class, the learned DiT attains the minimax-optimal prompt-length rate $\sqrt{\log n /n}$ in task-conditional total-variation error, given sufficient model capacity, pretraining episodes, and diffusion steps. This demonstrates the test-time adaptive generation of DiT.
\end{itemize}

The paper is organized as follows. Section~\ref{sec:main-setup-and-model} defines the data-generating process and evaluation criteria, then introduces the DiT, its training objective, and its reverse sampler. Section~\ref{sec:main-theoretical-guarantees} explains the task-family assumptions and the score-learning, prompt-uncertainty, and task-conditional generation guarantees, together with the main ideas behind their proofs. Section~\ref{sec:main-discussion} presents a summary and practical implications. The appendices contain the proofs and empirical illustrations.

\vspace{-5pt}
\subsection{Related Work}\label{sec:related-work}
\vspace{-5pt}

\textit{(A) ICL theory.}
Existing statistical theories of ICL primarily study point prediction, often under regression-type losses, and characterize how pretrained Transformers implement or approximate task-adaptive estimators \citep{bai2023transformers,lietal2023,zhang2024trained,kim2024transformers}. Bayesian analyses further interpret ICL through posterior prediction over latent tasks~\citep{wakayamasuzuki2025,ma2026optimal}. Our focus is different: we study generation of the full posterior-predictive output distribution and its finite-sample error under distributional loss.
\citet{kangetal2026} give a constructive Transformer approximation of the posterior-predictive distribution in Gaussian-process (GP) regression. However, their theory assumes a known GP model with fixed kernel and noise parameters and does not establish generalization from finite pretraining data. 
By contrast, our result learns a prompt-conditioned diffusion score from finitely many independent pretraining examples without knowing the H\"older task family or task prior, and controls expected posterior-predictive TV error of the resulting finite-step reverse diffusion sampler.

\textit{(B) Score learning with Transformers.}
Prior work shows Transformers can represent score estimators or denoising updates \citep{wuetal2025deepicl,smartetal2025denoising,rosuetal2025};
\citet{adibietal2026} construct frozen Transformers simulating closed-form diffusion recursions and approximating an estimation-free empirical sampler. Our focus is statistical rather than merely representational: we analyze a DiT learned from finite pretraining data and show that finite-step reverse diffusion achieves posterior-predictive accuracy, achieving the minimax-optimal rate with respect to prompt length.

Further related work on diffusion model theory is discussed in Appendix~\ref{app:sec:related-work-comparisons}.
\vspace{-5pt}

\section{In-Context Generation Setup and Diffusion Transformer}
\label{sec:main-setup-and-model}
\vspace{-5pt}

We first specify the data-generating process and the two output distributions used to evaluate in-context generation.
We then introduce the prompt-conditioned Diffusion Transformer, its training objective, and its reverse diffusion sampler.
The task-family assumptions used to prove the guarantees are stated separately in Section~\ref{sec:main-theoretical-guarantees}.
Appendix~\ref{app:sec:formal-setup} gives the formal probability conventions.

\vspace{-5pt}
\subsection{Data-generating process}
\label{subsec:main-episodic-framework}
\vspace{-5pt}

Fix a prompt length $n\ge1$ and an output dimension $d_z\ge1$, and let $\gZ=[0,1]^{d_z}$ be the output domain.
Let $\gT$ be a standard Borel task space, and let $\Pi$ be the Borel probability distribution over tasks encountered during pretraining.
The same distribution serves as the task prior in posterior prediction.
Adapting the task-mixture prompt-generating formulation~\citep{ching2026efficient,wakayamasuzuki2025,ma2026optimal} to distributional generation, we consider the following two models.

\begin{definition}[Prompt-generating process]
\label{def:main-prompt-generating-process}
Each prompt and its target are generated as follows.
\begin{enumerate}[labelindent=0pt,leftmargin=*,itemsep=2pt,parsep=1pt,topsep=0pt,partopsep=0pt]
  \item Sample a latent task $\Theta\sim\Pi$.

  \item Conditional on $\Theta$, generate $n+1$ complete examples independently from the task-specific distribution:
  \begin{itemize}
    \item For \textbf{sample-only prompts}, a task $\theta$ specifies a distribution $P_\theta$ on $\gZ$, where $\theta\mapsto P_\theta$ is a Borel probability kernel, and $Z_1,\ldots,Z_n,Z_\star \stackrel{\mathrm{iid}}{\sim}P_\Theta$.
    \item For \textbf{input--output prompts}, fix an input dimension $d_x\ge1$ and let $\gX=[0,1]^{d_x}$.
    A task $\theta$ specifies a joint distribution $P_\theta^{X,Z}$ on $\gX\times\gZ$, where $\theta\mapsto P_\theta^{X,Z}$ is a Borel probability kernel, and $(X_1,Z_1),\ldots,(X_n,Z_n),(X_\star,Z_\star) \stackrel{\mathrm{iid}}{\sim}P_\Theta^{X,Z}$.
  \end{itemize}

  \item Form the context from the first $n$ examples and hold out the output in the final example as the target: $C_n=Z_{1:n}, H_n=C_n$ for sample-only prompts and $C_n=((X_i,Z_i))_{i=1}^n, H_n=(C_n,X_\star)$ for input--output prompts.
\end{enumerate}
\end{definition}

% In both prompt models, the prompt $H_n$ excludes the target output $Z_\star$.
Write $\gH_n:=\gZ^n$ and $\gH_n:=(\gX\times\gZ)^n\times\gX$ for sample-only and input--output prompts, respectively, so that $H_n\in\gH_n$.
The prompt never includes the target output $Z_\star$.
% For input--output prompts, it includes the query $X_\star$.
% Because the input marginal may depend on the task, this observed query can also provide information about $\Theta$.
In the input--output model, it includes the observed query $X_\star$, which can itself provide information about $\Theta$ when the input marginal is task-dependent.
We call one draw of a latent task together with its prompt and held-out target an \textbf{episode}; its observed part is the prompt--target pair $(H_n,Z_\star)$.
Under either prompt model, the $n+1$ complete examples are conditionally i.i.d.\ given $\Theta$ and are generally dependent after marginalizing over their shared latent task.
In the input--output prompt model, $H_n$ contains $X_\star$, so $H_n$ and $Z_\star$ need not be conditionally independent given $\Theta$.

\textbf{Pretraining and test-time generation.} 
Repeating this procedure independently $N$ times yields the pretraining prompt--output pairs $((H_n^{(e)},Z_{\star}^{(e)}))_{e=1}^N$. The model is pretrained on these pairs without observing the task labels. At test time, a fresh episode is generated independently using the same meta-prior $\Pi$, and the pretrained model generates an output conditioned on the observed prompt $H_n$. For a fixed pretrained model and prompt, let $\widehat{P}(\cdot\,|\,H_n)$ denote the output distribution induced by the generation procedure's randomness.

\vspace{-5pt}
\subsection{Risk decomposition via a posterior-predictive oracle}
\label{subsec:main-generation-targets}
\vspace{-5pt}

\textbf{Task-conditional generation risk.}
% Our goal is to generate from the output distribution of the task underlying the prompt:
Our goal is distributional: generate from the task's output distribution underlying the prompt rather than recover its latent parameter $\Theta$:
\begin{equation*}
  P(\cdot\,|\,H_n,\Theta)=\Law(Z_\star\,|\,H_n,\Theta).
  % \label{eq:main-evaluation-distributions}
\end{equation*}
The task-conditional distribution equals $P_\Theta$ for sample-only prompts and $P_\Theta^{Z\,|\,X}(\cdot\mid X_\star)$ for input--output prompts, almost surely.
For any Borel probability kernel $Q(\cdot\mid H_n)$, define its generation risk by the total variation (TV) error averaged over a fresh test episode\footnote{Unless stated otherwise, $\E$ in a test-risk bound conditions on the completed training run and averages over the fresh task $\Theta$ and prompt $H_n$, including the query $X_{\star}$ in the input--output model.}:
\begin{equation}
  \gR_n(Q):=\E\TV\!\bigl(Q(\cdot\mid H_n),P(\cdot\mid H_n,\Theta)\bigr).
  \label{eq:main-task-conditional-generation-risk}
\end{equation}
The generator observes $H_n$ but not $\Theta$, and a finite prompt need not identify the task that produced it.

\textbf{A posterior-predictive oracle.}
The posterior predictive distribution $P(\cdot\mid H_n)=\Law(Z_\star\mid H_n)$ averages the task-conditional output distributions given the observed prompt:
\begin{equation}
    P(\cdot\mid h)=\int P(\cdot\mid h,\theta)\,\Pi(\od\theta\mid h),
    \qquad\text{for }\Law(H_n)\text{-almost every }h.
    \label{eq:main-posterior-predictive-mixture}
\end{equation}
In the input--output model, this posterior conditions on both $C_n$ and $X_\star$, so the query can update the task probabilities when the input marginal depends on the task.
For a fixed task family and prior, convexity of TV and the triangle inequality give $\inf_Q\gR_n(Q)\le\gR_n(P)\le2\inf_Q\gR_n(Q)$, where the infimum is over all prompt-based probability kernels, which may know the task family and prior $\Pi$. 
Thus, the posterior predictive distribution is a constant-factor oracle under this information constraint.

This oracle motivates (implicitly) learning the shared map $h\mapsto P(\cdot\mid h)$ as a latent teacher during pretraining. For DiT, this means learning the diffusion score of the posterior predictive distribution, which is also the target of the population denoising objective in Section~\ref{subsec:main-training-and-sampling}.
Using this implicit target in the error bound separates learning and sampling errors from the uncertainty left by a finite test prompt; i.e., for the learned generator $\widehat P$, the triangle inequality gives the principled decomposition
\begin{align}
  \underbrace{\E\TV\!\bigl(
    \widehat{P}(\cdot\,|\, H_n), P(\cdot\,|\, H_n,\Theta)
  \bigr)}_{\text{task-conditional generation error}}
  \!\le\!
  \underbrace{\E\TV\!\bigl(
    \widehat{P}(\cdot\,|\, H_n), P(\cdot\,|\, H_n)
  \bigr)}_{\text{posterior-predictive generation error}}
  +
  \underbrace{\E\TV\!\bigl(
    P(\cdot\,|\, H_n), P(\cdot\,|\, H_n,\Theta)
  \bigr)}_{\text{finite-prompt task uncertainty}}.
  \label{eq:main-generation-error-decomposition}
\end{align}
The first term measures how accurately the pretrained model and its sampler approximate the posterior-predictive target; it includes model approximation, estimation from finite pretraining data, optimization, and sampling errors.
The second term equals the oracle risk $\gR_n(P)$ and measures the uncertainty about the task's output distribution left by the finite test prompt.
It depends on the task's difficulty and prompt length $n$, and can remain positive even with exact posterior-predictive generation.
This uncertainty concerns output distributions rather than task-label identification: distinct tasks may induce the same output distribution.

\vspace{-5pt}
\subsection{Architecture: Prompt-conditioned Diffusion Transformer}
\label{subsec:main-diffusion-transformer}
\vspace{-5pt}

We use a prompt-conditioned DiT to learn the score for posterior-predictive generation. 
Following prior theoretical work on Transformer-based ICL~\citep[e.g.,][]{bai2023transformers,Ahn2023transformer,kim2024transformers}, we analyze a simplified Transformer architecture. 
Details are deferred to Appendices~\ref{app:sec:transformer-approximation} and~\ref{app:sec:sampling}.

\textbf{Conditional diffusion process.}
Keeping the prompt fixed, we perturb $Z_0\equiv Z_\star$ through the Ornstein--Uhlenbeck (OU) process:
\begin{equation}
\od Z_t=-Z_t\,\od t+\sqrt{2}\,\od W_t,\qquad Z_t\stackrel{\mathrm{d}}{=}\alpha_t Z_0+\sigma_t\xi,\quad \alpha_t=e^{-t},\quad \sigma_t^2=1-e^{-2t}.
\label{eq:main-forward-ou}
\end{equation}
Here $W_t$ is standard $d_z$-dimensional Brownian motion independent of $(Z_0,H_n)$, and $\xi\sim\gN(0,I_{d_z})$ is independent of the episode. Write $P_{0\mid h}=P(\cdot\mid h)$ and $P_{t\mid h}=\Law(Z_t\mid H_n=h)$. For $t>0$, the latter has a smooth density $p_t(\cdot\mid h)$ with score
\begin{equation}
s^\star(z,t\mid h)=\nabla_z\log p_t(z\mid h),\qquad z\in\R^{d_z}.
\end{equation}
This score determines the exact reverse drift~\citep{haussmannpardoux1986}. As $t\to\infty$, $P_{t\mid h}$ approaches $\gN(0,I_{d_z})$; starting from this Gaussian at finite $T$ incurs an initialization error.

\textbf{Prompt-conditioned DiT model.}
For the variance-balanced denoising objective in Section~\ref{subsec:main-training-and-sampling}, the network predicts the scaled score $\widetilde{s}^\star=\sigma_t s^\star$. 
Feedforward ReLU networks encode each context example and the query--state--time input $(X_\star,z,t)$, with $X_\star$ omitted for sample-only prompts. %The token sequence also includes auxiliary summary, latent-task, output, and null tokens.

The encoded token array $\mathbf E^{(0)}$ passes through four masked softmax-attention blocks, each with a residual connection followed by a token-wise coordinate update:
\begin{equation*}
  \mathbf E^{(k)}
  =\mathsf R_{k,n}\!\bigl(\mathbf E^{(k-1)}+\operatorname{Att}_{\mathbf v,k}(\mathbf E^{(k-1)})\bigr),
  \qquad k=1,\ldots,4.
\end{equation*}
Here $\operatorname{Att}_{\mathbf v,k}$ uses trainable query, key, value, and output projections with a fixed attention mask, whereas $\mathsf R_{k,n}$ is a fixed affine--ReLU map applied separately to each token.
A fixed ReLU readout maps the final output token to a $d_z$-dimensional vector, and radial projection produces $\widetilde{s}_{\mathbf v}(z,t,h)$.
%Note that, for the present setup, this architecture is sufficient. 
We defer the specific interpretation of the attention blocks to Section~\ref{subsec:main-score-approximation-learning}.

\textbf{Transformer score class.}
Fix $0<t_0<T$ and a deterministic output radius $R_{\mathrm{out},n}>0$. Each admissible parameter vector $\mathbf v$ defines a jointly Borel map
\begin{equation*}
\widetilde{s}_{\mathbf v}:\R^{d_z}\times[t_0,T]\times\gH_n\longrightarrow\{u\in\R^{d_z}:\|u\|_2\le R_{\mathrm{out},n}\},\qquad s_{\mathbf v}(z,t,h)=\frac{\widetilde{s}_{\mathbf v}(z,t,h)}{\sigma_t}.
\end{equation*}
The trainable coordinates comprise the designated encoder coefficients, prior logits, and attention projection matrices, and $\mathsf{N}_{\mathrm{param},n,\varepsilon}^{\mathrm{lat}}$ denotes their total number. We denote by $\gS_{n,\varepsilon}^{\mathrm{lat}}=\{s_{\mathbf v}\}$ the class whose architecture and parameter bounds are chosen using only $(n,\varepsilon)$ and the fixed dimensions and regularity bounds in Assumption~\ref{assump:smooth-latent-task}. The class is independent of the unknown task family, prior, training data, and confidence level. The reverse sampler uses the unscaled score $s_{\mathbf v}$.

\vspace{-5pt}
\subsection{DiT training and reverse diffusion sampling}
\label{subsec:main-training-and-sampling}
We connect DiT pretraining to posterior-predictive generation through three stages: a population variance-balanced DSM (bDSM) objective, its clipped finite-episode empirical counterpart, and randomized reverse diffusion sampling. Details are deferred to Appendices~\ref{app:sec:score-learning} and~\ref{app:sec:sampling}.

\textbf{Population score error and bDSM.}
For $0<t_0<T$, define the integrated score-estimation error
\begin{equation}
\label{eq:score-error}
\mathcal{E}_{\mathrm{score}}(s) := \mathbb{E}_{H_n}\int_{t_0}^{T}\bigl\|s(\cdot,t,H_n)-s^\star(\cdot,t\mid H_n)\bigr\|_{L^2(P_{t\mid H_n})}^{2}\odt.
\end{equation}
Since $s^\star$ is unknown, denoising score matching (DSM) instead uses the Gaussian perturbation as the regression target~\citep{vincent2011,songetal2021}. Predicting the scaled score $\tilde{s}=\sigma_t s$ replaces $-\xi_{\mathrm{dsm}}/\sigma_t$ with the unit-scale target $-\xi_{\mathrm{dsm}}$, but weights squared score error by $\sigma_t^2$. We compensate for this weighting by sampling the diffusion time with density $\varpi_{t_0,T}(t):=(A_{t_0,T}\sigma_t^2)^{-1}$, where $A_{t_0,T}:=\int_{t_0}^{T}\sigma_t^{-2}\,dt$. The positive cutoff $t_0>0$ makes this density normalizable.

Draw $\tau\sim\varpi_{t_0,T}$ and $\xi_{\mathrm{dsm}}\sim\mathcal{N}(0,I_{d_z})$ independently of each other and the episode, and set $Z_\tau=\alpha_\tau Z_0+\sigma_\tau\xi_{\mathrm{dsm}}$. The bDSM loss is
\begin{equation}
\label{eq:bdsm-population}
\mathcal{R}_{\mathrm{bDSM}}(s) := A_{t_0,T}\,\mathbb{E}\bigl\|\tilde{s}(Z_\tau,\tau,H_n)+\xi_{\mathrm{dsm}}\bigr\|_2^2.
\end{equation}
Under the finite-risk conditions in \cref{prop:balanced-dsm-identity}, we have
\begin{equation}
\label{eq:bdsm-excess-risk}
\mathcal{R}_{\mathrm{bDSM}}(s)-\mathcal{R}_{\mathrm{bDSM}}(s^\star)=\mathcal{E}_{\mathrm{score}}(s).
\end{equation}
Thus, excess bDSM risk controls the integrated score error required by the reverse-sampling analysis.

\textbf{Finite-episode pretraining.}
For each of the $N$ independent training episodes, independently draw $\tau_e\sim\varpi_{t_0,T}$ and $\xi_{\mathrm{dsm}}^{(e)}\sim\mathcal{N}(0,I_{d_z})$, and form $Z_{\tau_e}^{(e)}=\alpha_{\tau_e}Z_0^{(e)}+\sigma_{\tau_e}\xi_{\mathrm{dsm}}^{(e)}$. To obtain a bounded empirical loss, define $\xi_{R_{\mathrm{lab}}}^{\mathrm{clip},(e)}:=\operatorname{clip}_{R_{\mathrm{lab}}}(-\xi_{\mathrm{dsm}}^{(e)})$, where $\operatorname{clip}_R$ is radial projection onto the Euclidean ball of radius $R$. The empirical objective is
\[
\widehat{\mathcal{R}}_{N,\mathrm{bDSM}}^{(R_{\mathrm{lab}})}(s)=\frac{A_{t_0,T}}{N}\sum_{e=1}^{N}\bigl\|\tilde{s}(Z_{\tau_e}^{(e)},\tau_e,H_n^{(e)})-\xi_{R_{\mathrm{lab}}}^{\mathrm{clip},(e)}\bigr\|_2^2.
\]
Let $\widehat{s}_{N,n,\varepsilon}$ be an $\varepsilon_{\mathrm{opt}}$-approximate minimizer over $\mathcal{S}_{n,\varepsilon}^{\mathrm{lat}}$ of this $A_{t_0,T}$-scaled objective.

\textbf{Randomized reverse sampler.}
Given a new prompt $h$, we use an exponential--Euler reverse sampler on a randomly shifted time grid. Averaging over the shift connects score errors at the sampler's evaluation times to the integrated error controlled by bDSM.

Define the inverse-variance clock $u(t):=\int_{t_0}^t\sigma_r^{-2}\,\od r$, so its total length is $A_{t_0,T}=u(T)$. 
Choose an integer $M$ satisfying $M\ge\max\{\lceil2A_{t_0,T}\rceil,\lceil A_{t_0,T}/(T-t_0)\rceil\}$, as in \cref{eq:admissible-sampling-updates}, sample a uniform shift over one $u$-clock interval of length $A_{t_0,T}/M$, and let the resulting forward-time nodes be $v_{J_{\mathrm{grid}}}>\cdots>v_0>v_{-1}=t_0$, where the sampler makes $J_{\mathrm{grid}}+1\le M$ score evaluations almost surely. 
Set $\Delta_j=v_j-v_{j-1}$. 
Initialize $\widehat{Z}_{v_{J_{\mathrm{grid}}}}\sim\gN(0,I_{d_z})$ and iterate
\begin{equation*}
  \widehat{Z}_{v_{j-1}}
  =e^{\Delta_j}\widehat{Z}_{v_j}
  +2(e^{\Delta_j}-1)\widehat s_{N,n,\varepsilon}(\widehat{Z}_{v_j},v_j,h)
  +\sqrt{e^{2\Delta_j}-1}\,\xi_j,
  \qquad j=J_{\mathrm{grid}},\ldots,0.
\end{equation*}
The vectors $\xi_j$ are independent and standard Gaussian.
The law of $\proj_{\gZ}(\widehat{Z}_{v_{-1}})$, averaged over all sampler randomness, is $\widehat{P}_{0\mid h}^{(M)}=\widehat P(\cdot\mid h)$.
The sampler uses at most $M$ score evaluations almost surely. Because the reverse process stops at the positive time $t_0$, the projected endpoint incurs an early-stopping error, controlled in Appendix~\ref{app:sec:sampling}.

\vspace{-5pt}
\section{Theoretical Guarantees}
\label{sec:main-theoretical-guarantees}
\vspace{-5pt}

Our theoretical guarantees follow the decomposition in \cref{eq:main-generation-error-decomposition}.
\vspace{-5pt}
\begin{align*}
  \underbrace{\gR_n(\widehat P)}_{
    \substack{\text{total risk}}
  }
  \le
  \underbrace{
    \E\TV\!\bigl(
      \widehat P(\cdot\mid H_n),P(\cdot\mid H_n)
    \bigr)
  }_{
    \substack{\text{pretraining and sampling}\\
              \rightarrow \, \text{Section~\ref{subsec:main-score-approximation-learning}}}
  }
  +
  \underbrace{
    \E\TV\!\bigl(
      P(\cdot\mid H_n),P(\cdot\mid H_n,\Theta)
    \bigr)
  }_{
    \substack{\text{finite-prompt uncertainty}\\
              \rightarrow \, \text{Section~\ref{subsec:main-prompt-uncertainty}}}
  }
\lesssim
\underbrace{
  \textcolor{gray}{
    \Bigl(
      \begin{gathered}
        \text{minimax-optimal}\\[-4pt]
        \text{prompt rate}
      \end{gathered}
    \Bigr)
  }
}_{\rightarrow\,\text{Section~\ref{subsec:main-task-conditional-generation}}}
\end{align*}
Section~\ref{subsec:main-score-approximation-learning} shows that a DiT can approximate the posterior-predictive oracle from finitely many pretraining episodes without knowing the task family or prior. 
The analysis controls score learning from pretraining and translates score error into generation error through finite-step reverse diffusion. 
Also, finite prompts can leave task uncertainty even under exact posterior-predictive generation. Section~\ref{subsec:main-prompt-uncertainty} bounds this uncertainty and establishes a matching worst-case lower bound, even for generators that know the task family and prior. Section~\ref{subsec:main-task-conditional-generation} then combines the bounds at the prompt-uncertainty scale to obtain a task-conditional generation guarantee, establishing minimax optimality in the prompt length under the assumption below. %Under the assumption below, it establishes minimax optimality in the prompt length after averaging over training and gives sufficient bounds on model size, the number of pretraining episodes, and the number of score evaluations. 

\textbf{Task-family assumptions.}
The episodic setup in Section~\ref{sec:main-setup-and-model} does not itself require densities or regularity conditions.
The guarantees below specialize it to an idealized class of positive task densities with regularity in both the observations and the latent task parameter.
% Let $g_\theta$ denote the sample density for sample-only prompts and the joint density for input--output prompts; its domain dimension is $d_{\mathrm{stat}}=d_z$ or $d_x+d_z$, respectively.
The corresponding density domain has dimension $d_{\mathrm{stat}}=d_z$ in the sample-only model and $d_{\mathrm{stat}}=d_x+d_z$ in the input--output model. 
Write $\gF_{\beta,+}^d(L_\beta,\underline{p})$ for the densities on $[0,1]^d$ with $\beta$-H\"older norm at most $L_\beta$ and lower bound $\underline{p}>0$ (Appendix~\ref{app:subsec:regularity-conventions}).

\begin{assumption}[Smooth latent tasks with positive densities]
  \label{assump:smooth-latent-task}
  Tasks are indexed by a parameter $\theta\in\gT=[0,1]^{d_{\theta}}$, with $d_{\theta}\in\mathbb{N}$ fixed across the model class and independent of $n$, but the sampled $\Theta$ remains latent and $\Pi$ may be any Borel probability measure on $\gT$.
  Let $g_\theta$ denote the Lebesgue density of $P_\theta$ or $P_\theta^{X,Z}$ in the sample-only or input--output model, respectively. 
  Fix $\beta>0$, $L_{\beta}>1$, $\underline{p}\in(0,1]$, $\beta_{\theta}\in(0,1]$, and $H_{\theta}\in[0,\infty)$.
  The map $(\theta,u)\mapsto g_{\theta}(u)$ is jointly Borel and, for all $\theta,\theta'\in\gT$,
  \begin{equation*}
    g_{\theta}\in\gF_{\beta,+}^{d_{\mathrm{stat}}}(L_{\beta},\underline{p}),
    \qquad
    \sup_{u\in[0,1]^{d_{\mathrm{stat}}}}|g_{\theta}(u)-g_{\theta'}(u)|
    \le H_{\theta}\|\theta-\theta'\|_2^{\beta_{\theta}}.
    % \label{eq:smooth-latent-task-family}
  \end{equation*}
\end{assumption}

Fix either the sample-only or input--output prompt model.
Let $\mathfrak{M}_{\mathrm{lat}}$ denote the class of task families
\begin{equation*}
 \mathcal{M}=\bigl(\{g_\theta\}_{\theta\in\gT},\Pi\bigr)
\end{equation*}
that satisfy \cref{assump:smooth-latent-task} with the same displayed dimensions and regularity bounds.
% Each $\mathcal{M}\in\mathfrak{M}_{\mathrm{lat}}$ induces the training- and test-task distributions defined in Section~\ref{subsec:main-episodic-framework}.
Each $\mathcal{M}\in\mathfrak{M}_{\mathrm{lat}}$ induces a hierarchical distribution of $(\Theta,H_n,Z_{\star})$ through \cref{def:main-prompt-generating-process}.
Pretraining observes independent copies of $(H_n,Z_{\star})$, while test risk uses an independent fresh copy of $(\Theta,H_n,Z_{\star})$ and gives the trained generator only $H_n$. 
The learner knows the common dimensions and regularity bounds, but not $\mathcal{M}$, its task-density family, its task prior, or the sampled task parameters. 

% The regularity bounds hold at every latent task parameter, including points with zero prior mass. Spatial regularity controls neural approximation and early stopping, while latent regularity controls how finely the task family must be discretized. Positivity keeps log likelihoods bounded and stabilizes the score computation in Section~\ref{subsec:main-score-approximation-learning}.

This assumption is a model-level analogue of the H\"older regularity imposed on the Bayes predictor in \citet{wakayamasuzuki2025}. We instead impose the regularity on the task-generating densities rather than on the prompt-to-score map itself, making the condition explicit and readily verifiable for concrete distribution families; see Appendix~\ref{app:sec:smooth-latent-task-examples} for examples including a multi-modal task family.
% 
%For example, a latent coordinate can shift a positive density's peak while the prompt contains samples from one task. Appendix~\ref{app:sec:admissible-task-families} provides examples including multi-modal task family; Appendix~\ref{app:sec:oracle-diagnostics} illustrates the two error targets.
The sampling result also holds under weaker spatial conditions stated locally in \cref{cor:randomized-grid-update-bound}.

\vspace{-2pt}
\subsection{DiTs learn posterior-predictive generation during pretraining}
\label{subsec:main-score-approximation-learning}
\vspace{-2pt}

During pretraining, a DiT can learn to generate the posterior-predictive distribution for a new prompt.

\begin{theorem}[Posterior-predictive generation; informal version of Theorem~\ref{thm:main-generation-error-formal}]
\label{thm:main-generation-error}
Under \cref{assump:smooth-latent-task}, for each $n\ge2$ and sufficiently small $\varepsilon>0$, there exists a common DiT class and a choice of diffusion parameters with model size and forward-pass cost polynomial in $(n,\varepsilon^{-1})$. For every fixed $\mathcal{M}\in\mathfrak{M}_{\mathrm{lat}}$ and $\delta\in(0,1)$, pretraining on $N$ independent episodes to empirical optimization gap $\varepsilon_{\mathrm{opt}}\lesssim\varepsilon^2$ and sampling with $M$ admissible reverse steps, as in Section~\ref{subsec:main-training-and-sampling}, yield, with probability at least $1-\delta$ over the training episodes and their diffusion augmentations,
\begin{equation}
\E\TV\!\bigl(\widehat P^{(M)}_{0\mid H_n},P(\cdot\mid H_n)\bigr)
\lesssim
\varepsilon
+ \log(e/\varepsilon)\sqrt{\frac{\mathsf{N}_{\mathrm{param},n,\varepsilon}^{\mathrm{lat}}\log(en/\varepsilon)+\log(2/\delta)}{N}}
+ \frac{\log(e/\varepsilon)}{\sqrt M},
\label{eq:main-predictive-generation-bound}
\end{equation}
where $\mathsf{N}_{\mathrm{param},n,\varepsilon}^{\mathrm{lat}}$ is the number of trainable parameters.
\end{theorem}

The result goes beyond known-model representability results \citep[e.g.,][]{kangetal2026}: it gives a finite-sample guarantee for learning and generating the posterior-predictive distribution from finite pretraining data, \emph{without task labels or knowledge of the task family or prior}, and with explicit TV control for finite-step reverse diffusion. Thus, posterior-predictive generation is not merely representable under a known model; it is statistically learnable as an implicit target during pretraining.

For $\Ord(\varepsilon)$ predictive TV error, it suffices to choose
\begin{equation}
  N\gtrsim\varepsilon^{-2}\log^2(e/\varepsilon)\bigl\{\mathsf{N}_{\mathrm{param},n,\varepsilon}^{\mathrm{lat}}\log(en/\varepsilon)+\log(2/\delta)\bigr\} 
  \quad\text{and}\quad
  M\gtrsim\varepsilon^{-2}\log^2(e/\varepsilon).
  \label{eq:main-predictive-sufficient-choices}
\end{equation} 
% Thus, polynomial model size, number of pretraining episodes, and number of score evaluations suffice at the stated optimization accuracy.
Up to logarithmic factors, pretraining episodes scale linearly with parameter count, while score evaluations scale as $\varepsilon^{-2}$. Thus, all required resources are polynomial.

Under the task regularities (Assumption~\ref{assump:smooth-latent-task}), the main technical challenge is to obtain a common score class with resource bounds polynomial in $(n,\varepsilon^{-1})$ as the prompt length grows.
In the following two paragraphs, we explain how DiT realizes posterior averaging and why polynomially many representative tasks suffice.

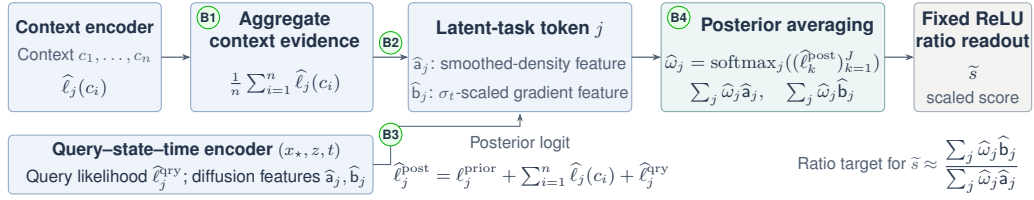
\begin{figure}[!tp]
\centering
\begingroup%
\definecolor{figInk}{HTML}{243349}%
\definecolor{figMuted}{HTML}{536174}%
\definecolor{figBlue}{HTML}{305D92}%
\definecolor{figBlueFill}{HTML}{EFF4FA}%
\definecolor{figTeal}{HTML}{246E70}%
\definecolor{figTealFill}{HTML}{EDF6F4}%
\definecolor{figLine}{HTML}{9CA8B6}%
\definecolor{figFixed}{HTML}{F3F3F1}%
\resizebox{0.97\linewidth}{!}{%
\begin{tikzpicture}[
  x=1mm,y=1mm,
  font=\sffamily\fontsize{9}{10.6}\selectfont,text=figInk,
  box/.style={rounded corners=.8mm,line width=.5pt},
  flow/.style={-{Stealth[length=1.5mm,width=1.1mm]},draw=figMuted,line width=.6pt},
  aux/.style={-{Stealth[length=1.4mm,width=1mm]},draw=figMuted,line width=.5pt},
  label/.style={font=\sffamily\fontsize{8}{9.2}\selectfont,text=figMuted},
  title/.style={font=\sffamily\bfseries\fontsize{9}{10.6}\selectfont},
  heading/.style={font=\sffamily\bfseries\fontsize{9.5}{11}\selectfont},
  every node/.style={inner sep=0pt,outer sep=0pt,align=center}
]
  \draw[box,draw=figBlue!55,fill=figBlueFill] (0,11) rectangle (26,29.5);
  \node[title,font=\sffamily\bfseries\fontsize{8.5}{10}\selectfont] at (13,25.5) {Context encoder};
  \node[label] at (13,20.6) {Context $c_1,\ldots,c_n$};
  \node at (13,15.3) {$\widehat\ell_j(c_i)$};

  \draw[box,draw=figBlue!55,fill=figBlueFill] (31,11) rectangle (62,29.5);
  \node[circle,draw=green!70!black,fill=green!8,inner sep=.65pt,font=\sffamily\bfseries\fontsize{6.2}{6.8}\selectfont] at (34,27) {B1};
  \node[title] at (47,25) {Aggregate\\context evidence};
  \node at (47,16) {$\frac{1}{n}\sum_{i=1}^n\widehat\ell_j(c_i)$};

  \draw[box,draw=figBlue!55,fill=figBlueFill] (68,11) rectangle (106,29.5);
  \node[title] at (87,25) {Latent-task token $j$};
  \node[label] at (87,18.8) {$\widehat{\mathsf a}_j$: smoothed-density feature};
  \node[label] at (87,14.3) {$\widehat{\mathsf b}_j$: $\sigma_t$-scaled gradient feature};

  \draw[box,draw=figTeal!60,fill=figTealFill] (111,11) rectangle (149,29.5);
  \node[circle,draw=green!70!black,fill=green!8,inner sep=.65pt,font=\sffamily\bfseries\fontsize{6.2}{6.8}\selectfont] at (114,27) {B4};
  \node[title] at (132.5,25.5) {Posterior averaging};
  \node[font=\fontsize{8.8}{10}\selectfont] at (130,19.5) {$\widehat\omega_j=\operatorname{softmax}_j((\widehat\ell_k^{\mathrm{post}})_{k=1}^J)$};
  \node at (130,14.3) {$\sum_j\widehat\omega_j\widehat{\mathsf a}_j,\quad\sum_j\widehat\omega_j\widehat{\mathsf b}_j$};

  \draw[box,draw=figLine,fill=figFixed] (154,11) rectangle (174,29.5);
  \node[title] at (164,25) {Fixed ReLU\\ratio readout};
  \node at (164,17.7) {$\widetilde s$};
  \node[label] at (164,13.4) {scaled score};

  \draw[flow] (26,20.25)--(31,20.25);
  \draw[flow] (62,20.25)--(68,20.25);
  \node[circle,draw=green!70!black,fill=green!8,inner sep=.65pt,font=\sffamily\bfseries\fontsize{6.2}{6.8}\selectfont] at (65,22.7) {B2};
  \draw[flow] (106,20.25)--(111,20.25);
  \draw[flow] (149,20.25)--(154,20.25);

  \draw[box,draw=figBlue!55,fill=figBlueFill] (0,-3) rectangle (62,7);
  \node[title,font=\sffamily\bfseries\fontsize{8.5}{10}\selectfont] at (32,4.4) {Query--state--time encoder $(x_\star,z,t)$};
  \node[font=\sffamily\fontsize{8.2}{9.4}\selectfont] at (32,-.3) {Query likelihood $\widehat\ell_j^{\mathrm{qry}}$; diffusion features $\widehat{\mathsf a}_j,\widehat{\mathsf b}_j$};
  \draw[aux] (62,2)--(65,2)--(65,8.5)--(87,8.5)--(87,10.7);
  \node[circle,draw=green!70!black,fill=green!8,inner sep=.65pt,font=\sffamily\bfseries\fontsize{6.2}{6.8}\selectfont] at (65,7.1) {B3};

  \node[label] at (87,5.6) {Posterior logit};
  \node[font=\fontsize{9}{10.5}\selectfont] at (89,.2){$\widehat\ell_j^{\mathrm{post}}=\ell_j^{\mathrm{prior}}+\sum_{i=1}^n\widehat\ell_j(c_i)+\widehat\ell_j^{\mathrm{qry}}$};
  \node[label,anchor=east] at (158,2) {Ratio target for $\widetilde s \approx$};
  \node[anchor=east,font=\fontsize{9.5}{10.5}\selectfont] at (172,2)
    {$\displaystyle\frac{\sum_j\widehat\omega_j\widehat{\mathsf b}_j}{\sum_j\widehat\omega_j\widehat{\mathsf a}_j}$};
\end{tikzpicture}%
}%
\endgroup
\caption{Conceptual view of the representability construction. 
The four attention blocks: 
B1 pools context log likelihoods; B2 copies the pooled values to each latent-task token; B3 forms task logits and attaches diffusion features; and B4 forms posterior-weighted averages of the density and $\sigma_t$-scaled gradient features. 
A fixed ReLU readout approximates their ratio to produce $\widetilde s$. 
}
\label{fig:main-smooth-latent-score-construction}
\end{figure}

\textbf{The structural alignment between Bayes’ rule and the attention circuit.}
A key property of DiT in in-context generation is that attention can implement posterior averaging by pooling example log likelihoods and using softmax to obtain posterior task weights.
For a prompt of length $n$, discretize the latent space into $J_{\theta,n}$ representative tasks $\{\theta_{j,n}\}_{j=1}^{J_{\theta,n}}$ at resolution $\eta$. 
For $\zeta\in(0,1]$, add $\zeta/J_{\theta,n}$ to each prior cell weight and renormalize. 
For brevity, set $J:=J_{\theta,n}$ and $\theta_j:=\theta_{j,n}$, and below omit the $n$-index from other grid quantities and the $\zeta$-dependence from the logits and weights.
Set $c_i=z_i$ in the sample-only model and $c_i=(x_i,z_i)$ in the input--output model. Let $\ell_j^{\mathrm{qry}}(x_\star)$ denote the query-input log marginal density, taking it to be zero in the sample-only model, and let $\ell_j^{\mathrm{prior}}$ denote the corresponding prior logit.
Bayes' rule gives the task logits:
\begin{equation}
  \ell_j^{\mathrm{post}}(h) =\ell_j^{\mathrm{prior}} +\sum_{i=1}^n\log g_{\theta_j}(c_i)+\ell_j^{\mathrm{qry}}(x_\star), \qquad \omega_j(h) =\operatorname{softmax}_j\!\bigl((\ell_k^{\mathrm{post}}(h))_{k=1}^{J}\bigr).
  \label{eq:main-smooth-latent-posterior-logit-update}
\end{equation}
Attention pools the example log likelihoods, while softmax converts the resulting logits into posterior weights.
Let $q_{j,t}:=q_{\theta_j,t}$ denote the diffused task-conditional density under representative task $\theta_j$, with conditioning on $x_\star$ omitted in the sample-only model. 
The finite model's predictive score is
\begin{equation}
  s^{\eta,\zeta}(z,t\mid h) =\frac{\sum_{j=1}^{J}\omega_j(h)\nabla_z q_{j,t}(z\mid x_\star)}{\sum_{j=1}^{J}\omega_j(h)q_{j,t}(z\mid x_\star)}.
  \label{eq:main-discretized-predictive-score}
\end{equation}
The construction in \cref{fig:main-smooth-latent-score-construction} uses four attention blocks (B1--B4) and a fixed readout to approximate this computation.
The final attention block averages density and $\sigma_t$-scaled gradient features rather than task scores; the readout approximates their ratio to produce $\widetilde{s}^{\eta,\zeta}=\sigma_t s^{\eta,\zeta}$.
Training over the full class need not preserve this internal posterior interpretation.
Appendices~\ref{app:subsec:episode-kl-score-approximation} and~\ref{app:subsubsec:transformer-construction} give the formal episode-level discretization, neural approximation, and token-level construction, respectively.
\vspace{-2pt}

\textbf{Why polynomially many representative tasks suffice.} 
% Pointwise posterior approximation is unnecessarily strong. 
Uniformly approximating the posterior for every prompt is stronger than the integrated accuracy required by score matching and reverse diffusion.
Instead, we compare the original and discretized episode distributions. The conditional KL chain rule and OU entropy dissipation yield
\vspace{-2pt}
\begin{equation}
  \int_{t_0}^T \E\bigl\|s^\star(Z_t,t\mid H_n)-s^{\eta,\zeta}(Z_t,t\mid H_n)\bigr\|_2^2\,\od t \lesssim (n+1)\eta^{2\beta_\theta}+\log(1+\zeta),
  \label{eq:main-integrated-grid-bound}
\end{equation}
where the expectation uses the original model. Conditional independence makes the KL cost additive over the $n+1$ examples in an episode. Then, choosing $\eta^{\beta_\theta}\asymp\varepsilon/\sqrt{n+1}$ and $\zeta\asymp\varepsilon^2$ gives $\Ord(\varepsilon^2)$ integrated score error with only $J = J_{\theta,n}\lesssim (1+\sqrt{n+1}/\varepsilon)^{d_\theta/\beta_\theta}$ representative tasks. 
Combining this bound with the parameter counts for the prior logits, the ReLU approximation of the fixed-dimensional likelihood and diffusion features, and all trainable attention projections, Appendix~\ref{app:subsec:smooth-latent-complexity} shows that the resulting total number of trainable scalar parameters $\mathsf{N}_{\mathrm{param},n,\varepsilon}^{\mathrm{lat}}$ is polynomial in $(n,\varepsilon^{-1})$. 
Crucially, $\mathsf{N}_{\mathrm{param},n,\varepsilon}^{\mathrm{lat}}$ controls the covering number of the same score class and yields the pretraining requirement in \cref{eq:main-predictive-sufficient-choices}. 
Finally, the randomized inverse-variance grid converts the integrated score error controlled by bDSM into finite-step TV accuracy. % without requiring a uniform-in-time score-error bound.
% Appendices~\ref{app:subsec:episode-kl-score-approximation}, \ref{app:subsubsec:transformer-construction}, and~\ref{app:subsec:posterior-predictive-generation-proof} give these three steps.

\subsection{A sharp information rate from finite prompts at test time}
\label{subsec:main-prompt-uncertainty}
\vspace{-3pt}

\textbf{A sharp rate for prompt-based generation.}
A finite prompt can leave several task-specific output distributions consistent with the demonstrations. The next result gives the sharp worst-case task-conditional accuracy available to any prompt-based generator. The lower bound lets the generator know the full task family and prior and includes generation randomness in its output distribution.

\begin{theorem}[Sharp finite-prompt rate; informal version of \cref{thm:main-task-uncertainty-and-minimax-rate}]
\label{thm:informal-task-uncertainty}
Under \cref{assump:smooth-latent-task}, the following statements hold.
\begin{itemize}
  \item \textbf{Uniform upper bound.}
  Every $\mathcal{M}\in\mathfrak{M}_{\mathrm{lat}}$ and every $n\ge2$ satisfy
  \begin{equation*}
    \E_{\mathcal{M}}\operatorname{TV}\!\left(P(\cdot\mid H_n),P(\cdot\mid H_n,\Theta)\right)
    \lesssim \sqrt{\log n/n}.
  \end{equation*}
  \item \textbf{Matching worst-case lower bound.}
  If, in addition, $H_\theta>0$, $\underline{p}<1$, and $L_\beta>1$, then for every sufficiently large $n$, there is a model $\mathcal{M}_n\in\mathfrak{M}_{\mathrm{lat}}$ for which
  \begin{equation*}
    \inf_{Q^{\mathcal{M}_n}}\E_{\mathcal{M}_n}\operatorname{TV}\!\left(Q^{\mathcal{M}_n}(\cdot\mid H_n),P(\cdot\mid H_n,\Theta)\right)
    \gtrsim \sqrt{\log n/n},
  \end{equation*}
  where the infimum is over prompt-only Borel probability kernels from $\gH_n$ to $\gZ$ that may know $\mathcal{M}_n$.
\end{itemize}
\end{theorem}

\textbf{Posterior prediction makes oracle use of the prompt.}
The lower bound is an information limit: it remains valid with arbitrarily many independent pretraining episodes and unrestricted computation at generation time. For each fixed task family and prior, the posterior predictive has expected task-conditional TV risk at most twice that of the best prompt-only generator (see \cref{eq:latent-oracle-uncertainty-comparison}), implying it is an oracle prompt-only generator up to a constant factor. \emph{What matters is uncertainty about the output distribution, not the task label alone.} Distinct tasks cause no generation error when they induce the same output distribution, while multimodality within a task should remain. The controlled experiment in Appendix~\ref{app:numerical-illustration} illustrates the posterior-predictive behavior: the generated distribution is broadly bimodal under a short, ambiguous context and contracts toward a narrow, essentially unimodal distribution as the context grows.

\textbf{Sharpness in a rich latent-task class.}
Fixed latent dimension still allows many weakly separated output distributions. The matching lower bound embeds about \(\log n\) weakly distinguishable changes along a single latent coordinate, showing that the logarithmic factor reflects the richness of the task class. This is a worst-case result: the hard family and prior may depend on \(n\). A fixed family may have a faster rate or zero uncertainty.

\subsection{The minimax-optimal rate with a pretrained DiT}
\label{subsec:main-task-conditional-generation}

\textbf{Reaching the known-model rate from data.}
Section~\ref{subsec:main-prompt-uncertainty} grants the oracle generator the task family and prior. The DiT learning rule receives neither and observes no task labels. Yet \cref{thm:main-generation-error,thm:informal-task-uncertainty} together give the same worst-case prompt-length rate.

\begin{corollary}[Task-conditional generation at the sharp prompt rate; informal version of \cref{cor:task-conditional-prompt-rate}]
\label{cor:informal-task-generation}
For sufficiently large $n$, set $\varepsilon=\sqrt{\log n/n}$, use the score class and estimator of \cref{thm:main-generation-error} with $\varepsilon_{\mathrm{opt}}\lesssim\log n/n$, and let $N$ and $M$ satisfy \cref{eq:main-predictive-sufficient-choices} at this value of $\varepsilon$. For every fixed $\mathcal{M}\in\mathfrak{M}_{\mathrm{lat}}$ and $\delta\in(0,1)$, with probability at least $1-\delta$ over pretraining and its diffusion augmentations,
\begin{equation}
  \E\operatorname{TV}\!\left(\widehat P^{(M)}_{0\mid H_n},P(\cdot\mid H_n,\Theta)\right)
  \lesssim \sqrt{\frac{\log n}{n}}.
  \label{eq:informal-task-generation}
\end{equation}
Here the expectation is over a fresh test episode conditional on the completed training run. 
\end{corollary}

\textbf{Matching learning and sampling accuracy to prompt information.}
Pretraining episodes estimate the shared prompt-to-distribution rule, model capacity controls its approximation, and diffusion steps control sampling accuracy. Additional demonstrations instead provide more information about the fresh task. Keeping the learning and sampling errors at $\Ord(\sqrt{\log n/n})$ lets the DiT attain the optimal rate for test-time generation. %Smaller predictive error remains possible at fixed $n$, but is not needed for the class-wide rate.

\textbf{A geometric view of in-context generation.}
Assumption~\ref{assump:smooth-latent-task} defines a H\"older-parameterized family of task distributions. 
A finite prompt induces a posterior over this family, and the posterior predictive is the corresponding posterior-weighted barycenter of the distributions, which need not itself belong to the family (see Example~\ref{ex:smooth-latent-peak-shift}, where task densities lie on a circle, while the posterior predictive can lie in its interior).
In our finite representation, attention realizes this geometry by assigning posterior weights to representative tasks, so that its softmax weights can be viewed as barycentric weights over the latent task family. The same low-dimensional regularity controls statistical complexity: only polynomially many representative tasks are needed to approximate the predictive score, while the cumulative information about the latent task grows only logarithmically with the prompt length. This yields the $\Ord(\sqrt{\log n/n})$ convergence of the posterior predictive toward the realized task distribution in expected total variation.

\section{Conclusion and Discussion}
\label{sec:main-discussion}

We established finite-sample guarantees for posterior-predictive generation with a pretrained DiT, separating learning and sampling errors from finite-prompt task uncertainty. Under our assumptions, sufficient model capacity, pretraining data, optimization accuracy, and sampling resources yield the minimax-optimal prompt-length rate in task-conditional total variation. 

\textbf{Preserving uncertainty is part of generative accuracy.} A finite prompt may leave several task hypotheses plausible, so multimodality may be appropriate and mode collapse may constitute error~\citep[cf.][]{Kohl2018A}. A sharper output distribution should not automatically count as improved quality. 
\begin{itemize}
    \item As an architectural implication, because the posterior predictive need not contract to any single task distribution, our results motivate soft representations of plausible tasks without mandatory commitment like hard routing or switching~\citep{JMLR:v23:21-0998}. 
    \item Our theory naturally raises a future research question: \emph{how to modify the test-time distribution?} Guidance~\citep{ho2021classifierfree,fuetal2024conditional}, reward-based reranking or best-of-$K$ selection~\citep{amini2025variational} can improve sample utility while changing the probability mass assigned to plausible alternatives. How to improve utility while preserving, or explicitly controlling departures from, posterior-predictive calibration remains an open direction for in-context generation.
\end{itemize}

\textbf{Resource allocation should follow the source of uncertainty.} Once learning and sampling errors are small relative to prompt uncertainty, additional diffusion steps cannot supply the missing task information. If the remaining multimodality reflects ambiguity between tasks, requesting another informative demonstration may therefore be more useful, motivating an extension to active demonstration acquisition \citep{rainforth2024modern}. For pretraining, our generalization analysis counts independent episodes: observations within an episode share a latent task and are generally dependent. Total token count alone consequently does not characterize the effective training sample size. The appendix summarizes the corresponding bottlenecks and interventions.

\section*{Acknowledgment}

TW was partially supported by JSPS KAKENHI (26K21188) and RIKEN Incentive Research Project.
RK was partially supported by JSPS KAKENHI (24K02905) and JST BOOST (JPMJBS2418).
WSL is supported by the Ministry of Education, Singapore, under its Academic Research Fund Tier 1 (A-8001814-00-00).
TS was partially supported by JST CREST (PMJCR2015). 
This work was partially supported by JST ERATO (JPMJER2601).
This research is supported by the National Research Foundation, Singapore and the Ministry of Digital Development and Information under the AI Visiting Professorship Programme (award number AIVP-2024-004). Any opinions, findings and conclusions or recommendations expressed in this material are those of the author(s) and do not reflect the views of National Research Foundation, Singapore and the Ministry of Digital Development and Information.

% \newpage

% https://iclr.cc/Conferences/2027/AIPolicyForAuthors 

\section*{AI use statement}
In this work, we used generative AI tools to assist in the writing of proofs, design or provide feedback on experiments, implement methods, and assist with translation. We have not used generative AI tools to generate synthetic datasets, help develop theoretical models or conceptual frameworks, formulate mathematical claims, provide critical ingredients for proving mathematical claims, propose or refine hypotheses, clean and reformat datasets, support qualitative and thematic data analysis, or interpret results. No other tasks with required disclosure are applicable to this work. Additionally, we used generative AI tools to suggest experimental parameters, search for and identify relevant literature, assist with \LaTeX{} formatting, and proofread the manuscript to improve readability. We have reviewed all AI-assisted work. We take responsibility for the final content of this work, including text, claims or artifacts produced with the aid of generative AI.

\bibliography{reference}
\bibliographystyle{iclr2027_conference}

\newpage

% {
% \hypersetup{linkcolor=black}
% \setcounter{tocdepth}{2}
% \tableofcontents
% }

\appendix

\addtocontents{toc}{\protect\setcounter{tocdepth}{2}}
{\small
\renewcommand{\contentsname}{Appendix Contents}
\tableofcontents
}

\clearpage

\section{Additional Related Work}
\label{app:sec:related-work-comparisons}

Section~\ref{sec:related-work} reviews the closely related work on ICL and Transformer-based score estimation. Here, we highlight the two diffusion-specific ingredients needed for our end-to-end analysis: learning a prompt-conditioned score from finitely many independent pretraining episodes, and separating posterior-predictive approximation error from the uncertainty that remains with a finite prompt.

Diffusion-model theory commonly separates score estimation from the stability and discretization analysis of the reverse-time process.
Score-to-sampling results translate suitable $L^2$ score errors along the forward marginals into distributional sampling guarantees under various moment and regularity conditions \citep{chenetal2023,bentonetal2024,lee2023convergence,chen2023improved,li2023towards}.
Complementary minimax analyses ask whether empirical score matching followed by reverse diffusion estimates a fixed data distribution at an optimal statistical rate~\citep{okoetal2023,wibisono2024optimal,chen2023score,zhangetal2024,fu2025approximation}.
Related work also exploits structural or manifold assumptions to obtain generalization, convergence, or intrinsic-dimensional guarantees~\citep{yakovlevpuchkin2025,tang2024adaptivity,azangulov2024convergence,fuetal2026}.
Our analysis uses the same broad score-to-distribution interface, but its input is the integrated prompt-conditioned risk of a score estimator learned from independent pretraining episodes.
The randomly shifted inverse-variance time grid converts this integrated-in-time control into an expected total-variation guarantee for a finite-step reverse sampler, without requiring a score-error bound that holds uniformly over time.

Conditional diffusion theory typically studies families of distributions indexed by an observed covariate \citep{fuetal2024conditional,tang2025conditional}, including approximation and estimation guarantees for conditional DiTs \citep{huetal2024conditionaldit}.
In our episodic model, by contrast, the conditioning object is a random prompt $H_n$ generated from an unobserved task.
The prompt provides evidence about the latent task but does not directly index a single realized-task distribution.
Consequently, the posterior-predictive distribution $P(\cdot\mid H_n)$ is the score-learning target, while its discrepancy from the realized-task distribution $P(\cdot\mid H_n,\Theta)$ is analyzed separately as finite-prompt task uncertainty.
Under the stated nondegeneracy conditions, our minimax prompt-length lower bound quantifies the unavoidable finite-prompt contribution to task-conditional total-variation risk.
%It is not a lower bound for posterior-predictive generation or for ordinary unconditional or covariate-conditional density estimation.

\clearpage

\section{Formal Setup and Conventions}
\label{app:sec:formal-setup}

\Cref{sec:main-setup-and-model} defines the episodic problem, the sample-only and input--output prompt models, evaluation distributions, and error decomposition.
This appendix fixes the formal conventions used throughout the proofs. It proceeds from probability metrics and regular conditioning to the two latent-task prompt models, then specifies independent-episode pretraining and test-time generation, and finally records the H\"older classes used in the regularity assumptions.

\subsection{Probability metrics and conditioning conventions}
\label{app:subsec:probability-conventions}

All spaces are standard Borel unless stated otherwise.
For a space $E$, $\mathscr{B}(E)$ and $\gP(E)$ denote its Borel $\sigma$-field and Borel probability measures, respectively.
We equip $\gP(E)$ with the canonical evaluation $\sigma$-field generated by the maps $\mu\mapsto\mu(A)$ for $A\in\mathscr{B}(E)$; this makes $\gP(E)$ standard Borel and agrees with the weak Borel $\sigma$-field under any compatible Polish realization of $E$.
For a measurable map $G$, $G_{\#} P$ denotes the pushforward of $P$.
For $P, Q \in \gP(E)$, we use
\begin{equation}
  \TV(P, Q) = \sup_{A \in \mathscr{B}(E)} |P(A)-Q(A)|,
  \qquad
  \KL(P\|Q) = \int \log\!\left(\frac{\od P}{\od Q}\right) \,\od P,
  \label{eq:metric-conventions}
\end{equation}
with $\KL(P\|Q) = +\infty$ when $P \not\ll Q$.
For densities $p, q$ relative to a common measure $\mu$, $\TV(P, Q) = \frac{1}{2}\int |p-q| \,\od\mu$.
We also write $\TV(p,q)$ and $\KL(p\|q)$ for the TV distance and KL divergence between their induced measures; the density $1$ denotes the uniform law when the domain is a unit-volume cube.
We write $\gN(m,\Sigma)$ for a Gaussian distribution and, for positive-definite $\Sigma$, $\gN(u;m,\Sigma)$ for its density evaluated at $u$.
For nonnegative quantities $a$ and $b$, we write $a\lesssim b$ when $a\le cb$ for a positive constant $c$ that may depend only on the fixed dimensions, regularity bounds, and auxiliary constructions stated in the result.
The hidden constant is independent of the varying quantities $n$, $N$, $M$, $\varepsilon$, and $\delta$, as well as the particular model, prior, and observed data.
We define $\gtrsim$ and $\asymp$ analogously; subscripts on these symbols record any additional allowed dependencies.
Named constants remain explicit when they are fixed and reused in a construction or in later parameter choices.

Upper-case $P,Q$ denote distributions or kernels and lower-case $p,q$ their densities; for task densities, $p_\theta$ denotes the sample-only density, $f_\theta$ the joint $(X,Z)$-density, and $g_\theta$ either one according to the prompt model.
Euclidean densities are relative to Lebesgue measure unless stated otherwise.
Lower-case $h$ and $\bc_n=(c_1,\ldots,c_n)$ denote realizations of $H_n$ and $C_n$, with $c_i=z_i$ in the sample-only prompt model and $c_i=(x_i,z_i)$ in the input--output prompt model.
We use $\ind\{A\}$ for an indicator and $\mathbf{1}_d$ for the all-ones vector.
Expectations average the stated randomness; test-risk expectations have the conditional meaning specified in \cref{app:subsec:episodic-learning}.

Fix $d_z \ge 1$, $\gZ = [0, 1]^{d_z}$, an integer prompt length $n \ge 1$, and a prompt $H_n$ taking values in a standard Borel space $\gH_n$.
When comparing a law on $\gZ$ with an OU-smoothed law on $\R^{d_z}$, we identify the former with its pushforward under the inclusion $\gZ\hookrightarrow\R^{d_z}$.
Write $\sP_{\mathrm{epi}}^{(n)}:=\Law(H_n,Z_{\star})$ for the distribution of observed episodes.
Standard-Borel regular conditioning permits us to choose a probability kernel
\begin{equation}
  h \mapsto P(\cdot \,|\, h) := \Law(Z_{\star} \,|\, H_n = h).
  \label{eq:posterior-predictive-kernel}
\end{equation}
Because $Z_{\star}\in\gZ$ almost surely, we choose this version to satisfy $P(\gZ\,|\, h)=1$ for every $h$, replacing it by a fixed point mass on the measurable exceptional set if necessary.
For $A\in\mathscr{B}(\gZ)$, $P(A\,|\,H_n)$ denotes the measurable map $h\mapsto P(A\,|\,h)$ evaluated at $H_n$, so $P(\cdot\,|\,H_n)$ is the corresponding random probability measure.
Density, smoothness, positivity, and within-episode conditional-independence conditions are imposed only by the task-family assumptions used for a given result.

\subsection{Latent-task prompt models}
\label{app:subsec:prompt-models}

We formalize the latent-task construction for the sample-only and input--output prompt models in \cref{subsec:main-episodic-framework}.
Fix a standard Borel task space $\gT$, a task prior $\Pi\in\gP(\gT)$, and Borel kernels $\theta\mapsto P_{H_n\,|\,\theta}$ and $(h,\theta)\mapsto P_{Z_{\star}\,|\,H_n,\theta}(\cdot\,|\,h)$.
The prior $\Pi$ is also the distribution over tasks in the pretraining data-generating process.
The induced distribution of observed episodes is
\begin{equation}
  \sP_{\mathrm{epi}}^{(n)}(\od h,\od z_{\star})
  =\int_{\gT}P_{H_n\,|\,\theta}(\od h)
  P_{Z_{\star}\,|\,H_n,\theta}(\od z_{\star}\,|\,h)\,\Pi(\od\theta).
  \label{eq:latent-task-episode-distribution}
\end{equation}
Only $(H_n,Z_{\star})$ is observed, while $\Theta$ is retained for task-conditional evaluation.
The standard-Borel assumptions give a regular conditional distribution $\Pi(\od\theta\,|\,h)$ of the task given the prompt.
Define the task-conditional output kernel by
\begin{equation}
  P(\cdot\,|\,h,\theta):=P_{Z_{\star}\,|\,H_n,\theta}(\cdot\,|\,h).
  \label{eq:task-conditional-kernel}
\end{equation}
By the tower property, we may choose the kernel in \cref{eq:posterior-predictive-kernel} so that the posterior-mixture identity \cref{eq:main-posterior-predictive-mixture} holds as an identity of probability measures for $\Law(H_n)$-almost every $h$.
A countable determining algebra and the monotone-class theorem give a single common exceptional null set for all Borel subsets of $\gZ$.

\subsubsection{Sample-only in-context generation}
\label{app:subsubsec:sample-only-model}

Without inputs, let $\theta\mapsto P_{\theta}$ be a Borel probability kernel on $\gZ$ and set
\begin{equation}
  Z_1,\ldots,Z_n,Z_{\star}\mid\Theta\stackrel{\iid}{\sim}P_{\Theta},
  \qquad
  H_n=Z_{1:n}.
  \label{eq:hierarchical-sample}
\end{equation}
The task-conditional output distribution is $P(\cdot\,|\,h,\theta)=P_{\theta}$ for $\Law(H_n,\Theta)$-almost every $(h,\theta)$.
If $P_{\theta}(\od z)=p_{\theta}(z)\,\od z$ with a jointly Borel density, then the posterior-predictive density is $p(z\,|\,h)=\int_{\gT}p_{\theta}(z)\,\Pi(\od\theta\,|\,h)$ for $\Law(H_n)$-almost every $h$.
The episodic formulation itself does not require densities.

\subsubsection{Input--output in-context generation}
\label{app:subsubsec:input-output-model}

Let $\gX$ be standard Borel and let $\theta\mapsto P_{\theta}^{X,Z}$ be a Borel probability kernel on $\gX\times\gZ$.
Conditional on the task, set
\begin{equation}
  (X_1,Z_1),\ldots,(X_n,Z_n),(X_{\star},Z_{\star})\mid\Theta
  \stackrel{\iid}{\sim}P_{\Theta}^{X,Z},
  \qquad
  H_n=\bigl(((X_i,Z_i))_{i=1}^n,X_{\star}\bigr).
  \label{eq:input-output-episode}
\end{equation}
Choose a jointly measurable regular conditional kernel $(\theta,x)\mapsto P_{\theta}^{Z\mid X}(\cdot\mid x)$.
For $\Law(H_n,\Theta)$-almost every $(h,\theta)$ with $h=(\bc_n,x_{\star})$, the task-conditional output distribution is
\begin{equation}
  P(\cdot\,|\,h,\theta)=P_{\theta}^{Z\mid X}(\cdot\mid x_{\star}).
  \label{eq:input-output-task-conditional-distribution}
\end{equation}
The posterior-mixture identity in \cref{eq:main-posterior-predictive-mixture} therefore conditions the task posterior on both the context and the observed query.
The query affects both the task-conditional output distribution and the task posterior when the input marginal is task-dependent.

\subsection{Independent-episode pretraining and test-time generation}
\label{app:subsec:episodic-learning}

Let the training sample contain $N\ge1$ independent copies of the observed episode:
\begin{equation}
  \gD_N^{\mathrm{epi}}
  :=\bigl((H_n^{(e)},Z_{\star}^{(e)})\bigr)_{e=1}^N,
  \qquad
  (H_n^{(e)},Z_{\star}^{(e)})\stackrel{\iid}{\sim}\sP_{\mathrm{epi}}^{(n)}.
  \label{eq:pretraining-episode-sample}
\end{equation}
For $N=0$, define $\gD_0^{\mathrm{epi}}$ as the deterministic empty sample.
Any declared sample splits are mutually independent, and a fresh test episode is independent of the training sample and all training randomness.
A possibly randomized learning procedure takes $\gD_N^{\mathrm{epi}}$ and an independent auxiliary seed as inputs and is jointly Borel in the training sample, seed, and prompt.
Every completed training run therefore defines a probability kernel
\begin{equation}
  h\longmapsto\widehat{P}(\cdot\,|\,h)\in\gP(\gZ).
  \label{eq:learned-generation-kernel}
\end{equation}
The kernel $\widehat{P}(\cdot\,|\,h)$ is the output distribution induced by the test-time sampler and therefore already integrates over generation-time randomness.

Unless stated otherwise, $\E$ in a test-risk bound is conditional on the completed training run and averages an independent fresh episode.
Under the latent-task model, this expectation includes $\Theta$ and $H_n$, including the fresh query $X_{\star}$ in the input--output model.
An outer expectation may additionally average the training sample and training randomness.
We use the three expected TV distances in \cref{eq:main-generation-error-decomposition} directly: predictive error, task-conditional error, and task uncertainty.
The last depends only on the episodic model and prompt length, not on training or the generator.

\subsection{H\"older density and task-family conventions}
\label{app:subsec:regularity-conventions}

The following conventions define the H\"older classes used in \cref{sec:main-theoretical-guarantees}.
Both the sampling and task-uncertainty analyses use these conventions.

Fix $\beta > 0$ and an integer $d \ge 1$.
Write $m_{\beta} := \lceil\beta\rceil-1$ and $\vartheta_{\beta} := \beta-m_{\beta} \in (0, 1]$.
For a function $g$ on $[0, 1]^d$ whose derivatives through order $m_{\beta}$ exist in the interior and extend continuously to the cube, define
\begin{equation}
  \|g\|_{C^{\beta}([0, 1]^d)}
  :=
  \max_{|a| \le m_{\beta}}\|D^ag\|_{\infty}
  +
  \max_{|a| = m_{\beta}}
  \sup_{u \ne v}
  \frac{|D^ag(u)-D^ag(v)|}{\|u-v\|_2^{\vartheta_{\beta}}},
  \label{eq:holder-norm-convention}
\end{equation}
where $a$ ranges over multi-indices.
With this convention, integer smoothness $\beta$ means $C^{\beta-1, 1}$ regularity.
Since $[0, 1]^d$ has unit volume and $\|p\|_{C^{\beta}} \ge \|p\|_{\infty} \ge 1$ for every density $p$, the class below is empty when $0 < L_{\beta} < 1$.
We therefore take $L_{\beta} \ge 1$ and let
\begin{equation}
  \gF_{\beta}^d(L_{\beta})
  :=
  \Bigl\{p:[0, 1]^d \to [0, \infty):
    \int_{[0, 1]^d}p(u) \,\od u = 1,
    \|p\|_{C^{\beta}([0, 1]^d)} \le L_{\beta}
  \Bigr\}.
  \label{eq:holder-density-class}
\end{equation}
For $\underline p\in(0,1]$, define the positive-density subclass
\begin{equation*}
  \gF_{\beta,+}^d(L_\beta,\underline p)
  :=
  \bigl\{p\in\gF_\beta^d(L_\beta):
  \inf_{u\in[0,1]^d}p(u)\ge\underline p\bigr\}.
\end{equation*}
Because the H\"older norm contains $\|p\|_{\infty}$, every member of the full class automatically satisfies $p(u) \le L_{\beta}$ on the cube.
At $L_{\beta} = 1$, the full class $\gF_{\beta}^d(L_{\beta})$ is the singleton containing the uniform density; when $L_{\beta} > 1$, that full class may contain densities that vanish and may have task-dependent supports.

The density classes above also support the density-estimation comparison.
\Cref{cor:conditional-tv-stopping,cor:randomized-grid-update-bound} state their spatial regularity conditions locally; they do not require the positivity or latent regularity imposed by \cref{assump:smooth-latent-task}.

\clearpage

\section{Conditional OU Scores and Variance-Balanced DSM}
\label{app:sec:diffusion}

This appendix supplies the diffusion identities that connect the statistical and sampling analyses. We first derive the conditional OU score and its positive-time bounds, and then show that inverse-variance time sampling makes the excess variance-balanced DSM risk exactly equal to the integrated score error. The reverse processes and sampling guarantees are deferred to \cref{app:sec:sampling}.

As in \cref{subsec:main-diffusion-transformer}, set $Z_0\equiv Z_{\star}$ and $Z_0^{(e)}\equiv Z_{\star}^{(e)}$; the zero subscript denotes diffusion time.
Write $P_{0\,|\,h}:=P(\cdot\,|\,h)$ for the same posterior-predictive distribution.

\subsection{OU smoothing and conditional score identities}
\label{app:subsec:ou-score-identities}

Fix $T>0$ and a prompt realization $h$.
The unit-rate OU process in \cref{eq:main-forward-ou} has the following fixed-time representation~\citep{karatzasshreve1991,songetal2021}.
For each $t\ge0$,
\begin{equation}
  Z_t \stackrel{\mathrm{d}}{=} \alpha_tZ_0+\sigma_t\xi,
  \qquad
  \alpha_t = e^{-t}, \quad \sigma_t^2 = 1-e^{-2t},
  \label{eq:forward-ou-fixed-time}
\end{equation}
where $\xi \sim \gN(0, I_{d_z})$ is independent of $Z_0$.
For any initial distribution $P$, write $P_t$ for the distribution of $\alpha_t Z+\sigma_t\xi$ with independent $Z\sim P$ and $\xi\sim\gN(0,I_{d_z})$.
In particular, $P_{t\,|\,h}:=\Law(Z_t\,|\,H_n=h)$ is the OU evolution of $P_{0\,|\,h}$.
For every $t > 0$, this distribution has an everywhere positive smooth density $p_t(\cdot \,|\, h)$, even if $P_{0 \,|\, h}$ is singular.
Its conditional score function is
\begin{equation}
  s^{\star}(z, t \,|\, h)
  := \nabla_z\log p_t(z \,|\, h).
  \label{eq:exact-prompt-score}
\end{equation}
The Gaussian-convolution formula and positivity provide a jointly Borel version of $(z, t, h) \mapsto s^{\star}(z, t \,|\, h)$ on $\R^{d_z} \times (0, T] \times \gH_n$; this version is used throughout.
Gaussian differentiation and Bayes' formula further give
\begin{align}
  s^{\star}(z,t\,|\, h)
  &=-\frac{z-\alpha_t\E[Z_0\,|\, Z_t=z,H_n=h]}{\sigma_t^2},
  \nonumber\\
  \nabla_zs^{\star}(z,t\,|\, h)
  &=-\frac{I_{d_z}}{\sigma_t^2}
  +\frac{\alpha_t^2}{\sigma_t^4}
  \Cov(Z_0\,|\, Z_t=z,H_n=h).
  \label{eq:positive-time-score-regularity}
\end{align}
Since $P_{0\,|\, h}$ is supported on $[0,1]^{d_z}$, for every $\varepsilon\in(0,T)$ and $t\in[\varepsilon,T]$,
\begin{equation}
  \|s^{\star}(z,t\,|\, h)\|_2
  \le \frac{\|z\|_2+\sqrt{d_z}}{\sigma_{\varepsilon}^2},
  \qquad
  \|\nabla_zs^{\star}(z,t\,|\, h)\|_{\mathrm{op}}
  \le \frac{1}{\sigma_{\varepsilon}^2}
  +\frac{d_z}{\sigma_{\varepsilon}^4}.
  \label{eq:positive-time-score-bounds}
\end{equation}
The denoising identity also yields $\E[\|s^{\star}(Z_t,t\,|\, H_n)\|_2^2\,|\, H_n]\le d_z/\sigma_t^2$.
These bounds justify positive-time reversal in \cref{app:subsec:exact-reverse-diffusion} and provide regularity estimates for score approximation.
The evaluations at the random prompt are denoted by $P_{t \,|\, H_n}$ and $s^{\star}(\cdot, t \,|\, H_n)$.

\subsection{Variance-balanced DSM and exact excess risk}
\label{app:subsec:variance-balanced-dsm}

Variance-balanced DSM trains the scaled-score predictor using the label $-\xi_{\mathrm{dsm}}$.
The clipped empirical objective is defined in \cref{app:subsec:clipped-dsm-empirical-objective}; here we establish the unclipped population identity.

Fix $0<t_0<T$.
For a measurable score function $s(z,t,h)$, recall the integrated squared score error introduced in Section~\ref{subsec:main-training-and-sampling}:
\begin{equation}
  \gE_{\mathrm{score}}(s)
  := \E_{H_n}\!\left[\int_{t_0}^T
  \bigl\|s(\cdot, t, H_n)-s^{\star}(\cdot, t \,|\, H_n)\bigr\|_{L^2(P_{t \,|\, H_n})}^2 \,\od t\right].
  \label{eq:population-score-error}
\end{equation}
$\gE_{\mathrm{score}}$ is not directly observable because the true score $s^{\star}$ is unknown.

The standard denoising identity~\citep{vincent2011} states that, for every fixed $t > 0$,
\begin{equation*}
  \E\!\left[-\frac{\xi_{\mathrm{dsm}}}{\sigma_t}
  \,|\, Z_t, t, H_n\right]
  = s^{\star}(Z_t, t \,|\, H_n),
  \qquad
  Z_t = \alpha_tZ_0+\sigma_t\xi_{\mathrm{dsm}}.
\end{equation*}
Under uniform time sampling, direct score regression uses the label $-\xi_{\mathrm{dsm}}/\sigma_t$, whose second moment grows as $t\downarrow0$.
Predicting $\sigma_t s$ instead uses the unit-scale label $-\xi_{\mathrm{dsm}}$, but yields a $\sigma_t^2$-weighted score error.
Converting that error to $\gE_{\mathrm{score}}$ by a worst-case bound costs a factor $\sigma_{t_0}^{-2}$.
Inverse-variance time sampling compensates for this weighting, so the balanced objective has excess risk exactly $\gE_{\mathrm{score}}$.

Draw $\xi_{\mathrm{dsm}} \sim \gN(0, I_{d_z})$ independently of an episode.
For any score function $s$, define its scaled representation by
\begin{equation}
  \widetilde{s}(z,t,h):=\sigma_t s(z,t,h),
  \qquad
  s(z,t,h)=\frac{\widetilde{s}(z,t,h)}{\sigma_t}.
  \label{eq:balanced-score-representation}
\end{equation}
Thus the same parameterized model predicts $\widetilde{s}$ in the scaled-score training loss and returns $s$ for use in the reverse sampler.

Define
\begin{equation}
  A_{t_0, T} := \int_{t_0}^T\frac{\od t}{\sigma_t^2}
  = \frac{1}{2}\log\left(\frac{e^{2T}-1}{e^{2t_0}-1}\right),
  \qquad
  \varpi_{t_0, T}(t) := \frac{1}{A_{t_0, T}\sigma_t^2}.
  \label{eq:balanced-time-density}
\end{equation}
For the variance-balanced DSM analysis, draw $\tau$ with density $\varpi_{t_0, T}$, independently of the episode and $\xi_{\mathrm{dsm}} \sim \gN(0, I_{d_z})$, and form $Z_{\tau}$ by
\begin{equation}
  Z_{\tau}
  := \alpha_{\tau} Z_0
  +\sigma_{\tau}\xi_{\mathrm{dsm}}.
  \label{eq:dsm-noised-state}
\end{equation}
The positive cutoff remains essential: because $\sigma_t^2 \sim 2t$ as $t\downarrow0$, $A_{0, T} = \infty$, so the balanced time density cannot be normalized at $t_0 = 0$.

As in the main paper, bDSM denotes variance-balanced denoising score matching.
Parameterize its population objective by the score function $s$:
\begin{equation}
  \gR_{\mathrm{bDSM}}(s)
  := A_{t_0, T}\,
  \E\left[\|\widetilde{s}(Z_{\tau},
  \tau, H_n)+\xi_{\mathrm{dsm}}\|_2^2\right].
  \label{eq:balanced-dsm-risk}
\end{equation}

\begin{proposition}[Exact variance-balanced DSM excess-risk identity]
\label{prop:balanced-dsm-identity}
Define $\widetilde{s}^{\star}(z,t\,|\, h):=\sigma_t s^{\star}(z,t\,|\, h)$.
Then
\begin{equation*}
  \E[-\xi_{\mathrm{dsm}} \,|\,
  Z_{\tau}, \tau, H_n]
  = \widetilde{s}^{\star}(Z_{\tau},
  \tau\,|\, H_n)
\end{equation*}
almost surely, and every jointly Borel score $s$ satisfying $\gR_{\mathrm{bDSM}}(s)<\infty$ obeys
\begin{equation}
  \gR_{\mathrm{bDSM}}(s)-\gR_{\mathrm{bDSM}}(s^{\star})
  = \gE_{\mathrm{score}}(s).
  \label{eq:balanced-dsm-excess-identity}
\end{equation}
\end{proposition}

\begin{proof}[Proof of \cref{prop:balanced-dsm-identity}]
The standard denoising identity \citep{vincent2011}, multiplied by $\sigma_t$, gives the conditional mean at each fixed $t>0$.
Joint measurability and disintegration under $\varpi_{t_0,T}$ give the asserted identity at $\tau$.
Conditional Jensen and $\xi_{\mathrm{dsm}}\in L^2$ imply $\widetilde{s}^{\star}\in L^2$ and finite risk.
For every finite-risk $s$, the squared-loss Pythagorean identity therefore yields
\begin{align*}
  \gR_{\mathrm{bDSM}}(s)-\gR_{\mathrm{bDSM}}(s^{\star})
  &=A_{t_0,T}\E\|\widetilde{s}(Z_{\tau},\tau,H_n)
       -\widetilde{s}^{\star}(Z_{\tau},\tau\mid H_n)\|_2^2\\
  &=\gE_{\mathrm{score}}(s),
\end{align*}
where the second equality uses $A_{t_0,T}\varpi_{t_0,T}(t)\sigma_t^2=1$.
\end{proof}

The identity controls the integrated score error without the factor $\sigma_{t_0}^{-2}$ incurred by a worst-case conversion from uniform-time noise prediction.
The inverse-variance time density compensates for the score scaling by $\sigma_t$; it is not claimed to minimize variance among all time-sampling schemes.

\clearpage

\section{Latent-Task Regularity and Episode-KL Approximation}
\label{app:sec:latent-task-approximation}

The score representation in \cref{app:sec:diffusion} averages over a continuous latent task space.
This appendix prepares its finite approximation in four steps: compact preprocessing and tail control; likelihood and conditional-density regularity; normalized diffusion features and latent-grid quadrature; and an episode-level KL comparison that yields integrated score error under the true prompt law. 
\Cref{app:sec:transformer-approximation} then approximates the resulting finite-mixture score with neural component maps and an exact four-block attention aggregation.
Unless displayed explicitly, dependence on the fixed diffusion window $(t_0,T)$ and deterministic state radius $R_{\mathrm{state},n}$ is suppressed.

\subsection{Compact preprocessing and score tails}
\label{app:subsec:score-preprocessing-and-tails}

The construction uses the sample-only and input--output prompt models on Euclidean cubes from \cref{app:subsec:prompt-models}.
Recall that $\bc_n=(c_1,\ldots,c_n)$ denotes a realized context.
We first define bounded input features and establish the tail bounds that convert uniform approximation on a compact set into integrated score-error bounds.

Define the query space
\begin{equation*}
  \gX_{\star}
  :=\begin{cases}
  \{0\},&\text{sample-only},\\
  [0,1]^{d_x},&\text{input--output}.
  \end{cases}
\end{equation*}
Write $h=\bc_n$ in the sample-only prompt model and $h=(\bc_n,x_{\star})$ in the input--output prompt model.
To use the same formulas in both cases, set $x_{\star}=0$ and $X_{\star}\equiv0$ in the sample-only prompt model.
Under either the sample-only or the input--output prompt model, the query--state--time feature defined below has dimension $d_{\mathrm{stat}}+1$.
For $R>0$, define radial clipping $\clip_R:\R^{d_z}\to\overline{\Ball}(0,R)$ by
\begin{equation}
  \clip_R(v)
  :=
  \begin{cases}
    v, &\|v\|_2\le R,\\
    Rv/\|v\|_2, &\|v\|_2>R,
  \end{cases}.
  \label{eq:radial-clipping}
\end{equation}
For each $n$, fix a deterministic state-clipping radius $R_{\mathrm{state},n}>\sqrt{d_z}$.
Define the fixed normalized input feature
\begin{equation}
  \mathsf{u}_n(z,t,h)
  :=\begin{cases}
  \left(\clip_{R_{\mathrm{state},n}}(z)/R_{\mathrm{state},n},
  2\dfrac{t-t_0}{T-t_0}-1\right),
  &\text{sample-only},\\
  \left(2x_{\star}-\mathbf{1}_{d_x},
  \clip_{R_{\mathrm{state},n}}(z)/R_{\mathrm{state},n},
  2\dfrac{t-t_0}{T-t_0}-1\right),
  &\text{input--output}.
  \end{cases}
  \label{eq:transformer-input-preprocessing}
\end{equation}
Here $x_{\star}$ is the query component of $h$ in the input--output prompt model.
The fixed map $\mathsf{u}_n$ takes values in $[-1,1]^{d_{\mathrm{stat}}+1}$; in particular, no additional query encoder is required.
The state radius enters only through the approximation domain, tail bounds, and regularity constants.
Recall from \cref{prop:balanced-dsm-identity} that $\widetilde{s}^{\star}(z,t\,|\, h)=\sigma_t s^{\star}(z,t\,|\, h)$.
The regression identity
\begin{equation}
  \gE_{\mathrm{score}}(s)
  =A_{t_0,T}\E\!\left[
  \|\widetilde{s}(Z_{\tau},\tau,H_n)-
  \widetilde{s}^{\star}(Z_{\tau},\tau\,|\, H_n)\|_2^2
  \right]
  \label{eq:score-gap-balanced-regression-form}
\end{equation}
shows why the learning analysis uses the scaled score rather than the unscaled score, which becomes singular as $t\downarrow0$.

For the smooth latent model considered below, every context array and query lies in the specified domain.
The positive-likelihood posterior kernel derived in \cref{app:subsubsec:diffusion-feature-regularity} is defined for every such prompt and agrees almost surely with the conditional distribution chosen in \cref{app:subsec:probability-conventions}.
Set
\begin{equation*}
  R_{\mathrm{bal},n}
  :=\frac{R_{\mathrm{state},n}+\sqrt{d_z}}{\sigma_{t_0}},
  \qquad
  \mathcal{G}_n
  :=\gH_n\times\overline{\Ball}(0,R_{\mathrm{state},n})\times[t_0,T],
\end{equation*}
and let $\delta_{\mathrm{tail},n}:=\Pr((H_n,Z_{\tau},\tau)\notin\mathcal{G}_n)$.

\begin{lemma}[Scaled-score bound and truncation error]
\label{lem:balanced-score-truncation}
For the cube-supported posterior-predictive distributions in \cref{app:subsec:probability-conventions}, $\|\widetilde{s}^{\star}(z,t\mid h)\|_2\le R_{\mathrm{bal},n}$ on $\mathcal{G}_n$, and
\begin{equation}
  \delta_{\mathrm{tail},n}
  \le \Pr\!\left(\|\xi\|_2>R_{\mathrm{state},n}-\sqrt{d_z}\right)<1,
  \qquad \xi\sim\gN(0,I_{d_z}).
  \label{eq:transformer-good-event-probability}
\end{equation}
The squared-score contribution outside the truncation region is finite and satisfies
\begin{align}
  \varepsilon_{\mathrm{tail},n}^2
  &:=A_{t_0,T}\E\!\left[
  \|\widetilde{s}^{\star}(Z_{\tau},\tau\mid H_n)\|_2^2
  \ind\{(H_n,Z_{\tau},\tau)\notin\mathcal{G}_n\}
  \right]
  \label{eq:transformer-exceptional-score-energy}\\
  &\le \frac{A_{t_0,T}}{\sigma_{t_0}^2}
  \E\!\left[
  (\|\xi\|_2+2\sqrt{d_z})^2
  \ind\{\|\xi\|_2>R_{\mathrm{state},n}-\sqrt{d_z}\}
  \right].
  \label{eq:derived-score-tail-bound}
\end{align}
\end{lemma}

\begin{proof}
By \cref{eq:positive-time-score-regularity} and cube support,
\begin{equation*}
  \|\widetilde{s}^{\star}(z,t\mid h)\|_2
  \le\frac{\|z\|_2+\sqrt{d_z}}{\sigma_{t_0}},
  \qquad t\in[t_0,T].
\end{equation*}
This proves the uniform bound on $\mathcal{G}_n$.
Since $\|Z_0\|_2\le\sqrt{d_z}$ and $\alpha_{\tau},\sigma_{\tau}\le1$, forward noising gives
\begin{equation*}
  \{\|Z_{\tau}\|_2>R_{\mathrm{state},n}\}
  \subseteq
  \{\|\xi\|_2>R_{\mathrm{state},n}-\sqrt{d_z}\},
  \qquad
  \|Z_{\tau}\|_2+\sqrt{d_z}\le\|\xi\|_2+2\sqrt{d_z}.
\end{equation*}
The Gaussian probability proves \cref{eq:transformer-good-event-probability}; its strict bound follows from $R_{\mathrm{state},n}>\sqrt{d_z}$.
Combining the score bound with these inequalities and applying Tonelli proves \cref{eq:derived-score-tail-bound}.
Joint Borel measurability follows from the chosen score version and the fixed clipping set, and Gaussian second moments give finiteness.
\end{proof}

\subsection{Likelihood and conditional-density regularity}
\label{app:subsec:likelihood-density-regularity}

Under \cref{assump:smooth-latent-task}, write $g_{\theta}=p_{\theta}$ for sample-only prompts and $g_{\theta}=f_{\theta}$ for input--output prompts.
The learner knows only the common dimensions, regularity bounds, and deterministic grid scheme; the unknown task family and prior enter the analysis through the existence of an approximating parameter vector. 
The prior may be any Borel probability measure, and the regularity and positivity bounds hold at every latent coordinate, including grid points with zero prior mass. 
The latent coordinates determine fixed grid indices but are neither observed nor supplied as numerical model inputs.

The next lemma derives likelihood and conditional-density bounds from the task assumptions.
The episode-KL bound in \cref{prop:episode-kl-score-approximation} uses these bounds to control latent discretization in expected integrated score error.

Set $\beta':=\beta\wedge1$ for the spatial modulus of the likelihood and conditional-density components.
In the input--output prompt model, define
\begin{equation*}
  \lambda_{\theta}(x):=\int_{[0,1]^{d_z}}g_{\theta}(x,y)\,\od y,
  \qquad
  q_{\theta}(z_0\,|\, x):=\frac{g_{\theta}(x,z_0)}{\lambda_{\theta}(x)},
\end{equation*}
and set $\ell_{\theta}:=\log g_{\theta}$ and $\ell_{\theta}^{\mathrm{qry}}:=\log\lambda_{\theta}$.
In the sample-only prompt model, set
\begin{equation*}
  q_{\theta}(z_0\,|\,0):=g_{\theta}(z_0),
  \qquad
  \lambda_{\theta}(0):=1,
  \qquad
  \ell_{\theta}:=\log g_{\theta},
  \qquad
  \ell_{\theta}^{\mathrm{qry}}(0):=0.
\end{equation*}

\begin{lemma}[Likelihood and conditional-density regularity]
\label{lem:positive-smooth-latent-derived-regularity}
Under \cref{assump:smooth-latent-task}, the preceding maps are jointly Borel, and $q_{\theta}(\cdot\,|\, x)$ is a probability density for every $(\theta,x)$.
There are fixed constants $B_{\ell,\mathrm{ctx}}\ge1$, $B_{\ell,\mathrm{qry}}\ge0$, $0<\underline{q}_{\mathrm{out}}\le\overline{q}_{\mathrm{out}}$, and finite H\"older constants such that, for every $\theta,\theta'\in\gT$, every $c,c'\in[0,1]^{d_{\mathrm{stat}}}$, every $x,x'\in\gX_{\star}$, and every $z_0\in[0,1]^{d_z}$,
\begin{equation}
  |\ell_{\theta}(c)|\le B_{\ell,\mathrm{ctx}},
  \qquad
  |\ell_{\theta}^{\mathrm{qry}}(x)|\le B_{\ell,\mathrm{qry}},
  \qquad
  \underline{q}_{\mathrm{out}}\le q_{\theta}(z_0\,|\, x)\le\overline{q}_{\mathrm{out}},
  \label{eq:smooth-latent-derived-envelope}
\end{equation}
and
\begin{align}
  |\ell_{\theta}(c)-\ell_{\theta'}(c')|
  &\le
  H_{\ell,\theta}\|\theta-\theta'\|_2^{\beta_{\theta}}
  {}+
  H_{\ell,c}\|c-c'\|_2^{\beta'},
  \label{eq:smooth-latent-context-regularity}\\
  |\ell_{\theta}^{\mathrm{qry}}(x)
  -\ell_{\theta'}^{\mathrm{qry}}(x')|
  &\le
  H_{\mathrm{qry},\theta}\|\theta-\theta'\|_2^{\beta_{\theta}}
  {}+
  H_{\mathrm{qry},x}\|x-x'\|_2^{\beta'},
  \label{eq:smooth-latent-query-regularity}\\
  \sup_{z_0\in[0,1]^{d_z}}
  |q_{\theta}(z_0\,|\, x)-q_{\theta'}(z_0\,|\, x')|
  &\le
  H_{q,\theta}\|\theta-\theta'\|_2^{\beta_{\theta}}
  {}+
  H_{q,x}\|x-x'\|_2^{\beta'}.
  \label{eq:smooth-latent-output-regularity}
\end{align}
In the sample-only prompt model,
\begin{equation*}
  (B_{\ell,\mathrm{qry}},H_{\mathrm{qry},\theta},
  H_{\mathrm{qry},x},H_{q,x})=(0,0,0,0).
\end{equation*}
The constants depend only on the parameters in \cref{assump:smooth-latent-task}; explicit choices for the sample-only and input--output prompt models are given in \cref{eq:positive-smooth-latent-derived-constants,eq:positive-smooth-latent-derived-constants-sample}.
\end{lemma}

We prove \cref{lem:positive-smooth-latent-derived-regularity} under \cref{assump:smooth-latent-task} and give the constants used in the construction.
Recall that $\beta'=\beta\wedge1$.

\begin{proof}[Proof of \cref{lem:positive-smooth-latent-derived-regularity}]
The defining H\"older bound for $\beta\le1$, and the mean-value theorem with the H\"older norm's gradient bound for $\beta>1$, give a constant
\begin{equation*}
  H_{\mathrm{dens}}:=C_{\beta,d_{\mathrm{stat}}}L_{\beta}
\end{equation*}
such that every $g_{\theta}$ in \cref{assump:smooth-latent-task} satisfies
\begin{equation*}
  |g_{\theta}(u)-g_{\theta}(u')|
  \le H_{\mathrm{dens}}\|u-u'\|_2^{\beta'}.
\end{equation*}
Put
\begin{equation*}
  B_{\log}:=1\vee|\log\underline{p}|\vee\log L_{\beta},
  \qquad
  K_q:=\underline{p}^{-1}+L_{\beta}\underline{p}^{-2}.
\end{equation*}
In the input--output prompt model, one may choose
\begin{align}
  B_{\ell,\mathrm{ctx}}=B_{\ell,\mathrm{qry}}&:=B_{\log},
  &
  (\underline{q}_{\mathrm{out}},\overline{q}_{\mathrm{out}})
  &:=\left(\frac{\underline{p}}{L_{\beta}},
  \frac{L_{\beta}}{\underline{p}}\right),
  \nonumber\\
  H_{\ell,\theta}=H_{\mathrm{qry},\theta}
  &:=\frac{H_{\theta}}{\underline{p}},
  &
  H_{\ell,c}=H_{\mathrm{qry},x}
  &:=\frac{H_{\mathrm{dens}}}{\underline{p}},
  \nonumber\\
  H_{q,\theta}&:=K_qH_{\theta},
  &
  H_{q,x}&:=K_qH_{\mathrm{dens}}.
  \label{eq:positive-smooth-latent-derived-constants}
\end{align}
In the sample-only prompt model, one may instead choose
\begin{equation}
  \begin{gathered}
  B_{\ell,\mathrm{ctx}}:=B_{\log},
  \qquad
  (\underline{q}_{\mathrm{out}},\overline{q}_{\mathrm{out}}):=(\underline{p},L_{\beta}),
  \qquad
  H_{\ell,\theta}:=H_{\theta}/\underline{p},
  \qquad
  H_{\ell,c}:=H_{\mathrm{dens}}/\underline{p},
  \\
  H_{q,\theta}:=H_{\theta},
  \qquad
  (B_{\ell,\mathrm{qry}},H_{\mathrm{qry},\theta},
  H_{\mathrm{qry},x},H_{q,x}):=(0,0,0,0).
  \end{gathered}
  \label{eq:positive-smooth-latent-derived-constants-sample}
\end{equation}

We verify the input--output case.
Because the output cube has unit volume,
\begin{equation*}
  \underline{p}\le\lambda_{\theta}(x)\le L_{\beta}
\end{equation*}
and
\begin{equation*}
  |\lambda_{\theta}(x)-\lambda_{\theta'}(x')|
  \le
  H_{\theta}\|\theta-\theta'\|_2^{\beta_{\theta}}
  +H_{\mathrm{dens}}\|x-x'\|_2^{\beta'}.
\end{equation*}
The mean-value theorem for the logarithm on $[\underline{p},L_{\beta}]$ proves the asserted bounds for $\ell_{\theta}$ and $\ell_{\theta}^{\mathrm{qry}}$.
The bounds on $g_{\theta}$ and $\lambda_{\theta}$ give the displayed values of $\underline{q}_{\mathrm{out}}$ and $\overline{q}_{\mathrm{out}}$.
Moreover, for $u=(x,z_0)$ and $u'=(x',z_0)$,
\begin{align*}
  |q_{\theta}(z_0\mid x)-q_{\theta'}(z_0\mid x')|
  &\le
  \frac{|g_{\theta}(u)-g_{\theta'}(u')|}{\underline{p}}
  +\frac{L_{\beta}
  |\lambda_{\theta}(x)-\lambda_{\theta'}(x')|}{\underline{p}^2}
  \\
  &\le
  K_q\left\{
  H_{\theta}\|\theta-\theta'\|_2^{\beta_{\theta}}
  +H_{\mathrm{dens}}\|x-x'\|_2^{\beta'}
  \right\}.
\end{align*}
This proves \cref{eq:smooth-latent-derived-envelope,eq:smooth-latent-context-regularity,eq:smooth-latent-query-regularity,eq:smooth-latent-output-regularity} with \cref{eq:positive-smooth-latent-derived-constants}.
The sample-only case follows directly from $q_{\theta}=g_{\theta}$, the density bounds, the latent modulus, and the mean-value bound for the logarithm, giving \cref{eq:positive-smooth-latent-derived-constants-sample}.
Joint Borel measurability follows from parameterized integration and composition with continuous functions on intervals bounded away from zero.
\end{proof}

\subsection{Diffusion features and latent-grid quadrature}
\label{app:subsec:latent-quadrature}

We now express the scaled score using the conditional densities just bounded.
The common normalization below keeps the denominator bounded away from zero, which allows us to compare the exact latent average with its grid approximation.

For $t>0$, define
\begin{equation*}
  G_t(z,y)
  :=(2\pi\sigma_t^2)^{-d_z/2}
  \exp\!\left\{-\frac{\|z-\alpha_ty\|_2^2}{2\sigma_t^2}\right\},
  \qquad
  I_t(z):=\int_{[0,1]^{d_z}}G_t(z,y)\,\od y.
\end{equation*}
For $\theta\in\gT$, set
\begin{align}
  \mathsf{a}_{\theta}(x,z,t)
  &:=\frac{\int q_{\theta}(y\,|\, x)G_t(z,y)\,\od y}{I_t(z)},
  \nonumber\\
  \mathsf{b}_{\theta}(x,z,t)
  &:=\frac{\sigma_t\int q_{\theta}(y\,|\, x)\nabla_zG_t(z,y)\,\od y}{I_t(z)}.
  \label{eq:smooth-latent-ou-bases}
\end{align}
Here $\mathsf{a}_{\theta}$ is the ratio of the smoothed density to $I_t(z)$, and $\mathsf{b}_{\theta}$ is the ratio of its $\sigma_t$-scaled gradient to $I_t(z)$.
The density bounds imply $\underline{q}_{\mathrm{out}}\le\mathsf{a}_{\theta}\le\overline{q}_{\mathrm{out}}$.
On $\mathcal{G}_n$, put
\begin{equation*}
  R_{\mathrm{num},n}^{\mathrm{lat}}
  :=\overline{q}_{\mathrm{out}}R_{\mathrm{bal},n}.
\end{equation*}
For every $\theta,\theta'\in\gT$ and every $(h,z,t)\in\mathcal{G}_n$, with $x_{\star}$ denoting the query component of $h$, the basis functions satisfy
\begin{equation}
  \begin{aligned}
    \|\mathsf{b}_{\theta}(x_{\star},z,t)\|_2 
    &\le 
    R_{\mathrm{num},n}^{\mathrm{lat}},
    \\
    |\mathsf{a}_{\theta}(x_{\star},z,t)-\mathsf{a}_{\theta'}(x_{\star},z,t)|
    +\|\mathsf{b}_{\theta}(x_{\star},z,t)-\mathsf{b}_{\theta'}(x_{\star},z,t)\|_2
    &\le
    (1+R_{\mathrm{bal},n})H_{q,\theta}
    \|\theta-\theta'\|_2^{\beta_{\theta}}.
  \end{aligned}
  \label{eq:smooth-latent-basis-latent-regularity}
\end{equation}

For an integer $m_{\theta,n}\ge1$, partition $\gT$ into $J_{\theta,n}:=m_{\theta,n}^{d_{\theta}}$ half-open cubes $(\mathcal{C}_{j,n})_{j=1}^{J_{\theta,n}}$ of side length $m_{\theta,n}^{-1}$, with terminal cells including the outer boundary.
Let $\theta_{j,n}$ be the midpoint of $\mathcal{C}_{j,n}$, and set
\begin{equation*}
  \eta_{\theta,n}:=\frac{\sqrt{d_{\theta}}}{m_{\theta,n}},
  \qquad
  w_{j,n}:=\Pi(\mathcal{C}_{j,n}).
\end{equation*}
Define the Borel cell-midpoint map $\pi_{\eta_{\theta,n}}(\theta):=\theta_{j,n}$ for $\theta\in\mathcal{C}_{j,n}$. Then $\|\theta-\pi_{\eta_{\theta,n}}(\theta)\|_2\le\eta_{\theta,n}$ for every $\theta\in\gT$, so the regular grid below is the discretization in \cref{prop:episode-kl-score-approximation} with $\eta=\eta_{\theta,n}$.
For a prompt $h$, let $\bc_n=(c_1,\ldots,c_n)$ and $x_{\star}$ denote its encoded context and query components, with $x_{\star}=0$ in the sample-only prompt model, and define
\begin{align*}
  \mathsf{Lik}_{j,n}(h)
  &:=\exp\!\left\{
  \ell_{\theta_{j,n}}^{\mathrm{qry}}(x_{\star})
  +\sum_{i=1}^n\ell_{\theta_{j,n}}(c_i)
  \right\},\\
  \omega_{j,n}(h)
  &:=\frac{w_{j,n}\mathsf{Lik}_{j,n}(h)}
  {\sum_{k=1}^{J_{\theta,n}}w_{k,n}
  \mathsf{Lik}_{k,n}(h)}.
\end{align*}
Here $\mathsf{Lik}_{j,n}(h)$ is the prompt likelihood at the $j$th grid point, and $\omega_{j,n}(h)$ is the approximate posterior mass assigned to latent cell $j$.
The resulting quadrature approximation to the scaled score is
\begin{equation}
  \widetilde{s}_n^{\mathrm{quad}}(z,t\,|\, h)
  :=
  \frac{\sum_j\omega_{j,n}(h)
  \mathsf{b}_{\theta_{j,n}}(x_{\star},z,t)}
  {\sum_j\omega_{j,n}(h)
  \mathsf{a}_{\theta_{j,n}}(x_{\star},z,t)}.
  \label{eq:smooth-latent-quadrature-score}
\end{equation}

\subsubsection{Score identity and regularity of the diffusion features}
\label{app:subsubsec:diffusion-feature-regularity}

For $z\in\R^{d_z}$, $t>0$, and $y\in[0,1]^{d_z}$, define
\begin{equation*}
  \nu_{z,t}(\od y)
  :=\frac{G_t(z,y)\,\od y}{I_t(z)},
  \qquad
  \Lambda_t(z,y)
  :=-\frac{z-\alpha_ty}{\sigma_t}.
\end{equation*}
Then $\nu_{z,t}$ is a probability distribution and
\begin{equation}
  \mathsf{a}_{\theta}(x,z,t)
  =\E_{\nu_{z,t}}q_{\theta}(Y\mid x),
  \qquad
  \mathsf{b}_{\theta}(x,z,t)
  =\E_{\nu_{z,t}}[q_{\theta}(Y\mid x)\Lambda_t(z,Y)].
  \label{eq:appendix-smooth-latent-normalized-bases}
\end{equation}
On $\mathcal{G}_n$,
\begin{equation*}
  \|\Lambda_t(z,y)\|_2
  \le\frac{R_{\mathrm{state},n}+\sqrt{d_z}}{\sigma_{t_0}}
  =R_{\mathrm{bal},n}.
\end{equation*}
The density bounds and \cref{eq:smooth-latent-output-regularity} therefore imply \cref{eq:smooth-latent-basis-latent-regularity}.

Conditional independence given the task makes the prompt likelihood the product of the context-example densities and, for input--output prompts, the query density.
Bayes' formula therefore gives a task posterior proportional to
\begin{equation*}
  \exp\!\left\{
  \ell_{\theta}^{\mathrm{qry}}(x_{\star})
  +\sum_{i=1}^n\ell_{\theta}(c_i)
  \right\}\Pi(\od\theta).
\end{equation*}
This formula is defined for every prompt and agrees almost surely with the task posterior.
Forward OU convolution of the posterior-predictive density and differentiation under the integral are justified by the bounded density and the Gaussian kernel.
The common factors $I_t(z)$ and the task-posterior normalizer cancel from the score ratio.
Consequently,
\begin{equation}
  \widetilde{s}^{\star}(z,t\mid h)
  =
  \frac{\int_{\gT} \mathsf{Lik}_n(\theta;h)
  \mathsf{b}_{\theta}(x_{\star},z,t)\,\Pi(\od\theta)}
  {\int_{\gT} \mathsf{Lik}_n(\theta;h)
  \mathsf{a}_{\theta}(x_{\star},z,t)\,\Pi(\od\theta)},
  \quad
  \mathsf{Lik}_n(\theta;h)
  :=\exp\!\Bigl\{
  \ell_{\theta}^{\mathrm{qry}}(x_{\star})
  +\sum_{i=1}^n\ell_{\theta}(c_i)
  \Bigr\}.
  \label{eq:smooth-latent-score-integral}
\end{equation}

\begin{lemma}[Normalized OU-basis regularity]
\label{lem:smooth-latent-normalized-basis-regularity}
Under \cref{assump:smooth-latent-task}, put $\mathcal{U}_{\mathrm{feat}}:=[-1,1]^{d_{\mathrm{stat}}+1}$
and $R_{\mathrm{geom},n}:=1+R_{\mathrm{state},n}+\sqrt{d_z}$.
For every $\theta\in\gT$, there exist continuous maps
\begin{equation*}
  \overline{\mathsf{a}}_{\theta,n}:\mathcal{U}_{\mathrm{feat}}
  \to[\underline{q}_{\mathrm{out}},\overline{q}_{\mathrm{out}}],
  \qquad
  \overline{\mathsf{b}}_{\theta,n}:\mathcal{U}_{\mathrm{feat}}
  \to\overline{\Ball}(0,R_{\mathrm{num},n}^{\mathrm{lat}}),
\end{equation*}
such that, for every $(h,z,t)\in\mathcal{G}_n$,
\begin{equation*}
  \overline{\mathsf{a}}_{\theta,n}(\mathsf{u}_n(z,t,h))
  =\mathsf{a}_{\theta}(x_{\star},z,t),
  \qquad
  \overline{\mathsf{b}}_{\theta,n}(\mathsf{u}_n(z,t,h))
  =\mathsf{b}_{\theta}(x_{\star},z,t).
\end{equation*}
Moreover, there is a single finite constant $H_{\mathrm{bas},n}$, independent of $\theta$, such that, for all $u,u'\in\mathcal{U}_{\mathrm{feat}}$,
\begin{equation*}
  |\overline{\mathsf{a}}_{\theta,n}(u)-\overline{\mathsf{a}}_{\theta,n}(u')|
  +\|\overline{\mathsf{b}}_{\theta,n}(u)-\overline{\mathsf{b}}_{\theta,n}(u')\|_2
  \le H_{\mathrm{bas},n}\|u-u'\|_2^{\beta'}.
\end{equation*}
The common H\"older constant satisfies
\begin{equation}
  \begin{aligned}
  H_{\mathrm{bas},n}\le{} C_{\mathrm{bas}}\big[&
  1+(1+R_{\mathrm{bal},n})H_{q,x}
  \\
  &+R_{\mathrm{state},n}
  \{R_{\mathrm{geom},n}\sigma_{t_0}^{-2}
  +\sigma_{t_0}^{-1}
  +R_{\mathrm{geom},n}^2\sigma_{t_0}^{-3}\}
  \\
  &+T\{R_{\mathrm{geom},n}^2\sigma_{t_0}^{-4}
  +\sigma_{t_0}^{-1}
  +R_{\mathrm{geom},n}\sigma_{t_0}^{-3}
  +R_{\mathrm{geom},n}^3\sigma_{t_0}^{-5}\}
  \big].
  \end{aligned}
  \label{eq:smooth-latent-basis-regularity}
\end{equation}
Here $0<C_{\mathrm{bas}}<\infty$ depends only on the fixed dimensions, $\beta'$, and the common density bounds.
\end{lemma}

\begin{proof}[Proof of \cref{lem:smooth-latent-normalized-basis-regularity}]
The Gaussian identities
\begin{align*}
  \nabla_z\Lambda_t(z,y)
  &=-\sigma_t^{-1}I_{d_z},\\
  \nabla_z\log G_t(z,y)
  &=\sigma_t^{-1}\Lambda_t(z,y),\\
  \partial_t\Lambda_t(z,y)
  &=-\frac{\alpha_ty}{\sigma_t}
  -\frac{\alpha_t^2}{\sigma_t^2}\Lambda_t(z,y)
  ,\\
  \partial_t\log G_t(z,y)
  &=-\frac{d_z\alpha_t^2}{\sigma_t^2}
  -\frac{(z-\alpha_ty)^{\top}\alpha_ty}{\sigma_t^2}
  +\frac{\alpha_t^2\|z-\alpha_ty\|_2^2}{\sigma_t^4}
\end{align*}
and the covariance identity below imply the derivative bounds
\begin{align*}
  \|\nabla_z\mathsf{a}_{\theta}\|_2
  &\lesssim R_{\mathrm{geom},n}\sigma_{t_0}^{-2},
  &
  \|\nabla_z\mathsf{b}_{\theta}\|_{\mathrm{op}}
  &\lesssim \sigma_{t_0}^{-1}+R_{\mathrm{geom},n}^2\sigma_{t_0}^{-3},
  \\
  |\partial_t\mathsf{a}_{\theta}|
  &\lesssim R_{\mathrm{geom},n}^2\sigma_{t_0}^{-4},
  &
  \|\partial_t\mathsf{b}_{\theta}\|_2
  &\lesssim \sigma_{t_0}^{-1}+R_{\mathrm{geom},n}\sigma_{t_0}^{-3}
  +R_{\mathrm{geom},n}^3\sigma_{t_0}^{-5}.
\end{align*}
The implicit constants depend only on the fixed dimensions and density bounds.
The denominator $I_t(z)$ causes no extra small-denominator factor because, for every bounded differentiable scalar- or vector-valued integrand $F$ and each coordinate derivative
\begin{equation*}
  D\in\{\partial_t,\partial_{z_1},\ldots,\partial_{z_{d_z}}\},
\end{equation*}
one has, componentwise in the vector-valued case,
\begin{equation}
  D\E_{\nu_{z,t}}[F(z,t,Y)]
  =\E_{\nu_{z,t}}[D F(z,t,Y)]
  +\operatorname{Cov}_{\nu_{z,t}}
  \bigl(F(z,t,Y),D\log G_t(z,Y)\bigr).
  \label{eq:appendix-smooth-latent-covariance-identity}
\end{equation}
Differentiation is justified by domination on the compact domain for $(x,z,t)$ defined above.
For the query coordinate, \cref{eq:smooth-latent-output-regularity} and \cref{eq:appendix-smooth-latent-normalized-bases} give a $\beta'$-H\"older modulus at most
\begin{equation*}
  (1+R_{\mathrm{bal},n})H_{q,x}.
\end{equation*}
To define the functions approximated by the network on the entire normalized cube, let $u\in\mathcal{U}_{\mathrm{feat}}$, written as $u=(u_x,u_z,u_t)$ in the input--output prompt model and $u=(u_z,u_t)$ in the sample-only prompt model, and set
\begin{equation*}
  \begin{aligned}
  x(u)&:=\begin{cases}
  0,&\text{sample-only},\\
  (u_x+\mathbf{1}_{d_x})/2,&\text{input--output},
  \end{cases}\\
  z(u)&:=R_{\mathrm{state},n}\clip_1(u_z),
  &
  t(u)&:=t_0+\frac{T-t_0}{2}(u_t+1).
  \end{aligned}
\end{equation*}
Define the extensions by
\begin{equation*}
  \overline{\mathsf{a}}_{\theta,n}(u)
  :=\mathsf{a}_{\theta}(x(u),z(u),t(u)),
  \qquad
  \overline{\mathsf{b}}_{\theta,n}(u)
  :=\mathsf{b}_{\theta}(x(u),z(u),t(u)).
\end{equation*}
Since $\clip_1$ is one-Lipschitz, the three coordinate maps above have Lipschitz constants bounded by a fixed constant, $R_{\mathrm{state},n}$, and $T$, respectively.
Multiplying the state and time derivative bounds by the Lipschitz constants of these coordinate maps and adding the query modulus gives the asserted H\"older inequality, with $H_{\mathrm{bas},n}$ satisfying \cref{eq:smooth-latent-basis-regularity}, after increasing $C_{\mathrm{bas}}$ by a fixed dimension-dependent factor.
By construction, the extensions take values in the ranges stated in the lemma.
For every $(h,z,t)\in\mathcal{G}_n$, the state block of $\mathsf{u}_n(z,t,h)$ lies in the unit ball and represents the unclipped state; substituting the feature coordinates into the definitions above therefore gives $\overline{\mathsf{a}}_{\theta,n}(\mathsf{u}_n(z,t,h))=\mathsf{a}_{\theta}(x_{\star},z,t)$ and $\overline{\mathsf{b}}_{\theta,n}(\mathsf{u}_n(z,t,h))=\mathsf{b}_{\theta}(x_{\star},z,t)$.
Since $\mathcal{U}_{\mathrm{feat}}$ has bounded diameter and $\beta'\le1$, the state--time Lipschitz bounds are also $\beta'$-H\"older after increasing their constant.
This proves all claims of the lemma.
\end{proof}

\subsection{Episode-level KL and integrated score approximation}
\label{app:subsec:episode-kl-score-approximation}

Uniform approximation over every prompt is stronger than the integrated error required by DSM and the sampler.
We therefore compare entire episode distributions before conditioning on the prompt.
This comparison also controls the addition of uniform prior mass.
All expectations in this subsection use the original model, including when the score being evaluated comes from a discretized model.

\begin{proposition}[Episode-level KL control of integrated score error]
\label{prop:episode-kl-score-approximation}
Under \cref{assump:smooth-latent-task}, fix an integer $J\ge1$, pairwise distinct grid points $\theta_1,\ldots,\theta_J\in\gT$, and $\eta\ge0$, and let $\pi_{\eta}:\gT\to\{\theta_1,\ldots,\theta_J\}$ be Borel with $\|\theta-\pi_{\eta}(\theta)\|_2\le\eta$ for every $\theta$.
Set $w_j=\Pr\!\bigl(\pi_{\eta}(\Theta)=\theta_j\bigr)$ and, for $\zeta\in[0,1]$, set
\begin{equation}
 w_j^{\zeta}:=\frac{w_j+\zeta/J}{1+\zeta}.
 \label{eq:integrated-prior-weights}
\end{equation}
Let $Q^{\eta,\zeta}_{\mathrm{epi}}$ be the observed-episode distribution obtained by drawing a grid task with weights $w_j^{\zeta}$ and then drawing $n+1$ conditionally independent examples from $g_{\theta_j}$.
In the input--output model, an example is a full pair $(X,Z)$ and the observed episode is $(C_n,X_\star,Z_\star)$.
Because every grid density is positive, let $Q^{\eta,\zeta}_{0\mid h}$ denote the everywhere-defined Bayes conditional law of $Z_\star$ given $H_n=h$ under $Q^{\eta,\zeta}_{\mathrm{epi}}$.
If $q_t^{\eta,\zeta}(\cdot\mid h)$ denotes its OU-smoothed density, define
$s^{\eta,\zeta}(z,t\mid h):=\nabla_z\log q_t^{\eta,\zeta}(z\mid h)$.
Define $D_n(\eta,\zeta)=(n+1)H_\theta^2\eta^{2\beta_\theta}/\underline{p}+\log(1+\zeta)$.
Then
\begin{align}
 \KL\!\left(\sP_{\mathrm{epi}}^{(n)}\middle\|Q^{\eta,\zeta}_{\mathrm{epi}}\right)
 &\le D_n(\eta,\zeta), \label{eq:integrated-episode-kl}\\
 \E_{H_n}\KL(P_{0\mid H_n}\|Q^{\eta,\zeta}_{0\mid H_n})
 &\le D_n(\eta,\zeta). \label{eq:integrated-predictive-kl}
\end{align}
Moreover, for every $0<t_0<T<\infty$,
\begin{equation}
 \int_{t_0}^T\E\left[
 \|s^\star(Z_t,t\mid H_n)-s^{\eta,\zeta}(Z_t,t\mid H_n)\|_2^2\right]\,\od t
 \le D_n(\eta,\zeta), \label{eq:integrated-latent-score-bound}
\end{equation}
where $(H_n,Z_t)$ has the original prompt and conditional diffusion law.
The result is uniform over Borel task priors, including priors with zero cell masses.
\end{proposition}

\begin{proof}
For densities $g,g'$ on a unit-volume cube with $g'\ge\underline{p}$, $\log u\le u-1$ gives
\begin{equation*}
 \KL(g\|g')\le\int\frac{(g-g')^2}{g'}
 \le\underline{p}^{-1}\|g-g'\|_\infty^2.
\end{equation*}
Set $U_i:=Z_i$ in the sample-only model and $U_i:=(X_i,Z_i)$ in the input--output model. Compare the two joint laws of $(\Theta,U_1,\ldots,U_{n+1})$ with the same prior $\Pi$, using $g_\Theta$ and $g_{\pi_\eta(\Theta)}$ as the conditional example densities.
Tensorization of conditional KL gives a bound $(n+1)H_\theta^2\eta^{2\beta_\theta}/\underline{p}$.
Marginalizing $\Theta$ gives the same bound for the full example sequence.
Reordering that sequence into $(H_n,Z_\star)$ is a measurable bijection under either the sample-only or the input--output prompt model, so it preserves KL.
This proves \cref{eq:integrated-episode-kl} for $\zeta=0$.

Let $Q^{\eta,\mathrm{unif}}_{\mathrm{epi}}$ use the uniform distribution over the same grid.
The weights in \cref{eq:integrated-prior-weights} give the identity
\begin{equation*}
 Q^{\eta,\zeta}_{\mathrm{epi}}
 =\frac{Q^{\eta,0}_{\mathrm{epi}}+\zeta Q^{\eta,\mathrm{unif}}_{\mathrm{epi}}}{1+\zeta}.
\end{equation*}
Its density is at least $(1+\zeta)^{-1}$ times the density of $Q^{\eta,0}_{\mathrm{epi}}$.
Integrating the resulting log-density inequality proves \cref{eq:integrated-episode-kl}.
The KL chain rule at $(H_n,Z_\star)$ proves \cref{eq:integrated-predictive-kl}: the conditional KL is averaged under the original prompt law, and the chain rule also contains a nonnegative KL term for the prompt marginal that can be dropped.

Fix $h$ in a $\Law(H_n)$-full-measure set on which the conditional laws are defined, and write $p_t,q_t$ for the OU-smoothed densities of $P_{0\mid h},Q^{\eta,\zeta}_{0\mid h}$.
Their initial densities on the output cube lie in $[\underline{q}_{\mathrm{out}},\overline{q}_{\mathrm{out}}]$ by \cref{lem:positive-smooth-latent-derived-regularity}.
Hence their ratio after convolution is bounded above and below by positive constants.
For $t\in[t_0,T]$, convolution with the OU kernel gives smooth positive densities with Gaussian tails.
Their spatial derivatives are bounded by Gaussian envelopes times polynomials in $\|z\|$, uniformly on this time interval.
In particular, the score identity from differentiating the kernel gives $\|\nabla\log p_t(z)\|+\|\nabla\log q_t(z)\|\lesssim_{t_0,T,d_z}1+\|z\|$.

Both densities solve $\partial_t p=\Delta p+\nabla\cdot(zp)$.
Differentiating $\int p_t\log(p_t/q_t)$ and integrating by parts therefore gives the relative-entropy dissipation identity
\begin{equation}
 \frac{\od}{\od t}\KL(p_t\|q_t)
 =-\int p_t\|\nabla\log p_t-\nabla\log q_t\|_2^2\,\od z.
 \label{eq:ou-relative-entropy-dissipation}
\end{equation}
The drift terms cancel and the coefficient is one for diffusion coefficient $\sqrt{2}$; compare \citet{chengeorgioupavon2016}.
To justify the calculation on $\R^{d_z}$, insert smooth cutoffs equal to one on a ball of radius $R$, with first and second derivatives of order $R^{-1}$ and $R^{-2}$, respectively, perform the calculation on their compact support, and let $R\to\infty$.
The bounded density ratio, the score growth bound, and the Gaussian bounds on the derivatives give an integrable upper bound for every differentiated and boundary term, so dominated convergence applies.
Integrating \cref{eq:ou-relative-entropy-dissipation} gives
\begin{align*}
 \int_{t_0}^T\!\int_{\R^{d_z}} p_t(z)
 \|s^\star(z,t\mid h)-s^{\eta,\zeta}(z,t\mid h)\|_2^2\,\od z\,\od t
 &=\KL(p_{t_0}\|q_{t_0})-\KL(p_T\|q_T)\\
 &\le\KL(P_{0\mid h}\|Q^{\eta,\zeta}_{0\mid h}),
\end{align*}
where the last inequality uses nonnegativity and contraction of KL under the OU kernel.
Tonelli's theorem and \cref{eq:integrated-predictive-kl} finish the proof.
\end{proof}

For the regular grid in \cref{app:subsec:latent-quadrature}, take $\eta=\eta_{\theta,n}$ and $\pi_\eta=\pi_{\eta_{\theta,n}}$. If $\omega_{j,n}^{\zeta}(h)$ denotes the posterior weight obtained from $w_{j,n}^{\zeta}=(w_{j,n}+\zeta/J_{\theta,n})/(1+\zeta)$ and $\mathsf{Lik}_{j,n}(h)$, then the everywhere-defined Bayes formula gives
\begin{equation}
 \sigma_t s^{\eta_{\theta,n},\zeta}(z,t\mid h)
 =\frac{\sum_j\omega_{j,n}^{\zeta}(h)
 \mathsf{b}_{\theta_{j,n}}(x_\star,z,t)}
 {\sum_j\omega_{j,n}^{\zeta}(h)
 \mathsf{a}_{\theta_{j,n}}(x_\star,z,t)}.
 \label{eq:regular-grid-floored-score-identity}
\end{equation}
Thus the floored quadrature score used in \cref{app:sec:transformer-approximation} is exactly the scaled score of the discretized model; only its neural component maps and final readout are approximated.

The piecewise-constant map $g_{\pi_\eta(\theta)}$ is used only to approximate the original task family; it need not be H\"older as a function of $\theta$.
Each grid task has the same spatial regularity and positivity as the original task family.
The neural construction in \cref{app:sec:transformer-approximation} approximates the component maps and ratio while its four attention blocks exactly aggregate the resulting finite-mixture terms.

\clearpage

\section{Transformer Construction, Approximation, and Class Complexity}
\label{app:sec:transformer-approximation}

The finite latent average from \cref{app:sec:latent-task-approximation} specifies the score computation to approximate.
This section proceeds in three stages: it first defines the scaled-score metric and the prompt-conditioned Transformer class, then uses four attention blocks to aggregate the approximated finite-mixture components exactly, and finally derives the integrated approximation and class-complexity bounds.
The last stage is needed because training need not return the specifically constructed parameters.

\subsection{Score Metric, Transformer Class, and Four-Block Realization}
\label{app:subsec:smooth-latent-transformer}

\subsubsection{Scaled-score metric and component classes}

For a nonempty class $\gS$ of jointly Borel score functions, define its population approximation error by
\begin{equation}
  \gE_{\mathrm{app}}(\gS)
  :=\inf_{s\in\gS}\gE_{\mathrm{score}}(s).
  \label{eq:score-class-approximation}
\end{equation}
For $s\in\gS$, write $\widetilde{s}:=\sigma_t s$ and define
\begin{equation}
  \widetilde{\gS}:=\{\widetilde{s}=\sigma_t s:s\in\gS\},
  \qquad
  d_{\mathrm{bal},\infty}(s,s')
  :=\sup_{(z,t,h)\in\R^{d_z}\times[t_0,T]\times\gH_n}
  \|\sigma_t s(z,t,h)-\sigma_t s'(z,t,h)\|_2.
  \label{eq:scaled-score-class}
\end{equation}
The value of $d_{\mathrm{bal},\infty}$ may be $+\infty$ for an unbounded class, so it is an extended metric until a finite uniform bound on the scaled-score output is imposed.
For a vector-valued function $f$ on this score domain, write $\|f\|_{\infty,2}:=\sup_{z,t,h}\|f(z,t,h)\|_2$.
Thus $d_{\mathrm{bal},\infty}(s,s')=\|\sigma_t(s-s')\|_{\infty,2}$.
For finite-dimensional vectors, $\|\cdot\|_\infty$ denotes the maximum absolute coordinate; for matrices, $\|A\|_{\max}:=\max_{i,j}|A_{ij}|$.
For a set $\mathcal{U}$ equipped with an extended metric $d$, the covering number $\Ncov(\varepsilon,\mathcal{U},d)$ is the smallest cardinality of an $\varepsilon$-cover whose centers belong to $\mathcal{U}$ and equals $+\infty$ if no finite cover exists~\citep[cf.][]{vaart_wellner_weak_2023,Wainwright_2019}.
This metric uses the same scaling as the variance-balanced DSM loss in \cref{app:sec:score-learning}.

For fixed input and output dimensions, $\nn(L,W,S,B)$ denotes the class of feedforward ReLU networks with $L$ affine layers, maximum hidden-layer width $W$, at most $S$ nonzero weights and biases, and parameter magnitudes at most $B$.
The output affine layer is included in $L$, biases are included in $S$, and the subsequent projection onto the prescribed output range is not part of $\nn(L,W,S,B)$.
Every component class used below has a fixed sparse wiring pattern, so its covering number does not include a separate search over sparsity patterns.
Basis, coordinate-update, and arithmetic weights that do not depend on the approximated function may be fixed as part of that wiring.

\subsubsection{Prompt-conditioned Transformer class}

The class takes a prompt, a noisy output, and diffusion time as inputs and returns a scaled score, followed by division by $\sigma_t$.
The grid and component architectures are fixed before training; the definition below specifies which coefficients can vary.
Its references to the component-approximation lemmas specify the network classes, while \cref{app:subsubsec:transformer-construction} gives the token layout and verifies the four-block computation at the constructed parameters.

The tolerances specify the required accuracy of each component: $\varepsilon_{\ell,n}$ and $\varepsilon_{\mathrm{qry},n}$ bound uniform log-likelihood errors, $\varepsilon_{a,n}$ and $\varepsilon_{b,n}$ bound errors in the scalar OU denominator and vector numerator, and $\varepsilon_{\mathrm{rat},n}$ bounds the ratio-readout error.
The vector errors are measured in Euclidean norm.

\begin{definition}[Prompt-conditioned score Transformer]
\label{def:smooth-latent-score-class}
Fix $n$ and a latent grid, and choose tolerances
\begin{equation*}
 \varepsilon_{\mathrm{floor},n},\ \varepsilon_{\ell,n},\
 \varepsilon_{a,n},\ \varepsilon_{b,n},\
 \varepsilon_{\mathrm{rat},n}\in(0,1].
\end{equation*}
Also fix a deterministic final output radius $0<R_{\mathrm{out},n}\le R_{\mathrm{bal},n}$.
The radius $R_{\mathrm{bal},n}$ remains the uniform bound used for latent quadrature and internal basis estimates; $R_{\mathrm{out},n}$ controls the score-network output and the training loss.
In the input--output prompt model, also fix $\varepsilon_{\mathrm{qry},n}\in(0,1]$; in the sample-only prompt model, set $\varepsilon_{\mathrm{qry},n}:=0$ and omit the query factor below.
Index the context-example tokens by $I_{\mathrm{ctx}}:=\{1,\ldots,n\}$ and the latent tokens by $I_{\mathrm{lat}}:=\{i_1^{\mathrm{lat}},\ldots,i_{J_{\theta,n}}^{\mathrm{lat}}\}$.
Let $i_{\mathrm{sum}}$, $i_{\mathrm{qry}}$, $i_{\mathrm{out}}$, and $i_{\varnothing}$ denote the summary, query--state--time, output, and null tokens, respectively.
Thus the sequence has $n+J_{\theta,n}+4$ positions.
Choose the explicit token width in \cref{eq:smooth-latent-explicit-token-width}, namely
\begin{equation*}
  d_{\mathrm{emb},n}^{\mathrm{lat}}
  =(d_z+5)J_{\theta,n}+d_{\mathrm{stat}}+2d_z+11
  \le C_{d_z}\{J_{\theta,n}+d_{\mathrm{stat}}+1\}.
\end{equation*}
\emph{Trainable parameter space.}
Let $\gV_n^{\mathrm{fit}}$ be the product of the compact parameter cubes for the likelihood approximators with bounded outputs in \cref{lem:fixed-dimensional-holder-relu-approximation}, the time-local OU-basis approximators in \cref{lem:smooth-latent-time-local-basis}, and the prior logits with radius \cref{eq:smooth-latent-prior-logit-radius}.
For attention, include every entry of each of the sixteen matrices
\begin{equation*}
 W_{Q,k},W_{K,k},W_{V,k},W_{O,k}\in
 [-B_{\mathrm{att},n},B_{\mathrm{att},n}]^{d_{\mathrm{emb},n}^{\mathrm{lat}}\times d_{\mathrm{emb},n}^{\mathrm{lat}}},
 \quad 1\le k\le4,
\end{equation*}
\begin{equation*}
 B_{\mathrm{att},n}=1\vee\sqrt{d_{\mathrm{emb},n}^{\mathrm{lat}}}.
\end{equation*}
Thus the attention part has exactly $16(d_{\mathrm{emb},n}^{\mathrm{lat}})^2$ trainable coordinates.
The sparse matrices in \cref{tab:smooth-latent-four-block-construction} specify a reference parameter within this entire score class.
The table describes the computation at that parameter; arbitrary trainable projections need not preserve its internal posterior interpretation or role coordinates.
The hard masks and coordinate-update maps are fixed for the entire class.
The query log-likelihood factor is omitted in the sample-only prompt model.
The fixed approximator architecture, latent identifiers, role codes, hard masks, coordinate-update maps, ratio readout, and radial output projection are not coordinates of $\gV_n^{\mathrm{fit}}$.
Let $\mathsf{N}_{\mathrm{param},n}^{\mathrm{lat}}:=\dim(\gV_n^{\mathrm{fit}})$, let $B_{\mathrm{lat},n}\ge1$ dominate its coordinate radii, and let
\begin{equation*}
  \proj_n^{\mathrm{par}}:
  [-B_{\mathrm{lat},n},B_{\mathrm{lat},n}]^{\mathsf{N}_{\mathrm{param},n}^{\mathrm{lat}}}
  \longrightarrow\gV_n^{\mathrm{fit}}
\end{equation*}
be coordinatewise projection onto the corresponding factor ranges.
This map is continuous, nonexpansive in $\ell_{\infty}$, and onto.

\emph{Token representation.}
For a parameter vector $\mathbf{v}$ in the displayed cube, put $\widehat{\mathbf{v}}:=\proj_n^{\mathrm{par}}(\mathbf{v})$.
We identify $\gV_n^{\mathrm{fit}}$ with its natural subset of the enclosing cube; on this subset, $\proj_n^{\mathrm{par}}$ is the identity.
Let $\mathbf{E}_n^{(0)}(\widehat{\mathbf{v}};z,t,h)$ be the encoded array of the $n$ context-example tokens, one summary token, one query--state--time token, $J_{\theta,n}$ latent tokens, one output token, and one null token.
In particular, every trainable component encoder is evaluated at the projected parameter vector $\widehat{\mathbf{v}}$.
The network-component specification in Appendix~\ref{app:subsubsec:transformer-construction} gives the initial value and shape of every token type, with ordered coordinate decomposition \cref{eq:smooth-latent-token-block-decomposition}.
The context-example network uses the fixed affine rescaling $c\mapsto2c-\mathbf{1}_{d_{\mathrm{stat}}}$.
In the input--output prompt model, the query log-likelihood network uses $x_{\star}\mapsto2x_{\star}-\mathbf{1}_{d_x}$.
The OU-basis network uses the globally bounded feature $\mathsf{u}_n(z,t,h)$ from \cref{eq:transformer-input-preprocessing}.
Its fixed time-local encoder combines overlapping time patches with a continuous ReLU partition of unity, as specified in Lemma~\ref{lem:smooth-latent-time-local-basis}; the four attention blocks are unchanged.
All occurrences of the noisy state in this encoded array use $\bar{z}=\clip_{R_{\mathrm{state},n}}(z)$; there is no unprocessed state input.
Consequently, for every parameter in the class,
\begin{equation}
 \widetilde{s}_{\mathbf{v}}(z,t,h)=\widetilde{s}_{\mathbf{v}}(\bar{z},t,h).
 \label{eq:global-state-clipping-reduction}
\end{equation}
This identity reduces the full-state supremum in $d_{\mathrm{bal},\infty}$ to the compact state ball.
We suppress the common arguments $(\widehat{\mathbf{v}};z,t,h)$ from the token arrays in the recursion below.

\emph{Attention blocks.}
For $1\le k\le4$ and every token row $i$, define
\begin{equation}
  \mathbf{E}_{n,i}^{(k)}
  :=
  \mathsf{R}_{k,n}\!\left[
  \mathbf{E}_{n,i}^{(k-1)}
  {}+
  \operatorname{Att}_{\mathbf{A}_k(\widehat{\mathbf{v}})}
  \!\left(
  \mathbf{E}_{n,i}^{(k-1)};
  (\mathbf{E}_{n,j}^{(k-1)})_{j\in\mathsf{M}_{k,n}(i)}
  \right)
  \right].
  \label{eq:smooth-latent-score-class-recursion}
\end{equation}
Here $\mathbf{A}_k(\widehat{\mathbf{v}})$ denotes the four dense bounded attention projection matrices, and $\mathsf{R}_{k,n}$ is the fixed affine--ReLU coordinate-update map after block $k$.
The non-null masks are
\begin{equation*}
  \mathsf{M}_{1,n}(i_{\mathrm{sum}})=I_{\mathrm{ctx}},
  \qquad
  \mathsf{M}_{2,n}(i_j^{\mathrm{lat}})=\{i_{\mathrm{sum}}\},
  \qquad
  \mathsf{M}_{3,n}(i_j^{\mathrm{lat}})=\{i_{\mathrm{qry}}\},
  \qquad
  \mathsf{M}_{4,n}(i_{\mathrm{out}})=I_{\mathrm{lat}},
\end{equation*}
for $i_j^{\mathrm{lat}}\in I_{\mathrm{lat}}$, and every undeclared row attends only to $i_{\varnothing}$.
The projection matrices of the constructed network, residual destinations, and states after $\mathsf{R}_{k,n}$ are listed in \cref{tab:smooth-latent-four-block-construction}; bounded coordinatewise projection precedes the exact selector in \cref{eq:appendix-smooth-latent-binary-gate}.

\emph{Score readout and class.}
Let $\mathsf{Read}_n$ first apply exact coordinatewise projection onto
\begin{equation*}
 [\underline{q}_{\mathrm{out}},\overline{q}_{\mathrm{out}}]\times
 [-R_{\mathrm{num},n}^{\mathrm{lat}}-1,R_{\mathrm{num},n}^{\mathrm{lat}}+1]^{d_z},
\end{equation*}
and then apply the fixed affine normalization and ReLU ratio approximator constructed below.
This preliminary projection is part of the readout for every parameter; it is inactive on the reference network's averaged features.
Let $\mathsf{P}_{\mathrm{out}}$ extract the output-token coordinates.
Define
\begin{equation}
  \widetilde{s}_{\mathbf{v}}(z,t,h)
  :=
  \clip_{R_{\mathrm{out},n}}
  \!\left[
  \mathsf{Read}_n\!\left(
  \mathsf{P}_{\mathrm{out}}\mathbf{E}_n^{(4)}(\widehat{\mathbf{v}};z,t,h)
  \right)
  \right],
  \qquad
  s_{\mathbf{v}}(z,t,h):=\frac{\widetilde{s}_{\mathbf{v}}(z,t,h)}{\sigma_t},
  \label{eq:smooth-latent-score-class-output}
\end{equation}
and set
\begin{equation}
  \gS_n^{\mathrm{lat}}
  :=
  \left\{
  s_{\mathbf{v}}:
  \mathbf{v}\in
  [-B_{\mathrm{lat},n},B_{\mathrm{lat},n}]^{\mathsf{N}_{\mathrm{param},n}^{\mathrm{lat}}}
  \right\}.
  \label{eq:smooth-latent-score-class-formal}
\end{equation}
Writing $\mathbf{A}_k(\widehat{\mathbf{v}})=(W_{Q,k},W_{K,k},W_{V,k},W_{O,k})$, the attention map in \cref{eq:smooth-latent-score-class-recursion} is $\operatorname{Att}_{W_{Q,k},W_{K,k},W_{V,k},W_{O,k}}$ from \cref{lem:attention-parameter-sensitivity} restricted to the displayed deterministic mask.
Here ``Transformer'' means the explicitly defined prompt-conditioned architecture: external ReLU input and readout modules surround four hard-masked residual softmax-attention blocks with fixed affine--ReLU coordinate updates and no layer normalization.
Exact containment in a fixed-width conventional Transformer with mandatory feed-forward sublayers or layer normalization is not claimed.
In particular, the definition does not cover a standard pre-normalized or post-normalized Transformer with fully learned coordinate updates.
\end{definition}

\subsubsection{Token layout and four-block computation}
\label{app:subsubsec:transformer-construction}

The following specification completes the class definition by giving the token coordinates and their initialization.
The construction lemma identifies parameter values at which the four attention blocks perform the required averaging and copy the relevant coordinates.
The integrated approximation error is analyzed in \cref{prop:integrated-transformer-approximation}.

\paragraph{Trainable and fixed components.}
Fix $n\ge1$, write $J:=J_{\theta,n}$, and let
\begin{equation}
  d_{\mathrm{emb},n}^{\mathrm{lat}}
  :=(d_z+5)J+d_{\mathrm{stat}}+2d_z+11.
  \label{eq:smooth-latent-explicit-token-width}
\end{equation}
For $\mathbf{v}\in[-B_{\mathrm{lat},n},B_{\mathrm{lat},n}]^{\mathsf{N}_{\mathrm{param},n}^{\mathrm{lat}}}$, write $\widehat{\mathbf{v}}:=\proj_n^{\mathrm{par}}(\mathbf{v})$.
Use the ordered token decomposition
\begin{equation}
  \R^{d_{\mathrm{emb},n}^{\mathrm{lat}}}
  =\mathsf{V}_{\mathrm{role}}\oplus\mathsf{V}_{\mathrm{raw}}
  \oplus\mathsf{V}_{\ell}\oplus\mathsf{V}_{\mathrm{pool}}
  \oplus\mathsf{V}_{\mathrm{qry}}\oplus\mathsf{V}_{\mathrm{id}}
  \oplus\mathsf{V}_{\mathrm{prior}}\oplus\mathsf{V}_{\mathrm{lat}}
  \oplus\mathsf{V}_{\mathrm{out}},
  \label{eq:smooth-latent-token-block-decomposition}
\end{equation}
whose block dimensions, in the displayed order, are
\begin{equation*}
  \bigl(6,d_{\mathrm{stat}}+1,J,J,J(d_z+2),J,1,d_z+2,d_z+1\bigr).
\end{equation*}
Index the rows by
\begin{align*}
  I_{\mathrm{ctx}}&:=\{1,\ldots,n\},
  &i_{\mathrm{sum}}&:=n+1,
  &i_{\mathrm{qry}}&:=n+2,\\
  i_j^{\mathrm{lat}}&:=n+2+j,
  &I_{\mathrm{lat}}&:=\{i_j^{\mathrm{lat}}:1\le j\le J\},
\end{align*}
with $i_{\mathrm{out}}:=n+J+3$ and $i_{\varnothing}:=n+J+4$.
Let $\mathbf{E}_n^{(0)}(\widehat{\mathbf{v}};z,t,h)\in\R^{(n+J+4)\times d_{\mathrm{emb},n}^{\mathrm{lat}}}$ have one-hot role codes for the context-example, summary, query, latent, output, and null rows.
Let $e_j\in\R^J$ be the $j$th standard coordinate vector.
The array's nonzero non-role blocks are
\begin{center}
  \small
  \begin{tabular}{@{}ll@{}}
    \toprule
    Rows & Initialized blocks \\
    \midrule
    Context example $i\in I_{\mathrm{ctx}}$
    & $\mathsf{V}_{\mathrm{raw}}$: padded normalized $c_i$;
      $\mathsf{V}_{\ell}$: $\widehat{\boldsymbol{\ell}}(c_i)$ \\
    Query
    & $\mathsf{V}_{\mathrm{raw}}$: $\mathsf{u}_n(z,t,h)$;
      $\mathsf{V}_{\mathrm{qry}}$: $\widehat{\mathbf{q}}(x_{\star},z,t)$ \\
    Latent $i_j^{\mathrm{lat}}$, $1\le j\le J$
    & $\mathsf{V}_{\mathrm{id}}$: $e_j$;
      $\mathsf{V}_{\mathrm{prior}}$: $\ell_{j,n}^{\mathrm{prior}}$ \\
    Summary, output, and null & none \\
    \bottomrule
  \end{tabular}
\end{center}
Here
\begin{equation*}
  \widehat{\boldsymbol{\ell}}(c_i)
  :=(\widehat{\ell}_j(c_i))_{j=1}^J
  \quad\text{and}\quad
  \widehat{\mathbf{q}}(x_{\star},z,t)
  :=
  (\widehat{\ell}_j^{\mathrm{qry}}(x_{\star}),
  \widehat{\mathsf{a}}_j(x_{\star},z,t),
  \widehat{\mathsf{b}}_j(x_{\star},z,t))_{j=1}^J.
\end{equation*}
All omitted coordinates are zero.

\begin{lemma}[Four-block Transformer construction]
\label{lem:smooth-latent-four-block-construction}
For the encoded array above, there exist four deterministic nonempty masks, projection matrices for the constructed network
\begin{equation*}
  W_{Q,k},W_{K,k},W_{V,k},W_{O,k}
  \in\R^{d_{\mathrm{emb},n}^{\mathrm{lat}}\times d_{\mathrm{emb},n}^{\mathrm{lat}}},
  \qquad 1\le k\le4,
\end{equation*}
and shared rowwise affine--ReLU maps
\begin{equation*}
  \mathsf{R}_{k,n}:\R^{d_{\mathrm{emb},n}^{\mathrm{lat}}}\to\R^{d_{\mathrm{emb},n}^{\mathrm{lat}}}
\end{equation*}
such that each residual block maps $\R^{(n+J+4)\times d_{\mathrm{emb},n}^{\mathrm{lat}}}$ to itself and implements the computation in \cref{tab:smooth-latent-four-block-construction}.
Every row not listed in that table attends only to the null token.
After block 4, the output block equals
\begin{equation}
  \left(
  \sum_{j=1}^J\widehat{\omega}_{j,n}\widehat{\mathsf{a}}_j,
  \sum_{j=1}^J\widehat{\omega}_{j,n}\widehat{\mathsf{b}}_j
  \right)\in\mathsf{V}_{\mathrm{out}},
  \qquad
  \widehat{\omega}_{j,n}
  =\operatorname{softmax}_j((\widehat{\ell}_{k,n}^{\mathrm{post}})_{k=1}^{J}),
  \label{eq:smooth-latent-four-block-output}
\end{equation}
where
\begin{equation}
  \widehat{\ell}_{j,n}^{\mathrm{post}}
  =\ell_{j,n}^{\mathrm{prior}}
  +n\overline{\widehat{\ell}}_j
  +\widehat{\ell}_j^{\mathrm{qry}},
  \qquad
  \overline{\widehat{\ell}}_j
  :=\frac{1}{n}\sum_{i=1}^n\widehat{\ell}_j(c_i).
  \label{eq:smooth-latent-posterior-logit-update}
\end{equation}
The four residual updates preserve all role codes and latent identifiers.
Every projection and coordinate-update coefficient of the constructed network has magnitude at most
\begin{equation*}
  1\vee n\vee B_{\mathrm{gate},n}^{\mathrm{lat}}
  \vee\sqrt{d_{\mathrm{emb},n}^{\mathrm{lat}}},
\end{equation*}
where $B_{\mathrm{gate},n}^{\mathrm{lat}}$ is defined in \cref{eq:smooth-latent-selector-radius}.
Hence the construction is an explicit member of the score class in \cref{def:smooth-latent-score-class}.
\end{lemma}

\begin{table}[H]
  \centering
  \caption{Complete computation implemented by the four masked attention blocks.
  A singleton mask makes its softmax weight exactly one.
  Values are projected from the indicated source block and injected into the residual destination in \cref{eq:smooth-latent-token-block-decomposition}.}
  \label{tab:smooth-latent-four-block-construction}
  \scriptsize \setlength{\tabcolsep}{2pt}
  \renewcommand{\arraystretch}{1.12}
  \begin{tabular}{@{}>{\centering\arraybackslash}p{0.035\linewidth}>{\raggedright\arraybackslash}p{0.18\linewidth}>{\raggedright\arraybackslash}p{0.18\linewidth}>{\raggedright\arraybackslash}p{0.255\linewidth}>{\raggedright\arraybackslash}p{0.275\linewidth}@{}}
    \toprule
    $k$ & Active query rows and keys & Constructed attention logits & Value and residual destination & State after fixed update $\mathsf{R}_{k,n}$ \\
    \midrule
    1
    & Summary row; keys $I_{\mathrm{ctx}}$
    & All zero, hence weights $1/n$
    & $\widehat{\boldsymbol{\ell}}(c_i)\in\R^J$ to $\mathsf{V}_{\mathrm{pool}}$
    & Summary pool is $\overline{\widehat{\boldsymbol{\ell}}}=n^{-1}\sum_i\widehat{\boldsymbol{\ell}}(c_i)\in\R^J$ \\
    2
    & Every latent row; singleton key $i_{\mathrm{sum}}$
    & Singleton, hence weight $1$
    & $\overline{\widehat{\boldsymbol{\ell}}}\in\R^J$ to the latent pool block
    & Every latent row contains the same pooled vector $\overline{\widehat{\boldsymbol{\ell}}}\in\R^J$ \\
    3
    & Every latent row; singleton key $i_{\mathrm{qry}}$
    & Singleton, hence weight $1$
    & $\widehat{\mathbf{q}}\in\R^{J(d_z+2)}$ to the latent query block
    & Identifier $e_j$ selects $(\widehat{\ell}_j^{\mathrm{qry}},\widehat{\mathsf{a}}_j,\widehat{\mathsf{b}}_j)$, and $\mathsf{V}_{\mathrm{lat}}$ becomes $(\widehat{\ell}_{j,n}^{\mathrm{post}},\widehat{\mathsf{a}}_j,\widehat{\mathsf{b}}_j)\in\R^{d_z+2}$ \\
    4
    & Output row; keys $I_{\mathrm{lat}}$
    & $\widehat{\ell}_{j,n}^{\mathrm{post}}$ for latent key $j$
    & $(\widehat{\mathsf{a}}_j,\widehat{\mathsf{b}}_j)\in\R^{d_z+1}$ to $\mathsf{V}_{\mathrm{out}}$
    & Output block is the posterior-weighted pair in \cref{eq:smooth-latent-four-block-output} \\
    \bottomrule
  \end{tabular}
\end{table}

\begin{proof}
All reference matrices are compositions of coordinate projections and injections associated with \cref{eq:smooth-latent-token-block-decomposition}, except for the scalar factors $n$ and $\sqrt{d_{\mathrm{emb},n}^{\mathrm{lat}}}$ described below.
For block 1, set the active query and key projections to zero, project the context-example values from $\mathsf{V}_{\ell}$, and inject the attention output into $\mathsf{V}_{\mathrm{pool}}$.
The summary mask contains exactly the $n$ context-example rows, so softmax attention returns their mean.
For blocks 2 and 3, the active key set of each latent row is a singleton, so the attention weight is exactly one independently of its logit.
Coordinate projection and injection therefore copy the summary pool and the full query--basis array into the indicated blocks of every latent row.

The fixed map $\mathsf{R}_{3,n}$ projects every selector input onto $[-B_{\mathrm{gate},n}^{\mathrm{lat}},B_{\mathrm{gate},n}^{\mathrm{lat}}]$ and uses the one-hot identifier $e_j$.
For $e\in\{0,1\}$ and $|x|\le B_{\mathrm{gate},n}^{\mathrm{lat}}$, the exact affine--ReLU identity
\begin{equation}
  ex=\operatorname{ReLU}(x)
  -\operatorname{ReLU}(x-B_{\mathrm{gate},n}^{\mathrm{lat}}e)
  -\operatorname{ReLU}(-x)
  +\operatorname{ReLU}(-x-B_{\mathrm{gate},n}^{\mathrm{lat}}e)
  \label{eq:appendix-smooth-latent-binary-gate}
\end{equation}
selects one coordinate exactly.
Parallel copies select the $j$th pooled likelihood, query likelihood, scalar basis value, and vector basis value.
An affine output layer then forms \cref{eq:smooth-latent-posterior-logit-update} and writes the resulting $(d_z+2)$-vector into $\mathsf{V}_{\mathrm{lat}}$.
The six one-hot role codes and the same binary-gate identity make these tokenwise coordinate updates shared across rows while leaving all other coordinate blocks unchanged.

For block 4, map the output-token role coordinate to $\sqrt{d_{\mathrm{emb},n}^{\mathrm{lat}}}e_1$ and map latent token $j$ to $\widehat{\ell}_{j,n}^{\mathrm{post}}e_1$ in key space.
The scaled dot product is therefore exactly $\widehat{\ell}_{j,n}^{\mathrm{post}}$.
Map the remaining latent coordinates to $(\widehat{\mathsf{a}}_j,\widehat{\mathsf{b}}_j)$ and inject the weighted attention output into $\mathsf{V}_{\mathrm{out}}$.
This proves \cref{eq:smooth-latent-four-block-output} and every row of \cref{tab:smooth-latent-four-block-construction}.
The null-token value is zero for all four value projections of the constructed network, so unlisted rows receive zero attention updates.
The direct-sum dimensions add to \cref{eq:smooth-latent-explicit-token-width}, and the displayed coefficient bound follows from the coordinate maps, selector threshold, posterior-logit multiplier, and scaling used in the dot product.
\end{proof}

\subsection{Component Approximation and Stability Bounds}
\label{app:subsec:shared-neural-estimates}

Fix the common parameters in \cref{assump:smooth-latent-task} and the regular grid, diffusion interval, clipping radii, and component tolerances of \cref{def:smooth-latent-score-class}.
Take $\varepsilon_{\mathrm{floor},n},\varepsilon_{\ell,n},\varepsilon_{a,n},\varepsilon_{b,n},\varepsilon_{\mathrm{rat},n}\in(0,1]$ and $0<R_{\mathrm{out},n}\le R_{\mathrm{bal},n}$; the query tolerance belongs to $(0,1]$ for input--output prompts and is zero for sample-only prompts.
The following error quantities for the component networks are used by both approximation analyses.
For the tolerances in \cref{def:smooth-latent-score-class}, set
\begin{align*}
  \mathsf{C}_{\mathrm{mix},n}
  &:=\frac{R_{\mathrm{num},n}^{\mathrm{lat}}}{\underline{q}_{\mathrm{out}}}
  +\frac{R_{\mathrm{num},n}^{\mathrm{lat}}\overline{q}_{\mathrm{out}}}
  {\underline{q}_{\mathrm{out}}^2},\\
  \mathsf{C}_{\mathrm{net},n}
  &:=1+2\mathsf{C}_{\mathrm{mix},n},\\
  \varepsilon_{\mathrm{net},n}
  &:=2\mathsf{C}_{\mathrm{mix},n}
  (n\varepsilon_{\ell,n}+\varepsilon_{\mathrm{qry},n})
  +\frac{\varepsilon_{b,n}}{\underline{q}_{\mathrm{out}}}
  +\frac{R_{\mathrm{num},n}^{\mathrm{lat}}\varepsilon_{a,n}}
  {\underline{q}_{\mathrm{out}}^2}
  +\varepsilon_{\mathrm{rat},n}.
\end{align*}

The constant $\mathsf{C}_{\mathrm{mix},n}$ controls the effect of perturbing posterior weights on the score ratio.

\subsubsection{Attention stability and ReLU approximation}

For $d,m\in\mathbb{N}$ and a query token $v\in\R^d$, key--value tokens $x_1,\ldots,x_m\in\R^d$, and matrices $W_Q,W_K,W_V,W_O\in\R^{d\times d}$, define the attention block~\citep{vaswani2017attention}
\begin{equation*}
  \operatorname{Att}_{W_Q,W_K,W_V,W_O}(v;x_{1:m})
  :=W_O\sum_{j=1}^m\pi_jW_Vx_j,
  \qquad
  \pi_j
  :=\frac{\exp\{(W_Qv)^{\top} W_Kx_j/\sqrt{d}\}}
  {\sum_{k=1}^m\exp\{(W_Qv)^{\top} W_Kx_k/\sqrt{d}\}}.
\end{equation*}

\begin{lemma}[Sensitivity of scaled dot-product attention]
\label{lem:attention-parameter-sensitivity}
Fix $B\ge1$.
Suppose every entry of the tokens and matrices in this attention map and in a primed copy is bounded in absolute value by $B$.
Let $\eta_{\mathrm{att}}$ be the maximum of the entrywise matrix differences and the $\ell_{\infty}$ distances between corresponding tokens.
Then
\begin{equation*}
  \left\|
  \operatorname{Att}_{W_Q,W_K,W_V,W_O}(v;x_{1:m})
  -\operatorname{Att}_{W'_Q,W'_K,W'_V,W'_O}(v';x'_{1:m})
  \right\|_{\infty}
  \le16d^5B^6\eta_{\mathrm{att}}.
\end{equation*}
The bound is independent of the number of keys $m$ and remains valid after restricting both maps by the same deterministic nonempty attention mask.
\end{lemma}
The proof follows the standard perturbation argument used in self-attention stability analyses~\citep{edelman22a,takakura23a}; we give the details because our version jointly perturbs the tokens and all four attention projections and must remain uniform in the number of masked keys.
\begin{proof}
The matrix inequality $\|Ax\|_{\infty}\le d\|A\|_{\max}\|x\|_{\infty}$ gives
\begin{equation*}
  \|W_Qv-W'_Qv'\|_{\infty}
  +\max_j\|W_Kx_j-W'_Kx_j'\|_{\infty}
  \le4dB\eta_{\mathrm{att}}.
\end{equation*}
A two-term telescoping argument for the scaled dot products therefore gives
\begin{equation*}
  \max_j\left|
  \frac{(W_Qv)^{\top} W_Kx_j-(W'_Qv')^{\top} W'_Kx_j'}{\sqrt{d}}
  \right|
  \le4d^{5/2}B^3\eta_{\mathrm{att}}.
\end{equation*}
The softmax map satisfies the following Lipschitzness~\citep[Corollary A.7]{edelman22a}:
\begin{equation*}
  \|\operatorname{softmax}(a)-\operatorname{softmax}(a')\|_1
  \le2\|a-a'\|_{\infty}.
\end{equation*}
The weighted value average consequently changes by at most
\begin{equation*}
  \left\|
  \sum_j\pi_jW_Vx_j-
  \sum_j\pi_j'W'_Vx_j'
  \right\|_{\infty}
  \le\{8d^{7/2}B^5+2dB\}\eta_{\mathrm{att}}.
\end{equation*}
One final telescoping step through $W_O$ yields
\begin{equation*}
  \|W_O-W'_O\|_{\max}\left\|\sum_j\pi_jW_Vx_j\right\|_1
  +dB\left\|\sum_j\pi_jW_Vx_j-\sum_j\pi_j'W'_Vx_j'\right\|_{\infty}
  \le16d^5B^6\eta_{\mathrm{att}}.
\end{equation*}
Applying the same argument to the restricted list of keys selected by the common mask proves the masked version.
\end{proof}

\begin{lemma}[Fixed-dimensional H\"older approximation by ReLU networks]
\label{lem:fixed-dimensional-holder-relu-approximation}
Fix $d,d_{\mathrm{out}}\in\mathbb{N}$ and $\beta_0\in(0,1]$.
There is an exponent $0<c<\infty$, depending only on $(d,d_{\mathrm{out}},\beta_0)$, with the following property; all implicit constants below have the same allowed dependencies.
Let $M_f,H\ge0$, let $\varepsilon\in(0,1]$, and let $f:[-1,1]^d\to[-M_f,M_f]^{d_{\mathrm{out}}}$ satisfy
\begin{equation*}
  \|f(x)-f(x')\|_{\infty}
  \le H\|x-x'\|_{\infty}^{\beta_0}.
\end{equation*}
There are integers $D_{\varepsilon}^{\mathrm{NN}},W_{\varepsilon},S_{\varepsilon}\ge1$, a bound $B_{\varepsilon}\ge1$, and a fixed-architecture class
\begin{equation*}
  \mathscr{N}_{\varepsilon}
  \subseteq
  \proj_{[-M_f,M_f]^{d_{\mathrm{out}}}}\circ
  \nn(D_{\varepsilon}^{\mathrm{NN}},W_{\varepsilon},S_{\varepsilon},B_{\varepsilon}),
\end{equation*}
chosen independently of $f$, that contains a member $f_{\varepsilon}$ satisfying
\begin{equation*}
  \sup_{x\in[-1,1]^d}
  \|f_{\varepsilon}(x)-f(x)\|_{\infty}\le\varepsilon.
\end{equation*}
The network sizes satisfy
\begin{equation}
  D_{\varepsilon}^{\mathrm{NN}}\vee W_{\varepsilon}\vee S_{\varepsilon}
  \lesssim
  \left(1+\frac{H\vee1}{\varepsilon}\right)^{d/\beta_0}
  \log^2\!\left(2+\frac{1+M_f+H}{\varepsilon}\right),
  \label{eq:fixed-dimensional-holder-network-size}
\end{equation}
Every scalar weight and bias has magnitude at most $B_{\varepsilon}$.
The tensor-hat and multiplication subnetworks are fixed, and only the output-layer interpolation coefficients vary over a compact cube.
After adding functionally inactive coordinates, write the class as
\begin{equation*}
  \mathscr{N}_{\varepsilon}
  =\{f_{\vartheta}:\vartheta\in[-B_{\varepsilon},B_{\varepsilon}]^{S_{\varepsilon}}\}.
\end{equation*}
The evaluation map $(x,\vartheta)\mapsto f_{\vartheta}(x)$ is jointly Borel and continuous in $\vartheta$.
There are constants $\Lip_{\mathrm{par}},\Lip_{\mathrm{in}}\in[0,\infty)$ such that
\begin{align*}
  \sup_x\|f_{\vartheta}(x)-f_{\vartheta'}(x)\|_{\infty}
  &\le \Lip_{\mathrm{par}}\|\vartheta-\vartheta'\|_{\infty},
  \\
  \sup_{\vartheta\in[-B_{\varepsilon},B_{\varepsilon}]^{S_{\varepsilon}}}
  \|f_{\vartheta}(x)-f_{\vartheta}(x')\|_{\infty}
  &\le \Lip_{\mathrm{in}}\|x-x'\|_{\infty},
\end{align*}
and
\begin{equation}
  B_{\varepsilon}+\Lip_{\mathrm{par}}+\Lip_{\mathrm{in}}
  \lesssim (1+M_f+H+\varepsilon^{-1})^c.
  \label{eq:fixed-dimensional-holder-parameter-stability}
\end{equation}
\end{lemma}

\begin{proof}
Affinely rescale the input cube to $[0,1]^d$ and take a regular tensor grid with mesh
\begin{equation*}
  \eta_{\mathrm{grid}}:=\min\left\{1,
  c_d\left(\frac{\varepsilon}{H\vee1}\right)^{1/\beta_0}\right\},
\end{equation*}
where $c_d>0$ is chosen below.
If $u_1,\ldots,u_{J_{\mathrm{hat}}}$ are the grid points and $\psi_1,\ldots,\psi_{J_{\mathrm{hat}}}$ the tensor-hat partition of unity, then $J_{\mathrm{hat}}\le C_d\{1+(H\vee1)/\varepsilon\}^{d/\beta_0}$.
The interpolant
\begin{equation*}
  I_{\eta_{\mathrm{grid}}}f(x):=\sum_{j=1}^{J_{\mathrm{hat}}} f(u_j)\psi_j(x)
\end{equation*}
satisfies $\|I_{\eta_{\mathrm{grid}}}f-f\|_{\infty}\le H(C_d\eta_{\mathrm{grid}})^{\beta_0}\le\varepsilon/2$ after decreasing $c_d$.

Every univariate hat is an exact ReLU function.
Use the ReLU multiplication construction in \citet[Proposition~3 and its proof]{yarotsky2017} on a fixed bounded interval, recursively approximate each product of $d$ univariate hats, and project the result onto $[0,1]$.
With
\begin{equation*}
  \varepsilon_{\mathrm{mult}}:=\frac{\varepsilon}{2J_{\mathrm{hat}}(M_f\vee1)},
\end{equation*}
each approximate tensor hat $\widehat{\psi}_j$ can be chosen so that $\|\widehat{\psi}_j-\psi_j\|_{\infty}\le\varepsilon_{\mathrm{mult}}$.
Consequently,
\begin{equation*}
  \left\|
  \sum_{j=1}^{J_{\mathrm{hat}}}f(u_j)\widehat{\psi}_j-I_{\eta_{\mathrm{grid}}}f
  \right\|_{\infty}
  \le J_{\mathrm{hat}}M_f\varepsilon_{\mathrm{mult}}\le\varepsilon/2.
\end{equation*}
The multiplication paths have fixed width and depth $\Ord_d\!\left(\log(2+\varepsilon_{\mathrm{mult}}^{-1})\right)$.
Parallelizing the $J_{\mathrm{hat}}$ subnetworks gives a fixed wiring pattern whose depth, width, and total number of scalar weights and biases obey \cref{eq:fixed-dimensional-holder-network-size}; the extra logarithm allows all coordinates to share one network graph.
Fix every hat and multiplication weight at the value used in these $J_{\mathrm{hat}}$ approximate basis functions.
For a coefficient array $\mathbf{a}=(a_{kj})_{1\le k\le d_{\mathrm{out}},1\le j\le J_{\mathrm{hat}}}$, define
\begin{equation*}
  \widetilde{f}_{\mathbf{a}}(x)
  :=\left(\sum_{j=1}^{J_{\mathrm{hat}}} a_{kj}\widehat{\psi}_j(x)\right)_{k=1}^{d_{\mathrm{out}}},
  \qquad
  f_{\mathbf{a}}:=\proj_{[-M_f,M_f]^{d_{\mathrm{out}}}}\circ\widetilde{f}_{\mathbf{a}}.
\end{equation*}
Choose $B_{\varepsilon}$ to bound the fixed network weights, the hat slopes, and the coefficient range $[-B_{\varepsilon},B_{\varepsilon}]^{d_{\mathrm{out}}J_{\mathrm{hat}}}$, with $B_{\varepsilon}\ge M_f$.
Only $\mathbf{a}$ is variable; all internal network parameters are fixed and therefore are not trainable.
Pad $\mathbf{a}$ by inactive coordinates to obtain a parameter vector $\vartheta\in[-B_{\varepsilon},B_{\varepsilon}]^{S_{\varepsilon}}$.
The resulting class and its fixed wiring depend only on $(d,d_{\mathrm{out}},\beta_0,M_f,H,\varepsilon)$ and not on $f$.
Taking $a_{kj}=f_k(u_j)$ gives the approximant established above.

Because $0\le\widehat{\psi}_j\le1$ and Euclidean projection onto the box is nonexpansive,
\begin{equation*}
  \sup_x\|f_{\mathbf{a}}(x)-f_{\mathbf{a}'}(x)\|_{\infty}
  \le J_{\mathrm{hat}}\|\mathbf{a}-\mathbf{a}'\|_{\infty}.
\end{equation*}
Thus one may take $\Lip_{\mathrm{par}}=J_{\mathrm{hat}}$.
Each exact univariate hat is $C_d\eta_{\mathrm{grid}}^{-1}$-Lipschitz.
The multiplication paths have fixed width, uniformly bounded weights, and depth $\Ord_d(\log(2+\varepsilon_{\mathrm{mult}}^{-1}))$.
Write $\operatorname{Lip}_{\infty}(g)$ for the Lipschitz constant of a map $g$ with respect to the $\ell_{\infty}$ norm.
For a constant $\kappa_d>0$ depending only on $d$, the layerwise operator-norm bound gives
\begin{equation*}
  \max_{1\le j\le J_{\mathrm{hat}}}\operatorname{Lip}_{\infty}(\widehat{\psi}_j)
  \le C_d\eta_{\mathrm{grid}}^{-1}(2+\varepsilon_{\mathrm{mult}}^{-1})^{\kappa_d}
  =:\Lip_{\psi,\varepsilon}.
\end{equation*}
Consequently, uniformly over the coefficient cube,
\begin{equation*}
  \|f_{\mathbf{a}}(x)-f_{\mathbf{a}}(x')\|_{\infty}
  \le J_{\mathrm{hat}}B_{\varepsilon}\Lip_{\psi,\varepsilon}\|x-x'\|_{\infty}.
\end{equation*}
Hence one may take $\Lip_{\mathrm{in}}=J_{\mathrm{hat}}B_{\varepsilon}\Lip_{\psi,\varepsilon}$.
The quantities $J_{\mathrm{hat}}$, $\eta_{\mathrm{grid}}^{-1}$, $\varepsilon_{\mathrm{mult}}^{-1}$, and $B_{\varepsilon}$ are polynomial in $(1+M_f+H+\varepsilon^{-1})$ because $d$ is fixed.
This proves \cref{eq:fixed-dimensional-holder-parameter-stability}.
All remaining measurability and continuity claims follow from finite continuous composition.
\end{proof}

\subsubsection{Time-local approximation and output clipping}
\label{app:subsubsec:time-local-approximation-and-clipping}

The OU-basis functions vary most rapidly near the early-stopping time.
We approximate them on overlapping time intervals before combining the local approximations with a fixed partition of unity.
Separate tolerances for the denominator and numerator account for their different effects on the score ratio.

\begin{lemma}[Time-local OU-basis approximation]
\label{lem:smooth-latent-time-local-basis}
Under \cref{assump:smooth-latent-task}, use the constants from \cref{lem:positive-smooth-latent-derived-regularity}, and fix $0<t_0<T$, the state radius, and $\varepsilon_{a,n},\varepsilon_{b,n}\in(0,1]$.
Define
\begin{equation*}
 K:=\max\{1,\lfloor\log_2(T/t_0)\rfloor\},\qquad
 t_j^{\mathrm{knot}}:=2^jt_0\ (0\le j<K),\qquad t_K^{\mathrm{knot}}:=T.
\end{equation*}
For $0\le j\le K$, set
\begin{equation*}
 a_j:=t_{\max\{0,j-1\}}^{\mathrm{knot}},\qquad
 b_j:=t_{\min\{K,j+1\}}^{\mathrm{knot}},\qquad
 h_j^{\mathrm{loc}}:=b_j-a_j.
\end{equation*}
Let $R_{\mathrm{geom},n}:=1+R_{\mathrm{state},n}+\sqrt{d_z}$, and fix a sufficiently large constant $C_{\mathrm{loc}}$ depending only on the fixed dimensions, $\beta'$, and the common density bounds.
Apply the extension construction of \cref{lem:smooth-latent-normalized-basis-regularity} on each interval, using the query and state normalization in \cref{eq:transformer-input-preprocessing} and affine normalization of the local time interval to $[-1,1]$.
The resulting denominator and numerator maps have respective $\beta'$-H\"older moduli bounded by
\begin{align}
 H_{a,j,n}
 &:=C_{\mathrm{loc}}\big[
 1+H_{q,x}
 +R_{\mathrm{state},n}R_{\mathrm{geom},n}\sigma_{a_j}^{-2}
 +h_j^{\mathrm{loc}}R_{\mathrm{geom},n}^2\sigma_{a_j}^{-4}\big],
 \label{eq:smooth-latent-local-denominator-modulus}\\
 H_{b,j,n}
 &:=C_{\mathrm{loc}}\big[
 1+R_{\mathrm{bal},n}H_{q,x}
 +R_{\mathrm{state},n}
   \{\sigma_{a_j}^{-1}+R_{\mathrm{geom},n}^2\sigma_{a_j}^{-3}\}
 \nonumber\\
 &\hspace{18mm}
 +h_j^{\mathrm{loc}}\{\sigma_{a_j}^{-1}
 +R_{\mathrm{geom},n}\sigma_{a_j}^{-3}
 +R_{\mathrm{geom},n}^3\sigma_{a_j}^{-5}\}\big].
 \label{eq:smooth-latent-local-numerator-modulus}
\end{align}
Write
\begin{equation*}
 H_{\max,n}:=\max_{0\le j\le K}(H_{a,j,n}\vee H_{b,j,n}),
\end{equation*}
and
\begin{equation*}
 L_{\mathrm{bas},n}
 :=\log\!\left(
 2+K+\frac{1+R_{\mathrm{num},n}^{\mathrm{lat}}+H_{\max,n}}
 {\varepsilon_{a,n}\wedge\varepsilon_{b,n}}
 \right).
\end{equation*}
For each latent-grid point, there is a shared fixed-architecture ReLU class containing approximations with uniform errors at most $\varepsilon_{a,n}$ for $\mathsf{a}$ and $\varepsilon_{b,n}$ in Euclidean norm for $\mathsf{b}$ on $\mathcal{G}_n$.
Across all $J_{\theta,n}$ points, the number of trainable parameters, number of stored nonzero coefficients, and evaluation cost are bounded by
\begin{equation}
 \mathsf{N}_{\mathrm{basis},n}^{\mathrm{lat}}
 :=C_{\mathrm{loc}} J_{\theta,n}L_{\mathrm{bas},n}^2
 \sum_{j=0}^K\left[
 \left(1+\frac{H_{a,j,n}}{\varepsilon_{a,n}}\right)^{(d_{\mathrm{stat}}+1)/\beta'}
 +d_z\left(1+\frac{H_{b,j,n}}{\varepsilon_{b,n}}\right)^{(d_{\mathrm{stat}}+1)/\beta'}
 \right].
 \label{eq:smooth-latent-time-local-basis-size}
\end{equation}
Every member is projected onto $[\underline{q}_{\mathrm{out}},\overline{q}_{\mathrm{out}}]$ in its denominator coordinate and onto $[-R_{\mathrm{num},n}^{\mathrm{lat}}-1,R_{\mathrm{num},n}^{\mathrm{lat}}+1]$ in each numerator coordinate.
Only output-layer interpolation coefficients are trainable; they range over a compact cube, with inactive padding coordinates allowed.
The parameter radius and the input and parameter Lipschitz constants are polynomial in
\begin{equation*}
 1+K+T+t_0^{-1}+R_{\mathrm{state},n}
 +R_{\mathrm{num},n}^{\mathrm{lat}}+H_{\max,n}
 +(\varepsilon_{a,n}\wedge\varepsilon_{b,n})^{-1}.
\end{equation*}
The class depends only on the displayed common quantities, not on the task map or task prior.
\end{lemma}

\begin{proof}
The covariance calculation in \cref{app:subsubsec:diffusion-feature-regularity}, in particular \cref{eq:appendix-smooth-latent-covariance-identity}, applies on every interval $[a_j,b_j]$ with $\sigma_{t_0}$ replaced by $\sigma_{a_j}$.
Affine time normalization multiplies the time derivative by $h_j^{\mathrm{loc}}/2$.
The denominator has query modulus $H_{q,x}$, whereas the numerator has query modulus at most $R_{\mathrm{bal},n}H_{q,x}$.
Combining these observations with the separate state and time derivatives proves \cref{eq:smooth-latent-local-denominator-modulus,eq:smooth-latent-local-numerator-modulus}.

Let $\chi_0,\ldots,\chi_K$ be the continuous piecewise-linear nodal hats on the knots $t_0^{\mathrm{knot}},\ldots,t_K^{\mathrm{knot}}$, restricted to $[t_0,T]$.
They are nonnegative, sum to one, and $\chi_j$ is supported on $[a_j,b_j]$.
Each is an exact affine--ReLU function of the globally normalized time coordinate.
The final interval includes any shorter remaining interval: when $K\ge2$, consecutive knot gaps are at least $t_0$ and $h_j^{\mathrm{loc}}\le7a_j$.
When $K=1$, both supports are the whole interval and normalization cancels $T-t_0$.
Thus the hats and the local coordinates after coordinatewise projection
\begin{equation*}
 u_{t,j}:=
 2\min\{1,\max\{0,(t-a_j)/h_j^{\mathrm{loc}}\}\}-1
\end{equation*}
have coefficients and input Lipschitz constants polynomial in $1+T+t_0^{-1}$ when expressed in the globally normalized input.

Consider one scalar function $f$, with range bound $B:=1+\overline{q}_{\mathrm{out}}+R_{\mathrm{num},n}^{\mathrm{lat}}$ and desired error $\varepsilon$.
Use $\varepsilon=\varepsilon_{a,n}$ for the denominator and $\varepsilon=\varepsilon_{b,n}/\sqrt{d_z}$ for each numerator coordinate.
On the $j$th normalized cube, construct the tensor-hat interpolant $I_jf$ from the proof of Lemma~\ref{lem:fixed-dimensional-holder-relu-approximation}, with error at most $\varepsilon/2$.
Let $H_{f,j,n}$ be $H_{a,j,n}$ for the denominator and $H_{b,j,n}$ for a numerator coordinate.
Its number of grid points is $\lesssim(1+H_{f,j,n}/\varepsilon)^{(d_{\mathrm{stat}}+1)/\beta'}$.
Write its tensor hats as $\psi_{j,r}$ and its nodal coefficients as $c_{j,r}$.
Then
\begin{equation*}
 I f:=\sum_{j,r}c_{j,r}\chi_j(t)\psi_{j,r}(u_x,u_z,u_{t,j})
\end{equation*}
has error at most $\varepsilon/2$: every nonzero $\chi_j$ uses a local interpolant valid at that time, and their weights sum to one.

Let $J_f^{\mathrm{loc}}$ be the total number of these local grid points.
Each product $\chi_j\psi_{j,r}$ has $d_{\mathrm{stat}}+2$ bounded univariate factors.
Use the fixed multiplication paths in the proof of Lemma~\ref{lem:fixed-dimensional-holder-relu-approximation} and project their outputs onto $[0,1]$.
Approximate each product to accuracy $\varepsilon/\{2J_f^{\mathrm{loc}}(B\vee1)\}$.
The accumulated error after weighting by the interpolant coefficients is at most $\varepsilon/2$.
Only $c_{j,r}$ vary, and the final coordinatewise projections cannot increase error.
The number of coefficients and operations is $\lesssim J_f^{\mathrm{loc}}\log^2(2+J_f^{\mathrm{loc}} B/\varepsilon)$.
Summing the coordinates and latent-grid points gives \cref{eq:smooth-latent-time-local-basis-size} after absorbing fixed dimension factors.

For arbitrary trainable coefficients, every approximate hat remains in $[0,1]$, so parameter sensitivity is at most $J_f^{\mathrm{loc}}$ before the final nonexpansive projection.
The univariate hat slopes and time-coordinate maps have polynomial bounds.
The multiplication paths have fixed width, bounded weights, and logarithmic depth; their input Lipschitz constants are therefore polynomial in the displayed quantities, by the same layerwise bound used in \cref{eq:fixed-dimensional-holder-parameter-stability}.
Summing $J_f^{\mathrm{loc}}$ paths with bounded output-layer coefficients preserves polynomial input sensitivity.
Running the component networks in parallel and adding inactive coordinates yields the stated common compact parameter cube.
All computations are continuous, so evaluation is jointly Borel and continuous in the parameters.
\end{proof}

\begin{lemma}[Clipping error for the scaled score]
\label{lem:smooth-latent-output-clipping}
For a deterministic radius $R_{\mathrm{out},n}>0$, set
\begin{equation}
 \varepsilon_{\mathrm{out,clip},n}^2
 :=A_{t_0,T}\E(\|\xi\|_2-R_{\mathrm{out},n})_{+}^2,
 \qquad \xi\sim\gN(0,I_{d_z}).
 \label{eq:smooth-latent-output-clipping-error}
\end{equation}
Under the episodic OU model,
\begin{equation}
 \E_{H_n}\int_{t_0}^T
 \frac{\E_{Z_t\mid H_n}
 \|\widetilde{s}^{\star}(Z_t,t\mid H_n)
 -\clip_{R_{\mathrm{out},n}}\widetilde{s}^{\star}(Z_t,t\mid H_n)\|_2^2}
 {\sigma_t^2}\,\od t
 \le\varepsilon_{\mathrm{out,clip},n}^2.
 \label{eq:smooth-latent-output-clipping-bias}
\end{equation}
If $R_{\mathrm{out},n}\ge\sqrt{d_z}$, then
\begin{equation}
 \varepsilon_{\mathrm{out,clip},n}^2
 \le2A_{t_0,T}
 \exp\!\left\{-\tfrac{1}{2}(R_{\mathrm{out},n}-\sqrt{d_z})^2\right\}.
 \label{eq:smooth-latent-output-clipping-tail}
\end{equation}
\end{lemma}

\begin{proof}
The denoising identity gives $\widetilde{s}^{\star}(Z_t,t\mid H_n)=-\E[\xi\mid Z_t,H_n]$.
The squared distance to a closed convex ball is convex.
Conditional Jensen therefore bounds its expectation at this conditional mean by $\E(\|\xi\|_2-R_{\mathrm{out},n})_{+}^2$.
Integration against $\sigma_t^{-2}\od t$ proves \cref{eq:smooth-latent-output-clipping-bias}.
More generally, Gaussian norm concentration gives, for every $R\ge\sqrt{d_z}$ and $r=R-\sqrt{d_z}$,
\begin{equation}
 \E(\|\xi\|_2-R)_{+}^2
 \le\int_0^\infty 2u e^{-(r+u)^2/2}\,\od u
 \le2e^{-r^2/2}.
 \label{eq:gaussian-radial-tail-moment}
\end{equation}
Taking $R=R_{\mathrm{out},n}$ proves \cref{eq:smooth-latent-output-clipping-tail}.
\end{proof}

\subsubsection{Prior logits and stability of the score ratio}

Define
\begin{equation*}
  w_{j,n}^{\mathrm{floor}}
  :=\frac{w_{j,n}+\varepsilon_{\mathrm{floor},n}/J_{\theta,n}}
  {1+\varepsilon_{\mathrm{floor},n}},
  \qquad
  \ell_{j,n}^{\mathrm{prior}}
  :=\log\{J_{\theta,n}w_{j,n}^{\mathrm{floor}}\}.
\end{equation*}
Here $\mathbf{w}_n=(w_{j,n})_{j=1}^{J_{\theta,n}}$ and $\mathbf{w}_n^{\mathrm{floor}}=(w_{j,n}^{\mathrm{floor}})_{j=1}^{J_{\theta,n}}$.
Write $\boldsymbol{\omega}_n(h)$ and $\boldsymbol{\omega}_n^{\mathrm{floor}}(h)$ for the posterior-weight vectors obtained from these two prior vectors, respectively.
This is the specialization of the generic prior-floor construction in \cref{prop:episode-kl-score-approximation} to $\zeta=\varepsilon_{\mathrm{floor},n}$; thus $w_{j,n}^{\mathrm{floor}}=w_{j,n}^{\zeta}$ and $\omega_{j,n}^{\mathrm{floor}}=\omega_{j,n}^{\zeta}$ at this value of $\zeta$.
As usual, let $x_{\star}$ denote the query component of $h$, with $x_{\star}=0$ in the sample-only prompt model.
Define the score obtained from these modified prior weights by
\begin{equation*}
  \widetilde{s}_n^{\mathrm{floor}}(z,t\mid h)
  :=
  \frac{\sum_j\omega_{j,n}^{\mathrm{floor}}(h)
  \mathsf{b}_{\theta_{j,n}}(x_{\star},z,t)}
  {\sum_j\omega_{j,n}^{\mathrm{floor}}(h)
  \mathsf{a}_{\theta_{j,n}}(x_{\star},z,t)}.
\end{equation*}
Because
\begin{equation*}
  \frac{\varepsilon_{\mathrm{floor},n}}{2J_{\theta,n}}
  \le w_{j,n}^{\mathrm{floor}}\le1,
\end{equation*}
one has
\begin{equation*}
  \frac{\varepsilon_{\mathrm{floor},n}}{2}
  \le J_{\theta,n}w_{j,n}^{\mathrm{floor}}
  \le J_{\theta,n}.
\end{equation*}
Thus
\begin{equation}
  |\ell_{j,n}^{\mathrm{prior}}|
  \le\log\!\left(\frac{2J_{\theta,n}}{\varepsilon_{\mathrm{floor},n}}\right).
  \label{eq:smooth-latent-prior-logit-radius}
\end{equation}

For two probability vectors $\omega,\omega'$, the corresponding score ratios satisfy
\begin{equation*}
  \left\|
  \frac{\sum_j\omega_j\mathsf{b}_j}{\sum_j\omega_j\mathsf{a}_j}
  -
  \frac{\sum_j\omega'_j\mathsf{b}_j}{\sum_j\omega'_j\mathsf{a}_j}
  \right\|_2
  \le\mathsf{C}_{\mathrm{mix},n}\|\omega-\omega'\|_1.
\end{equation*}
This follows by bounding the numerator change by $R_{\mathrm{num},n}^{\mathrm{lat}}\|\omega-\omega'\|_1$, the denominator change by $\overline{q}_{\mathrm{out}}\|\omega-\omega'\|_1$, and both denominators below by $\underline{q}_{\mathrm{out}}$.

\subsubsection{Neural approximation of the finite-mixture score}
\label{app:subsubsec:shared-neural-approximation}

Apply \cref{lem:fixed-dimensional-holder-relu-approximation} to the $J_{\theta,n}$ context-example log-likelihood maps and, when present, the query log-likelihood maps, using the common ranges and moduli in \cref{lem:positive-smooth-latent-derived-regularity}.
Use Lemma~\ref{lem:smooth-latent-time-local-basis} for the OU-basis coordinates, with denominator tolerance $\varepsilon_{a,n}$ and Euclidean numerator tolerance $\varepsilon_{b,n}$.
Apply exact coordinatewise ReLU projection onto $[-B_{\ell,\mathrm{ctx}},B_{\ell,\mathrm{ctx}}]$ for every context-example log-likelihood output and onto $[-B_{\ell,\mathrm{qry}},B_{\ell,\mathrm{qry}}]$ for every query log-likelihood output.
These intervals contain the corresponding log-likelihood functions, so projection cannot enlarge their approximation errors and gives a common range for the entire score class.
The time-local basis class allocates scalar numerator tolerances $\varepsilon_{b,n}/\sqrt{d_z}$, with fixed dimension factors absorbed in $\mathsf{N}_{\mathrm{basis},n}^{\mathrm{lat}}$.
Project every approximate $\mathsf{a}$-coordinate onto $[\underline{q}_{\mathrm{out}},\overline{q}_{\mathrm{out}}]$ using its exact scalar ReLU representation.
Do not radially project the approximate $\mathsf{b}$-vectors inside the ReLU--attention network.
Instead, project their coordinates onto
\begin{equation*}
  [-R_{\mathrm{num},n}^{\mathrm{lat}}-1,
  R_{\mathrm{num},n}^{\mathrm{lat}}+1],
\end{equation*}
which again cannot enlarge the approximation error because every exact numerator coordinate lies in the smaller interval with radius $R_{\mathrm{num},n}^{\mathrm{lat}}$.
At the reference coefficients, their approximation guarantee and $\varepsilon_{b,n}\le1$ give
\begin{equation*}
  \|\widehat{\mathsf{b}}_j\|_2
  \le R_{\mathrm{num},n}^{\mathrm{lat}}+1.
\end{equation*}
Extend the ratio map to the enclosing rectangular box
\begin{equation*}
  [\underline{q}_{\mathrm{out}},\overline{q}_{\mathrm{out}}]
  \times[-R_{\mathrm{num},n}^{\mathrm{lat}}-1,
  R_{\mathrm{num},n}^{\mathrm{lat}}+1]^{d_z}
\end{equation*}
by
\begin{equation*}
  (a,b)\longmapsto
  \frac{b}{\min\{\overline{q}_{\mathrm{out}},\max\{\underline{q}_{\mathrm{out}},a\}\}},
\end{equation*}
and affinely rescale this box to a unit cube.
After affine normalization, this map has range bounded by
\begin{equation*}
  1+\frac{R_{\mathrm{num},n}^{\mathrm{lat}}+1}{\underline{q}_{\mathrm{out}}}
\end{equation*}
and a Lipschitz constant bounded, up to a fixed dimension factor, by
\begin{equation*}
  (1+\overline{q}_{\mathrm{out}})(2+R_{\mathrm{num},n}^{\mathrm{lat}})
  (1+\underline{q}_{\mathrm{out}}^{-2}).
\end{equation*}
It agrees with $b/a$ on the specified ball domain and satisfies the range and Lipschitz bounds used in $\mathsf{N}_{\mathrm{rat},n}^{\mathrm{lat}}$.
Apply \cref{lem:fixed-dimensional-holder-relu-approximation} to each output coordinate with scalar tolerance $\varepsilon_{\mathrm{rat},n}/\sqrt{d_z}$, so its vector error is at most $\varepsilon_{\mathrm{rat},n}$ in Euclidean norm.

Put $\varepsilon_{\mathrm{logit},n}:=n\varepsilon_{\ell,n}+\varepsilon_{\mathrm{qry},n}$.
Write $\widehat{\ell}_j(c)$ for the $j$th context-example log-likelihood-network output, set $\overline{\widehat{\ell}}_j:=n^{-1}\sum_{i=1}^n\widehat{\ell}_j(c_i)$, and write $\widehat{\ell}_j^{\mathrm{qry}}(x_{\star})$ for the corresponding query log-likelihood-network output.
In the sample-only prompt model, use the convention $\widehat{\ell}_j^{\mathrm{qry}}(0):=0$.
For the modified prior weights, the exact quadrature logits and their neural approximations are
\begin{align*}
  \ell_{j,n}^{\mathrm{post}}(h)
  &:=
  \ell_{j,n}^{\mathrm{prior}}
  +\sum_{i=1}^n\ell_{\theta_{j,n}}(c_i)
  +\ell_{\theta_{j,n}}^{\mathrm{qry}}(x_{\star}),\\
  \widehat{\ell}_{j,n}^{\mathrm{post}}(h)
  &:=
  \ell_{j,n}^{\mathrm{prior}}
  +n\overline{\widehat{\ell}}_j
  +\widehat{\ell}_j^{\mathrm{qry}}(x_{\star}).
\end{align*}
Their uniform coordinatewise difference is at most $\varepsilon_{\mathrm{logit},n}$.
The softmax map satisfies
\begin{equation*}
  \|\operatorname{softmax}(a)-\operatorname{softmax}(a')\|_1
  \le2\|a-a'\|_{\infty}.
\end{equation*}
Hence the posterior-weight error is at most $2\varepsilon_{\mathrm{logit},n}$.
The corresponding weighted numerator and denominator errors are at most
\begin{equation*}
  2R_{\mathrm{num},n}^{\mathrm{lat}}\varepsilon_{\mathrm{logit},n}
  +\varepsilon_{b,n}
\end{equation*}
and
\begin{equation*}
  2\overline{q}_{\mathrm{out}}\,\varepsilon_{\mathrm{logit},n}
  +\varepsilon_{a,n},
\end{equation*}
respectively.
Both approximate denominators remain in $[\underline{q}_{\mathrm{out}},\overline{q}_{\mathrm{out}}]$.
The same decomposition of the difference between two ratios used above, followed by the ratio-readout error, gives
\begin{align}
  &\frac{2R_{\mathrm{num},n}^{\mathrm{lat}}\varepsilon_{\mathrm{logit},n}
  +\varepsilon_{b,n}}{\underline{q}_{\mathrm{out}}}
  +
  \frac{R_{\mathrm{num},n}^{\mathrm{lat}}
  \{2\overline{q}_{\mathrm{out}}\,\varepsilon_{\mathrm{logit},n}
  +\varepsilon_{a,n}\}}{\underline{q}_{\mathrm{out}}^2}
  +\varepsilon_{\mathrm{rat},n}
  \le\varepsilon_{\mathrm{net},n}.
  \label{eq:smooth-latent-network-error}
\end{align}
\textbf{Masked Transformer implementation.}
Set
\begin{equation}
  B_{\mathrm{gate},n}^{\mathrm{lat}}
  :=1+\log\!\left(\frac{2J_{\theta,n}}{\varepsilon_{\mathrm{floor},n}}\right)
  +nB_{\ell,\mathrm{ctx}}+B_{\ell,\mathrm{qry}}
  +\overline{q}_{\mathrm{out}}+R_{\mathrm{num},n}^{\mathrm{lat}}.
  \label{eq:smooth-latent-selector-radius}
\end{equation}
Every coordinate used by the selector is bounded by \cref{eq:smooth-latent-selector-radius}; the preliminary projection in $\mathsf{R}_{3,n}$ is therefore inactive at the reference parameter.
Apply \cref{lem:smooth-latent-four-block-construction} with $J=J_{\theta,n}$, the coordinatewise projections of the component-network outputs constructed above, and the logits of the modified prior weights.
The lemma specifies the exact token width, initial token shapes, attention masks, projections, residual destinations, and selector network.
It follows from \cref{eq:smooth-latent-four-block-output} that the output token contains the approximate posterior-weighted denominator and numerator exactly.
The fixed ReLU ratio readout acts on the output token, and the fixed radial output projection completes the score map.
Deterministic nonempty masks, fixed role and identifier codes, and fixed selector and ratio networks introduce no trainable coordinates.
The context-example and query log-likelihood, OU-basis, prior-logit, and attention-projection coordinates all range over declared compact cubes and are jointly trainable.
Their product is $\gV_n^{\mathrm{fit}}$ in \cref{def:smooth-latent-score-class}, and coordinatewise projection from the enclosing parameter cube gives exactly the class representation in \cref{eq:smooth-latent-score-class-formal}.
The constructed approximator coefficients, logits of the modified prior weights, and the four sets of projection matrices in \cref{tab:smooth-latent-four-block-construction} define a point of this product whose output is the constructed score. Its finite-mixture error is bounded by \cref{eq:smooth-latent-network-error}.
The architectures of the approximators and the radii of their parameter cubes depend only on the shared bounds, dimensions, grid, and tolerances, so the same class applies to every task map and task prior satisfying the assumptions.

Consequently, the constructed output $f$ before the final radial projection satisfies $\|f-\widetilde{s}_n^{\mathrm{floor}}\|_2\le\varepsilon_{\mathrm{net},n}$ on $\mathcal{G}_n$.
This estimate compares the network with the discretized score computed from the modified prior weights; it does not compare that score with the original posterior predictive.
The integrated approximation guarantee below uses this same construction.

\subsection{Integrated Transformer Approximation Guarantee}
\label{app:subsec:integrated-transformer-approximation}

\begin{proposition}[Integrated score approximation by the Transformer class]
\label{prop:integrated-transformer-approximation}
Use the regular grid and cell-midpoint map from \cref{app:subsec:latent-quadrature}, set $\eta=\eta_{\theta,n}$ and $\zeta=\varepsilon_{\mathrm{floor},n}>0$ in \cref{prop:episode-kl-score-approximation}, and use the class in \cref{def:smooth-latent-score-class}.
Let $\varepsilon_{\mathrm{net},n}$ be the component-error expression in \cref{app:subsec:shared-neural-estimates}.
Then
\begin{align}
 \gE_{\mathrm{app}}(\gS_n^{\mathrm{lat}})
 \le {}&4D_n(\eta_{\theta,n},\varepsilon_{\mathrm{floor},n})
       +4A_{t_0,T}\varepsilon_{\mathrm{net},n}^2
       +2\varepsilon_{\mathrm{out,clip},n}^2 \nonumber\\
       &+2A_{t_0,T}R_{\mathrm{out},n}^2\delta_{\mathrm{tail},n}
       +2\varepsilon_{\mathrm{tail},n}^2.
 \label{eq:integrated-transformer-approximation}
\end{align}
The parameters of the approximating network may depend on the task family and prior.
Its architecture, coordinate bounds, and approximation guarantee depend only on the common bounds, grid, and tolerances.
\end{proposition}

\begin{proof}
By \cref{eq:regular-grid-floored-score-identity}, the construction in \cref{app:subsubsec:shared-neural-approximation}, in particular \cref{eq:smooth-latent-network-error}, gives an output $f$ before the final radial projection satisfying $\|f-\sigma_t s^{\eta_{\theta,n},\zeta}\|_2\le\varepsilon_{\mathrm{net},n}$ on $\mathcal{G}_n$. 
No uniform comparison of $s^{\eta_{\theta,n},\zeta}$ with $s^\star$ is used here.
Recall that $\widetilde{s}^\star=\sigma_t s^\star$, and let $\mathsf{P}$ be radial projection onto the ball of radius $R_{\mathrm{out},n}$.
On $\mathcal{G}_n$, nonexpansiveness and two applications of the squared triangle inequality give
\begin{equation*}
 \|\mathsf{P} f-\widetilde{s}^\star\|_2^2
 \le4\varepsilon_{\mathrm{net},n}^2
   +4\sigma_t^2\|s^{\eta_{\theta,n},\zeta}-s^\star\|_2^2
   +2\|\mathsf{P}\widetilde{s}^\star-\widetilde{s}^\star\|_2^2.
\end{equation*}
Integrate against $\sigma_t^{-2}P_{t\mid H_n}\,\od t$ and the original prompt law.
The first term integrates to at most $4A_{t_0,T}\varepsilon_{\mathrm{net},n}^2$, the second is bounded by \cref{eq:integrated-latent-score-bound}, and the third by \cref{lem:smooth-latent-output-clipping}.
On $\mathcal{G}_n^c$, use $\|\mathsf{P} f-\widetilde{s}^\star\|^2\le2R_{\mathrm{out},n}^2+2\|\widetilde{s}^\star\|^2$ and the two bounds in \cref{lem:balanced-score-truncation}.
This proves \cref{eq:integrated-transformer-approximation}.
\end{proof}

For a desired integrated score error of order $\varepsilon^2$, it is enough to take $\eta_{\theta,n}^{\beta_\theta}$ of order $\varepsilon/\sqrt{n+1}$, prior floor of order $\varepsilon^2$, and neural scaled-score error of order $\varepsilon/\sqrt{A_{t_0,T}}$.
The conclusion is averaged under the true episode law, which is the criterion used by the score-learning and sampling analyses; it does not assert a pointwise approximation guarantee for every prompt.

\subsection{Class Complexity: Parameters, Sensitivity, and Covering}
\label{app:subsec:smooth-latent-complexity}

\Cref{prop:integrated-transformer-approximation} constructs one accurate parameter vector for each task family and task prior.
Training can select any vector in the class, so the learning analysis requires a bound on output sensitivity throughout the parameter cube, not just near that approximator.
The next proposition gives this bound and the resulting covering number, together with counts of trainable parameters, stored nonzero scalar coefficients, and sparse arithmetic operations in one forward pass.
Masks, role labels, latent-cell indices, and their bit encodings are excluded from scalar-coefficient storage.

\begin{proposition}[Transformer size and covering numbers]
\label{prop:smooth-latent-complexity}
Under \cref{assump:smooth-latent-task}, fix any regular latent grid and component tolerances satisfying \cref{def:smooth-latent-score-class}.
Let $\mathsf{N}_{\mathrm{param},n}^{\mathrm{lat}}$ count trainable scalar coordinates, let $\mathsf{N}_{\mathrm{store},n}^{\mathrm{lat}}$ count stored nonzero scalar coefficients under this convention, and let $\mathsf{N}_{\mathrm{op},n}^{\mathrm{lat}}$ count sparse arithmetic operations in one forward evaluation.
With the indexing in \cref{def:smooth-latent-score-class}, the score class has the compact trainable-parameter representation
\begin{equation*}
  \gS_n^{\mathrm{lat}}
  =\{s_{\mathbf{v}}:\mathbf{v}\in[-B_{\mathrm{lat},n},B_{\mathrm{lat},n}]^{\mathsf{N}_{\mathrm{param},n}^{\mathrm{lat}}}\}.
\end{equation*}
There is a deterministic constant $\Lip_{\mathrm{lat},n}^{\mathrm{par}}$ such that
\begin{equation}
  d_{\mathrm{bal},\infty}(s_{\mathbf{v}},s_{\mathbf{v}'})
  \le
  \Lip_{\mathrm{lat},n}^{\mathrm{par}}
  \|\mathbf{v}-\mathbf{v}'\|_{\infty}.
  \label{eq:smooth-latent-parameter-sensitivity}
\end{equation}
This bound holds for all parameter vectors $\mathbf{v}$ and $\mathbf{v}'$ in the displayed cube.
Consequently, for every $\varepsilon_{\mathrm{cov}}>0$,
\begin{equation}
  \log\Ncov\!\left(
  \varepsilon_{\mathrm{cov}},\gS_n^{\mathrm{lat}},d_{\mathrm{bal},\infty}
  \right)
  \le
  \mathsf{N}_{\mathrm{param},n}^{\mathrm{lat}}
  \log\!\left(
  1+\frac{2B_{\mathrm{lat},n}
  \Lip_{\mathrm{lat},n}^{\mathrm{par}}}{\varepsilon_{\mathrm{cov}}}
  \right).
  \label{eq:smooth-latent-score-cover}
\end{equation}
The empirical variance-balanced DSM loss admits a Borel-measurable exact ERM over this compact parameter cube.
Explicit bounds for $\mathsf{N}_{\mathrm{param},n}^{\mathrm{lat}}$, $\mathsf{N}_{\mathrm{store},n}^{\mathrm{lat}}$, and $\mathsf{N}_{\mathrm{op},n}^{\mathrm{lat}}$ are given in \cref{eq:smooth-latent-trainable-parameter-count,eq:smooth-latent-coefficient-storage,eq:smooth-latent-forward-cost}.
These conclusions hold uniformly over the task maps in the fixed regularity class of \cref{assump:smooth-latent-task} and every Borel task prior on $\gT$.
\end{proposition}

The bounds on the number of parameters, stored coefficients, and arithmetic operations use exact real arithmetic, counting each scalar arithmetic operation, comparison, exponential, logarithm, and square root at unit cost, and do not establish finite-precision stability or optimization runtime.
For a larger trainable architecture, the covering analysis would need its full parameter count, radius, and sensitivity.
The class is indexed by prompt length and approximation choices; no parameter sharing or extrapolation across prompt lengths is established.

\subsubsection{Parameter counts and global parameter sensitivity}
\label{app:subsubsec:parameter-counts-and-sensitivity}

We count the construction's trainable coordinates, stored coefficients, and sparse operations, and then combine these bounds with global parameter sensitivity to obtain the covering-number bound used in the finite-episode learning analysis.
The parameter choices for a target accuracy are deferred to \cref{app:subsec:accuracy-parameter-choices}.
\Cref{app:subsec:accuracy-parameter-choices} later specifies the grid and approximation tolerances for a prescribed predictive accuracy; the present bounds do not require those specific choices.

For fixed $d$ and $\beta_0\in(0,1]$, define
\begin{equation*}
  \mathfrak{C}_{d,\beta_0}(B,H,\varepsilon)
  :=
  C_{d,\beta_0}
  \left(1+\frac{H\vee1}{\varepsilon}\right)^{d/\beta_0}
  \log^2\!\left(2+\frac{1+B+H}{\varepsilon}\right).
\end{equation*}
This is the ReLU-network size needed for a fixed-dimensional $\beta_0$-H\"older approximation with range bound $B$, H\"older constant $H$, and approximation error $\varepsilon$.
Set
\begin{align*}
  \mathsf{N}_{\ell,n}^{\mathrm{lat}}
  &:=
  J_{\theta,n}
  \mathfrak{C}_{d_{\mathrm{stat}},\beta'}
  (B_{\ell,\mathrm{ctx}},H_{\ell,c},\varepsilon_{\ell,n}),\\
  \mathsf{N}_{\mathrm{qry},n}^{\mathrm{lat}}
  &:=
  \begin{cases}
    0,
    &\text{sample-only prompt model},\\
    J_{\theta,n}\mathfrak{C}_{d_x,\beta'}
    (B_{\ell,\mathrm{qry}},H_{\mathrm{qry},x},
    \varepsilon_{\mathrm{qry},n}),
    &\text{input--output prompt model},
  \end{cases}\\
  \mathsf{N}_{\mathrm{rat},n}^{\mathrm{lat}}
  &:=
  d_z\mathfrak{C}_{d_z+1,1}
  \left(
  1+\frac{R_{\mathrm{num},n}^{\mathrm{lat}}+1}{\underline{q}_{\mathrm{out}}},
  (1+\overline{q}_{\mathrm{out}})
  (2+R_{\mathrm{num},n}^{\mathrm{lat}})
  (1+\underline{q}_{\mathrm{out}}^{-2}),
  \varepsilon_{\mathrm{rat},n}
  \right).
\end{align*}
Use $\mathsf{N}_{\mathrm{basis},n}^{\mathrm{lat}}$ from \cref{eq:smooth-latent-time-local-basis-size}; it counts all local hats, multiplication paths, interpolation coefficients, and final coordinatewise projections.
In the sample-only prompt model, omit the query log-likelihood network.
With the counting conventions of \cref{prop:smooth-latent-complexity}, the bounds on the numbers of trainable parameters, stored coefficients, and arithmetic operations per forward pass are
\begin{align}
  \mathsf{N}_{\mathrm{param},n}^{\mathrm{lat}}
  &\lesssim
  (d_{\mathrm{emb},n}^{\mathrm{lat}})^2+J_{\theta,n}
  +\mathsf{N}_{\ell,n}^{\mathrm{lat}}
  +\mathsf{N}_{\mathrm{qry},n}^{\mathrm{lat}}
  +\mathsf{N}_{\mathrm{basis},n}^{\mathrm{lat}},
  \label{eq:smooth-latent-trainable-parameter-count}\\
  \mathsf{N}_{\mathrm{store},n}^{\mathrm{lat}}
  &\lesssim
  (d_{\mathrm{emb},n}^{\mathrm{lat}})^2
  +\mathsf{N}_{\ell,n}^{\mathrm{lat}}
  +\mathsf{N}_{\mathrm{qry},n}^{\mathrm{lat}}
  +\mathsf{N}_{\mathrm{basis},n}^{\mathrm{lat}}
  +\mathsf{N}_{\mathrm{rat},n}^{\mathrm{lat}},
  \label{eq:smooth-latent-coefficient-storage}\\
  \mathsf{N}_{\mathrm{op},n}^{\mathrm{lat}}
  &\lesssim
  n\mathsf{N}_{\ell,n}^{\mathrm{lat}}
  +\mathsf{N}_{\mathrm{qry},n}^{\mathrm{lat}}
  +\mathsf{N}_{\mathrm{basis},n}^{\mathrm{lat}}
  +(n+J_{\theta,n})(d_{\mathrm{emb},n}^{\mathrm{lat}})^2
  +\mathsf{N}_{\mathrm{rat},n}^{\mathrm{lat}}.
  \label{eq:smooth-latent-forward-cost}
\end{align}
The parameter count includes inactive coordinates used to embed every component network into a common compact parameter cube.
The storage and operation counts exclude structural zeros, masks, cell and role indices, and their bit encodings, but include every nonzero fixed or trainable scalar coefficient.

\begin{lemma}[Token bounds and parameter sensitivity of the four-attention-block network]
\label{lem:smooth-latent-four-block-sensitivity}
Let $\mathbf{v},\mathbf{v}'\in\gV_n^{\mathrm{fit}}$ collect trainable coordinates of the component networks, prior logits, and entries of the attention projection matrices.
All token-array bounds below hold uniformly over $(h,z,t)\in\mathcal{G}_n$, and their common input arguments are suppressed.
There are deterministic quantities $B_{0,n},\Lip_{0,n},B_{\mathrm{par},n}\ge1$ and fixed-map constants $B_{\mathrm{rt},k,n},\Lip_{\mathrm{rt},k,n}\ge1$, $1\le k\le4$, such that
\begin{align}
  \max_i\|\mathbf{E}_{n,i}^{(0)}(\mathbf{v})\|_{\infty}
  &\le B_{0,n},
  \label{eq:smooth-latent-initial-token-envelope}\\
  \max_i\|\mathbf{E}_{n,i}^{(0)}(\mathbf{v})-\mathbf{E}_{n,i}^{(0)}(\mathbf{v}')\|_{\infty}
  &\le \Lip_{0,n}\|\mathbf{v}-\mathbf{v}'\|_{\infty},
  \label{eq:smooth-latent-initial-token-sensitivity}
\end{align}
and every trainable or fixed coefficient used by an attention block has magnitude at most $B_{\mathrm{par},n}$.
Define recursively, for $1\le k\le4$,
\begin{align}
  \overline{B}_{k-1,n}
  &:=\max\{1,B_{\mathrm{par},n},B_{k-1,n}\},
  \nonumber\\
  B_{k,n}
  &:=B_{\mathrm{rt},k,n}\left\{1+B_{k-1,n}+(d_{\mathrm{emb},n}^{\mathrm{lat}})^2\overline{B}_{k-1,n}^3\right\},
  \label{eq:smooth-latent-recursion-envelope}\\
  \Lip_{k,n}
  &:=\Lip_{\mathrm{rt},k,n}\left[
  \Lip_{k-1,n}+16(d_{\mathrm{emb},n}^{\mathrm{lat}})^5\overline{B}_{k-1,n}^6(1+\Lip_{k-1,n})
  \right].
  \label{eq:smooth-latent-recursion-sensitivity}
\end{align}
Then, for every $0\le k\le4$,
\begin{align}
  \max_i\|\mathbf{E}_{n,i}^{(k)}(\mathbf{v})\|_{\infty}
  &\le B_{k,n},
  \label{eq:smooth-latent-token-envelope}\\
  \max_i\|\mathbf{E}_{n,i}^{(k)}(\mathbf{v})-\mathbf{E}_{n,i}^{(k)}(\mathbf{v}')\|_{\infty}
  &\le \Lip_{k,n}\|\mathbf{v}-\mathbf{v}'\|_{\infty}.
  \label{eq:smooth-latent-token-sensitivity}
\end{align}
The quantities $B_{0,n}$, $\Lip_{0,n}$, $B_{\mathrm{par},n}$, $B_{\mathrm{rt},k,n}$, and $\Lip_{\mathrm{rt},k,n}$ are polynomial in the displayed component-network bounds, $d_{\mathrm{emb},n}^{\mathrm{lat}}$, $J_{\theta,n}$, $B_{\mathrm{gate},n}^{\mathrm{lat}}$, and $\log(2J_{\theta,n}/\varepsilon_{\mathrm{floor},n})$.
If $\Lip_{\mathrm{out},n}$ is an input-Lipschitz constant of the fixed ratio readout followed by radial projection, then
\begin{equation}
  d_{\mathrm{bal},\infty}(s_{\mathbf{v}},s_{\mathbf{v}'})
  \le \Lip_{\mathrm{out},n}\Lip_{4,n}\|\mathbf{v}-\mathbf{v}'\|_{\infty},
  \label{eq:smooth-latent-four-block-final-sensitivity}
\end{equation}
and $\Lip_{\mathrm{out},n}$ is polynomial in the same inputs and $\varepsilon_{\mathrm{rat},n}^{-1}$.
\end{lemma}

\begin{proof}
The likelihood-network bounds in \cref{eq:fixed-dimensional-holder-parameter-stability}, the time-local basis bounds in Lemma~\ref{lem:smooth-latent-time-local-basis}, the projections of component outputs, and \cref{eq:smooth-latent-prior-logit-radius} give \cref{eq:smooth-latent-initial-token-envelope,eq:smooth-latent-initial-token-sensitivity} and a polynomial choice of $B_{\mathrm{par},n}$.
After each attention update, let $\mathsf{R}_{k,n}$ denote the fixed affine--ReLU coordinate-update map, with the selector from \cref{eq:appendix-smooth-latent-binary-gate} included after the third update.
The four coordinate-update networks have a universal depth bound $D_{\mathrm{rt}}$, width at most $C_{d_z}(d_{\mathrm{emb},n}^{\mathrm{lat}})^2$, and coefficient magnitude at most
\begin{equation*}
  B_{\mathrm{rt},n}
  :=1\vee n\vee B_{\mathrm{gate},n}^{\mathrm{lat}}
  \vee\sqrt{d_{\mathrm{emb},n}^{\mathrm{lat}}}.
\end{equation*}
An affine layer of width at most $w$ and coefficient bound $B$ satisfies
\begin{equation*}
  1+\|Wx+b\|_{\infty}\le2w(B\vee1)(1+\|x\|_{\infty}),
  \qquad
  \|W(x-x')\|_{\infty}\le w(B\vee1)\|x-x'\|_{\infty}.
\end{equation*}
ReLU is one-Lipschitz and does not increase the bound on token magnitudes, so iteration over the $D_{\mathrm{rt}}$ layers gives constants $B_{\mathrm{rt},k,n},\Lip_{\mathrm{rt},k,n}\ge1$ satisfying
\begin{align*}
  \|\mathsf{R}_{k,n}(x)\|_{\infty}
  &\le B_{\mathrm{rt},k,n}(1+\|x\|_{\infty}),\\
  \|\mathsf{R}_{k,n}(x)-\mathsf{R}_{k,n}(x')\|_{\infty}
  &\le \Lip_{\mathrm{rt},k,n}\|x-x'\|_{\infty}.
\end{align*}
More precisely, one may take both constants no larger than
\begin{equation*}
  \{2C_{d_z}(d_{\mathrm{emb},n}^{\mathrm{lat}})^2B_{\mathrm{rt},n}\}^{D_{\mathrm{rt}}}.
\end{equation*}
This proves the asserted polynomial dependence for the coordinate-update networks.
If the tokens and attention matrices entering block $k$ are bounded by $\overline{B}_{k-1,n}$, then its attention output is bounded coordinatewise by $(d_{\mathrm{emb},n}^{\mathrm{lat}})^2\overline{B}_{k-1,n}^3$.
Adding the residual and applying $\mathsf{R}_{k,n}$ proves \cref{eq:smooth-latent-recursion-envelope,eq:smooth-latent-token-envelope} by induction.
At block $k$, the maximum perturbation among its tokens and trainable matrices is at most $(1+\Lip_{k-1,n})\|\mathbf{v}-\mathbf{v}'\|_{\infty}$.
The attention perturbation bound in \cref{lem:attention-parameter-sensitivity}, followed by the residual update and the $\Lip_{\mathrm{rt},k,n}$-Lipschitz coordinate-update map, gives \cref{eq:smooth-latent-recursion-sensitivity,eq:smooth-latent-token-sensitivity}.
The preliminary readout projection maps all inputs to the compact approximation box specified above and is nonexpansive.
Thus \cref{eq:fixed-dimensional-holder-parameter-stability} gives a global polynomial input-Lipschitz constant for the fixed ratio readout; radial output projection is also nonexpansive.
The state-clipping identity \cref{eq:global-state-clipping-reduction} extends the token bounds from $\mathcal{G}_n$ to the full state domain.
Absorbing the fixed conversion between $\ell_{\infty}$ and Euclidean norms into $\Lip_{\mathrm{out},n}$ proves \cref{eq:smooth-latent-four-block-final-sensitivity}.
\end{proof}

\begin{proof}[Proof of \cref{prop:smooth-latent-complexity}]
Work under \cref{assump:smooth-latent-task}, with the constants given in \cref{lem:positive-smooth-latent-derived-regularity}.
We first count the coefficients and operations in the attention blocks and component networks.
We then use parameter sensitivity to obtain a covering bound and a measurable ERM, while tracking the dependence on the declared grid and component tolerances.

\textbf{Model size and forward-pass cost.}
Exclude the component encoders and the fixed ratio readout, but include every nonzero fixed or trainable scalar coefficient in the four attention blocks and their coordinate-update maps.
The four-attention-block network above can be stored using at most
\begin{equation}
  16(d_{\mathrm{emb},n}^{\mathrm{lat}})^2+C_{d_z}d_{\mathrm{emb},n}^{\mathrm{lat}}
  \le C_{d_z}(d_{\mathrm{emb},n}^{\mathrm{lat}})^2
  \label{eq:smooth-latent-update-storage-bound}
\end{equation}
such coefficients and evaluated using at most
\begin{equation}
  C_{d_z}(n+J_{\theta,n}+4)(d_{\mathrm{emb},n}^{\mathrm{lat}})^2
  \label{eq:smooth-latent-update-operation-bound}
\end{equation}
sparse arithmetic operations.
Structural zeros, masks, role labels, latent-cell indices, and their bit encodings are excluded under the counting convention of \cref{prop:smooth-latent-complexity}.
Every fixed coordinate-update coefficient has magnitude at most
\begin{equation}
  1\vee n\vee B_{\mathrm{gate},n}^{\mathrm{lat}}\vee\sqrt{d_{\mathrm{emb},n}^{\mathrm{lat}}}.
  \label{eq:smooth-latent-update-coefficient-bound}
\end{equation}

Each block has four $d_{\mathrm{emb},n}^{\mathrm{lat}}\times d_{\mathrm{emb},n}^{\mathrm{lat}}$ attention projections, so all four blocks contain at most $16(d_{\mathrm{emb},n}^{\mathrm{lat}})^2$ nonzero projection coefficients.
The role and block injections, coordinate copies, residual writes, scalar projections, and posterior-logit affine sum use at most $C_{d_z}d_{\mathrm{emb},n}^{\mathrm{lat}}$ shared nonzero coefficients.
For the selector, one copy of \cref{eq:appendix-smooth-latent-binary-gate} is used for each of the $J_{\theta,n}$ candidate entries in each of the $d_z+2$ selected likelihood-and-basis coordinates.
It therefore uses at most $C_{d_z}J_{\theta,n}\le C_{d_z}d_{\mathrm{emb},n}^{\mathrm{lat}}$ shared nonzero coefficients.
This proves \cref{eq:smooth-latent-update-storage-bound}.

For a forward evaluation, the number of operations needed to apply the attention projections in all four blocks is $\lesssim(n+J_{\theta,n}+4)(d_{\mathrm{emb},n}^{\mathrm{lat}})^2$.
The masks in \cref{tab:smooth-latent-four-block-construction} contain $n+3J_{\theta,n}$ non-null query--key pairs, whose dot products and value aggregation require $\lesssim(n+J_{\theta,n})d_{\mathrm{emb},n}^{\mathrm{lat}}$ further operations.
The shared sparse coordinate-update maps require at most $C_{d_z}d_{\mathrm{emb},n}^{\mathrm{lat}}$ operations per token row, including the $J_{\theta,n}(d_z+2)$ selector gates on a latent row.
Residual additions and fixed coordinate injections cost $\lesssim(n+J_{\theta,n}+4)d_{\mathrm{emb},n}^{\mathrm{lat}}$ operations.
Since $d_{\mathrm{emb},n}^{\mathrm{lat}}\ge1$, these bounds imply \cref{eq:smooth-latent-update-operation-bound}.
Among the fixed coordinate-update coefficients, only the factor $n$ in the posterior logit and the projection and gate thresholds can exceed one.
The trainable attention coefficients are bounded by $1\vee\sqrt{d_{\mathrm{emb},n}^{\mathrm{lat}}}$ by the parameter-space definition.
This proves \cref{eq:smooth-latent-update-coefficient-bound}.

We now add the component networks to these attention-block counts.
The trainable output-layer coefficients of the context-example log-likelihood, query log-likelihood, and OU-basis networks are bounded, up to a fixed dimension-dependent factor, by
\begin{equation*}
  \mathsf{N}_{\ell,n}^{\mathrm{lat}}
  +\mathsf{N}_{\mathrm{qry},n}^{\mathrm{lat}}
  +\mathsf{N}_{\mathrm{basis},n}^{\mathrm{lat}}.
\end{equation*}
The fixed nonzero internal coefficients of these three network families obey the same bound.
The prior contributes $J_{\theta,n}$ trainable logits.
By \cref{eq:smooth-latent-update-storage-bound}, the fixed selector and coordinate-update networks and the trainable attention matrices contribute $\lesssim(d_{\mathrm{emb},n}^{\mathrm{lat}})^2$ nonzero stored coordinates.
The fixed ratio readout contributes at most $\mathsf{N}_{\mathrm{rat},n}^{\mathrm{lat}}$ nonzero stored coordinates.
This proves \cref{eq:smooth-latent-trainable-parameter-count,eq:smooth-latent-coefficient-storage}.
Evaluating the shared context-example network $n$ times and the query networks and ratio readout once, and then applying \cref{eq:smooth-latent-update-operation-bound}, proves \cref{eq:smooth-latent-forward-cost}.
In particular, only the summary row attends to all $n$ context-example tokens, the two coordinate-copy stages use singleton-key rows, and the output row attends to the $J_{\theta,n}$ latent tokens; no dense quadratic all-token attention matrix is evaluated.

\textbf{Covering numbers and measurable empirical minimization.}
To pass from the parameter count to a covering number, we bound the parameter radius and the sensitivity of the score to its parameters.
The final posterior-attention logit uses the standard $(d_{\mathrm{emb},n}^{\mathrm{lat}})^{-1/2}$ scaled dot product, so one fixed query or key coordinate is multiplied by $\sqrt{d_{\mathrm{emb},n}^{\mathrm{lat}}}$; this factor is included in the fixed attention and coordinate-update radius.
By construction, the radii of the parameter-space factors in \cref{def:smooth-latent-score-class} are the component-network parameter radii, the prior-logit radius in \cref{eq:smooth-latent-prior-logit-radius}, and the attention-projection radius containing the reference matrices.
Take $B_{\mathrm{lat},n}$ to be the largest of these radii.
This choice is polynomial in the displayed inputs and $\log(\varepsilon_{\mathrm{floor},n}^{-1})$ by \cref{eq:fixed-dimensional-holder-parameter-stability} and Lemma~\ref{lem:smooth-latent-time-local-basis}.
Because $\proj_n^{\mathrm{par}}$ is nonexpansive in $\ell_{\infty}$, applying \cref{lem:smooth-latent-four-block-sensitivity} after this projection and setting $\Lip_{\mathrm{lat},n}^{\mathrm{par}}:=\Lip_{\mathrm{out},n}\Lip_{4,n}$ proves the required bound.
This proves \cref{eq:smooth-latent-parameter-sensitivity} and makes the claimed polynomial dependence explicit.
A grid contained in the compact trainable-parameter cube proves \cref{eq:smooth-latent-score-cover}.
The resulting finite network is jointly Borel in the data and continuous in its parameters.
% The measurable maximum theorem therefore guarantees the existence of a Borel-measurable exact empirical-risk minimizer.
Since the parameter cube is a compact metric space and the empirical risk is Borel measurable in the training sample and continuous in the parameter (hence Carath\'eodory), the measurable maximum theorem \citep[Theorem~18.19]{aliprantisborder2006}, applied to the negative empirical risk, yields a Borel-measurable parameter selector whose induced score is an exact empirical-risk minimizer; see also \citet{brownpurves1973}.

\end{proof}

\clearpage

\section{Pretraining and Score Learning from Independent Episodes}
\label{app:sec:score-learning}

This section turns the approximation class from \cref{app:sec:transformer-approximation} into a finite-episode learning guarantee.
We first define clipped variance-balanced DSM and its population target, then establish uniform convergence in the scaled-score metric, and finally combine approximation, clipping, optimization, covering, and estimation errors for the Transformer class.

\subsection{Clipped Labels and the Empirical bDSM Objective}
\label{app:subsec:clipped-dsm-empirical-objective}

The training sample consists of $N$ independent episodes, each augmented with one independent diffusion time and Gaussian perturbation.
Bounding the scaled-score output and clipping the Gaussian label yield a bounded squared loss.
The resulting theorem separates approximation, clipping, optimization, covering, and estimation errors.

\subsubsection{Population regression function and clipping bias}
\label{app:subsubsec:regression-function-and-clipping}

For a DSM-label clipping radius $R_{\mathrm{lab}} > 0$, let $\clip_{R_{\mathrm{lab}}}$ be the radial clipping map in \cref{eq:radial-clipping}.
Define the clipped DSM label and a measurable version of its regression function by
\begin{align}
  \xi_{R_{\mathrm{lab}}}^{\mathrm{clip}}
  & := \clip_{R_{\mathrm{lab}}}(-\xi_{\mathrm{dsm}}),
  \label{eq:clipped-dsm-label}\\
  \widetilde{s}_{R_{\mathrm{lab}}}^{\mathrm{lab}}(z, t, h)
  & := \E\!\left[\xi_{R_{\mathrm{lab}}}^{\mathrm{clip}}
  \,|\, Z_{\tau} = z,
  \tau = t, H_n = h\right].
  \label{eq:clipped-dsm-regression-function}
\end{align}
For a score $s$, define the clipped population risk and the label-clipping bias by
\begin{align}
  \gR_{\mathrm{bDSM}}^{(R_{\mathrm{lab}})}(s)
  & := A_{t_0, T}\E\!\left[
  \|\widetilde{s}(Z_{\tau}, \tau, H_n)
  -\xi_{R_{\mathrm{lab}}}^{\mathrm{clip}}\|_2^2\right],
  \label{eq:population-clipped-balanced-risk}\\
  \varepsilon_{\mathrm{lab,clip}}^2(R_{\mathrm{lab}})
  & := A_{t_0, T}\E\!\left[
  \|\widetilde{s}_{R_{\mathrm{lab}}}^{\mathrm{lab}}
  (Z_{\tau}, \tau, H_n)
  -\widetilde{s}^{\star}(Z_{\tau}, \tau\,|\, H_n)\|_2^2\right].
  \label{eq:balanced-clipping-bias-definition}
\end{align}
The Pythagorean identity for squared loss and the squared triangle inequality give
\begin{align}
  &\gR_{\mathrm{bDSM}}^{(R_{\mathrm{lab}})}(s)
  -A_{t_0, T}\E\!\left[
  \|\widetilde{s}_{R_{\mathrm{lab}}}^{\mathrm{lab}}-
  \xi_{R_{\mathrm{lab}}}^{\mathrm{clip}}\|_2^2\right]
  = A_{t_0, T}\E\!\left[
  \|\widetilde{s}-\widetilde{s}_{R_{\mathrm{lab}}}^{\mathrm{lab}}\|_2^2\right],
  \nonumber\\
  &\gE_{\mathrm{score}}(s)
  \le 2A_{t_0, T}\E\!\left[
  \|\widetilde{s}-\widetilde{s}_{R_{\mathrm{lab}}}^{\mathrm{lab}}\|_2^2\right]
  +2\varepsilon_{\mathrm{lab,clip}}^2(R_{\mathrm{lab}}).
  \label{eq:clipped-to-exact-score}
\end{align}
Here, $\widetilde{s}$ and $\widetilde{s}_{R_{\mathrm{lab}}}^{\mathrm{lab}}$ denote their evaluations at the random training input $(Z_{\tau}, \tau, H_n)$.
The term $\varepsilon_{\mathrm{lab,clip}}^2(R_{\mathrm{lab}})$ measures the bias introduced by clipping the DSM label.
This label clipping, together with the uniform bound imposed on the scaled-score predictor in \cref{eq:score-class-envelope}, keeps the empirical squared loss uniformly bounded.
This uniform loss bound uses the final output radius, rather than the potentially larger internal feature radius.
It is separate from state clipping and from the score or scaled-score clipping built into the score network.

\begin{proposition}[Bias from clipping Gaussian DSM labels]
\label{prop:balanced-clipping-bias}
For every $R_{\mathrm{lab}} > 0$,
\begin{equation}
  \varepsilon_{\mathrm{lab,clip}}^2(R_{\mathrm{lab}})
  \le A_{t_0, T}\E(\|\xi\|_2-R_{\mathrm{lab}})_{+}^2,
  \qquad \xi \sim \gN(0, I_{d_z}).
  \label{eq:balanced-clipping-bias-gaussian}
\end{equation}
Consequently, for every $x \ge 0$, taking $R_{\mathrm{lab}} \ge \sqrt{d_z}+x$ gives
\begin{equation}
  \varepsilon_{\mathrm{lab,clip}}^2(R_{\mathrm{lab}})
  \le 2A_{t_0, T}e^{-x^2/2}.
  \label{eq:balanced-clipping-bias}
\end{equation}
\end{proposition}

\begin{proof}[Proof of \cref{prop:balanced-clipping-bias}]
Let $\xi_{R_{\mathrm{lab}}}^{\mathrm{clip}}:=\clip_{R_{\mathrm{lab}}}(-\xi_{\mathrm{dsm}})$.
By \cref{prop:balanced-dsm-identity}, $\widetilde{s}^{\star}(Z_{\tau},\tau\,|\, H_n)=\E[-\xi_{\mathrm{dsm}}\,|\, Z_{\tau},\tau,H_n]$, whereas $\widetilde{s}_{R_{\mathrm{lab}}}^{\mathrm{lab}}(Z_{\tau},\tau,H_n)=\E[\xi_{R_{\mathrm{lab}}}^{\mathrm{clip}}\,|\, Z_{\tau},\tau,H_n]$.
Conditional Jensen's inequality therefore gives
\begin{align*}
  \left\|\widetilde{s}_{R_{\mathrm{lab}}}^{\mathrm{lab}}
  (Z_{\tau},\tau,H_n)-\widetilde{s}^{\star}(Z_{\tau},\tau\,|\, H_n)
  \right\|_2^2
  &\le
  \E\!\left[
  \|\xi_{R_{\mathrm{lab}}}^{\mathrm{clip}}
  +\xi_{\mathrm{dsm}}\|_2^2 \,|\, Z_{\tau},\tau,H_n\right].
\end{align*}
For radial clipping, $\|\clip_R(u)-u\|_2 = (\|u\|_2-R)_{+}$.
Hence
\begin{equation*}
  \varepsilon_{\mathrm{lab,clip}}^2(R_{\mathrm{lab}})
  \le A_{t_0, T}\E(\|\xi_{\mathrm{dsm}}\|_2-R_{\mathrm{lab}})_{+}^2,
\end{equation*}
which proves \cref{eq:balanced-clipping-bias-gaussian}.
The Gaussian tail moment bound \cref{eq:gaussian-radial-tail-moment}, with $R=\sqrt{d_z}+x$, proves \cref{eq:balanced-clipping-bias}.
\end{proof}

\subsubsection{Empirical risk minimization over independent pretraining episodes}
\label{app:subsubsec:independent-episodes-erm}

For each independent training episode $e$, draw $\tau_e \sim \varpi_{t_0, T}$ and $\xi_{\mathrm{dsm}}^{(e)} \sim \gN(0, I_{d_z})$, independently of the episode and of all other draws, and set
\begin{equation*}
  Z_{\tau_e}^{(e)}
  := \alpha_{\tau_e}Z_0^{(e)}
  +\sigma_{\tau_e}\xi_{\mathrm{dsm}}^{(e)},
  \qquad
  \xi_{R_{\mathrm{lab}}}^{\mathrm{clip},(e)}
  := \clip_{R_{\mathrm{lab}}}(-\xi_{\mathrm{dsm}}^{(e)}).
\end{equation*}
The augmented training examples are independent copies of $(H_n,Z_{\tau},\tau,\xi_{R_{\mathrm{lab}}}^{\mathrm{clip}})$:
\begin{equation}
  (H_n^{(e)}, Z_{\tau_e}^{(e)}, \tau_e,
  \xi_{R_{\mathrm{lab}}}^{\mathrm{clip},(e)})
  \stackrel{\iid}{\sim}\Law(H_n,Z_{\tau},\tau,\xi_{R_{\mathrm{lab}}}^{\mathrm{clip}}),
  \qquad e = 1, \ldots, N.
  \label{eq:dsm-augmented-iid}
\end{equation}
This augmentation transforms the episodic training sample $\gD_N^{\mathrm{epi}}$ into
\begin{equation}
  \gD_N^{\mathrm{dsm}}
  := \bigl((H_n^{(e)}, Z_{\tau_e}^{(e)}, \tau_e,
  \xi_{R_{\mathrm{lab}}}^{\mathrm{clip},(e)})\bigr)_{e = 1}^N.
  \label{eq:dsm-training-sample}
\end{equation}
Here and below, we suppress the dependence on the fixed episode distribution and time window $(t_0,T)$ in the notation for the augmented distribution, sample, and clipping-bias functional; we write the label radius explicitly when it is relevant.
Let the score hypothesis class
\begin{equation*}
  \gS \subseteq \{s:\R^{d_z} \times [t_0, T] \times \gH_n \to \R^{d_z}:
  s\text{ is jointly Borel}\}
\end{equation*}
be deterministic and fixed before $\gD_N^{\mathrm{dsm}}$ is observed.
The empirical variance-balanced DSM risk is
\begin{equation}
  \widehat{\gR}_{N, \mathrm{bDSM}}^{(R_{\mathrm{lab}})}(s)
  := \frac{A_{t_0, T}}{N}\sum_{e = 1}^N
  \|\widetilde{s}(Z_{\tau_e}^{(e)}, \tau_e, H_n^{(e)})
  -\xi_{R_{\mathrm{lab}}}^{\mathrm{clip},(e)}\|_2^2.
  \label{eq:empirical-balanced-dsm-risk}
\end{equation}
For each fixed $s$, this empirical risk is unbiased for $\gR_{\mathrm{bDSM}}^{(R_{\mathrm{lab}})}(s)$.
A jointly measurable approximate ERM is a data-dependent $\widehat{s}_N \in \gS$ satisfying
\begin{equation}
  \widehat{\gR}_{N, \mathrm{bDSM}}^{(R_{\mathrm{lab}})}(\widehat{s}_N)
  \le \inf_{s \in \gS}
  \widehat{\gR}_{N, \mathrm{bDSM}}^{(R_{\mathrm{lab}})}(s)
  +\varepsilon_{\mathrm{opt}},
  \qquad \varepsilon_{\mathrm{opt}} \ge 0.
  \label{eq:approximate-erm}
\end{equation}
Here $\varepsilon_{\mathrm{opt}}$ is the optimization error in the scaled empirical objective above.
Equivalently, the mean-squared-loss gap without the factor $A_{t_0,T}$ is at most $\varepsilon_{\mathrm{opt}}/A_{t_0,T}$.
An exact ERM corresponds to $\varepsilon_{\mathrm{opt}}=0$; the Transformer score theorem also allows a specified nonzero optimization gap.
Joint measurability means that evaluation of $\widehat{s}_N$ is Borel jointly in the DSM training sample and $(z,t,h)$.
If training uses auxiliary randomization, impose the same requirement jointly in that randomization; because the proof event is uniform over $s\in\gS$, the probability statement below may equivalently include the training randomization.
Training evaluates the learned score $\widehat{s}_N$ through its scaled representation $\sigma_t\widehat{s}_N$.
For this data-dependent score, $\gE_{\mathrm{score}}(\widehat{s}_N)$ is evaluated on a fresh episode, conditionally on $\gD_N^{\mathrm{dsm}}$.

\subsection{Scaled-Score Covering Metric}
\label{app:subsec:scaled-score-covering}

Variance-balanced DSM evaluates the rescaled prediction $\widetilde{s}=\sigma_t s$, so it suffices to establish uniform convergence in a metric on these rescaled predictions.
Recall the scaled class $\widetilde{\gS}$ and the supremum metric on scaled scores $d_{\mathrm{bal},\infty}$ from \cref{eq:scaled-score-class}.
The two classes describe the same predictors: because $\sigma_t>0$ on $[t_0,T]$, the map $s\mapsto\widetilde{s}$ is a bijection, with inverse $\widetilde{s}\mapsto\widetilde{s}/\sigma_t$.
Since $\sigma_{t_0}\le\sigma_t\le1$,
\begin{equation*}
  d_{\mathrm{bal},\infty}(s,s')
  \le \|s-s'\|_{\infty,2}
  \le \sigma_{t_0}^{-1}d_{\mathrm{bal},\infty}(s,s'),
  \qquad s,s'\in\gS,
\end{equation*}
and the best-in-class score error is unchanged by this representation:
\begin{equation*}
  \inf_{\widetilde{s}\in\widetilde{\gS}}
  \gE_{\mathrm{score}}(\widetilde{s}/\sigma_t)
  = \gE_{\mathrm{app}}(\gS).
\end{equation*}
Moreover, $s\mapsto\widetilde{s}$ is an isometry from $(\gS,d_{\mathrm{bal},\infty})$ onto $(\widetilde{\gS},\|\cdot\|_{\infty,2})$.
Hence, for every $\varepsilon\ge0$,
\begin{equation*}
  \Ncov(\varepsilon,\gS,d_{\mathrm{bal},\infty})
  =\Ncov(\varepsilon,\widetilde{\gS},\|\cdot\|_{\infty,2})
  \le \Ncov(\varepsilon, \gS, \|\cdot\|_{\infty,2}).
\end{equation*}
The proof of \cref{thm:balanced-score-oracle} uses a finite $\varepsilon$-cover in $d_{\mathrm{bal},\infty}$ only to establish uniform convergence.
The ERM itself remains defined over the entire score class $\gS$; the cover is used only in the proof and does not define a discretized hypothesis class or an implementation requirement.
This supremum metric on scaled scores also avoids the factor $\sigma_{t_0}^{-1}$ that can arise when distances are instead measured on the unscaled score functions.

\subsection{Finite-Episode Score-Learning Theorem}
\label{app:subsec:finite-sample-score-learning}

The next result converts approximate empirical minimization into a conditional score-error guarantee on a fresh episode.
It applies to any deterministic score class whose scaled outputs are uniformly bounded and admit a finite cover in $d_{\mathrm{bal},\infty}$; the Transformer class is substituted later in \cref{app:subsec:posterior-predictive-generation-proof}.

\begin{theorem}[Finite-episode score learning with variance-balanced DSM]
\label{thm:balanced-score-oracle}
Fix $N \ge 1$, $R_{\mathrm{lab}} > 0$, $0 \le R_{\mathrm{pred}}<\infty$, and $0\le\varepsilon_{\mathrm{cov}}\le R_{\mathrm{pred}}+R_{\mathrm{lab}}$.
Let $\gD_N^{\mathrm{dsm}}$ be the i.i.d.\ augmented training sample in \cref{eq:dsm-augmented-iid,eq:dsm-training-sample}, and let $\gS$ be a nonempty deterministic class of jointly Borel scores satisfying
\begin{align}
  \sup_{s\in\gS}\sup_{z,t,h}\|\sigma_t s(z,t,h)\|_2
  &\le R_{\mathrm{pred}},
  \nonumber\\
  \Ncov(\varepsilon_{\mathrm{cov}},\gS,
  d_{\mathrm{bal},\infty})&<\infty.
  \label{eq:score-class-envelope}
\end{align}
Let $\widehat{s}_N \in \gS$ be a jointly measurable approximate ERM satisfying \cref{eq:approximate-erm} with deterministic optimization error $\varepsilon_{\mathrm{opt}}$.
For every $\delta \in (0, 1)$, with probability at least $1-\delta$ over $\gD_N^{\mathrm{dsm}}$,
\begin{align}
  \gE_{\mathrm{score}}(\widehat{s}_N)
  &\lesssim
  \gE_{\mathrm{app}}(\gS)
  +\varepsilon_{\mathrm{lab,clip}}^2(R_{\mathrm{lab}})
  +\varepsilon_{\mathrm{opt}}
  \nonumber\\
  &\quad+A_{t_0, T}(R_{\mathrm{pred}}+R_{\mathrm{lab}})\varepsilon_{\mathrm{cov}}
  \nonumber\\
  &\quad+A_{t_0, T}(R_{\mathrm{pred}}+R_{\mathrm{lab}})^2
  \frac{\log\!\bigl(2\Ncov(\varepsilon_{\mathrm{cov}},
  \gS,d_{\mathrm{bal},\infty})/\delta\bigr)}{N}.
  \label{eq:balanced-score-oracle}
\end{align}
\end{theorem}

The probability event concerns the $N$ independent augmented training episodes, while $\gE_{\mathrm{score}}(\widehat{s}_N)$ evaluates the fitted score on a fresh episode conditional on that training sample.
The inequality compares the fitted score with the best approximation in $\gS$ and shows separately how the error depends on label clipping, empirical optimization, the finite-cover approximation used to extend the bound to the entire class, and estimation from finitely many episodes.

% Multiplication or division by the deterministic continuous function $t\mapsto\sigma_t$ preserves joint Borel evaluation and pointwise continuity in every finite-dimensional parameter.
% Thus, when the parameter space is a compact metric space, and the sample space is standard Borel, the balanced empirical risk is a Carath\'eodory objective and the measurable maximum theorem provides a measurable ERM selection.
For the Transformer class, the measurable-selection argument in \cref{prop:smooth-latent-complexity} supplies a measurable exact ERM.
The theorem itself permits any jointly measurable approximate ERM satisfying \cref{eq:approximate-erm}; this assumption concerns the estimator's measurability and does not provide an algorithm for finding it.
This establishes the existence of a measurable minimizer, not an algorithm for finding it.
The optimization gap in \cref{eq:approximate-erm} accounts for approximate empirical minimization.

\begin{proof}[Proof of \cref{thm:balanced-score-oracle}]
We first show that the excess loss has variance controlled by its expectation.
This gives a Bernstein bound on a finite score cover, which we extend to the entire score class before applying the approximate-ERM inequality and converting the clipped regression error to score error.

The clipped regression function satisfies
\begin{equation*}
  \widetilde{s}_{R_{\mathrm{lab}}}^{\mathrm{lab}}(Z_{\tau}, \tau, H_n)
  = \E[\xi_{R_{\mathrm{lab}}}^{\mathrm{clip}}\,|\, Z_{\tau}, \tau, H_n],
  \qquad
  \|\widetilde{s}_{R_{\mathrm{lab}}}^{\mathrm{lab}}(Z_{\tau}, \tau, H_n)\|_2
  \le R_{\mathrm{lab}}
\end{equation*}
almost surely.
For $s \in \gS$, define its clipped excess loss by
\begin{align*}
  L_s^{\mathrm{exc}}
  &:= \|\widetilde{s}(Z_{\tau}, \tau, H_n)-\xi_{R_{\mathrm{lab}}}^{\mathrm{clip}}\|_2^2
  -\|\widetilde{s}_{R_{\mathrm{lab}}}^{\mathrm{lab}}(Z_{\tau}, \tau, H_n)-\xi_{R_{\mathrm{lab}}}^{\mathrm{clip}}\|_2^2,
  \\
  \E[L_s^{\mathrm{exc}}]
  &= \E\|\widetilde{s}(Z_{\tau}, \tau, H_n)
  -\widetilde{s}_{R_{\mathrm{lab}}}^{\mathrm{lab}}(Z_{\tau}, \tau, H_n)\|_2^2.
\end{align*}
The second identity is the regression Pythagorean identity, and all expectations below are over one episode augmented for variance-balanced DSM.
For brevity, put
\begin{equation*}
  R_{\mathrm{loss}} := R_{\mathrm{pred}}+R_{\mathrm{lab}},
  \qquad
  \Gamma_{\mathrm{cov}}(\delta)
  := \log\!\bigl(2\Ncov(\varepsilon_{\mathrm{cov}},\gS,d_{\mathrm{bal},\infty})
  /\delta\bigr).
\end{equation*}
The difference of squared losses factors as
\begin{equation*}
  L_s^{\mathrm{exc}}
  = \left\langle
  \widetilde{s}(Z_{\tau}, \tau, H_n)
  -\widetilde{s}_{R_{\mathrm{lab}}}^{\mathrm{lab}}(Z_{\tau}, \tau, H_n),
  \widetilde{s}(Z_{\tau}, \tau, H_n)
  +\widetilde{s}_{R_{\mathrm{lab}}}^{\mathrm{lab}}(Z_{\tau}, \tau, H_n)
  -2\xi_{R_{\mathrm{lab}}}^{\mathrm{clip}}
  \right\rangle.
\end{equation*}
The first factor is the prediction error relative to the clipped regression function; the bounds on the scaled scores and clipped labels control the second factor.
Together with the regression identity above, this gives
\begin{equation*}
  |L_s^{\mathrm{exc}}| \le 3R_{\mathrm{loss}}^2,
  \qquad
  \E[(L_s^{\mathrm{exc}})^2]
  \le 9R_{\mathrm{loss}}^2\E[L_s^{\mathrm{exc}}].
\end{equation*}
In particular, $|L_s^{\mathrm{exc}}-\E[L_s^{\mathrm{exc}}]| \le 6R_{\mathrm{loss}}^2$ and $\Var(L_s^{\mathrm{exc}}) \le 9R_{\mathrm{loss}}^2\E[L_s^{\mathrm{exc}}]$.
Thus the excess losses obey the variance--expectation condition underlying standard fast-rate regression bounds \citep{bartlettbousquetmendelson2005,koltchinskii2011oracle}.
Here a finite-cover argument suffices; we do not invoke a theorem based on local Rademacher complexity for the nonconvex Transformer class.
Under the convention in \cref{app:sec:transformer-approximation} that cover centers belong to the score class, let $\gS_{\mathrm{cov}}(\varepsilon_{\mathrm{cov}})\subseteq\gS$ be a minimum-cardinality $\varepsilon_{\mathrm{cov}}$-net in $d_{\mathrm{bal},\infty}$.
This net is used only in the proof; empirical minimization remains over the entire score class $\gS$.
Let $\widehat{\E}_N$ denote the empirical average over the $N$ episodes augmented for variance-balanced DSM.
Bernstein's inequality \citep{boucheron2013concentration} and a union bound imply that, with probability at least $1-\delta$, simultaneously for every $u \in \gS_{\mathrm{cov}}(\varepsilon_{\mathrm{cov}})$,
\begin{equation}
  |\widehat{\E}_N[L_u^{\mathrm{exc}}]-\E[L_u^{\mathrm{exc}}]|
  \le
  \sqrt{\frac{18R_{\mathrm{loss}}^2\E[L_u^{\mathrm{exc}}]\Gamma_{\mathrm{cov}}(\delta)}{N}}
  +\frac{4R_{\mathrm{loss}}^2\Gamma_{\mathrm{cov}}(\delta)}{N}.
  \label{eq:appendix-bernstein-two-sided}
\end{equation}
For every $s \in \gS$, choose $s^{\sharp} \in \gS_{\mathrm{cov}}(\varepsilon_{\mathrm{cov}})$ satisfying
\begin{equation*}
  \|\sigma_t s-\sigma_t s^{\sharp}\|_{\infty,2}
  \le \varepsilon_{\mathrm{cov}}.
\end{equation*}
The difference of squares and the bounds on the scaled score maps and clipped labels give
\begin{equation}
  |L_s^{\mathrm{exc}}-L_{s^{\sharp}}^{\mathrm{exc}}|
  \le 2R_{\mathrm{loss}}\varepsilon_{\mathrm{cov}},
  \qquad
  \E[L_{s^{\sharp}}^{\mathrm{exc}}]
  \le 2\E[L_s^{\mathrm{exc}}]+2\varepsilon_{\mathrm{cov}}^2.
  \label{eq:appendix-cover-transfer}
\end{equation}
The first inequality controls both its empirical and population averages.
Combining \cref{eq:appendix-bernstein-two-sided,eq:appendix-cover-transfer}, using $\varepsilon_{\mathrm{cov}} \le R_{\mathrm{loss}}$, and applying $ab \le (a^2+b^2)/2$ to the resulting cross term yields, on the same event,
\begin{equation}
  |\widehat{\E}_N[L_s^{\mathrm{exc}}]-\E[L_s^{\mathrm{exc}}]|
  \lesssim
  \sqrt{\frac{R_{\mathrm{loss}}^2\E[L_s^{\mathrm{exc}}]\Gamma_{\mathrm{cov}}(\delta)}{N}}
  +\frac{R_{\mathrm{loss}}^2\Gamma_{\mathrm{cov}}(\delta)}{N}
  +R_{\mathrm{loss}}\varepsilon_{\mathrm{cov}},
  \qquad s \in \gS.
  \label{eq:appendix-full-class-bernstein}
\end{equation}

For any deterministic $s\in\gS$, subtract the empirical loss of the clipped regression function from \cref{eq:approximate-erm}.
Apply \cref{eq:appendix-full-class-bernstein} to both $\widehat{s}_N$ and $s$ on its simultaneous event.
Young's inequality absorbs the square-root term involving $\E[L_{\widehat{s}_N}^{\mathrm{exc}}]$ into the left side and bounds the comparator's square-root term by its risk plus the complexity term.
Multiplying by $A_{t_0,T}$ yields
\begin{equation*}
  A_{t_0, T}\E[L_{\widehat{s}_N}^{\mathrm{exc}}]
  \lesssim
  A_{t_0, T}\E[L_s^{\mathrm{exc}}]
  +A_{t_0, T}R_{\mathrm{loss}}^2\frac{\Gamma_{\mathrm{cov}}(\delta)}{N}
  +A_{t_0, T}R_{\mathrm{loss}}\varepsilon_{\mathrm{cov}}
  +\varepsilon_{\mathrm{opt}}.
\end{equation*}
For every $s \in \gS$, the squared triangle inequality and the inverse-variance time distribution give
\begin{equation*}
  A_{t_0, T}\E[L_s^{\mathrm{exc}}]
  \le 2\gE_{\mathrm{score}}(s)
  +2\varepsilon_{\mathrm{lab,clip}}^2(R_{\mathrm{lab}}),
\end{equation*}
while \cref{eq:clipped-to-exact-score} gives
\begin{equation*}
  \gE_{\mathrm{score}}(\widehat{s}_N)
  \le 2A_{t_0, T}\E[L_{\widehat{s}_N}^{\mathrm{exc}}]
  +2\varepsilon_{\mathrm{lab,clip}}^2(R_{\mathrm{lab}}).
\end{equation*}
By the convention of \cref{app:subsec:episodic-learning}, for each realized training sample this population error is exactly the conditional error on a fresh episode.
Combining the preceding bounds and taking the infimum over $s \in \gS$ proves \cref{eq:balanced-score-oracle} with a universal implicit constant.
Every concentration variable above is one complete episode augmented for variance-balanced DSM; auxiliary perturbations of the same episode do not increase the independent sample size beyond $N$.
\end{proof}

\subsection{Error Decomposition and Tuning Guidance}

The first three terms in \cref{eq:balanced-score-oracle} measure model approximation, label clipping, and optimization.
The final two arise from the uniform-convergence argument: one accounts for extending the bound from a finite cover to the entire score class, and the other is the covering-number generalization term.
Only these final two terms depend on the cover radius, and the generalization term also depends logarithmically on the cover cardinality.

For a desired clipping-bias tolerance $\varepsilon_{\mathrm{lab,tar}}^2 > 0$, the explicit choice
\begin{equation*}
  R_{\mathrm{lab}} = \sqrt{d_z}+
  \left[2\max\left\{0,
  \log\frac{2A_{t_0, T}}{\varepsilon_{\mathrm{lab,tar}}^2}\right\}\right]^{1/2}
\end{equation*}
gives $\varepsilon_{\mathrm{lab,clip}}^2(R_{\mathrm{lab}}) \le \varepsilon_{\mathrm{lab,tar}}^2$.
Increasing $R_{\mathrm{lab}}$ reduces clipping bias but increases the bound $R_{\mathrm{pred}}+R_{\mathrm{lab}}$ on the difference between the prediction and label, and hence the bound on the squared loss.
All tuning parameters may depend deterministically on $(n,N,t_0,T)$ and on fixed class parameters, but not on the observed training sample.

A convenient choice of the cover radius $\varepsilon_{\mathrm{cov}}$ as a function of the episode count $N$ is
\begin{equation}
  \varepsilon_{\mathrm{cov}}
  = \frac{\min\{1,R_{\mathrm{pred}}+R_{\mathrm{lab}}\}}{(N+1)^2}.
  \label{eq:convenient-score-cover-radius}
\end{equation}
Whenever the corresponding cover is finite, substituting \cref{eq:convenient-score-cover-radius} into \cref{eq:balanced-score-oracle} gives an explicit score-error bound whose square root enters the reverse-SDE stability analysis linearly.
With this choice, the $\varepsilon$-net approximation term is lower order than the $N^{-1}$ generalization term in \cref{eq:balanced-score-oracle}.
Later, $\varepsilon_{\mathrm{cov},n}$ denotes a different radius for the Transformer class, chosen as a function of the prompt length before the $N$ training episodes are observed.

\clearpage

\section{Reverse Diffusion Sampling and Posterior-Predictive Generation}
\label{app:sec:sampling}

This section converts integrated score error into a posterior-predictive generation guarantee.
We first define the randomized inverse-variance time grid and establish the exact reverse-process and endpoint estimates.
We then combine a finite-energy KL comparison with exact-score dynamics to obtain a positive-time discretization bound.
The final four subsections choose the accuracy parameters and tail radii, combine approximation, learning, and sampling, and record the scalar sample-only specialization.
The finite-step theorem applies, conditionally on training, to any finite-valued jointly Borel score with finite integrated score error; its output law averages over the independent grid shift and Gaussian sampling noise.
Throughout, $M\in\mathbb{N}$ is the number, or an almost-sure upper bound on the number, of reverse updates.

A small time-integrated error does not ensure accuracy at predetermined evaluation times for a merely measurable score.
The randomized grid in Section~\ref{subsec:main-training-and-sampling} addresses this mismatch: averaging the variance-weighted sum of score errors over its shift recovers the training integral.
This makes the score-learning guarantee applicable to the sampler without an additional time-regularity assumption.

\Cref{cor:randomized-grid-update-bound} bounds predictive error by initialization, the square root of integrated score error, early stopping, and discretization.
The bound concerns the sampler's output distribution averaged over its independent grid shift and uses at most $M$ score evaluations almost surely under \cref{eq:admissible-sampling-updates}.

\subsection{Randomized Inverse-Variance Time Grid}
\label{app:subsec:randomized-grid-sampler}

The randomly shifted inverse-variance grid is chosen to match the error controlled by score learning.
The discrete score-error contribution is bounded by a random quadrature whose expectation equals the integrated squared score error. 
Thus, this random upper bound is unbiased for the integrated criterion, although the discrete contribution itself need not be.
No Lipschitz bound on the learned score is needed.

Let $T > t_0 > 0$.
For $t \in [t_0, T]$, write
\begin{equation*}
  \mathsf{A}_{t_0}(t) := A_{t_0, t}
  = \int_{t_0}^t\frac{\od u}{\sigma_u^2}.
\end{equation*}
The quantity $A_{t_0,T}$ is the length of the sampling interval in inverse-variance time.

Fix a grid spacing
\begin{equation*}
  0 < h_{\mathrm{grid}}
  \le \min\{1/2, A_{t_0,T}, T - t_0\}.
\end{equation*}
Independently of all other randomness, draw the grid shift $U_{\mathrm{grid}} \sim \Unif(0, h_{\mathrm{grid}})$ and put
\begin{align*}
  J_{\mathrm{grid}}
  &:= \max\{j \in \mathbb{N}_0:
  U_{\mathrm{grid}} + jh_{\mathrm{grid}} \le A_{t_0,T}\},
  &v_{-1}&:=t_0,\\
  v_j
  &:=\mathsf{A}_{t_0}^{-1}(U_{\mathrm{grid}}+jh_{\mathrm{grid}}),
  &\widetilde{T}_{\mathrm{grid}}&:=v_{J_{\mathrm{grid}}}.
\end{align*}
Starting from $\widehat{Z}_{J_{\mathrm{grid}}+1}^{\mathrm{grid}} \sim \gN(0,I_{d_z})$, for $j = J_{\mathrm{grid}}, J_{\mathrm{grid}}-1, \ldots, 0$ use the exponential--Euler update
\begin{equation}
  \widehat{Z}_j^{\mathrm{grid}}
  = e^{\Delta_j}\widehat{Z}_{j+1}^{\mathrm{grid}}
  +2(e^{\Delta_j} - 1)\widehat{s}(\widehat{Z}_{j+1}^{\mathrm{grid}}, v_j, h)
  +\sqrt{e^{2\Delta_j} - 1}\,\xi_j^{\mathrm{grid}},
  \qquad
  \Delta_j := v_j - v_{j-1}.
  \label{eq:randomized-grid-euler-update}
\end{equation}
Here $\widehat Z_{j+1}^{\mathrm{grid}}$ is the state at $v_j$ before update $j$, and $\widehat Z_j^{\mathrm{grid}}$ is the state at $v_{j-1}$ afterward. Thus, the initial index $J_{\mathrm{grid}}+1$ does not denote an additional time node.
The vectors $\xi_j^{\mathrm{grid}}$ are independent and standard Gaussian.
Let $\widehat{Q}_{t_0 \,|\, h}^{\mathrm{grid}, h_{\mathrm{grid}}}$ denote the distribution of $\widehat{Z}_0^{\mathrm{grid}}$ after averaging over $U_{\mathrm{grid}}$ and the Gaussian sampler noise.
Joint Borel measurability of the learned score makes $h\mapsto\widehat{Q}_{t_0 \,|\, h}^{\mathrm{grid}, h_{\mathrm{grid}}}$ a Borel probability kernel.
Define the corresponding data-space output kernel by
\begin{equation}
  \widehat{P}_{0\,|\, h}^{\mathrm{grid},h_{\mathrm{grid}}}
  :=(\proj_{\gZ})_{\#}
  \widehat{Q}_{t_0\,|\, h}^{\mathrm{grid},h_{\mathrm{grid}}}.
  \label{eq:randomized-grid-output-kernel}
\end{equation}
The guarantees below concern this kernel, which averages over the random grid shift, and do not assert the same bound for every fixed realization of $U_{\mathrm{grid}}$.

\subsection{Exact Reverse Process on Positive Times}
\label{app:subsec:exact-reverse-diffusion}

Recall the conditional OU marginals and score from \cref{app:subsec:ou-score-identities}.
Let $r \in [0, T)$ be increasing reverse time, so that the corresponding forward time is $t = T-r$.
For fixed $h$, let $Z_r^{\leftarrow} := Z_{T-r}$ and equip the reversed path with its natural reversed filtration.
Smooth positivity and \cref{eq:positive-time-score-bounds}, applied to the forward segment $[\varepsilon,T]$, verify the local regularity, linear-growth, and finite-energy hypotheses of the diffusion time-reversal theorem.
The standard diffusion time-reversal theorem~\citep{haussmannpardoux1986} implies that, for every $\varepsilon \in (0, T)$, its restriction to $0 \le r \le T-\varepsilon$ is a weak solution of
\begin{equation}
  \od Z_r^{\leftarrow}
  = \left\{Z_r^{\leftarrow}
  +2s^{\star}(Z_r^{\leftarrow}, T-r \,|\, h)\right\} \,\od r
  +\sqrt{2}\,\od W_r^{\leftarrow},
  \qquad
  Z_0^{\leftarrow} \sim P_{T \,|\, h},
  \label{eq:exact-reverse}
\end{equation}
and its marginal at every $0 \le r < T$ is $P_{T-r \,|\, h}$.
Its marginals therefore converge weakly to the posterior-predictive distribution $P_{0 \,|\, h}$ as $r\uparrow T$.
No regularity of the original output distribution $P_{0 \,|\, h}$ is needed because the SDE representation is asserted only on intervals bounded away from forward time zero.

For each problem instance, fix a deterministic $0 < t_0 < T$; later bounds allow $(t_0, T)$ to depend on $n$.
Score learning is restricted to forward diffusion times $t \in [t_0, T]$, and the approximate reverse sampler is run only through reverse time $T-t_0$.
Thus $t_0$ is a cutoff in diffusion time, not an optimization stopping rule.
This positive cutoff avoids potentially nonuniform score behavior as $t\downarrow0$ and corresponds numerically to a minimum positive noise level.

The exact unprojected reverse endpoint has distribution $P_{t_0 \,|\, h}$ rather than $P_{0 \,|\, h}$; after Euclidean projection onto $\gZ$, the exact stopped distribution is $(\proj_{\gZ})_{\#}P_{t_0 \,|\, h}$.

\subsection{Endpoint Errors: Gaussian Initialization and Early Stopping}
\label{app:subsec:endpoint-errors}

The reverse sampler incurs approximation errors at both endpoints of the diffusion interval.
At terminal time $T$, it starts from a standard Gaussian approximation to the noised posterior-predictive distribution, whereas at time $t_0$ it stops at positive diffusion time.
We treat these errors separately because they depend on different assumptions and tuning parameters.

\subsubsection{Terminal Gaussian initialization error}
\label{app:subsubsec:terminal-gaussian-initialization}

At terminal time $T$, the learned reverse process is initialized from the standard Gaussian distribution $\gN(0,I_{d_z})$ instead of the exact noised posterior-predictive distribution $P_{T \,|\, H_n}$.
Define this initialization error by
\begin{equation}
  \varepsilon_{\mathrm{init}}(T)
  := \E\left[\TV\left(
  P_{T \,|\, H_n}, \gN(0,I_{d_z})
  \right)\right].
  \label{eq:terminal-error-definition}
\end{equation}
By the output-space convention in \cref{app:subsec:probability-conventions}, $P_{0 \,|\, H_n}$ is supported on $[0, 1]^{d_z}$.
The resulting explicit initialization bound is
\begin{equation}
  \varepsilon_{\mathrm{init}}(T)
  \le \frac{1}{2}\sqrt{-d_z\log(1-e^{-2T})},
  \qquad T>0.
  \label{eq:terminal-error-explicit}
\end{equation}
For the chosen horizons $T\ge1$, the same calculation gives $\varepsilon_{\mathrm{init}}(T)\le\sqrt{d_z/2}\,e^{-T}$.
Conditionally on $Z_0=z$, the forward distribution at time $T$ is
\begin{equation*}
  \gN\{e^{-T}z,(1-e^{-2T})I_{d_z}\}.
\end{equation*}
Convexity of KL, the Gaussian KL formula, and $\|z\|_2^2\le d_z$ give
\begin{equation*}
  \KL(P_{T\,|\, H_n}\|\gN(0,I_{d_z}))
  \le -\frac{d_z}{2}\log(1-e^{-2T}).
\end{equation*}
Pinsker's inequality and Jensen's inequality yield \cref{eq:terminal-error-explicit}.
For $T\ge1$, the elementary inequality $-\log(1-e^{-2T})\le2e^{-2T}$ gives $\varepsilon_{\mathrm{init}}(T)\le\sqrt{d_z/2}\,e^{-T}$.

\subsubsection{Early-stopping error}
\label{app:subsubsec:early-stopping}

Recall that $\beta' = \beta \wedge 1$.
Reverse sampling stops at $t_0>0$, so projecting the positive-time distribution back to the data cube introduces an early-stopping error.
The smoothness exponent is capped at one because even an infinitely smooth density generally moves mass of order $\sqrt{t}$ across the cube boundary before projection.

\begin{lemma}[H\"older early-stopping error bound on the cube]
\label{lem:cube-holder-tv-stopping}
Let $P(\od z) = p(z) \,\od z$ on $\gZ = [0, 1]^{d_z}$ with $p \in \gF_{\beta}^{d_z}(L_{\beta})$, and let $P_t$ be its OU evolution in \cref{eq:forward-ou-fixed-time}.
There is $C_{\mathrm{stop}} = C_{\mathrm{stop}}(\beta, d_z, L_{\beta}) < \infty$ such that, for every $0 < t \le 1$,
\begin{align*}
  \TV\left((\proj_{\gZ})_{\#}P_t, P\right)
  &\le \TV(P_t, P)
  \le C_{\mathrm{stop}}t^{\beta'/2}.
\end{align*}
No lower density bound is required.
\end{lemma}

\begin{proof}[Proof of \cref{lem:cube-holder-tv-stopping}]
Write $d=d_z$, and let $\bar{p}$ be the zero extension of $p$ from $\gZ=[0,1]^d$ to $\R^d$.
First note the global translation modulus
\begin{equation}
  \|\bar{p}(\cdot-y)-\bar{p}\|_{L^1(\R^d)}
  \le H_{\mathrm{tr}}\|y\|_2^{\beta'},
  \qquad y \in \R^d,
  \label{eq:appendix-zero-extension-translation}
\end{equation}
for a constant $H_{\mathrm{tr}}=H_{\mathrm{tr}}(d,\beta,L_{\beta})$.
Indeed, on $\gZ\cap(\gZ+y)$ the H\"older condition gives a bound of order $\|y\|_2^{\beta'}$ (the bounded gradient gives power one when $\beta>1$), while the two boundary layers have Lebesgue measure at most $C_d\min\{1,\|y\|_2\}$.
Since $\|p\|_{\infty}\le L_{\beta}$, their contribution has the same power for $\|y\|_2\le1$; for larger translations the trivial $L^1$ bound by two proves \cref{eq:appendix-zero-extension-translation} after increasing $H_{\mathrm{tr}}$.

For $0 < \alpha \le 1$, set $D_{\alpha}\bar{p}(x) := \alpha^{-d}\bar{p}(x/\alpha)$.
Uniformly for $e^{-1} \le \alpha \le 1$,
\begin{equation}
  \|D_{\alpha}\bar{p}-\bar{p}\|_1
  \le H_{\mathrm{dil}}(1-\alpha)^{\beta'},
  \label{eq:appendix-zero-extension-dilation}
\end{equation}
for a constant $H_{\mathrm{dil}}=H_{\mathrm{dil}}(d,\beta,L_{\beta})$.
To see this, split the integral over $\alpha\gZ$ and $\gZ\setminus\alpha\gZ$.
The normalization difference and the boundary layer are $\Ord(1-\alpha)$.
On $\alpha\gZ$ compare $p(x/\alpha)$ with $p(x)$; their arguments are at distance at most $\sqrt{d}(\alpha^{-1}-1)$, so H\"older continuity (or the bounded gradient) gives $\Ord((1-\alpha)^{\beta'})$.
This proves \cref{eq:appendix-zero-extension-dilation}.
A change of variables in \cref{eq:appendix-zero-extension-translation} also gives
\begin{equation}
  \|D_{\alpha}\bar{p}(\cdot-y)-D_{\alpha}\bar{p}\|_1
  \le H_{\mathrm{tr}}\alpha^{-\beta'}\|y\|_2^{\beta'}.
  \label{eq:appendix-dilated-translation}
\end{equation}

Put $\alpha = \alpha_t = e^{-t}$ and $\sigma = \sigma_t$.
Let $\varphi_{\sigma}^{(d)}$ be the centered $d$-variate Gaussian density with covariance $\sigma^2 I_d$.
The OU density is $p_t = \varphi_{\sigma}^{(d)}*D_{\alpha}\bar{p}$.
Minkowski's inequality, \cref{eq:appendix-zero-extension-dilation,eq:appendix-dilated-translation}, and $\xi \sim \gN(0, I_d)$ yield
\begin{align}
  \|p_t-\bar{p}\|_1
  &\le
  H_{\mathrm{tr}}\alpha^{-\beta'}\sigma^{\beta'}\E[\|\xi\|_2^{\beta'}]
  +H_{\mathrm{dil}}(1-\alpha)^{\beta'}.
  \label{eq:appendix-ou-holder-l1}
\end{align}
For $0 < t \le 1$, $\alpha^{-\beta'} \le e$, $\sigma \le \sqrt{2t}$, $1-\alpha \le t$, and $\E[\|\xi\|_2^{\beta'}] \le d^{\beta'/2}$.
Dividing \cref{eq:appendix-ou-holder-l1} by two gives $\TV(P_t, P) \le C_{\mathrm{stop}}t^{\beta'/2}$.
Finally, projection is a Markov kernel and fixes every point of $\gZ$, so
\begin{equation*}
  \TV((\proj_{\gZ})_{\#}P_t, P)
  = \TV((\proj_{\gZ})_{\#}P_t, (\proj_{\gZ})_{\#}P)
  \le \TV(P_t, P).
\end{equation*}
\end{proof}

\begin{corollary}[Conditional early-stopping error bound]
\label{cor:conditional-tv-stopping}
Fix $\beta>0$ and $L_\beta\ge1$, and consider either the sample-only or the input--output prompt model in \cref{eq:hierarchical-sample,eq:input-output-episode} with any Borel task prior $\Pi$.
Suppose the task laws admit jointly Borel densities $p_\theta(z)$ in the sample-only model or $f_\theta(x,z)$ in the input--output model, with $p_\Theta\in\gF_\beta^{d_z}(L_\beta)$ or $f_\Theta\in\gF_\beta^{d_x+d_z}(L_\beta)$, respectively, $\Pi$-almost surely.
Writing $\beta'=\beta\wedge1$, there is a constant $C_{\mathrm{stop}}$ depending only on $(\beta,L_\beta)$ and the dimensions such that
\begin{equation}
  \E\left[\TV\left(
  (\proj_{\gZ})_{\#}P_{t \,|\, H_n}, P_{0 \,|\, H_n}
  \right)\right]
  \le C_{\mathrm{stop}}t^{\beta'/2},
  \qquad 0 < t \le 1.
  \label{eq:posterior-predictive-tv-stopping}
\end{equation}
For any standard Borel space $\mathcal{U}$ and Borel map $\phi_{\mathrm{pr}}:\gH_n\to\mathcal{U}$, set $U=\phi_{\mathrm{pr}}(H_n)$.
Write $P_{0\mid U}:=\Law(Z_0\mid U)$ and $P_{t\mid U}:=\Law(Z_t\mid U)$.
The same bound holds with $H_n$ replaced by $U$.
In the input--output model, the expectation is over the query drawn from its input marginal; no lower bound on its marginal density is assumed.
\end{corollary}

\begin{proof}[Proof of \cref{cor:conditional-tv-stopping}]
In the sample-only task family, for almost every prompt $h$, the density of $P_{0 \,|\, h}$ is the posterior mixture of the task densities.
The H\"older ball is convex, so \cref{lem:cube-holder-tv-stopping} applies with the same constant and averaging proves \cref{eq:posterior-predictive-tv-stopping}.

For the input--output task family, condition first on the demonstrations $C_n = c$.
Let
\begin{equation*}
  f_c(x,z):=\int f_{\theta}(x,z)\Pi(\od\theta\,|\, C_n=c),
  \qquad
  \lambda_c(x):=\int_{\gZ} f_c(x,z)\,\od z.
\end{equation*}
The query has density $\lambda_c$, and where $\lambda_c(x)>0$ the predictive output density is $f_c(x,\cdot)/\lambda_c(x)$.
Let $\mathsf{T}_t^z$ denote unprojected OU evolution in the $z$ coordinate and extend $f_c(x, \cdot)$ by zero outside the cube.
Projection contraction, linearity, and cancellation of the query marginal give
\begin{align*}
 \E\!\left[
 \TV((\proj_{\gZ})_{\#}P_{t \,|\, H_n}, P_{0 \,|\, H_n})
 \,|\, C_n = c\right]
 \le
 \frac{1}{2}\int_{[0, 1]^{d_x}}
 \int_{\R^{d_z}}
 \left|\mathsf{T}_t^z f_c(x, z)-f_c(x, z)\right|\od z \,\od x.
\end{align*}
The proof of \cref{lem:cube-holder-tv-stopping} applies verbatim to each nonnegative subdensity $f_c(x, \cdot)$: its mass and its $z$-H\"older norm are bounded uniformly by a constant depending only on $L_{\beta}$.
Integrating over the unit-volume input cube proves the asserted bound without dividing by a lower bound for $\lambda_c$.

For $U=\phi_{\mathrm{pr}}(H_n)$, both conditional laws are mixtures of their counterparts given $H_n$.
Convexity of TV and the tower property give
\begin{equation*}
 \E\TV((\proj_{\gZ})_\#P_{t\mid U},P_{0\mid U})
 \le\E\TV((\proj_{\gZ})_\#P_{t\mid H_n},P_{0\mid H_n}),
\end{equation*}
proving the claim with conditioning on $U$.
\end{proof}

For a target early-stopping error $\varepsilon_{\mathrm{stop}}\in(0,1]$, the upper bound permits $t_0\asymp\varepsilon_{\mathrm{stop}}^{2/\beta'}$.

For TV accuracy $\varepsilon$, \cref{eq:accuracy-diffusion-times} chooses $t_0\asymp\varepsilon^{2/(\beta\wedge1)}$ and fixes the constants in the early-stopping bound.

\subsection{Finite-Energy KL Comparison for Reverse Diffusions}
\label{app:subsec:reverse-diffusion-kl-comparison}

The following lemma states the stopped finite-energy form of Girsanov's theorem used below; see \citet{liptser2001} for the standard localization approach.
We include the proof because no global Novikov condition is assumed.

\begin{lemma}[Finite-energy KL bound for diffusions with matched noise]
\label{lem:finite-energy-girsanov}
Fix $d\in\mathbb{N}$, a finite path horizon $T_{\mathrm{path}}>0$, and an initial distribution $\mu\in\gP(\R^d)$.
Let $\mathsf{Q}$ be the law on $C([0,T_{\mathrm{path}}];\R^d)$ of
\begin{equation*}
  \od Y_r = b(r, Y_r) \,\od r+\sqrt{2}\,\od W_r,
  \qquad Y_0 \sim \mu,
\end{equation*}
where $b:[0,T_{\mathrm{path}}]\times\R^d\to\R^d$ is Borel, globally Lipschitz in its state variable uniformly in $r$, and has linear growth.
Let $b^{\star}:[0,T_{\mathrm{path}}]\times\R^d\to\R^d$ be Borel, and let $\mathsf{P}^{\star}$ be the law of a weak solution with the same initial distribution and drift $b^{\star}$.
If
\begin{equation*}
  \E_{\mathsf{P}^{\star}}\int_0^{T_{\mathrm{path}}}
  \|b^{\star}(r, Y_r)-b(r, Y_r)\|_2^2 \,\od r < \infty,
\end{equation*}
then
\begin{equation}
  \KL(\mathsf{P}^{\star}\|\mathsf{Q})
  \le \frac{1}{4}\E_{\mathsf{P}^{\star}}\int_0^{T_{\mathrm{path}}}
  \|b^{\star}(r, Y_r)-b(r, Y_r)\|_2^2 \,\od r.
  \label{eq:finite-energy-girsanov}
\end{equation}
The same conclusion holds when the reference process is defined by a recursively frozen nonanticipative drift.
Specifically, choose a deterministic partition $0=r_0<\cdots<r_K=T_{\mathrm{path}}$ and finite-valued Borel maps $a_k:\R^d\to\R^d$, $0\le k<K$, and set
\begin{equation*}
  b_{\mathrm{fr}}(r,\omega)
  :=\omega_r+2a_k(\omega_{r_k}),
  \qquad r\in[r_k,r_{k+1}),
\end{equation*}
with the last interval closed at $T_{\mathrm{path}}$.
In this alternative case, let $\mathsf{Q}$ denote the law of the unique interval-by-interval solution of $\od Y_r=b_{\mathrm{fr}}(r,Y_{[0,r]})\,\od r+\sqrt{2}\,\od W_r$ with $Y_0\sim\mu$.
In this case, replace $b(r,Y_r)$ in the energy condition and in \cref{eq:finite-energy-girsanov} by $b_{\mathrm{fr}}(r,Y_{[0,r]})$.
The construction includes $a_k(x)=\widehat{s}(x,\vartheta_k,h)$ for fixed score-evaluation times $\vartheta_k$ whenever $\widehat{s}$ is Borel and finite-valued.
The reference process and every post-stopping continuation are then constructed explicitly, interval by interval, so no state regularity of the maps $a_k$ is required.
\end{lemma}

\begin{proof}
Put $\Delta b_r=b^{\star}(r,Y_r)-b(r,Y_r)$ in the Markov case and $\Delta b_r=b^{\star}(r,Y_r)-b_{\mathrm{fr}}(r,Y_{[0,r]})$ in the recursively frozen case, and define
\begin{equation*}
  \upsilon_m := \inf\left\{r:\int_0^r\|\Delta b_u\|_2^2 \,\od u \ge m\right\}
  \wedge T_{\mathrm{path}}.
\end{equation*}
On the filtered weak-solution space carrying $\mathsf{P}^{\star}$, the process $\Delta b$ is progressively measurable, so $\upsilon_m$ is a stopping time.
On a realization of the weak solution under $\mathsf{P}^{\star}$, retain its path through $\upsilon_m$ and, after $\upsilon_m$, continue with the reference drift and the same Brownian increments.
Global Lipschitzness of $b$ makes this continuation well defined in the Markov case; the recursively frozen case is constructed successively on the remaining grid intervals.
Indeed, conditional on the left endpoint of an interval $[r_k,r_{k+1}]$, its frozen map value is finite and measurable, and the resulting affine SDE has a unique explicit solution.
Induction over the deterministic partition gives weak uniqueness for the recursively frozen reference process and for every post-stopping continuation.
Denote its law by $\mathsf{P}_m$.
The coupling gives
\begin{equation*}
  \TV(\mathsf{P}_m, \mathsf{P}^{\star})
  \le \mathsf{P}^{\star}(\upsilon_m < T_{\mathrm{path}})
  \le \frac{1}{m}\E_{\mathsf{P}^{\star}}\int_0^{T_{\mathrm{path}}}\|\Delta b_r\|_2^2 \,\od r
  \longrightarrow 0.
\end{equation*}
Relative to $\mathsf{Q}$, the drift change under $\mathsf{P}_m$ is $\ind\{r\le\upsilon_m\}\Delta b_r$.
Its accumulated squared norm is at most $m$.
Hence Novikov's criterion holds because
\begin{equation*}
  \E_{\mathsf{P}_m}\exp\!\left\{
  \frac{1}{4}\int_0^{\upsilon_m}\|\Delta b_r\|_2^2\,\od r
  \right\}
  \le e^{m/4}<\infty.
\end{equation*}
If $W^{(m)}$ denotes the Brownian motion driving the process obtained by joining the two segments, define the stochastic exponential
\begin{equation*}
  D_u^{(m)}
  := \exp\left\{-\frac{1}{\sqrt{2}}\int_0^{u\wedge\upsilon_m}
  \Delta b_r^{\top} \,\od W_r^{(m)}
  -\frac{1}{4}\int_0^{u\wedge\upsilon_m}\|\Delta b_r\|_2^2 \,\od r\right\},
  \qquad 0\le u\le T_{\mathrm{path}}.
\end{equation*}
This process is therefore a true, uniformly integrable martingale of mean one.
Define $\mathsf{Q}_m$ by $\od\mathsf{Q}_m/\od\mathsf{P}_m=D_{T_{\mathrm{path}}}^{(m)}$. Girsanov's theorem then identifies its drift.
Under $\mathsf{Q}_m$ the coordinate process has the reference drift $b$ in the Markov case and $b_{\mathrm{fr}}$ in the recursively frozen case.
Weak uniqueness for the globally Lipschitz reference equation identifies $\mathsf{Q}_m = \mathsf{Q}$ in the Markov case; the interval-by-interval construction gives the same identification in the recursively frozen case.
Taking the $\mathsf{P}_m$-expectation of $\log(\od\mathsf{P}_m/\od\mathsf{Q})$ eliminates the stochastic integral and gives
\begin{equation*}
  \KL(\mathsf{P}_m\|\mathsf{Q})
  = \frac{1}{4}\E_{\mathsf{P}_m}
  \int_0^{\upsilon_m}\|\Delta b_r\|_2^2 \,\od r
  = \frac{1}{4}\E_{\mathsf{P}^{\star}}
  \int_0^{\upsilon_m}\|\Delta b_r\|_2^2 \,\od r.
\end{equation*}
Lower semicontinuity of relative entropy under $\mathsf{P}_m \to \mathsf{P}^{\star}$ in total variation, followed by monotone convergence, proves \cref{eq:finite-energy-girsanov}.
\end{proof}

\subsection{Positive-Time Score Dynamics and Discretization}
\label{app:subsec:score-dynamics-and-discretization}

The next lemma controls the change of the exact score within a time-grid cell.
Together with the KL bound for diffusions with finite drift energy, it gives the positive-time sampling bound below.

\begin{lemma}[Exact reverse-score dynamics on a positive-time interval]
\label{lem:exact-reverse-score-dynamics}
Fix a prompt $h$ for which $P_{0\,|\, h}$ is supported on $[0,1]^{d_z}$, and let $p_t(\cdot\,|\, h)$ and $s_t^{\star}(\cdot\,|\, h)=\nabla\log p_t(\cdot\,|\, h)$ be its OU marginal density and score.
For $0<a<b<\infty$, let $(Y_r)_{0\le r\le b-a}$ be the exact reverse process
\begin{equation*}
  \od Y_r
  =\{Y_r+2s_{b-r}^{\star}(Y_r\,|\, h)\}\,\od r
  +\sqrt{2}\,\od W_r,
  \qquad
  Y_0\sim P_{b\,|\, h}.
\end{equation*}
Then $Y_r\sim P_{b-r\,|\, h}$, and the forward representation $Z_t=\alpha_tZ_0+\sigma_t\xi$ gives
\begin{equation}
  s_t^{\star}(Z_t \,|\, h)
  = -\sigma_t^{-1}\E[\xi \,|\, Z_t, h],
  \qquad
  \nabla s_t^{\star}(Z_t \,|\, h)
  = \sigma_t^{-2}\{\Cov(\xi \,|\, Z_t, h)-I_{d_z}\}.
  \label{eq:appendix-randomized-grid-score-identities}
\end{equation}
Consequently,
\begin{align}
  \E\|s_t^{\star}(Z_t \,|\, h)\|_2^2
  &\le \frac{d_z}{\sigma_t^2},
  \nonumber\\
  \E\|\nabla s_t^{\star}(Z_t \,|\, h)\|_{\mathrm{F}}^2
  &\le \frac{d_z(d_z+3)}{\sigma_t^4}.
  \label{eq:appendix-randomized-grid-fisher-hessian}
\end{align}
Moreover, $S_r:=s_{b-r}^{\star}(Y_r\,|\, h)$ is a continuous semimartingale satisfying
\begin{equation}
  \od S_r
  =-S_r\,\od r
  +\sqrt{2}\,\nabla s_{b-r}^{\star}(Y_r\,|\, h)\,\od W_r.
  \label{eq:appendix-randomized-grid-score-process}
\end{equation}
All identities hold on the full interval $0\le r\le b-a$ in the usual localized It\^o sense.
\end{lemma}

\begin{proof}
At every positive time, Gaussian convolution makes $p_t(\cdot\,|\, h)$ strictly positive and smooth.
Gaussian differentiation gives \cref{eq:appendix-randomized-grid-score-identities}.
Conditional Jensen proves the score bound.
For $\Sigma_{t,h}:=\Cov(\xi\,|\, Z_t,h)$, one has $\tr(\Sigma_{t,h}^2)\le\{\tr(\Sigma_{t,h})\}^2$ and
\begin{equation*}
  \E\{\tr(\Sigma_{t,h})\}^2
  \le \E\|\xi\|_2^4
  =d_z(d_z+2).
\end{equation*}
Since $\|\Sigma_{t,h}-I_{d_z}\|_{\mathrm{F}}^2\le\{\tr(\Sigma_{t,h})\}^2+d_z$, this proves the Hessian bound.
The positive-time diffusion time-reversal identity gives $Y_r\sim P_{b-r\,|\, h}$.
Stop the process when it leaves a compact state set and apply It\^o's formula to $(r,y)\mapsto s_{b-r}^{\star}(y\,|\, h)$.
Differentiating the OU Fokker--Planck equation gives
\begin{equation*}
  \partial_t s_t^{\star}
  =s_t^{\star}+(\nabla s_t^{\star})y
  +\nabla\operatorname{div}s_t^{\star}
  +2(\nabla s_t^{\star})s_t^{\star}.
\end{equation*}
Substitution of the exact reverse drift cancels all finite-variation terms except $-S_r\,\od r$, proving \cref{eq:appendix-randomized-grid-score-process} up to the stopping time.
The moment bounds in \cref{eq:appendix-randomized-grid-fisher-hessian}, together with $b-r\ge a$, give the square integrability required to let the localization radii tend to infinity and obtain the identity on the full interval.
\end{proof}

\subsubsection{Discretization on a positive-time interval}

\begin{proposition}[Randomized time-grid discretization bound]
\label{prop:randomized-grid-discretization}
Consider the randomized time-grid sampler in \cref{app:subsec:randomized-grid-sampler}.
Let $\widehat{s}$ be trained independently of the test episode.
Conditionally on the completed training run, assume that $\widehat{s}$ is finite-valued and jointly Borel and that $\gE_{\mathrm{score}}(\widehat{s})<\infty$.
Then
\begin{align}
  &\E\TV\bigl(
  \widehat{Q}_{t_0 \,|\, H_n}^{\mathrm{grid},h_{\mathrm{grid}}},
  P_{t_0 \,|\, H_n}\bigr)
  \le
  \varepsilon_{\mathrm{init}}(T-h_{\mathrm{grid}})
  +\sqrt{\frac{\gE_{\mathrm{score}}(\widehat{s})}{2}}
  +\varepsilon_{\mathrm{grid}}(h_{\mathrm{grid}}),
  \label{eq:randomized-grid-discretization-bound}\\
  &\varepsilon_{\mathrm{grid}}(h_{\mathrm{grid}})
  :=
  \left\{
    \frac{d_z(T-t_0)h_{\mathrm{grid}}^2}{6}
    +d_z(d_z+3)e^{2h_{\mathrm{grid}}}h_{\mathrm{grid}} A_{t_0,T}
  \right\}^{1/2},
  \quad h_{\mathrm{grid}}>0.
  \label{eq:randomized-grid-discretization-error}
\end{align}
The number of reverse updates is at most $\lceil A_{t_0,T}/h_{\mathrm{grid}}\rceil$ almost surely.
\end{proposition}

\begin{proof}[Proof of \cref{prop:randomized-grid-discretization}]
We first compare the exact reverse process with the process that uses the same initialization and freezes the score on each time-grid cell.
The score difference splits into variation of the exact score within each cell and learned-score error at the grid nodes; averaging over the random shift controls the latter by the integrated score error.
The finite-energy KL bound then gives a TV bound, after which we account for Gaussian initialization.

Condition on the learned score and fix a fresh prompt $h$ in a common full-measure set on which the conditional integrated squared score error is finite.

To bound the variation of the exact score, consider one time-grid cell.
Let $Y_r$ be the exact reverse path initialized at its upper forward time and write $S_r = s_{v_j-r}^{\star}(Y_r \,|\, h)$ for $0 \le r \le \Delta_j = v_j-v_{j-1}$.
Define the integrated squared error from freezing the exact score within each time-grid cell by
\begin{equation*}
  \mathcal{F}_h(U_{\mathrm{grid}})
  :=\sum_{j=0}^{J_{\mathrm{grid}}}
  \int_0^{\Delta_j}\E\|S_r-S_0\|_2^2 \,\od r.
\end{equation*}
Applying \cref{lem:exact-reverse-score-dynamics} with $(a,b)=(v_{j-1},v_j)$ gives \cref{eq:appendix-randomized-grid-fisher-hessian,eq:appendix-randomized-grid-score-process} on this cell.

The time grid satisfies
\begin{equation*}
  \frac{\od}{\od\mathsf{A}_{t_0}(t)}\log\sigma_t^2
  = 2e^{-2t} \le 2.
\end{equation*}
Hence, whenever $t \in [v_{j-1}, v_j]$,
\begin{equation}
  \Delta_j \le h_{\mathrm{grid}}\sigma_{v_j}^2
  \le h_{\mathrm{grid}}e^{2h_{\mathrm{grid}}}
  \sigma_t^2,
  \qquad
  \Delta_j \le h_{\mathrm{grid}}.
  \label{eq:appendix-randomized-grid-cell-length}
\end{equation}
Solving \cref{eq:appendix-randomized-grid-score-process} on the cell gives
\begin{equation*}
  S_r-S_0 = (e^{-r}-1)S_0
  +\sqrt{2}\int_0^re^{-(r-u)}
  \nabla s_{v_j-u}^{\star}(Y_u \,|\, h) \,\od W_u.
\end{equation*}
The two terms are orthogonal in $L^2$.
Since $\int_0^{\Delta_j}(1-e^{-r})^2 \,\od r \le \Delta_j^3/3$, $\E\|S_0\|_2^2 \le d_z/\sigma_{v_j}^2$, and $\Delta_j/\sigma_{v_j}^2 \le h_{\mathrm{grid}}$, the summed drift contribution is at most $d_z(T-t_0)h_{\mathrm{grid}}^2/3$.
It\^o isometry bounds the summed martingale contribution by
\begin{equation*}
  2\sum_j\Delta_j\int_{v_{j-1}}^{v_j}
  \E\|\nabla s_t^{\star}(Z_t \,|\, h)\|_{\mathrm{F}}^2 \,\od t.
\end{equation*}
Use $\Delta_j/\sigma_t^2 \le h_{\mathrm{grid}}e^{2h_{\mathrm{grid}}}$ and \cref{eq:appendix-randomized-grid-fisher-hessian}.
Summing the cells gives
\begin{equation}
  \mathcal{F}_h(U_{\mathrm{grid}})
  \le \frac{d_z}{3}(T-t_0)h_{\mathrm{grid}}^2
  +2d_z(d_z+3)e^{2h_{\mathrm{grid}}}
  h_{\mathrm{grid}}A_{t_0,T}.
  \label{eq:appendix-randomized-grid-freezing-energy}
\end{equation}
This bound is uniform in $h$ and $U_{\mathrm{grid}}$.

We next control the learned-score error at the grid nodes.
Since the learned score is only measurable, we use averaging over the random shift rather than a time-continuity bound.
Put
\begin{equation*}
  e_h(t) := \E_{P_{t \,|\, h}}
  \|s_t^{\star}(Z_t \,|\, h)-\widehat{s}(Z_t, t, h)\|_2^2.
\end{equation*}
Joint Borel measurability of the scores and the conditional OU kernel makes $e_h$ Lebesgue measurable, and the finite conditional integrated squared score error makes it integrable.
A random grid shift hits any null set of exceptional times with probability zero; the countable union over lattice indices is still null.
By \cref{eq:appendix-randomized-grid-cell-length}, the learned-score grid contribution is at most
\begin{equation*}
  \mathcal{E}_h^{\mathrm{grid}}(U_{\mathrm{grid}})
  := h_{\mathrm{grid}}
  \sum_{j = 0}^{J_{\mathrm{grid}}}\sigma_{v_j}^2e_h(v_j).
\end{equation*}
Averaging the shifted lattice over $U_{\mathrm{grid}}$ and applying the change of variables $x = \mathsf{A}_{t_0}(t)$ gives exactly
\begin{align}
  \E_{U_{\mathrm{grid}}}\mathcal{E}_h^{\mathrm{grid}}(U_{\mathrm{grid}})
  &= \int_0^{A_{t_0,T}}\sigma_{\mathsf{A}_{t_0}^{-1}(x)}^2
  e_h\{\mathsf{A}_{t_0}^{-1}(x)\} \,\od x
  = \int_{t_0}^Te_h(t) \,\od t.
  \label{eq:appendix-randomized-grid-quadrature}
\end{align}

The two integrated squared-error bounds now control the difference between the exact and frozen reverse drifts.
For almost every grid shift $u$, let $\mathsf{P}_{h,u}^{\star}$ be the exact reverse path law initialized from $P_{\widetilde{T}_{\mathrm{grid}}(u)\,|\, h}$, and let $\mathsf{Q}_{h,u}$ be the equally initialized recursively frozen path law.
Represent both paths in the decreasing forward-time coordinate $t\in[t_0,\widetilde{T}_{\mathrm{grid}}(u)]$, write $Z_t^{\leftarrow}$ for the coordinate process under $\mathsf{P}_{h,u}^{\star}$, and let $\kappa_u(t)=v_j$ for almost every $t\in[v_{j-1},v_j]$.
Define the integrated squared score error along the full reverse path
\begin{equation*}
  \mathcal{I}_h(u)
  :=\int_{t_0}^{\widetilde{T}_{\mathrm{grid}}(u)}
  \E_{\mathsf{P}_{h,u}^{\star}}\!\left[
  \left\|s_t^{\star}(Z_t^{\leftarrow}\,|\, h)
  -\widehat{s}(Z_{\kappa_u(t)}^{\leftarrow},\kappa_u(t),h)\right\|_2^2
  \right] \,\od t.
\end{equation*}
The reference process is defined interval by interval by the explicit exponential--Euler formula, so finite-valued Borel evaluation of $\widehat{s}$ suffices to construct it with the score frozen on each interval.
The reverse-drift difference is twice the score difference defining $\mathcal{I}_h(u)$, so the factor $1/4$ in \cref{eq:finite-energy-girsanov} gives the first inequality below.
Minkowski's inequality with respect to time and the exact reverse-path law, followed by \cref{eq:appendix-randomized-grid-cell-length}, gives the remaining inequalities:
\begin{align*}
  \KL\!\left(\mathsf{P}_{h,u}^{\star}\|\mathsf{Q}_{h,u}\right)
  &\le \mathcal{I}_h(u),
  \\
  \mathcal{I}_h(u)^{1/2}
  &\le \mathcal{F}_h(u)^{1/2}
  +\left\{\sum_{j=0}^{J_{\mathrm{grid}}}\Delta_j e_h(v_j)\right\}^{1/2}
  \\
  &\le \mathcal{F}_h(u)^{1/2}+\mathcal{E}_h^{\mathrm{grid}}(u)^{1/2}.
\end{align*}
Pinsker's inequality, Jensen's inequality in $U_{\mathrm{grid}}$, \cref{eq:appendix-randomized-grid-freezing-energy}, and \cref{eq:appendix-randomized-grid-quadrature} yield the last two terms in \cref{eq:randomized-grid-discretization-bound} after averaging over $h$.

It remains to replace the exact terminal distribution by the sampler's Gaussian initialization.
Since $A_{t_0,T}-\mathsf{A}_{t_0}(\widetilde{T}_{\mathrm{grid}}) < h_{\mathrm{grid}}$ and $\od t/\od\mathsf{A} = \sigma_t^2 \le 1$, we have $\widetilde{T}_{\mathrm{grid}} \in [T-h_{\mathrm{grid}}, T]$.
Invariance of $\gN(0,I_{d_z})$ under the OU semigroup and TV contraction give the terminal term $\varepsilon_{\mathrm{init}}(T-h_{\mathrm{grid}})$.
Convexity of TV in the first argument permits averaging the output distribution over $U_{\mathrm{grid}}$.
The update count follows because, with $x = A_{t_0,T}/h_{\mathrm{grid}}$ and $y = U_{\mathrm{grid}}/h_{\mathrm{grid}} \in (0, 1)$, $J_{\mathrm{grid}}+1 = \lfloor x-y\rfloor+1 \le \lceil x\rceil$ almost surely.
\end{proof}

\subsubsection{Generation error with a prescribed number of updates}

The proof combines the grid-spacing bound in \cref{prop:randomized-grid-discretization} with the conditional stopping bound in \cref{cor:conditional-tv-stopping}.
The main theorem uses the following equivalent form, which is indexed by the number of steps.

For $0<t_0<T$ with $t_0\le1$, write $A=A_{t_0,T}$ and choose an integer
\begin{equation}
 M\ge\max\{\lceil2A\rceil,\lceil A/(T-t_0)\rceil\}.
 \label{eq:admissible-sampling-updates}
\end{equation}

\begin{corollary}[Randomized time-grid bound for at most $M$ updates]
\label{cor:randomized-grid-update-bound}
Fix $\beta>0$ and $L_\beta\ge1$, and consider either the sample-only or the input--output prompt model in \cref{eq:hierarchical-sample,eq:input-output-episode} with any Borel task prior $\Pi$.
Suppose the task laws admit jointly Borel densities $p_\theta(z)$ in the sample-only model or $f_\theta(x,z)$ in the input--output model, with $p_\Theta\in\gF_\beta^{d_z}(L_\beta)$ or $f_\Theta\in\gF_\beta^{d_x+d_z}(L_\beta)$, respectively, $\Pi$-almost surely.
Let $T>t_0>0$ with $t_0\le1$, and write $\beta'=\beta\wedge1$.
Suppose the learned score is independent of the test episode and, conditionally on the completed training run, finite-valued and jointly Borel with finite integrated squared score error.
Let $M \in \mathbb{N}$ satisfy
\begin{equation*}
  M \ge
  \max\left\{\lceil 2A_{t_0,T}\rceil,
  \left\lceil\frac{A_{t_0,T}}{T-t_0}\right\rceil\right\},
  \qquad
  h_{\mathrm{grid}} := \frac{A_{t_0,T}}{M}.
\end{equation*}
When $T-t_0 \ge 1/2$, as under the parameter choices below for all sufficiently large prompt lengths, the simpler condition $M \ge \lceil 2A_{t_0,T}\rceil$ implies both requirements.
Then the sampler uses at most $M$ reverse updates almost surely, and its output distribution, averaged over the random grid shift, satisfies
\begin{align}
  \E\TV\left(
  \widehat{P}_{0 \,|\, H_n}^{(M)}, P_{0 \,|\, H_n}\right)
  &\le
  \varepsilon_{\mathrm{init}}(T - A_{t_0,T}/M)
  +\sqrt{\frac{\gE_{\mathrm{score}}(\widehat{s})}{2}}
  +C_{\mathrm{stop}}t_0^{\beta'/2}
  \nonumber\\
  &\qquad
  +\frac{A_{t_0,T}}{M}\sqrt{\frac{d_z(T-t_0)}{6}}
  +e^{A_{t_0,T}/M}\frac{A_{t_0,T}}{\sqrt{M}}\sqrt{d_z(d_z + 3)}.
  \label{eq:randomized-grid-update-bound}
\end{align}
Here $\widehat{P}_{0 \,|\, h}^{(M)}:=\widehat{P}_{0 \,|\, h}^{\mathrm{grid},A_{t_0,T}/M}$, so the parenthesized argument denotes a bound on the number of updates rather than a grid spacing.
Moreover, there is a constant $0<C_{\mathrm{disc}}=C_{\mathrm{disc}}(d_z)<\infty$ such that, for every $0<\varepsilon_{\mathrm{disc}}\le1$, the displayed conditions on $M$ and $M\ge C_{\mathrm{disc}}\varepsilon_{\mathrm{disc}}^{-2}(A_{t_0,T}+1)^2$ make the sum of the two discretization terms at most $\varepsilon_{\mathrm{disc}}$.
Under the parameter choices used in \cref{eq:main-predictive-sufficient-choices}, $A_{t_0,T}\lesssim\log(e/\varepsilon)$ and $\log(e/\varepsilon)\ge1$, so $A_{t_0,T}+1\lesssim\log(e/\varepsilon)\lesssim\log\varepsilon^{-1}$ for $\varepsilon\le1/2$.
Thus, for a sufficiently large implicit constant, the condition on $M$ in \cref{eq:main-predictive-sufficient-choices} implies both inequalities in \cref{eq:admissible-sampling-updates} and the preceding discretization requirement with $\varepsilon_{\mathrm{disc}}=\varepsilon$.
This is a sufficient number of sampling steps, not a minimax lower bound for the number of reverse updates.
\end{corollary}

\begin{proof}[Proof of \cref{cor:randomized-grid-update-bound}]
The conditions on $M$ give $A_{t_0,T}/M \le \min\{1/2, A_{t_0,T}, T-t_0\}$, so \cref{prop:randomized-grid-discretization} applies with $h_{\mathrm{grid}}=A_{t_0,T}/M$.
Projection contracts TV, and the triangle inequality adds the conditional stopping error from \cref{cor:conditional-tv-stopping}.
With $h_{\mathrm{grid}}=A_{t_0,T}/M$, subadditivity of the square root in \cref{eq:randomized-grid-discretization-error} gives
\begin{equation*}
  \varepsilon_{\mathrm{grid}}(A_{t_0,T}/M)
  \le
  \frac{A_{t_0,T}}{M}\sqrt{\frac{d_z(T-t_0)}{6}}
  +e^{A_{t_0,T}/M}\frac{A_{t_0,T}}{\sqrt{M}}\sqrt{d_z(d_z+3)}.
\end{equation*}
The update count is $\lceil A_{t_0,T}/(A_{t_0,T}/M)\rceil = M$.
Substitution proves the claim.
Since $\sigma_t^2\le1$, one has $T-t_0\le A_{t_0,T}$, and the condition $M\ge2A_{t_0,T}$ gives $e^{A_{t_0,T}/M}\le e^{1/2}$.
If $M\ge C_{\mathrm{disc}}\varepsilon_{\mathrm{disc}}^{-2}(A_{t_0,T}+1)^2$, the sum of the two discretization terms is at most
\begin{equation*}
  \frac{\sqrt{d_z/6}}{C_{\mathrm{disc}}}
  \frac{A_{t_0,T}^{3/2}}{(A_{t_0,T}+1)^2}\varepsilon_{\mathrm{disc}}^2
  +\frac{e^{1/2}\sqrt{d_z(d_z+3)}}{\sqrt{C_{\mathrm{disc}}}}
  \frac{A_{t_0,T}}{A_{t_0,T}+1}\varepsilon_{\mathrm{disc}}.
\end{equation*}
Choosing $C_{\mathrm{disc}}(d_z)$ sufficiently large makes this expression at most $\varepsilon_{\mathrm{disc}}$ for every $0<\varepsilon_{\mathrm{disc}}\le1$.
\end{proof}

This bound for at most $M$ updates applies to any learned score with finite integrated squared score error, regardless of the training objective.
In the finite-sample results, the score-error bound comes from variance-balanced DSM, and \cref{eq:randomized-grid-update-bound} converts it into an explicit almost-sure bound on the number of reverse updates.

The H\"older condition is needed only to pass from the positive-time distribution to the data distribution; the positive-time discretization bound requires only the standing cube support and its stated score assumptions.

Under the same conditions, the preceding corollary gives the simpler form
\begin{align}
  \E\TV\bigl(
    \widehat{P}_{0\,|\, H_n}^{(M)},P(\cdot\,|\, H_n)
  \bigr)
  \lesssim
  \underbrace{\varepsilon_{\mathrm{init}}(T-A_{t_0,T}/M)}_{\text{Gaussian initialization}}
  +\underbrace{\sqrt{\gE_{\mathrm{score}}(\widehat{s})}}_{\text{learned score}}
  \!+\!\underbrace{t_0^{(\beta \wedge 1)/2}}_{\text{early stopping}}
  +\underbrace{A_{t_0,T}/\sqrt{M}}_{\text{discretization}}.
  \label{eq:main-finite-step-sampling-bound}
\end{align}
The implicit constant depends only on the fixed dimension and regularity bounds.
Indeed, $T-t_0\le A_{t_0,T}\le M/2$, so both discretization terms in \cref{eq:randomized-grid-update-bound} are $\Ord(A_{t_0,T}/\sqrt{M})$.
The square root converts integrated squared-score error $\Ord(\varepsilon^2)$ into generation error $\Ord(\varepsilon)$.

The formal posterior-predictive guarantee is stated in \cref{thm:main-generation-error}.
The next three subsections complete its proof by choosing the common score class and diffusion parameters, verifying the required tail bounds, and combining the approximation, learning, and sampling results.

\subsection{End-to-End Parameter Choices for Target Accuracy}
\label{app:subsec:accuracy-parameter-choices}

Fix the common parameters in \cref{assump:smooth-latent-task}, and put $\beta'=\beta\wedge1$.
All constants below depend only on these parameters and dimensions.
Choose $0<\varepsilon\le\varepsilon_0$, where a sufficiently small fixed $\varepsilon_0$ ensures $0<t_0<T$, $t_0\le1$, and the radius inequalities below.
For every integer $n\ge1$, set
\begin{align}
 t_0&=\left(\frac{\varepsilon}{4C_{\mathrm{stop}}}\right)^{2/\beta'},
 &T&=\max\left\{1,1+\log\frac{4\sqrt{d_z/2}}{\varepsilon}\right\},
 &A&=A_{t_0,T}.
 \label{eq:accuracy-diffusion-times}
\end{align}
Here $C_{\mathrm{stop}}>0$ is any fixed valid constant in the early-stopping TV bound.
Choose fixed constants satisfying
\begin{equation*}
 c_{\mathrm{state}}\ge2\sqrt{d_z},\qquad
 c_{\mathrm{state}}^2/16>2+2/\beta',\qquad c_{\mathrm{lab}}^2/2>2,
\end{equation*}
and choose
\begin{align}
 R_{\mathrm{state},n}&=\sqrt{d_z}+c_{\mathrm{state}}\sqrt{\log(e/\varepsilon)},
 &R_{\mathrm{out},n}=R_{\mathrm{lab},n}&=\sqrt{d_z}+c_{\mathrm{lab}}\sqrt{\log(e/\varepsilon)},\nonumber\\
 R_{\mathrm{bal},n}&=\frac{R_{\mathrm{state},n}+\sqrt{d_z}}{\sigma_{t_0}}.
 \label{eq:smooth-latent-explicit-radii}
\end{align}
The dependence on $\varepsilon$ in all quantities carrying a subscript $n$ is suppressed only in this construction for a fixed accuracy.
We denote the resulting class and trainable count explicitly by $\gS_{n,\varepsilon}^{\mathrm{lat}}$ and $\mathsf{N}_{\mathrm{param},n,\varepsilon}^{\mathrm{lat}}$.
The symbols $\gS_n^{\mathrm{lat}}$, $\mathsf{N}_{\mathrm{param},n}^{\mathrm{lat}}$, and $B_{\mathrm{lat},n}$ in the supporting lemmas suppress this fixed accuracy; in the main text they are written $\gS_{n,\varepsilon}^{\mathrm{lat}}$, $\mathsf{N}_{\mathrm{param},n,\varepsilon}^{\mathrm{lat}}$, and $B_{\mathrm{lat},n,\varepsilon}$.
When the theorem is applied at the prompt rate, these symbols always mean $\varepsilon=\sqrt{\log(en)/n}$.

For a fixed sufficiently small $c_{\mathrm{acc}}>0$, choose the grid by
\begin{align}
 m_{\theta,n}
 &=\left\lceil\sqrt{d_\theta}
 \left(1+\frac{(n+1)H_\theta^2}
 {c_{\mathrm{acc}}\underline{p}\,\varepsilon^2}\right)^{1/(2\beta_\theta)}\right\rceil,
 &J_{\theta,n}&=m_{\theta,n}^{d_\theta},\nonumber\\
 \eta_{\theta,n}&=\sqrt{d_\theta}/m_{\theta,n},
 &\varepsilon_{\mathrm{floor},n}&=c_{\mathrm{acc}}\varepsilon^2.
 \label{eq:accuracy-integrated-grid}
\end{align}
Use the regular cells and representatives of \cref{app:subsec:latent-quadrature}.
Their maximum distance is at most $\eta_{\theta,n}$.
This gives
\begin{equation}
 D_n(\eta_{\theta,n},\varepsilon_{\mathrm{floor},n})\le2c_{\mathrm{acc}}\varepsilon^2,
 \qquad
 J_{\theta,n}\lesssim\left(1+\frac{\sqrt{n+1}}{\varepsilon}\right)^{d_\theta/\beta_\theta}.
 \label{eq:accuracy-grid-size}
\end{equation}
The extra one in \cref{eq:accuracy-integrated-grid} makes this choice valid also when $H_\theta=0$.
The prior-logit radius is bounded by $\log(2J_{\theta,n}/\varepsilon_{\mathrm{floor},n})=\Ord(\log(en/\varepsilon))$.

Set $\varepsilon_{\mathrm{tar},n}=c_{\mathrm{acc}}\varepsilon/\sqrt{A}$ and choose the neural tolerances
\begin{align}
 \varepsilon_{\ell,n}&=\frac{\varepsilon_{\mathrm{tar},n}}{18n\mathsf{C}_{\mathrm{net},n}},
 &\varepsilon_{\mathrm{qry},n}&=\frac{\varepsilon_{\mathrm{tar},n}}{18\mathsf{C}_{\mathrm{net},n}},\nonumber\\
 \varepsilon_{a,n}&=\frac{\varepsilon_{\mathrm{tar},n}}{18(1+R_{\mathrm{num},n}^{\mathrm{lat}}\underline{q}_{\mathrm{out}}^{-2})},
 &\varepsilon_{b,n}&=\frac{\varepsilon_{\mathrm{tar},n}}{18(1+\underline{q}_{\mathrm{out}}^{-1})},\nonumber\\
 \varepsilon_{\mathrm{rat},n}&=\varepsilon_{\mathrm{tar},n}/18,
 &\varepsilon_{\mathrm{cov},n}&=\frac{c_{\mathrm{acc}}\varepsilon^2}{A(R_{\mathrm{out},n}+R_{\mathrm{lab},n})}.
 \label{eq:accuracy-component-tolerances}
\end{align}
For sample-only prompts set $\varepsilon_{\mathrm{qry},n}=0$ and omit that module.
Decreasing $\varepsilon_0$ makes all nonzero component tolerances at most one and $R_{\mathrm{out},n}\le R_{\mathrm{bal},n}$.
These choices are deterministic, independent of the unknown family and prior, and fixed before observing training data.
They do not depend on $N$ or the requested confidence level.

\subsection{Tail Bounds for the Selected Radii}
\label{app:subsec:accuracy-tail-bounds}

\begin{lemma}[Gaussian-tail bounds for the selected radii]
\label{lem:smooth-latent-explicit-tail-bounds}
Use the diffusion times and radii in \cref{eq:accuracy-diffusion-times,eq:smooth-latent-explicit-radii}, with $\beta'=\beta\wedge1$.
Let $\delta_{\mathrm{tail},n}$ and $\varepsilon_{\mathrm{tail},n}^2$ be the probability and score-energy terms in \cref{eq:transformer-good-event-probability,eq:transformer-exceptional-score-energy}.
There is $\varepsilon_1\in(0,1]$, depending only on $d_z$, $\beta'$, $C_{\mathrm{stop}}$, $c_{\mathrm{state}}$, and $c_{\mathrm{lab}}$, such that, for $0<\varepsilon\le\varepsilon_1$ and $c_{\mathrm{state}}\ge2\sqrt{d_z}$, the following bounds hold with implicit constants depending only on the same quantities:
\begin{align}
  \sigma_{t_0}^{-1}
  &\lesssim \varepsilon^{-1/\beta'},
  &
  A
  &\lesssim \log(e/\varepsilon),
  \label{eq:smooth-latent-tail-time-prefactors}\\
  R_{\mathrm{bal},n}^2
  &\lesssim \log(e/\varepsilon)\varepsilon^{-2/\beta'},
  &
  \delta_{\mathrm{tail},n}
  &\le \varepsilon^{c_{\mathrm{state}}^2/8},
  \label{eq:smooth-latent-tail-radius-probability}\\
  \varepsilon_{\mathrm{tail},n}^2
  &\lesssim \log(e/\varepsilon)
  \varepsilon^{c_{\mathrm{state}}^2/16-2/\beta'},
  \label{eq:smooth-latent-tail-energy-rate}\\
  AR_{\mathrm{bal},n}^2
  \delta_{\mathrm{tail},n}
  &\lesssim \log^2(e/\varepsilon)
  \varepsilon^{c_{\mathrm{state}}^2/8-2/\beta'},
  \label{eq:smooth-latent-tail-envelope-rate}\\
  \varepsilon_{\mathrm{lab,clip}}^2
  \{R_{\mathrm{lab},n}\}
  &\lesssim \log(e/\varepsilon)\varepsilon^{c_{\mathrm{lab}}^2/2}.
  \label{eq:smooth-latent-label-tail-rate}
\end{align}
Consequently, if
\begin{equation}
  \frac{c_{\mathrm{state}}^2}{16}
  >2+\frac{2}{\beta'},
  \qquad
  \frac{c_{\mathrm{lab}}^2}{2}>2,
  \label{eq:smooth-latent-tail-constant-condition}
\end{equation}
then
\begin{equation}
  4AR_{\mathrm{bal},n}^2
  \delta_{\mathrm{tail},n}
  +4\varepsilon_{\mathrm{tail},n}^2
  +\varepsilon_{\mathrm{lab,clip}}^2
  \{R_{\mathrm{lab},n}\}
  =o(\varepsilon^2).
  \label{eq:smooth-latent-tail-total-rate}
\end{equation}
In particular, the left-hand side is at most $c_1\varepsilon^2$ for every fixed $c_1>0$ and all sufficiently small $\varepsilon$.
\end{lemma}

\begin{proof}
For $0<t\le1$, concavity of $t\mapsto1-e^{-2t}$ gives
\begin{equation*}
  \sigma_t^2=1-e^{-2t}\ge(1-e^{-2})t.
\end{equation*}
After decreasing $\varepsilon_1$ so that $t_0\le1$, this proves the first bound in \cref{eq:smooth-latent-tail-time-prefactors}.
Moreover, $e^{2t_0}-1\ge2t_0$ and $e^{2T}-1\le e^{2T}$ imply
\begin{equation*}
  A
  \le T+\frac{1}{2}\log\{1/(2t_0)\}
  \lesssim \log(e/\varepsilon).
\end{equation*}
The definition of $R_{\mathrm{bal},n}$ now gives its bound in \cref{eq:smooth-latent-tail-radius-probability}.

Let $X:=\|\xi\|_2$ for $\xi\sim\gN(0,I_{d_z})$ and put $a:=c_{\mathrm{state}}\sqrt{\log(e/\varepsilon)}$.
The Gaussian norm concentration inequality gives
\begin{equation*}
  \Pr(X>a)
  \le\exp\{-\tfrac{1}{2}(a-\sqrt{d_z})^2\}
  \le e^{-a^2/8}
  \le\varepsilon^{c_{\mathrm{state}}^2/8},
\end{equation*}
where the second inequality uses $a\ge2\sqrt{d_z}$.
This and \cref{eq:transformer-good-event-probability} prove the probability bound.
Cauchy--Schwarz and the finite fourth Gaussian moment give
\begin{equation*}
  \E\!\left[
  (X+2\sqrt{d_z})^2\ind\{X>a\}
  \right]
  \le
  \left\{\E(X+2\sqrt{d_z})^4\right\}^{1/2}
  \Pr(X>a)^{1/2}
  \lesssim\varepsilon^{c_{\mathrm{state}}^2/16}.
\end{equation*}
Substitution into \cref{eq:derived-score-tail-bound}, together with \cref{eq:smooth-latent-tail-time-prefactors}, proves \cref{eq:smooth-latent-tail-energy-rate}.
Multiplying the bounds for $A$, $R_{\mathrm{bal},n}^2$, and $\delta_{\mathrm{tail},n}$ proves \cref{eq:smooth-latent-tail-envelope-rate}.
Finally, apply \cref{eq:balanced-clipping-bias} with $x=c_{\mathrm{lab}}\sqrt{\log(e/\varepsilon)}$ to obtain \cref{eq:smooth-latent-label-tail-rate}.
Under \cref{eq:smooth-latent-tail-constant-condition}, each exponent remaining after division by $\varepsilon^2$ is strictly positive.
Since $\varepsilon^a\{1+\log(\varepsilon^{-1})\}^k\to0$ for every $a>0$ and fixed $k$, \cref{eq:smooth-latent-tail-total-rate} follows.
\end{proof}

\subsection{Proof of the Posterior-Predictive Generation Guarantee}
\label{app:subsec:posterior-predictive-generation-proof}

For the score class in \cref{thm:main-generation-error-formal}, combine its logarithmic covering number with the confidence term:
\begin{equation*}
 \Gamma_{n,\varepsilon,\delta}
 :=\log\!\left(\frac{2\Ncov(\varepsilon_{\mathrm{cov},n},
 \gS_{n,\varepsilon}^{\mathrm{lat}},d_{\mathrm{bal},\infty})}{\delta}\right)
 =\log\Ncov\!\left(\varepsilon_{\mathrm{cov},n},
 \gS_{n,\varepsilon}^{\mathrm{lat}},d_{\mathrm{bal},\infty}\right)
 +\log(2/\delta).
\end{equation*}
The cover centers belong to the class, and its radius is fixed by \cref{eq:accuracy-component-tolerances}.
The following calculation establishes the approximation and complexity claims used in \cref{thm:main-generation-error-formal}.
It combines \cref{prop:integrated-transformer-approximation} with the covering-number bound for the entire score class in \cref{prop:smooth-latent-complexity}.
In particular, the approximation and parameter-count bounds hold for the same deterministic class used in training.

The chosen scaled-output and label radii $R_{\mathrm{out},n}$ and $R_{\mathrm{lab},n}$ are both $\Ord(\sqrt{\log(e/\varepsilon)})$, and $A\lesssim \log(e/\varepsilon)$.
On the training event in \cref{thm:main-generation-error-formal}, the more detailed score and predictive bounds are
\begin{equation}\begin{aligned}
 \gE_{\mathrm{score}}(\widehat{s}_{N,n,\varepsilon})
 &\lesssim\varepsilon^2+\varepsilon_{\mathrm{opt}}
       +\frac{A(R_{\mathrm{out},n}+R_{\mathrm{lab},n})^2\Gamma_{n,\varepsilon,\delta}}{N},
\end{aligned}
 \label{eq:main-balanced-score-oracle}
\end{equation}
\begin{equation*}
    \E\TV\!\bigl(\widehat{P}_{0\,|\,H_n}^{(M)},P(\cdot\mid H_n)\bigr)
    \lesssim
    \varepsilon+\sqrt{\varepsilon_{\mathrm{opt}}}
    +(R_{\mathrm{out},n}+R_{\mathrm{lab},n})\sqrt{\frac{A\Gamma_{n,\varepsilon,\delta}}{N}}
    +\frac{A}{\sqrt{M}}.
\end{equation*}
Using the definition of $\Gamma_{n,\varepsilon,\delta}$, the radius bounds, and \cref{eq:accuracy-cover-bound}, and imposing $\varepsilon_{\mathrm{opt}}\lesssim\varepsilon^2$, gives \cref{eq:main-predictive-generation-bound} and the sufficient choices in \cref{eq:main-predictive-sufficient-choices}.

\begin{theorem}[Posterior-predictive generation; formal version of Theorem~\ref{thm:main-generation-error}]
\label{thm:main-generation-error-formal}
Under \cref{assump:smooth-latent-task}, for each $n\ge2$ and sufficiently small $\varepsilon>0$, there exists a common DiT score class $\gS_{n,\varepsilon}^{\mathrm{lat}}$ and diffusion parameters that depend on $(n,\varepsilon)$ and the fixed dimensions and regularity bounds. 
Its number of trainable parameters $\mathsf{N}_{\mathrm{param},n,\varepsilon}^{\mathrm{lat}}$, stored nonzero coefficients, and arithmetic operations per forward pass are polynomial in $(n,\varepsilon^{-1})$. Use the estimator and reverse sampler of Section~\ref{subsec:main-training-and-sampling}, with $N\ge1$ independent pretraining episodes, empirical optimization gap $\varepsilon_{\mathrm{opt}}\lesssim\varepsilon^2$, and an integer $M$ satisfying \cref{eq:admissible-sampling-updates}. 
For every fixed $\mathcal{M}\in\mathfrak{M}_{\mathrm{lat}}$ and $\delta\in(0,1)$, with probability at least $1-\delta$ over the training episodes and their diffusion augmentations,
\begin{equation}
  \E\TV\!\bigl(\widehat P^{(M)}_{0\mid H_n},P(\cdot\mid H_n)\bigr) \lesssim \varepsilon +\log(e/\varepsilon)\sqrt{\frac{\mathsf{N}_{\mathrm{param},n,\varepsilon}^{\mathrm{lat}}\log(en/\varepsilon)+\log(2/\delta)}{N}} +\frac{\log(e/\varepsilon)}{\sqrt M}.
\end{equation}
The expectation averages a fresh prompt conditional on the completed training run.
\end{theorem}
\begin{proof}[Proof of \cref{thm:main-generation-error-formal}]
We first verify approximation and covering for the deterministic class.
We then apply the score-learning inequality and transfer its error to the randomized sampler on the same training event.

The grid and floor choices give \cref{eq:accuracy-grid-size}.
Substitution of \cref{eq:accuracy-component-tolerances} into the definition of $\varepsilon_{\mathrm{net},n}$ gives $\varepsilon_{\mathrm{net},n}\le\varepsilon_{\mathrm{tar},n}/3$.
\Cref{lem:smooth-latent-explicit-tail-bounds,lem:smooth-latent-output-clipping} give
\begin{equation*}
 A R_{\mathrm{bal},n}^2\delta_{\mathrm{tail},n}
 +\varepsilon_{\mathrm{tail},n}^2
 +\varepsilon_{\mathrm{out,clip},n}^2
 +\varepsilon_{\mathrm{lab,clip}}^2(R_{\mathrm{lab},n})
 \lesssim\varepsilon^2.
\end{equation*}
In particular the output-clipping contribution is at most $2A\exp(-c_{\mathrm{lab}}^2 \log(e/\varepsilon)/2)=o(\varepsilon^2)$.
Applying \cref{prop:integrated-transformer-approximation} proves
\begin{equation*}
 \gE_{\mathrm{app}}(\gS_{n,\varepsilon}^{\mathrm{lat}})\lesssim\varepsilon^2.
\end{equation*}

The quantitative size and sensitivity bounds in \cref{prop:smooth-latent-complexity} hold for arbitrary grids and component tolerances.
They therefore apply directly to the parameter choices made above for the integrated error.
Here $t_0^{-1}$, $T$, $R_{\mathrm{bal},n}$, the grid size, all inverse nonzero tolerances, and the prior-logit radius are polynomially bounded in $n$ and $\varepsilon^{-1}$.
The local OU moduli and the numbers of components are likewise polynomial.
The fixed-depth attention sensitivity recursion and the clipped readout give polynomial parameter radius and parameter sensitivity over the entire class.
Consequently, using the bounds $B_{\mathrm{lat},n}$ and $\Lip_{\mathrm{lat},n}^{\mathrm{par}}$ from \cref{prop:smooth-latent-complexity}, evaluated at the choices above,
\begin{align}
 \log\Ncov(u,\gS_{n,\varepsilon}^{\mathrm{lat}},d_{\mathrm{bal},\infty})
 &\le\mathsf{N}_{\mathrm{param},n,\varepsilon}^{\mathrm{lat}}
 \log\left(1+2B_{\mathrm{lat},n}\Lip_{\mathrm{lat},n}^{\mathrm{par}}/u\right),\nonumber\\
 \Gamma_{n,\varepsilon,\delta}
 &\lesssim\mathsf{N}_{\mathrm{param},n,\varepsilon}^{\mathrm{lat}}\log(en/\varepsilon)
       +\log(2/\delta).
 \label{eq:accuracy-cover-bound}
\end{align}
This establishes the approximation and complexity claims used in \cref{thm:main-generation-error-formal}.
The exact empirical minimizer exists measurably by the compactness and continuity argument in \cref{prop:smooth-latent-complexity}; the theorem permits any measurable approximate minimizer with the specified scaled-objective gap.

Apply \cref{thm:balanced-score-oracle} with its predictor radius $R_{\mathrm{pred}}=R_{\mathrm{out},n}$, the cover radius in \cref{eq:accuracy-component-tolerances}, and optimization tolerance $\varepsilon_{\mathrm{opt}}$.
With probability at least $1-\delta$, this gives \cref{eq:main-balanced-score-oracle}.
The uniform output bound and $t\ge t_0>0$ ensure that the learned score is finite-valued and Borel on the full state domain, as required by \cref{cor:randomized-grid-update-bound}.
In that corollary, $A/M\le1/2$ makes the shifted terminal-time error $\lesssim\varepsilon$ by \cref{eq:accuracy-diffusion-times}; the early-stopping error is at most $\varepsilon/4$.
Also $A\ge T-t_0$ and $M\ge2A$, so
\begin{equation*}
 \frac{A\sqrt{T-t_0}}{M}\le\frac{A}{\sqrt{M}}.
\end{equation*}
Taking square roots of \cref{eq:main-balanced-score-oracle} proves the main predictive bound.
% The task bound follows from \cref{eq:main-generation-error-decomposition}, with the same score and sampling event.
The task-conditional consequence is derived separately in \cref{cor:task-conditional-prompt-rate} by combining this predictive bound with \cref{eq:main-generation-error-decomposition}.
Finally $A\lesssim \log(e/\varepsilon)$ and $(R_{\mathrm{out},n}+R_{\mathrm{lab},n})^2\lesssim \log(e/\varepsilon)$ give the displayed logarithmic form.
The probability statement is for each model separately, with uniformly bounded failure probability and common deterministic constants.
\end{proof}

The sufficient choices in \cref{eq:main-predictive-sufficient-choices} make the predictive error $\Ord(\varepsilon)$.
Choose their structural constant so that both inequalities required of the integer-valued choices hold.
They are statistical and real-arithmetic bounds; they do not guarantee efficient optimization of the empirical loss.

\subsection{Scalar Sample-Only Complexity Specialization}

For sample-only prompts with $d_z=d_\theta=\beta=\beta_\theta=1$, $\delta=n^{-1}$, and $\varepsilon=\sqrt{\log(en)/n}$, the parameter choices based on the episode-level KL bound give the following sufficient bounds under exact real arithmetic:
\begin{equation}
 \begin{gathered}
 J_{\theta,n}=\widetilde{\Ord}(n),\qquad
 \mathsf{N}_{\mathrm{param},n}^{\mathrm{lat}}=\widetilde{\Ord}(n^5),\qquad
 \mathsf{N}_{\mathrm{store},n}^{\mathrm{lat}}=\widetilde{\Ord}(n^5),\qquad
 \mathsf{N}_{\mathrm{op},n}^{\mathrm{lat}}=\widetilde{\Ord}(n^5),\\
 N=\widetilde{\Ord}(n^6),\qquad M=\widetilde{\Ord}(n).
 \end{gathered}
 \label{eq:smooth-latent-canonical-counts}
\end{equation}
Here $\widetilde{\Ord}$ suppresses powers of $\log(en)$.

\begin{proof}[Derivation of \cref{eq:smooth-latent-canonical-counts}]
Let $d_z=d_\theta=1$, $\beta=\beta'=\beta_\theta=1$, with sample-only prompts and $\delta=n^{-1}$.
In this paragraph $\widetilde{\Ord}$ hides powers of $\log(en)$ only.
Let $t_{0,n}$ and $T_n$ denote the choices in \cref{eq:accuracy-diffusion-times} evaluated at $\varepsilon=\sqrt{\log(en)/n}$.
With $\varepsilon=\sqrt{\log(en)/n}$, the diffusion and radius choices give
\begin{equation*}
 t_{0,n}\asymp \log(en)/n,\qquad
 \sigma_{t_{0,n}}^{-1},R_{\mathrm{bal},n},R_{\mathrm{num},n}^{\mathrm{lat}}
 =\widetilde{\Ord}(n^{1/2}),\qquad J_{\theta,n}=\widetilde{\Ord}(n).
\end{equation*}
The state, final-output, and label radii, $T_n$, and $A_{t_{0,n},T_n}$ are logarithmic in $n$.
Formula \cref{eq:accuracy-component-tolerances} gives
\begin{equation*}
 \varepsilon_{\ell,n}^{-1}=\widetilde{\Ord}(n^2),\quad
 \varepsilon_{a,n}^{-1}=\widetilde{\Ord}(n),\quad
 \varepsilon_{b,n}^{-1}=\widetilde{\Ord}(n^{1/2}),\quad
 \varepsilon_{\mathrm{rat},n}^{-1}=\widetilde{\Ord}(n^{1/2}).
\end{equation*}
There are $\Ord(\log n)$ overlapping time patches.
On patches with $a_j\le1$, $h_j^{\mathrm{loc}}\le7a_j$ and $\sigma_{a_j}^2\asymp a_j$; on later patches, $\sigma_{a_j}^{-1}=\Ord(1)$ and $h_j^{\mathrm{loc}}\le T_n$.
Since $H_{q,x}=0$, \cref{eq:smooth-latent-local-denominator-modulus,eq:smooth-latent-local-numerator-modulus} give
\begin{equation*}
 \max_j H_{a,j,n}=\widetilde{\Ord}(n),\qquad
 \max_j H_{b,j,n}=\widetilde{\Ord}(n^{3/2}).
\end{equation*}
The two-dimensional state--time basis size is therefore
\begin{equation*}
 \mathsf{N}_{\mathrm{basis},n}^{\mathrm{lat}}
 =\widetilde{\Ord}\!\left(n\{(n\,n)^2+(n^{3/2}n^{1/2})^2\}\right)
 =\widetilde{\Ord}(n^5).
\end{equation*}
The likelihood and fixed ratio module sizes are $\widetilde{\Ord}(n^3)$ and $\widetilde{\Ord}(n^2)$, respectively.
The token width is $\Ord(J_{\theta,n})$; the attention projection count is $\widetilde{\Ord}(n^2)$ and its forward cost is $\widetilde{\Ord}((n+J_{\theta,n})J_{\theta,n}^2)=\widetilde{\Ord}(n^3)$.
Repeated likelihood evaluation costs $\widetilde{\Ord}(n^4)$.
Thus the basis term dominates the numbers of trainable parameters, stored coefficients, and forward-pass operations, each bounded by $\widetilde{\Ord}(n^5)$.
Substitution into \cref{eq:main-predictive-sufficient-choices} with $\varepsilon=\sqrt{\log(en)/n}$ gives $N=\widetilde{\Ord}(n^6)$ and $M=\widetilde{\Ord}(n)$.
\end{proof}

The direct implementation uses at most $M\mathsf{N}_{\mathrm{op},n}^{\mathrm{lat}}=\widetilde{\Ord}(n^6)$ arithmetic operations for a full sample in this scalar example.
The training sample contains $N(n+1)=\widetilde{\Ord}(n^7)$ observed outputs; these outputs are not generally independent within episodes.
The large exponents make the sufficient nature of the construction explicit.
They are not empirical scaling laws.

\clearpage

\section{Sharp Finite-Prompt Task Uncertainty}
\label{app:sec:task-uncertainty}

This appendix first states the finite-prompt uncertainty theorem, then proves it in three steps: posterior averaging compares task uncertainty with oracle risk, predictive information gives the upper bound, and an Assouad construction gives the matching lower bound. The final subsection combines this result with \cref{thm:main-generation-error} to obtain the task-conditional guarantee.
The input--output argument keeps the fresh query in the conditioning variables, and the family used for the lower bound varies along only one latent coordinate while the class bounds remain fixed.

\subsection{Finite-Prompt Theorem and Oracle Risk}
\label{app:subsec:formal-task-uncertainty}

Even exact posterior prediction can differ from the realized task distribution.
We therefore quantify this difference over the model class $\mathfrak{M}_{\mathrm{lat}}$ defined after \cref{assump:smooth-latent-task}, whose members share the same dimensions and regularity bounds.
For each $\mathcal M$, the symbols $\Pr_{\mathcal M}$, $\E_{\mathcal M}$, and $\Law_{\mathcal M}$ refer to its induced joint law; the dependence of the kernels $P$ on $\mathcal M$ is suppressed.
For a known model $\mathcal{M}\in\mathfrak{M}_{\mathrm{lat}}$, define the oracle task risk
\begin{equation}
 \mathcal{R}_n^{\mathrm{or}}(\mathcal{M})
 :=\inf_{Q^{\mathcal{M}}}\E_{\mathcal{M}}
 \TV\bigl(Q^{\mathcal{M}}(\cdot\mid H_n),P(\cdot\mid H_n,\Theta)\bigr),
 \label{eq:latent-oracle-risk}
\end{equation}
where the infimum is over Borel probability kernels from $\gH_n$ to $\gZ$ that may use the entire task family and task prior.
Generation randomness is included in the output distribution.
An independent random seed cannot improve the infimum by convexity of TV.
\Cref{lem:posterior-averaging} and the Borel posterior-predictive kernel fixed in \cref{eq:posterior-predictive-kernel} imply
\begin{equation}
 \mathcal{R}_n^{\mathrm{or}}(\mathcal{M})
 \le \E_{\mathcal{M}}\TV\!\bigl(P(\cdot\mid H_n),P(\cdot\mid H_n,\Theta)\bigr)
 \le2\mathcal{R}_n^{\mathrm{or}}(\mathcal{M}).
 \label{eq:latent-oracle-uncertainty-comparison}
\end{equation}
A lower bound on this oracle risk therefore limits every generator, including one with unlimited independent pretraining.

\begin{theorem}[Uniform upper and matching lower bounds for finite-prompt task uncertainty]
\label{thm:main-task-uncertainty-and-minimax-rate}
\emph{Uniform upper bound.}
For every $\mathcal{M}\in\mathfrak{M}_{\mathrm{lat}}$ and $n\ge1$, the following bound holds with an implicit constant depending only on the common class parameters:
\begin{equation}
 \E_{\mathcal{M}}\KL\!\bigl(
 P(\cdot\mid H_n,\Theta)\,\|\,P(\cdot\mid H_n)\bigr)
 \lesssim\frac{\log(en)}{n}.
 \label{eq:main-predictive-information-bound}
\end{equation}
Consequently,
\begin{equation*}
 \E_{\mathcal{M}}\TV\!\bigl(P(\cdot\mid H_n),P(\cdot\mid H_n,\Theta)\bigr)
 \lesssim \sqrt{\frac{\log(en)}{n}}
\end{equation*}
uniformly over $\mathcal{M}\in\mathfrak{M}_{\mathrm{lat}}$.

\emph{Matching lower bound.}
If $H_\theta>0$, $\underline{p}<1$, and $L_\beta>1$, then, for all sufficiently large $n$,
\begin{equation}
 \begin{aligned}
 \sqrt{\frac{\log(en)}{n}}&\lesssim\sup_{\mathcal{M}\in\mathfrak{M}_{\mathrm{lat}}}\mathcal{R}_n^{\mathrm{or}}(\mathcal{M})\\
 &\le\sup_{\mathcal{M}\in\mathfrak{M}_{\mathrm{lat}}}
 \E_{\mathcal{M}}\TV\!\bigl(P(\cdot\mid H_n),P(\cdot\mid H_n,\Theta)\bigr)
 \lesssim \sqrt{\frac{\log(en)}{n}}.
 \end{aligned}
 \label{eq:main-task-uncertainty-bound}
\end{equation}
The lower bound allows any number of independent training episodes, including a countably infinite sequence, and unrestricted computation at generation time.
The implicit constants and the sufficiently-large-$n$ threshold depend only on the common class parameters.
The worst-case model and prior may depend on $n$, while all class parameters remain fixed.
\end{theorem}

\textbf{Asymptotic interpretation of the upper bound.}
The uniform upper bound gives 
\begin{equation*}
  \E_{\mathcal M}\TV\!\bigl(P(\cdot\mid H_n),P(\cdot\mid H_n,\Theta)\bigr) 
  \lesssim
  \sqrt{\log(en)/n}. 
\end{equation*}
Thus the expected TV error tends to zero. In particular, for every fixed $\rho_{\mathrm{TV}}>0$, Markov's inequality gives
\begin{align*}
 &\Pr_{\mathcal M}\bigl(
 \TV\!\bigl(P(\cdot\mid H_n),P(\cdot\mid H_n,\Theta)\bigr)>\rho_{\mathrm{TV}}\bigr)\\
 &\qquad\le \frac{1}{\rho_{\mathrm{TV}}}\E_{\mathcal M}
 \TV\!\bigl(P(\cdot\mid H_n),P(\cdot\mid H_n,\Theta)\bigr)
 \lesssim \frac{1}{\rho_{\mathrm{TV}}}\sqrt{\frac{\log(en)}{n}}.
\end{align*}
Hence the TV error converges to zero in probability under the joint law of the realized task and prompt for each fixed $\mathcal M$.
This is convergence of the posterior-predictive output distribution to the realized task-conditional output distribution. 
It does not establish convergence for every fixed $\theta$, almost-sure convergence, or contraction of $\Pi(\cdot\mid H_n)$ to $\delta_\Theta$. 
Recovery of a unique latent parameter would additionally require identifiability; predictive convergence itself does not, since different parameters may induce the same distribution.

The upper bound applies to every $\mathcal{M}\in\mathfrak{M}_{\mathrm{lat}}$ without the additional nondegeneracy conditions.
The matching lower bound shows that the upper bound's dependence on $n$ cannot be improved uniformly over the stated nondegenerate class, even when the family and prior are known; it does not assert that every fixed family or prior attains this rate.
The predictive KL upper bound is proved in \cref{app:subsec:latent-predictive-information}, and the matching lower bound is proved in \cref{app:subsec:sharp-prompt-lower-bound}.

\subsection{Posterior Averaging and Oracle Comparison}
\label{app:subsec:posterior-averaging}

The posterior-predictive distribution averages the task law conditional on the prompt.
The following lemma bounds its expected task-conditional TV error by twice that of any prompt-based probability kernel.

\begin{lemma}[Posterior averaging bound]
\label{lem:posterior-averaging}
Suppose the posterior-mixture identity in \cref{eq:main-posterior-predictive-mixture} holds, and let $h \mapsto Q(\cdot \,|\, h)$ be any measurable probability kernel from $\gH_n$ to $\gZ$.
Then
\begin{equation*}
  \E\left[\TV\!\left(P(\cdot \,|\, H_n), P(\cdot \,|\, H_n, \Theta)\right)\right]
  \le
  2\E\left[\TV\!\left(Q(\cdot\mid H_n), P(\cdot \,|\, H_n, \Theta)\right)\right].
\end{equation*}
\end{lemma}

\begin{proof}
Write $Q_h=Q(\cdot\mid h)$ and $R_{h,\theta}=P(\cdot\mid h,\theta)$.
On the common full-measure set where the posterior-mixture identity holds, convexity of TV gives
\begin{equation*}
 \TV(Q_h,P(\cdot\mid h))
 \le\int\TV(Q_h,R_{h,\theta})\Pi(\od\theta\mid h).
\end{equation*}
These TV distances are measurable and bounded by one: on a standard Borel space, the defining supremum can be taken over a countable generating algebra.
Taking expectations and applying the tower property yields
\begin{equation*}
 \E\TV(Q(\cdot\mid H_n),P(\cdot\mid H_n))
 \le\E\TV(Q(\cdot\mid H_n),P(\cdot\mid H_n,\Theta)).
\end{equation*}
The triangle inequality through $Q(\cdot\mid H_n)$ proves the claim.
\end{proof}

The lemma uses only the posterior-mixture identity in \cref{eq:main-posterior-predictive-mixture} and the measurability of $Q(\cdot\mid H_n)$ as a function of the prompt.
The TV bound requires no density, smoothness, conditional i.i.d., or diffusion assumption.
Taking the infimum over $Q$ gives the factor-two oracle comparison in \cref{eq:latent-oracle-uncertainty-comparison}; the reverse inequality follows by choosing the posterior-predictive kernel.

Under TV loss, the posterior-predictive distribution need not minimize posterior expected TV loss: with posterior weights $0.9$ and $0.1$ on $\delta_0$ and $\delta_1$, its posterior expected loss is $0.18$, whereas $\delta_0$ has loss $0.1$.
Thus, \cref{lem:posterior-averaging} is a factor-two posterior-averaging bound, not an exact Bayes-optimality statement.

\subsection{Upper Bound via Predictive Information}
\label{app:subsec:latent-predictive-information}

Fix either the sample-only or the input--output prompt model and the common dimensions and regularity bounds in \cref{assump:smooth-latent-task}, and recall the corresponding model class $\mathfrak{M}_{\mathrm{lat}}$ defined after that assumption.

Under the positive latent-family assumption, the information supplied by the next output is bounded in terms of the cumulative information in the prompt.
This comparison concerns the realized task and its exact posterior predictive.
The episode-KL argument in \cref{prop:episode-kl-score-approximation} instead compares true and discretized models to control score approximation.

For standard-Borel random variables, write
\begin{equation*}
 I(A;B\mid C):=
 \E\KL\!\bigl(
 \Law(B\mid A,C)
 \,\|\,\Law(B\mid C)\bigr)
\end{equation*}
for conditional mutual information, and omit $C$ for unconditional mutual information.
The reference-mixture bound below ensures finiteness of the information quantities used in the proof, whose identities follow from the KL chain rule.

\begin{lemma}[From a latent KL modulus to one-step prediction]
\label{lem:latent-predictive-information}
Fix $d_\theta\in\mathbb{N}$ and a probability kernel $\theta\mapsto G_\theta$ on a standard Borel observation space.
Let $\Theta$ have any Borel prior $\Pi$ on $[0,1]^{d_\theta}$, and, conditionally on $\Theta=\theta$, let $U_1,U_2,\ldots$ be i.i.d. with law $G_\theta$.
Suppose that, for some $K_{\mathrm{KL}}\in[0,\infty)$ and $\beta_\theta\in(0,1]$,
\begin{equation}
 \KL(G_\theta\|G_{\theta'})
 \le K_{\mathrm{KL}}\|\theta-\theta'\|_2^{2\beta_\theta}
 \quad\text{for all }\theta,\theta'.
 \label{eq:latent-kl-modulus}
\end{equation}
There is a constant $C_0$, depending only on $(d_\theta,\beta_\theta,K_{\mathrm{KL}})$, such that for every integer $m\ge1$,
\begin{equation}
 I(\Theta;U_{1:m})
 \le C_0+\frac{d_\theta}{2\beta_\theta}\log m.
 \label{eq:latent-cumulative-information}
\end{equation}
Moreover, for every $n\ge0$,
\begin{equation}
 \E\KL\!\bigl(
 G_\Theta\,\|\,\Law(U_{n+1}\mid U_{1:n})\bigr)
 \le
 \frac{C_0+\frac{d_\theta}{2\beta_\theta}\log(n+1)}{n+1}.
 \label{eq:latent-one-step-kl}
\end{equation}
\end{lemma}

\begin{proof}
Fix $\varepsilon\in(0,1]$.
The latent cube admits a Euclidean $\varepsilon$-net $\{\theta_1,\ldots,\theta_{J_\varepsilon}\}$ with $J_\varepsilon\le C_{d_\theta}\varepsilon^{-d_\theta}$.
For example, subdivide each coordinate into intervals of length at most $\varepsilon/\sqrt{d_\theta}$ and use their centers.
Define the reference law
\begin{equation*}
 Q_{m,\varepsilon}:=
 J_\varepsilon^{-1}\sum_{j=1}^{J_\varepsilon}G_{\theta_j}^{\otimes m}.
\end{equation*}
For each $\theta$, choose a nearest net point.
The mixture dominates $J_\varepsilon^{-1}G_{\theta_j}^{\otimes m}$, so product additivity of KL and \cref{eq:latent-kl-modulus} give
\begin{equation}
 \KL(G_\theta^{\otimes m}\|Q_{m,\varepsilon})
 \le \log J_\varepsilon+mK_{\mathrm{KL}}\varepsilon^{2\beta_\theta}.
 \label{eq:latent-reference-mixture-bound}
\end{equation}
Let $\overline G_m:=\int G_\theta^{\otimes m}\Pi(\od\theta)$ be the actual marginal law of the examples.
Splitting the log density ratio through $\overline G_m$ and integrating gives
\begin{equation*}
 \int\KL(G_\theta^{\otimes m}\|Q_{m,\varepsilon})
       \Pi(\od\theta)
 =I(\Theta;U_{1:m})+
   \KL(\overline G_m\|Q_{m,\varepsilon}).
\end{equation*}
The nonnegative final term can be dropped.
Choose $\varepsilon=m^{-1/(2\beta_\theta)}$ in \cref{eq:latent-reference-mixture-bound}.
This proves \cref{eq:latent-cumulative-information} after absorbing $\log C_{d_\theta}+K_{\mathrm{KL}}$ into $C_0$.
The equal weights in $Q_{m,\varepsilon}$ are used only to obtain the bound; the actual prior $\Pi$ has not been replaced or restricted.

Put $\mathcal{I}_k=I(\Theta;U_{k+1}\mid U_{1:k})$.
The chain rule gives
\begin{equation*}
 \sum_{k=0}^{m-1}\mathcal{I}_k=I(\Theta;U_{1:m}).
\end{equation*}
To prove monotonicity, set $F=U_{1:k}$, $A=U_{k+1}$, and $B=U_{k+2}$.
Applying the mutual-information chain rule to $I((\Theta,A);B\mid F)$ in the two possible orders gives
\begin{equation*}
 I(\Theta;B\mid F)+I(A;B\mid\Theta,F)
 = I(A;B\mid F)+I(\Theta;B\mid A,F).
\end{equation*}
Conditional independence makes $I(A;B\mid\Theta,F)=0$.
Conditional exchangeability identifies the first term with $\mathcal{I}_k$ and the final term with $\mathcal{I}_{k+1}$.
Hence
\begin{equation*}
 \mathcal{I}_k-\mathcal{I}_{k+1}=I(U_{k+1};U_{k+2}\mid U_{1:k})\ge0.
\end{equation*}
Therefore
\begin{equation*}
 (n+1)\mathcal{I}_n\le\sum_{k=0}^n\mathcal{I}_k
 =I(\Theta;U_{1:n+1}).
\end{equation*}
Finally, conditional independence gives
\begin{equation*}
 \mathcal{I}_n=\E\KL\!\bigl(
 G_\Theta\,\|\,\Law(U_{n+1}\mid U_{1:n})\bigr).
\end{equation*}
Combining the last two displays with \cref{eq:latent-cumulative-information} proves \cref{eq:latent-one-step-kl}.
The preceding finite bound also justifies every mutual-information chain-rule identity used here.
\end{proof}

\begin{proof}[Proof of the upper bounds in \cref{thm:main-task-uncertainty-and-minimax-rate}]
Fix $\mathcal{M}\in\mathfrak{M}_{\mathrm{lat}}$.
Let $U_i=Z_i$ in the sample-only model and $U_i=(X_i,Z_i)$ in the input--output model.
Let $G_\theta$ denote its one-example law, with density $g_\theta$ on the unit-volume cube.
Using $\log x\le x-1$ and the normalization of the two densities,
\begin{equation*}
 \KL(G_\theta\|G_{\theta'})
 \le\int\frac{(g_\theta-g_{\theta'})^2}{g_{\theta'}}
 \le\frac{1}{\underline{p}}
        \|g_\theta-g_{\theta'}\|_\infty^2
 \le\frac{H_\theta^2}{\underline{p}}
       \|\theta-\theta'\|_2^{2\beta_\theta}.
\end{equation*}
Thus Lemma~\ref{lem:latent-predictive-information} applies uniformly over all allowed task maps and all Borel priors.

In the sample-only model, the one-step predictive law is $P(\cdot\mid C_n)$, and the oracle is $P_\Theta$.
\Cref{eq:latent-one-step-kl} proves the asserted KL bound.

In the input--output model, put $\mathcal{I}_n=I(\Theta;U_{n+1}\mid U_{1:n})$.
Conditional independence and the KL chain rule give
\begin{align*}
 &\E_{\mathcal{M}}\KL\!\bigl(
 P^{Z\mid X}_\Theta(\cdot\mid X_\star)
 \,\|\,P(\cdot\mid C_n,X_\star)\bigr)\\
 &\quad=I(\Theta;Z_\star\mid C_n,X_\star)\\
 &\quad\le I(\Theta;(X_\star,Z_\star)\mid C_n)=\mathcal{I}_n.
\end{align*}
The inequality follows by dropping the nonnegative term $I(\Theta;X_\star\mid C_n)$.
In particular, the argument does not require a task-independent input marginal and still conditions the predictive distribution on the query $X_\star$.
This proves \cref{eq:main-predictive-information-bound} in the second model.

Pinsker's inequality $\TV(P,Q)^2\le\KL(P\|Q)/2$ and Jensen's inequality yield
\begin{equation*}
 \begin{aligned}
 &\E_{\mathcal{M}}\TV\!\bigl(P(\cdot\mid H_n),P(\cdot\mid H_n,\Theta)\bigr)\\
 &\quad\le\sqrt{\frac{1}{2}\E_{\mathcal{M}}
 \KL\!\bigl(P(\cdot\mid H_n,\Theta)
       \,\|\,P(\cdot\mid H_n)\bigr)}
 \lesssim \sqrt{\frac{\log(en)}{n}}.
 \end{aligned}
\end{equation*}
\end{proof}

The lower bound in \cref{thm:main-task-uncertainty-and-minimax-rate} is proved in \cref{app:subsec:sharp-prompt-lower-bound}.
Together with the preceding upper bounds and the posterior-averaging inequality, it gives the stated oracle-risk comparison.
The constants depend only on the common class parameters.
When $H_\theta=0$ or $\underline{p}=1$, all tasks within a model have the same law and both predictive KL and task uncertainty are zero.

\subsection{Lower Bound via a Smooth Latent Assouad Family}
\label{app:subsec:sharp-prompt-lower-bound}

\begin{proof}[Proof of the lower bound in \cref{thm:main-task-uncertainty-and-minimax-rate}]
For the lower bound, we construct a hypercube of smooth densities, apply an Assouad testing argument~\citep{tsybakov2009introduction}, and embed the entire family into a one-dimensional latent path.

Choose a nonzero $\psi\in C_c^\infty((0,1))$ with $\int\psi=0$ and $\|\psi\|_\infty\le1$, extended by zero outside $(0,1)$.
Write $b_1=\int_0^1|\psi|>0$ and $b_2=\int_0^1\psi^2>0$.
For sufficiently large $n$, set
\begin{equation*}
 m=\left\lfloor\tfrac{1}{4}\log_2 n\right\rfloor\ge1,
 \qquad a=\kappa\sqrt{m/n},\qquad K=2^m,
\end{equation*}
where $\kappa>0$ is a fixed small constant chosen below.
For $v\in\{-1,1\}^m$, define
\begin{equation*}
 p_v(z)=1+a\sum_{j=1}^m v_j\psi(mz_1-j+1),
 \qquad z\in[0,1]^{d_z}.
\end{equation*}
The perturbations have disjoint supports and zero integral, so $p_v$ is a density with $1-a\le p_v\le1+a$.

We check spatial regularity under \cref{eq:holder-norm-convention}, including integer $\beta$.
Recall $m_{\beta}=\lceil\beta\rceil-1$ and $\vartheta_{\beta}=\beta-m_{\beta}\in(0,1]$ from Appendix~\ref{app:subsec:regularity-conventions}, and write $\widetilde{p}_v=(p_v-1)/a$.
Disjointness gives $\|\partial_1^s \widetilde{p}_v\|_\infty\le C_s m^s$ for every fixed integer $s\ge0$; derivatives involving other coordinates vanish.
The $m_{\beta}$th derivative has both a uniform bound $Cm^{m_{\beta}}$ and a Lipschitz bound $Cm^{m_{\beta}+1}$.
Thus
\begin{equation*}
 |\partial_1^{m_{\beta}}\widetilde{p}_v(z)-\partial_1^{m_{\beta}}\widetilde{p}_v(z')|
 \lesssim\min\{m^{m_{\beta}},m^{m_{\beta}+1}\|z-z'\|_2\}
 \le Cm^\beta\|z-z'\|_2^{\vartheta_{\beta}}.
\end{equation*}
Consequently $\|p_v\|_{C^\beta}\le1+C_{\beta,d_z,\psi}a m^\beta$.
Since $a m^\beta\to0$, all these densities satisfy the fixed spatial norm bound $L_\beta>1$ and lower density bound $\underline{p}<1$ for large $n$.

Let $P_v$ have density $p_v$, let $\E_v:=\E_{P_v^{\otimes n}}$, and let $d_{\mathrm{H}}$ denote Hamming distance.
Direct integration gives
\begin{equation*}
 \TV(P_v,P_w)=\frac{a b_1}{m}d_{\mathrm{H}}(v,w),
 \qquad
 \KL(P_v^{\otimes n}\|P_w^{\otimes n})
 \le\frac{8na^2b_2}{m}=8\kappa^2b_2
 \quad\text{if }d_{\mathrm{H}}(v,w)=1.
\end{equation*}
For the KL bound, take $a\le1/2$ and use $\KL(P_v\|P_w)\le\int(p_v-p_w)^2/p_w$ on the single differing bump.
Choose $0<\kappa\le\min\{1,(8\sqrt{b_2})^{-1}\}$.
Pinsker's inequality then gives $\TV(P_v^{\otimes n},P_w^{\otimes n})\le1/2$ for neighboring vertices.

For any Borel probability-valued estimator $Q_D$, where $D=Z_{1:n}$, let $\widehat{v}(Q_D)$ be the lexicographically first minimizer of $w\mapsto\TV(Q_D,P_w)$.
This is measurable: TV from a fixed probability measure on the cube is a supremum over a countable generating algebra.
The definition of $\widehat{v}(Q_D)$ and the triangle inequality imply
\begin{equation*}
 \TV(Q_D,P_v)\ge\tfrac{1}{2}\TV(P_{\widehat{v}(Q_D)},P_v)
 =\frac{a b_1}{2m}d_{\mathrm{H}}(\widehat{v}(Q_D),v).
\end{equation*}
For each bit, pair the vertices that differ only in that bit.
The sum of the two testing-error probabilities in each pair is at least $1-\TV(P_v^{\otimes n},P_w^{\otimes n})\ge1/2$.
Averaging uniformly over all vertices therefore gives error probability at least $1/4$ for each bit.
Summing over the $m$ bits yields
\begin{equation*}
 2^{-m}\sum_v\E_v\TV(Q_D,P_v)
 \ge\frac{a b_1}{8}\gtrsim \sqrt{\frac{\log(en)}{n}},
\end{equation*}
because $m\asymp\log(en)$.
The estimator need not have a density.
Randomized probability-valued estimators satisfy the same bound by averaging their output measures over the auxiliary randomization and using convexity of TV.

We next embed all $K=2^m$ densities into the fixed latent class.
Enumerate the vertices as $v^{(1)},\ldots,v^{(K)}$, put $\lambda_j=(j-1)/(K-1)$, and let $\bar p_\lambda$ interpolate linearly between $p_{v^{(j)}}$ and $p_{v^{(j+1)}}$ on each interval $[\lambda_j,\lambda_{j+1}]$.
Convexity preserves normalization, positivity, and the spatial norm.
Since $\|p_v-p_w\|_\infty\le2a$, summing the intervalwise Lipschitz bounds gives
\begin{equation*}
 \|\bar p_\lambda-\bar p_{\lambda'}\|_\infty\le2a(K-1)|\lambda-\lambda'|,
 \qquad
 2aK\le Cn^{-1/4}\sqrt{\log n}\longrightarrow0.
\end{equation*}
For sufficiently large $n$, this Lipschitz constant is at most the fixed $H_\theta>0$.
Set $g_\theta(z)=\bar p_{\theta_1}(z)$ for $\theta\in[0,1]^{d_\theta}$.
This map is jointly Borel and satisfies
\begin{equation*}
 \|g_\theta-g_{\theta'}\|_\infty
 \le H_\theta|\theta_1-\theta_1'|
 \le H_\theta|\theta_1-\theta_1'|^{\beta_\theta}
 \le H_\theta\|\theta-\theta'\|_2^{\beta_\theta},
\end{equation*}
using $|\theta_1-\theta_1'|\le1$.
Choose $\Pi_n=K^{-1}\sum_{j=1}^K\delta_{(\lambda_j,0,\ldots,0)}$.
The resulting environment belongs to $\mathfrak{M}_{\mathrm{lat}}$ with all class bounds fixed, and its oracle risk is at least $c\sqrt{\log(en)/n}$ by the testing bound obtained above under the uniform prior on the vertices.
Although the hypercube has $m$ bits, its embedding uses only one latent coordinate.

For input--output prompts, set $g_\theta(x,z)=\bar p_{\theta_1}(z)$ on the unit input--output cube.
The same spatial and latent regularity checks apply.
The inputs, including $X_\star$, are uniform and independent of the task and output context.
Conditional on all these inputs, any kernel $Q(\cdot\mid C_n,X_\star)$ is an estimator based on the output context to which the preceding bound applies.
Averaging over the inputs gives the same lower bound.

Finally, fix the constructed environment $\mathcal{M}_n$.
For any rule using independent training episodes, write $\widehat Q_N^{\mathcal{M}_n}(\cdot\mid \mathsf d,h)$ for its output kernel given a training-data realization $\mathsf d$ and prompt $h$, and define the Borel kernel
\begin{equation*}
 \overline{Q}_h(A)
 :=\int \widehat Q_N^{\mathcal{M}_n}(A\mid \mathsf d,h)
 \Law_{\mathcal{M}_n}(\gD_N^{\mathrm{epi}})(\od \mathsf d).
\end{equation*}
Independence of training from the fresh task and prompt, together with convexity of TV, shows that $\overline{Q}$ has no greater expected loss than the original rule.
The same construction applies to a countably infinite training sequence on its standard-Borel product space.
Hence independent pretraining cannot reduce the lower bound even when the entire environment is known.
\end{proof}

The lower bound is a uniform statement over a fixed class.
It does not assert the rate $\sqrt{\log(en)/n}$ for every fixed latent family or fixed prior.
Additional identifiability or local regularity may yield a log-free parametric rate for particular models.
No positive lower bound in \cref{eq:main-task-uncertainty-bound} is asserted when $H_\theta=0$ or $\underline{p}=1$.

\subsection{Task-Conditional Consequence and Minimaxity}

Combining the predictive guarantee with the finite-prompt uncertainty bound gives a bound on the learned generator's error relative to the task distribution.
To compare it with a single learning rule that does not know the family or prior, define, for finite $N\ge0$, the unknown-model minimax task risk
\begin{equation}
 \mathcal{R}_{n,N}^{\mathrm{lat}}
 :=\inf_Q\sup_{\mathcal{M}\in\mathfrak{M}_{\mathrm{lat}}}
 \E_{\mathcal{M}}\TV\bigl(Q(\cdot\mid\gD_N^{\mathrm{epi}},H_n),
 P(\cdot\mid H_n,\Theta)\bigr).
 \label{eq:latent-unknown-model-risk}
\end{equation}
Here $\E_{\mathcal{M}}$ averages training and the fresh episode.
The infimum ranges over Borel learning rules, including randomized rules, that know the common class parameters but are not given the family or prior.
In $\E_{\mathrm{train},\mathcal{M}}\E_{\mathcal{M}}$, the inner expectation averages a fresh task and prompt conditional on training, and the outer expectation averages training, augmentations, and fitting randomness under $\mathcal{M}$.
The learning rule does not depend on which model in the class generates the data.

\begin{corollary}[Task-conditional generation at the sharp prompt rate]
\label{cor:task-conditional-prompt-rate}
For all sufficiently large $n$, fix $\mathcal{M}\in\mathfrak{M}_{\mathrm{lat}}$, $N\ge1$, $\delta\in(0,1)$, and an integer $M$ satisfying \cref{eq:admissible-sampling-updates}. Set $\varepsilon=\sqrt{\log(en)/n}$, and use the score class, estimator, and sampler of \cref{thm:main-generation-error} with $\varepsilon_{\mathrm{opt}}\le c_{\mathrm{opt}}\log(en)/n$. With probability at least $1-\delta$ over the training episodes and their diffusion augmentations,
\begin{equation}
 \begin{aligned}
 &\E\TV\!\bigl(\widehat{P}_{0\,|\,H_n}^{(M)},P(\cdot\mid H_n,\Theta)\bigr)\\
 &\quad\lesssim\sqrt{\frac{\log(en)}{n}}+\log(en)\sqrt{\frac{
 \mathsf{N}_{\mathrm{param},n,\sqrt{\log(en)/n}}^{\mathrm{lat}}\log(en)
 +\log(2/\delta)}{N}}
       +\frac{\log(en)}{\sqrt{M}}.
 \end{aligned}
 \label{eq:main-smooth-latent-generation-rate}
\end{equation}
On the same event, the displayed bound also holds for $\E\TV\!\bigl(\widehat{P}_{0\,|\,H_n}^{(M)},P(\cdot\mid H_n)\bigr)$. If $N$ and $M$ satisfy \cref{eq:main-predictive-sufficient-choices} with $\varepsilon=\sqrt{\log(en)/n}$, then both TV errors are $\lesssim\sqrt{\log(en)/n}$ and the integrated score error in \cref{eq:main-balanced-score-oracle} is $\lesssim\log(en)/n$. The implicit constants may also depend on $c_{\mathrm{opt}}$. Under the nondegeneracy conditions of \cref{thm:main-task-uncertainty-and-minimax-rate}, take $\delta=n^{-1}$ and these sufficient choices of $N$ and $M$. Then
\begin{equation}
 \begin{aligned}
 \sqrt{\frac{\log(en)}{n}}&\lesssim \mathcal{R}_{n,N}^{\mathrm{lat}}\\
 &\le \sup_{\mathcal{M}\in\mathfrak{M}_{\mathrm{lat}}}
       \E_{\mathrm{train},\mathcal{M}}\E_{\mathcal{M}}\TV\!\bigl(\widehat{P}_{0\,|\,H_n}^{(M)},P(\cdot\mid H_n,\Theta)\bigr)
 \lesssim \sqrt{\frac{\log(en)}{n}}.
 \end{aligned}
 \label{eq:main-latent-minimax-sandwich}
\end{equation}
The lower bound holds for every number of independent training episodes, even if the model is known. The upper bound requires the stated choices of $N$, $M$, and the optimization tolerance.
\end{corollary}

The proof below combines posterior-predictive error with finite-prompt task uncertainty. Since $\log(en)\asymp\log n$ for sufficiently large $n$, this is the same rate stated in the main text and gives the same orders of $N$ and $M$. A smaller target $\varepsilon$ can reduce posterior-predictive error at fixed $n$, but it does not reduce the uncertainty caused by observing only $n$ examples from the fresh task; individual task families may nevertheless have faster rates or zero uncertainty.

\subsubsection{Proof of the task-conditional prompt-rate corollary}
\label{app:subsec:task-conditional-prompt-rate-proof}

\begin{proof}[Proof of \cref{cor:task-conditional-prompt-rate}]
Set $\varepsilon=\sqrt{\log(en)/n}$ and restrict to $n$ large enough that $\varepsilon\le\varepsilon_0$. Then $\log(e/\varepsilon)\lesssim\log(en)$ and $\log(en/\varepsilon)\lesssim\log(en)$. Applying \cref{thm:main-generation-error} and the covering bound in \cref{eq:accuracy-cover-bound} gives the posterior-predictive part of \cref{eq:main-smooth-latent-generation-rate}; \cref{eq:main-balanced-score-oracle} gives the stated score bound. The uniform upper bound in \cref{thm:main-task-uncertainty-and-minimax-rate} is $\lesssim\sqrt{\log(en)/n}$, so the triangle inequality gives the task-conditional bound on the same training event.

For the upper minimax bound, the score class chosen above and the fitting rule are common to all $\mathcal{M}\in\mathfrak{M}_{\mathrm{lat}}$.
Under each model, the conditional task risk is $\lesssim \sqrt{\log(en)/n}$ on an event of probability at least $1-n^{-1}$ and at most one otherwise.
Hence
\begin{equation*}
 \E_{\mathrm{train},\mathcal{M}}\E_{\mathcal{M}}\TV\!\bigl(\widehat{P}_{0\,|\,H_n}^{(M)},P(\cdot\mid H_n,\Theta)\bigr)
 \lesssim \sqrt{\frac{\log(en)}{n}}+n^{-1}
 \lesssim\sqrt{\frac{\log(en)}{n}}.
\end{equation*}
The constants and sufficient choices are common, so taking the supremum over models preserves this bound.
The minimax infimum is no larger than the risk of this particular common rule.
Existence is ensured by the Borel exact ERM in \cref{prop:smooth-latent-complexity}; any common Borel fitting rule meeting the prescribed approximate-ERM tolerance satisfies the same upper bound.

For the lower bound, the construction in \cref{app:subsec:sharp-prompt-lower-bound} gives, for each sufficiently large $n$, one model in $\mathfrak{M}_{\mathrm{lat}}$ on which every probability-valued estimator has expected task risk at least $c\sqrt{\log(en)/n}$, even if given the model and independent training data.
This applies to every unknown-model learning rule in \cref{eq:latent-unknown-model-risk}.
Taking its supremum risk and then the infimum over rules proves the first inequality in \cref{eq:main-latent-minimax-sandwich}.
Averaging the predictive bound in the same way gives expected posterior-predictive TV $\lesssim \sqrt{\log(en)/n}$, without a matching $\sqrt{\log(en)/n}$ lower claim for that risk.
\end{proof}

\clearpage

\section{Examples of Smooth Latent Task Families}
\label{app:sec:smooth-latent-task-examples}

This section illustrates the scope of \cref{assump:smooth-latent-task} in two layers.
We first verify one sample-only and one input--output family directly, and then give three reusable construction templates for generating further admissible families.

\subsection{Direct Sample-Only and Input--Output Families}

The following family illustrates the common-support and latent-regularity conditions in \cref{assump:smooth-latent-task}.
Its posterior predictive is examined in \cref{app:sec:oracle-diagnostics}.

\begin{example}[Smoothly varying sample densities]
\label{ex:smooth-latent-peak-shift}
In the sample-only model, let $d_z=d_\theta=1$ and
\begin{equation*}
 g_\theta(z)=1+\tfrac{1}{4}\cos\bigl(2\pi(z-\theta)\bigr),
 \qquad (\theta,z)\in[0,1]^2.
\end{equation*}
The latent coordinate shifts the density's peak without changing its support.
Integration over one period gives normalization, and
\begin{equation*}
 \tfrac{3}{4}\le g_\theta\le\tfrac{5}{4},\qquad
 \|g_\theta\|_{C^1}\le\tfrac{5}{4}+\tfrac{\pi}{2},\qquad
 \sup_z|g_\theta(z)-g_{\theta'}(z)|\le\tfrac{\pi}{2}|\theta-\theta'|,
\end{equation*}
since both $|\partial_zg_\theta|$ and $|\partial_\theta g_\theta|$ are at most $\pi/2$.
Joint continuity and these uniform bounds verify \cref{assump:smooth-latent-task} with $\beta=\beta_\theta=1$, $\underline{p}=3/4$, $L_1=5/4+\pi/2$, and $H_\theta=\pi/2$, for any Borel prior.
The prompt consists of samples from one task; its posterior predictive averages over possible peak locations.

This family also has a simple coefficient-space geometry. Expanding the shifted cosine gives
\begin{equation*}
 g_\theta(z)
 =1+\tfrac14\cos(2\pi\theta)\cos(2\pi z)
   +\tfrac14\sin(2\pi\theta)\sin(2\pi z).
\end{equation*}
Thus, the coefficient vector
$\gamma(\theta)=\tfrac14\bigl(\cos(2\pi\theta),\sin(2\pi\theta)\bigr)$
traces a circle of radius $1/4$. Given $H_n$, the posterior-predictive density is
\begin{equation*}
 \begin{aligned}
 \bar g_n(z)
 &:=\int g_\theta(z)\,\Pi(\od\theta\mid H_n)
 =1+\bar a_n\cos(2\pi z)+\bar b_n\sin(2\pi z),\\
 (\bar a_n,\bar b_n)
 &:=\int \gamma(\theta)\,\Pi(\od\theta\mid H_n).
 \end{aligned}
\end{equation*}
Consequently, $\sqrt{\bar a_n^2+\bar b_n^2}\le1/4$, with strict inequality unless the posterior distribution of $\gamma(\Theta)$ is a point mass. The task densities therefore lie on the circle, while their posterior predictive can lie in its interior.
\end{example}

\begin{example}[Smoothly varying conditional output densities]
\label{ex:input-output-task-family}
Let $d_x=d_z=1$, let $\psi(x,z)=(2x-1)(2z-1)$, and set
\begin{equation*}
  g_\theta(x,z)=1+\rho(2\theta_1-1)\psi(x,z),
  \qquad \theta\in[0,1]^{d_\theta},\quad (x,z)\in[0,1]^2,
  \quad 0<\rho<1.
\end{equation*}
These are jointly continuous probability densities with $g_\theta\ge1-\rho$ and a common finite spatial $C^\beta$ bound for every fixed $\beta>0$.
Moreover,
\begin{equation*}
  \sup_{x,z}|g_\theta(x,z)-g_{\theta'}(x,z)|
  \le2\rho\|\psi\|_\infty\|\theta-\theta'\|_2.
\end{equation*}
Thus, for sufficiently small $\rho$, this family satisfies \cref{assump:smooth-latent-task} with density floor $1-\rho$ and $\beta_\theta=1$, for any Borel prior.
Because $\int_0^1\psi(x,z)\,\od z=0$ for every $x$, the input marginal is uniform, while the conditional output density varies with the task.
\end{example}

\subsection{General Construction Templates}

\begin{example}[Smooth perturbations of the uniform density]
\label{ex:smooth-latent-families}
Let $\psi\in C^\infty([0,1]^{d_{\mathrm{stat}}})$ satisfy $\int_{[0,1]^{d_{\mathrm{stat}}}}\psi(u)\,\od u=0$ and $\|\psi\|_\infty\le1$, and define
\begin{equation*}
g_\theta(u)=1+\rho(2\theta_1-1)\psi(u), \qquad \theta\in[0,1]^{d_\theta}.
\end{equation*}
For sufficiently small $\rho>0$, this family satisfies \cref{assump:smooth-latent-task} with $\beta_\theta=1$.
\end{example}

\begin{example}[Smoothly varying multimodal mixtures]
\label{ex:smooth-latent-multimodal-mixtures}
Let $\varphi_1,\ldots,\varphi_K$ be fixed smooth densities with separated modes, let $\Delta^{K-1}:=\{w\in[0,1]^K:\sum_{k=1}^K w_k=1\}$, and let $w:[0,1]^{d_\theta}\to\Delta^{K-1}$ be H\"older continuous. Then
\begin{equation*}
g_\theta(u)=(1-\gamma_{\mathrm{mix}})+\gamma_{\mathrm{mix}}\sum_{k=1}^K w_k(\theta)\varphi_k(u)
\end{equation*}
satisfies \cref{assump:smooth-latent-task} for sufficiently small $\gamma_{\mathrm{mix}}>0$. Thus, the relative weights of the modes may vary across tasks.
\end{example}

\begin{example}[Normalized exponential families]
\label{ex:smooth-latent-exponential-families}
Fix $r\in\mathbb{N}$, let $\phi:[0,1]^{d_{\mathrm{stat}}}\to\mathbb{R}^r$, and let $\eta_{\mathrm{nat}}:[0,1]^{d_\theta}\to\mathbb{R}^r$. The following is a normalized exponential family:
\begin{equation*}
g_\theta(u)=\frac{\exp\{\eta_{\mathrm{nat}}(\theta)^\top\phi(u)\}}{\int_{[0,1]^{d_{\mathrm{stat}}}}\exp\{\eta_{\mathrm{nat}}(\theta)^\top\phi(v)\}\,\od v}.
\end{equation*}
It satisfies \cref{assump:smooth-latent-task} when the coordinates of $\phi$ have bounded $C^\beta$ norms, $\eta_{\mathrm{nat}}([0,1]^{d_\theta})$ is compact, and $\eta_{\mathrm{nat}}$ is $\beta_\theta$-H\"older continuous. Hence, the class includes exponential families with bounded parameters, as well as truncated Gaussian families.
\end{example}

\clearpage

\section{Numerical Illustrations}
\label{app:numerical-illustration}

\subsection{Posterior-Predictive Generation in a Synthetic Image Task}
To illustrate the distributional behavior motivated by our theory, we construct a synthetic episodic image-transformation task with an exact, generally bimodal posterior-predictive distribution.
The task is designed so that a short context leaves the sign of a latent transformation rule ambiguous, whereas a longer context identifies the sign and contracts the query distribution.
We use the experiment to ask whether a diffusion Transformer can represent both modes under an ambiguous context, contract toward the appropriate mode as more demonstrations are observed, and learn this prompt-to-distribution map from increasing numbers of pretraining episodes while preserving the query content.

Following controlled synthetic experiments in previous ICL studies
\citep[e.g.,][]{panwar2024incontext,wakayamasuzuki2025,ma2026optimal}, we use a simple task family in which the latent rule and its posterior predictive can be calculated exactly.

\subsubsection{Experimental Setup}
\label{app:experiments-setup}

Each episode has a shared binary latent variable
\begin{equation}
    G\in\{-1,+1\},
    \qquad
    \Pr(G=-1)=\Pr(G=+1)=\frac{1}{2}.
\end{equation}
Conditioned on $G$, the episode-specific position rule is
\begin{equation}
    c_i=a+b r_i+\varepsilon_i,
    \quad
    a\sim\mathcal{N}(0.5,0.04^2),
    \quad
    b\mid G\sim\mathcal{N}(0.30G,0.06^2),
    \quad
    \varepsilon_i\sim\mathcal{N}(0,0.05^2),
    \label{eq:experiment-position-rule}
\end{equation}
where $r_i$ is a normalized descriptor of object size and $c_i$ is the latent horizontal-position coordinate of the output.
The corresponding horizontal translation in pixels is
\begin{equation}
    \Delta x_i=18(c_i-0.5).
    \label{eq:experiment-image-translation}
\end{equation}
No vertical translation is applied.
The variables $a$, $b$, and $G$ are shared within an episode, whereas the observation noises are independent.

For the context examples, we draw $r_i\sim\operatorname{Unif}[-0.2,0.2]$, while the query descriptor is fixed at $r_q=0.9$.
Thus, the context observes the rule only near its intercept, whereas the query requires extrapolation away from the observed range.
With few examples, the context may be compatible with both signs of $b$, which produces left and right predictive modes at the query.
As the context grows, the posterior probability of one sign concentrates and the posterior-predictive distribution contracts toward the corresponding mode.

Images are $32\times32$ RGB and are generated from six shape families: circle, triangle, star, square, blob, and glyph.
The query shape family is excluded from its context, so the task cannot be solved by memorizing a transformation associated with a particular family.
An input $X_i$ contains a dark object on a bright background, and the corresponding output $Y_i$ contains the same object in blue after applying the horizontal translation in Eq.~\cref{eq:experiment-image-translation}.
We write $C_n=\{(X_i,Y_i)\}_{i=1}^n$ for the rendered image context.
We do not use shadows, textures, or vertical transformations.

The exact posterior predictive is available in closed form~\citep{gelman2013bda3}.
Let $\theta=(a,b)^\top$, $\phi(r)=(1,r)^\top$, $m_g=(0.5,0.30g)^\top$, $S_0=\operatorname{diag}(0.04^2,0.06^2)$, and $\sigma_\varepsilon=0.05$.
For computing the exact reference law for $C_n$, let $r_{1:n}$ and $c_{1:n}$ denote the simulator-side descriptors and latent output positions used to render its context pairs, and write $\Phi_n\in\mathbb{R}^{n\times2}$ for the matrix with rows $\phi(r_i)^\top$ and $c_{1:n}=(c_1,\ldots,c_n)^\top$.
Conditioned on $G=g$, Gaussian conjugacy gives
\begin{align}
    S_{g,n}
    =
    \left(
        S_0^{-1}+\sigma_\varepsilon^{-2}\Phi_n^\top\Phi_n
    \right)^{-1},
    \quad
    m_{g,n}
    =
    S_{g,n}
    \left(
        S_0^{-1}m_g+\sigma_\varepsilon^{-2}\Phi_n^\top c_{1:n}
    \right).
\end{align}
The posterior sign probability is obtained from the Gaussian marginal likelihood,
\begin{equation}
    \pi_{g,n}
    :=
    \Pr(G=g\mid C_n)
    =
    \frac{
        \tfrac{1}{2}\,\mathcal{N}\!\left(
            c_{1:n};\Phi_n m_g,
            \sigma_\varepsilon^2 I_n+\Phi_n S_0\Phi_n^\top
        \right)
    }{
        \displaystyle
        \sum_{g'\in\{-1,+1\}}
        \tfrac{1}{2}\,\mathcal{N}\!\left(
            c_{1:n};\Phi_n m_{g'},
            \sigma_\varepsilon^2 I_n+\Phi_n S_0\Phi_n^\top
        \right)
    }.
\end{equation}
Hence, with
\begin{equation}
    \mu_{g,n}=\phi(r_q)^\top m_{g,n},
    \qquad
    \sigma_{g,n}^2
    =
    \sigma_\varepsilon^2
    +\phi(r_q)^\top S_{g,n}\phi(r_q),
\end{equation}
the exact query distribution is the two-component Gaussian mixture
\begin{equation}
    p(c_q\mid C_n)
    =
    \sum_{g\in\{-1,+1\}}
    \pi_{g,n}\,
    \mathcal{N}(c_q;\mu_{g,n},\sigma_{g,n}^2).
    \label{eq:experiment-exact-posterior-predictive}
\end{equation}
This distribution is used as ground truth in all distributional evaluations.
Although the simulator retains $r_{1:n}$ and $c_{1:n}$ to compute \Cref{eq:experiment-exact-posterior-predictive}, the learned model receives only the image context $C_n$ and the diffusion-conditioning variables described below.

We train a cross-attention DiT~\citep{peebles2023dit} on this episodic task.
Each context pair $(X_i,Y_i)$ is concatenated along the channel dimension and mapped by a convolutional pair encoder to a $2\times2$ grid of four $512$-dimensional tokens.
The model accepts at most $64$ pairs, or $256$ context tokens in total.
No positional embedding identifies the order of the context pairs.
The query input and the noisy target are also concatenated channel-wise and embedded by a convolution with patch size $2$, producing a $16\times16$ grid of $256$ query tokens.
The network contains $12$ DiT blocks, each consisting of self-attention among the query tokens, cross-attention from the query tokens to the context tokens, and an MLP.
The hidden dimension is $512$, the number of attention heads is $8$, the feed-forward dimension is $2{,}048$, and the model has $80{,}271{,}180$ trainable parameters.
Diffusion-time and context-length embeddings condition the blocks through adaptive layer normalization (AdaLN), with the latter computed from the number of valid context pairs; the implementation follows the public DiT codebase.\footnote{\url{https://github.com/facebookresearch/DiT/blob/main/models.py}}
The model is given the context input--output images, the query input image, the noisy target, the diffusion time, and a valid-token mask, but not $G$, $a$, $b$, $r_i$, $c_i$, or the analytical posterior.

Training uses freshly generated episodes at every optimizer update, similar to \cite{garg2022case}~and~\citet{wakayamasuzuki2025}.
We sample one context length uniformly from
$n\in\{1,2,4,8,16,32,64\}$ and process $64$ independent episodes with that same length.
The microbatch size is $8$ and gradients are accumulated for $8$ microbatches, giving an effective batch size of $64$.
We train for $40{,}000$ optimizer updates, corresponding to $2.56$ million episodes.
The clean query output $x_0$ is scaled to $[-1,1]$.
A diffusion time is sampled from $1{,}000$ steps with a cosine noise schedule~\citep{nichol2021improved}, Gaussian noise is added to obtain $x_t$, and the model directly predicts the clean target rather than the noise:
\begin{equation}
    \mathcal{L}
    =
    \mathbb{E}\left[
        \left\|
            \widehat x_0(x_t,t,X_q,C_n)-x_0
        \right\|_2^2
    \right].
\end{equation}
We use AdamW~\citep{kingma2015adam,paszke2019pytorch} with learning rate $2\times10^{-4}$, weight decay $0.01$, a $1{,}000$-step linear warmup followed by cosine decay, and gradient clipping at $1.0$. At evaluation time, we draw images with deterministic DDIM sampling~\citep{song2021ddim} using $200$ reverse steps and $\eta=0$. All reported results use the raw model weights.

We evaluate generated images by recovering the object center.
For test task $b\in\{1,\ldots,B\}$, context length $n$, and generated sample $k$, let $\widehat c_{b,k,n}$ denote the recovered latent position and define
\begin{equation}
    \widehat P_{b,n}^{c}
    =
    \frac{1}{K}\sum_{k=1}^K\delta_{\widehat c_{b,k,n}},
    \qquad K=256.
\end{equation}
Let $P_{b,n}^{c,\mathrm{post}}$ denote the exact mixture in \cref{eq:experiment-exact-posterior-predictive}.
We report three quantitative metrics.
\begin{itemize}
    \item \emph{Right-probability error.}
    The boundary $c=0.5$ separates the left and right predictive modes.
    We compare the generated right-mode frequency with its exact posterior-predictive probability:
    \begin{equation}
        E_{\mathrm{right}}(b,n)
        =
        \left|
            \frac{1}{K}\sum_{k=1}^K
            \mathbf{1}[\widehat c_{b,k,n}>0.5]
            -
            \Pr(c_q>0.5\mid C_{b,n})
        \right|.
    \end{equation}
    This metric isolates whether the model assigns the correct probability mass to the two alternatives.

    \item \emph{Posterior-distribution accuracy.}
    We measure the discrepancy in position space by
    \begin{equation}
        E_{W_1}(b,n)
        =
        W_1\!\left(
            \widehat P_{b,n}^{c},
            P_{b,n}^{c,\mathrm{post}}
        \right).
    \end{equation}
    Unlike the right-probability error, this quantity jointly reflects mixture weights, mode locations, and within-mode dispersion.

    \item \emph{Aligned content preservation.}
    To separate shape preservation from position inference, we undo the horizontal translation estimated from each generated image before comparing object masks.
    Let $M_b^{\mathrm{query}}$ be the query mask and $\widetilde M_{b,k,n}^{\mathrm{gen}}$ the generated mask after alignment.
    We report
    \begin{equation}
        \operatorname{IoU}_{\mathrm{align}}(n)
        =
        \frac{1}{BK}
        \sum_{b=1}^B\sum_{k=1}^K
        \frac{
            \left|
                \widetilde M_{b,k,n}^{\mathrm{gen}}
                \cap M_b^{\mathrm{query}}
            \right|
        }{
            \left|
                \widetilde M_{b,k,n}^{\mathrm{gen}}
                \cup M_b^{\mathrm{query}}
            \right|
        }.
    \end{equation}

    % \item \emph{Posterior contraction.}
    % For a fixed task and query, we compare the standard deviation of the recovered generated positions,
    % \begin{equation}
    %     \widehat s_{b,n}^{\mathrm{gen}}
    %     =
    %     \left[
    %         \frac{1}{K-1}
    %         \sum_{k=1}^K
    %         \left(
    %             \widehat c_{b,k,n}-\overline c_{b,n}
    %         \right)^2
    %     \right]^{1/2},
    % \end{equation}
    % with the exact mixture standard deviation
    % $
    %     s_{b,n}^{\mathrm{post}}
    %     =
    %     \sqrt{\operatorname{Var}(c_q\mid C_{b,n})}.
    % $
    % This measures uncertainty over output position for a fixed query, rather than pixel-level variation.
\end{itemize}
To visualize posterior contraction directly, we additionally compare the full empirical distribution of recovered generated positions with the exact posterior-predictive distribution for a fixed task and query across nested context lengths.
We also inspect independent generated images and their pixel-wise average
$\overline Y_K=K^{-1}\sum_{k=1}^K Y^{(k)}$.
For a multimodal position distribution, this average can be spatially blurred even when individual samples are sharp and valid.

\subsubsection{Results}
\label{app:experiments-results}

We first examine the effect of the number of training episodes.
Fig.~\ref{fig:training-sample-scaling} evaluates checkpoints at $5{,}000$, $10{,}000$, $20{,}000$, $30{,}000$, and $40{,}000$ optimizer updates, corresponding to $N\in\{0.32,0.64,1.28,1.92,2.56\}$ million episodes.
The same $64$ fixed test tasks are evaluated at context length $n=64$ with $256$ generations per task using the raw weights.
As $N$ increases, the mean right-probability error decreases from $0.448$ to $0.016$, and position $W_1$ decreases from $0.252$ to $0.041$.
Both distributional metrics improve most sharply between $0.64$ and $1.28$ million episodes.
In contrast, aligned content IoU remains approximately $0.83$--$0.85$ throughout training.
Thus, the model learns to render and preserve the query object early, while additional episodes primarily improve inference of the latent rule and calibration of the posterior-predictive position distribution.

\begin{figure}[t]
    \centering
    \includegraphics[width=\linewidth]{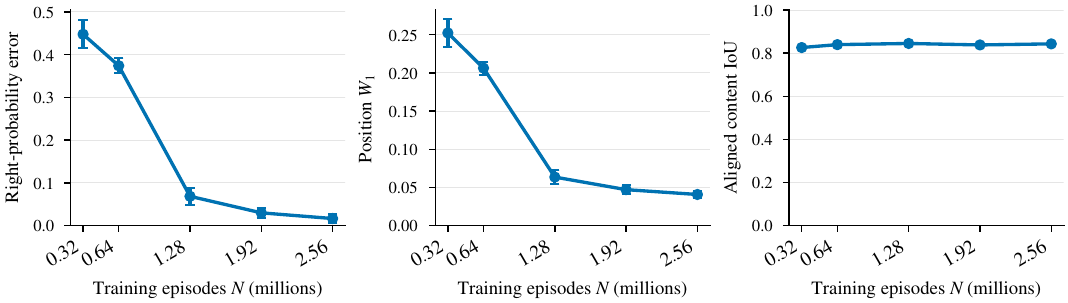}
    \caption{
    \textbf{Scaling with the number of training episodes.}
    Performance at $0.32$, $0.64$, $1.28$, $1.92$, and $2.56$ million episodes.
    From left to right, the panels show the right-probability error, the $1$-Wasserstein distance between the recovered generated-position distribution and the exact two-component Gaussian posterior predictive, and aligned content IoU.
    Evaluation uses the same $64$ fixed test tasks, $n=64$, $256$ generations per task, and the raw model weights.
    Error bars indicate the standard error across tasks.
    }
    \label{fig:training-sample-scaling}
\end{figure}

We next study test-time scaling with the context length.
Fig.~\ref{fig:posterior-contraction} shows the recovered position distribution for a representative fixed task (task~12) under nested contexts of lengths $n\in\{1,2,4,8,16,32\}$.
At $n=1$, the sign of the latent slope is not identified, and both the exact posterior predictive and the model distribution have left and right modes.
As more examples are observed, the posterior probability over the slope sign concentrates and the exact distribution contracts toward the left mode.
The model reproduces this overall transition, but contraction is delayed in the intermediate regime.
In particular, at $n=8$, the exact posterior predictive has already placed almost all of its mass on the left mode, whereas the model still assigns visible probability mass to the right mode.
By $n=16$, the generated distribution is narrow and essentially unimodal, closely matching the exact posterior predictive, and the agreement persists at larger context lengths.
%This transition directly visualizes an amortized posterior-concentration lag that is distinct from the uncertainty present in the exact finite-context posterior itself.
Because the panels follow one realized episode, newly observed noise can transiently change the posterior mode probabilities, so the degree of concentration need not evolve monotonically at every intermediate context length.

\begin{figure}[t]
    \centering
    \includegraphics[width=\linewidth]{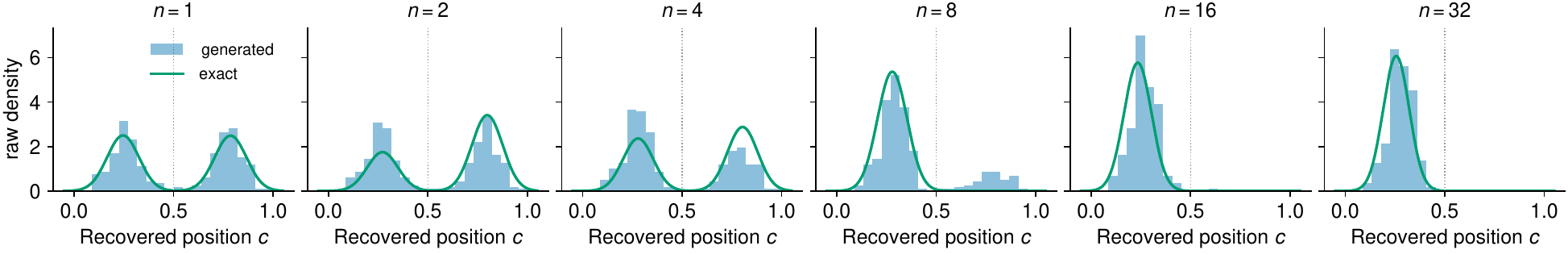}
    \caption{\textbf{Posterior-predictive contraction with context length.}
    Each column corresponds to $n\in\{1,2,4,8,16,32\}$ for the same fixed task and query. Histograms show positions recovered from $256$ generated images using the raw model weights, and green curves show the analytically exact posterior-predictive distribution. The dotted line at $c=0.5$ separates the left and right modes. The model captures the transition from a broad bimodal predictive law at short contexts to a narrow, essentially unimodal law at long contexts, with visible excess right-mode mass in the intermediate regime.}
    \label{fig:posterior-contraction}
\end{figure}

Finally, Fig.~\ref{fig:qualitative-generation} gives qualitative examples of posterior-predictive generation for a circle query and a triangle query, comparing $n=1$ and $n=8$ with the raw weights.
For the circle task, the exact right-mode probability changes from $0.550$ to $0.999$, while the generated frequency changes from $0.562$ to $0.965$. For the triangle task, the exact right-mode probability instead changes from $0.527$ to $0.009$, while the generated frequency changes from $0.410$ to $0.055$.
Thus, a short context leaves both horizontal placements plausible, whereas the longer context selects the right mode for the circle and the left mode for the triangle.
The individual generations preserve the query shape in both cases.
The pixel averages at $n=1$ are blurred along the horizontal direction because they combine samples from both modes, whereas the averages at $n=8$ are sharp after the posterior predictive has concentrated.
A single averaged image obscures the multimodal uncertainty represented by the generated sample distribution.

\begin{figure}[t]
    \centering
    \includegraphics[width=\linewidth]{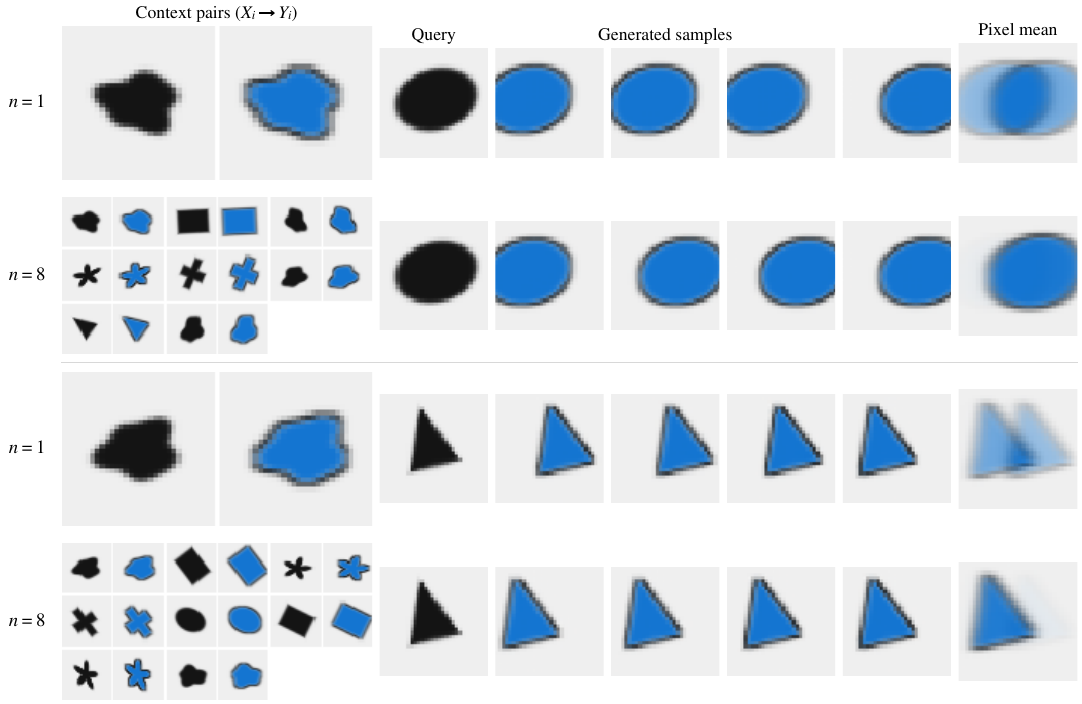}
    \caption{
    \textbf{Qualitative posterior-predictive generation.}
    The top half shows a task with a circle query, and the bottom half shows a task with a triangle query; each task is displayed at $n=1$ and $n=8$.
    From left to right, each row contains the available context input--output pairs $X_i\to Y_i$, the untransformed query input, four independent generations for that query, and the pixel-wise average of $256$ generations.
    All samples use the raw model weights.
    For the circle task, the exact/generated right-mode probabilities are $0.550/0.562$ at $n=1$ and $0.999/0.965$ at $n=8$.
    For the other task, they are $0.527/0.410$ at $n=1$ and $0.009/0.055$ at $n=8$.
    The horizontally blurred averages at $n=1$ reflect bimodality, while the sharp averages at $n=8$ reflect posterior-predictive contraction.
    }
    \label{fig:qualitative-generation}
\end{figure}

\subsection{Oracle Diagnostics for Predictive and Task-Conditional Error}
\label{app:sec:oracle-diagnostics}
The image experiment above evaluates posterior-predictive learning in a trained diffusion model. We next complement it with a separate oracle diagnostic in a simple known-density family, isolating task uncertainty from predictive approximation without neural training or reverse sampling.

\subsubsection{Task uncertainty and predictive approximation in a known family}
\label{app:subsec:oracle-error-targets}

The two error targets can be seen in a sample-only family with known densities and prior: $g_\theta(z)=1+\tfrac{1}{4}\cos(2\pi(z-\theta))$ for $z,\theta\in[0,1]$, with $\Pi=\Unif[0,1]$.
This positive sinusoidal family satisfies the task assumptions (\cref{ex:smooth-latent-peak-shift}).
We compute its exact predictive density and compare it with both the realized task density and a predictive approximation on a finite latent grid.

\begin{figure}[!htbp]
 \centering
 \includegraphics[width=\linewidth]{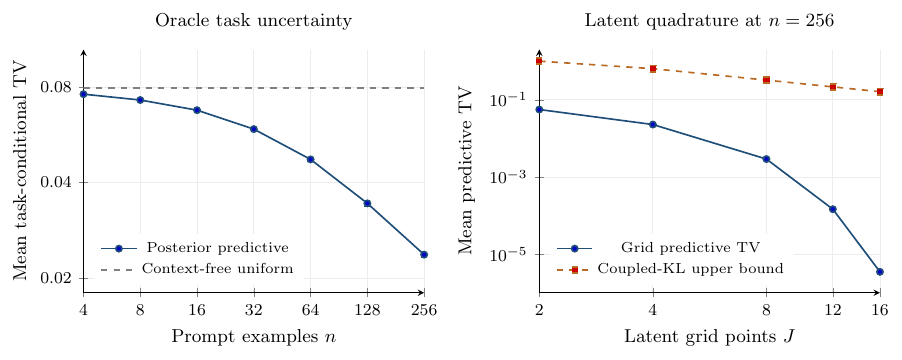}
 \caption{Oracle illustration for the positive sinusoidal family.
 Left: task uncertainty with nested prompts, compared with the context-free uniform predictor.
 Right: predictive TV from latent quadrature at $n=256$, compared with the coupled episode-KL upper bound.
 Means use $5000$ independent tasks; bars show $1.96$ standard errors across tasks.
 TV between densities is calculated analytically.
 This fixed-family diagnostic does not test the worst-case rate, neural training, or reverse sampling.}
 \label{fig:oracle-diagnostics}
\end{figure}

In \cref{fig:oracle-diagnostics}, longer prompts reduce the observed mean task uncertainty.
At fixed prompt length, a finer grid instead reduces error against the posterior predictive.
These are distinct effects: the former uses more examples from the task, while the latter refines the calculation on the same prompt.
The calculations below give the analytical TV formulas, quadrature checks, seeds, and recorded per-episode results.

\subsubsection{Analytical TV calculation and numerical protocol}

We examine how task uncertainty decreases with the prompt length and how a finite latent grid approximates the posterior predictive distribution in the sample-only family considered here.
These are oracle calculations with a known family and prior.
They do not train a network, estimate an episode-complexity rate, or evaluate a reverse sampler.

Use the family in \cref{ex:smooth-latent-peak-shift},
\begin{equation*}
 g_\theta(z)=1+a\cos(2\pi(z-\theta)),\qquad a=1/4,\quad
 z,\theta\in[0,1],\qquad\Pi=\Unif[0,1].
\end{equation*}
The fixed positivity and Lipschitz bounds are verified in that example.
Write $v(\theta)=(\cos(2\pi\theta),\sin(2\pi\theta))$ and $m(C_n)=\E[v(\Theta)\mid C_n]$.
The posterior has density proportional to $\prod_{i=1}^n g_\theta(z_i)$ relative to $\Pi$; its normalizer lies in $[(3/4)^n,(5/4)^n]$, so it is positive for every context.
The predictive density equals $1+a\{m_1\cos(2\pi z)+m_2\sin(2\pi z)\}$.
The moment vector has norm at most one, so posterior averaging may flatten the density rather than produce the density at a posterior-mean coordinate.
For any two such densities with coefficient vectors $v,w$, their continuous TV distance is exactly
\begin{equation}
 \TV(p_v,p_w)=\frac{a}{\pi}\|v-w\|_2.
 \label{eq:oracle-harmonic-tv}
\end{equation}
Indeed, the difference is a sinusoid of amplitude $a\|v-w\|_2$, and $\int_0^1|\cos(2\pi z)|\,\od z=2/\pi$.
Thus all per-episode TV values below are density-to-density integrals computed analytically, with no histogram or empirical-measure TV surrogate.

For each of five independent seeds, draw $1000$ tasks and $256$ context examples per task by rejection sampling from the displayed density, using the uniform proposal and envelope $1+a$.
Use nested prefixes of lengths $n\in\{4,8,16,32,64,128,256\}$.
At each length, evaluate the posterior moments with the midpoint periodic trapezoidal rule using $J_{\mathrm{ref}}=2^{\lceil\log_2(n+2)\rceil}$ equally weighted latent points.
The likelihood is a trigonometric polynomial of degree at most $n$, and multiplication by either first moment raises that degree to at most $n+1$.
Since $J_{\mathrm{ref}}>n+1$, this quadrature integrates the numerator and denominator exactly in exact arithmetic.
The computation uses log likelihoods with a common normalization to avoid underflow.
Doubling the reference grid for the first $32$ episodes at every seed and length changed predictive TV by at most $2.3\times10^{-16}$.

The left panel of \cref{fig:oracle-diagnostics} evaluates expected task uncertainty by averaging \cref{eq:oracle-harmonic-tv} between the predictive moment and the true task moment over $5000$ independent tasks.
Its mean decreases from $0.07598$ at $n=4$ to $0.02372$ at $n=256$; the corresponding standard errors are $0.00018$ and $0.00024$.
The context-free uniform density has task TV exactly $a/\pi=0.07958$.
The error bars are the mean plus or minus $1.96$ times the empirical standard error across tasks.
They measure simulation uncertainty, not variability across neural-network training runs.
The same tasks are used across prompt lengths.
This fixed-family curve does not test the uniform lower bound, whose hard families may change with $n$.

At $n=256$, evaluate grid predictive moments using $J\in\{2,4,8,12,16\}$ midpoint cells and the exact cell weights $1/J$.
The mean predictive TV errors are $0.05582$, $0.02280$, $0.002908$, $0.0001460$, and $3.482\times10^{-6}$, respectively.
The right panel compares these errors with the Pinsker bound obtained from the coupled episode KL:
\begin{equation*}
 \E\TV(P_{0\mid C_n},Q^J_{0\mid C_n})
 \le\min\left\{1,\sqrt{\frac{n+1}{2}
 \E_{\Theta}\KL(g_\Theta\|g_{\pi_J(\Theta)})}\right\}.
\end{equation*}
This follows from the conditional KL chain rule in \cref{prop:episode-kl-score-approximation}, Pinsker's inequality, and Jensen's inequality.
The plotted KL uses the actual single-task divergence, instead of the coarser common H\"older bound.
Periodicity reduces its prior integral to a displacement uniform on $[-1/(2J),1/(2J)]$.
We compute the displacement and output integrals using $128$-node Gauss--Legendre rules and check them against $256$ nodes; the coupled episode values agree within $10^{-10}$.
The resulting upper bounds range from $1$ to $0.16177$ and are conservative for this family.
The prior is uniform, so adding a uniform floor would not change these particular grid weights; this diagnostic does not test prior smoothing for a singular prior.

\clearpage
\section{Qualitative Diagnostics and Resource Allocation} \label{app:qualitative_diagnostics} 

The decomposition in Section~\ref{subsec:main-generation-targets} separates posterior-predictive generation error from finite-prompt task uncertainty. Table~\ref{tab:bottlenecks_and_interventions} translates this decomposition into qualitative diagnostics for experiments and resource allocation. The entries are not mutually exclusive and should be interpreted as diagnostic hypotheses rather than as additional theorem statements. 

\begin{table}[htp]
  \centering
  \small
  \setlength{\tabcolsep}{4pt}
  \renewcommand{\arraystretch}{1.15}
  \caption{Qualitative diagnostics and resource allocation for in-context generation. }
  \label{tab:bottlenecks_and_interventions}
  \begin{tabularx}{\linewidth}{@{}
    >{\raggedright\arraybackslash}p{0.28\linewidth}
    >{\raggedright\arraybackslash}p{0.21\linewidth}
    >{\raggedright\arraybackslash}X@{}}
    \toprule
    \textbf{Observed pattern} &
    \textbf{Possible bottleneck} &
    \textbf{Suggested intervention} \\
    \midrule
    The task posterior remains spread over tasks with
    different output laws.
    & Finite-prompt task uncertainty
    & Request additional demonstrations that distinguish the plausible
      task-specific output laws. \\
    \addlinespace[3pt]
    The oracle is concentrated, but the model remains diffuse
    after checking sampler convergence.
    & Amortization or learning error
    & Increase independent pretraining episodes; improve model capacity,
      optimization, or use of context. \\
    \addlinespace[3pt]
    Error relative to the oracle varies substantially with
    sampler steps or schedule.
    & Sampling error or score--sampler interaction
    & Refine the sampler steps and schedule; check convergence
      with the learned score held fixed. \\
    \addlinespace[3pt]
    Individual samples are plausible, but their mode frequencies
    disagree with the oracle.
    & Probability calibration error
    & Evaluate mode probabilities, improve episodic distributional
      training, and check sensitivity to the sampler. \\
    \addlinespace[3pt]
    Learning and sampling errors are small relative to
    finite-prompt task uncertainty.
    & Prompt information
    & Prioritize more informative demonstrations over further
      reductions in learning or sampling error. \\
    \addlinespace[3pt]
    A large token budget contains few independent episodes.
    & Limited independent training data
    & Track and increase independent episode count; account for
      within-episode dependence when allocating pretraining data. \\
    \bottomrule
  \end{tabularx}
\end{table}

\end{document}